\documentclass[runningheads]{llncs}
\usepackage{graphicx}
\usepackage{amsmath,amssymb} % define this before the line numbering.

\usepackage{algpseudocode}
\usepackage{algorithm}
\usepackage{dsfont}
\usepackage{multirow}

\newcommand{\rr}{\mathbb{R}}

\newcommand{\ical}{\mathcal{I}}
\newcommand{\scal}{\mathcal{S}}
\newcommand{\acal}{\mathcal{A}}
\newcommand{\gcal}{\mathcal{G}}

\newcommand{\wtr}{\widetilde R}
\newcommand{\wtrs}{\widetilde r}

\newcommand{\polar}{\mathop{Polar}}

\newcommand{\norm}[1]{\left\|#1\right\|}
\newcommand{\inner}[1]{\left\langle#1\right\rangle}
\newcommand{\ceil}[1]{\left\lceil#1\right\rceil}
\newcommand{\rbra}[1]{\left(#1\right)}
\newcommand{\cbra}[1]{\left\{#1\right\}}

\def\argmin{\mathop{\rm arg\,min}\limits}
\def\min{\mathop{\rm min}\limits}
\def\max{\mathop{\rm max}\limits}

\let\emptyset\varnothing

\begin{document}
\pagestyle{headings}
\mainmatter

%
% Insert your submission number here
%
\def\DAGM12SubNumber{109}  

%
% Replace with your title
%
\title{Robust PCA: Optimization of the Robust Reconstruction Error over the Stiefel Manifold}

%
% DO NOT MODIFY these for the draft version that is used for the
% review process.
% 
%
\titlerunning{Robust PCA: Optimization of the Robust Reconstruction Error}
\authorrunning{A. Podosinnikova, S. Setzer, and M. Hein}
\author{Anastasia Podosinnikova, Simon Setzer, and Matthias Hein}
\institute{Computer Science Department, Saarland University, Saarbr\"ucken, Germany}

\maketitle

\begin{abstract}
It is well known that Principal Component Analysis (PCA) is strongly affected by outliers and a lot of effort has been put into robustification of PCA. In this paper we present a new algorithm for robust PCA minimizing the trimmed reconstruction error. By directly minimizing over the Stiefel manifold, we avoid  deflation as often used by projection pursuit methods. In distinction to other methods for robust PCA, our method has no free parameter and is computationally very efficient. We illustrate the performance on various datasets including an application to background modeling and subtraction. Our method performs better or similar to current state-of-the-art methods while being faster.
\end{abstract}

%%%%%%%%%%%%%%%%%%%%%%%%%%%%%%%%%%%%%%%%%%%%%%%%%%%%
\section{Introduction}
PCA is probably the most common tool for exploratory data analysis, dimensionality reduction and clustering, e.g.,~\cite{Jol2002}. It can either be seen as finding the best low-dimensional subspace approximating the data or as finding the subspace of highest variance. However, due to the fact that the variance is not robust, PCA can be strongly influenced by outliers. Indeed, even one outlier can change the principal components (PCs) drastically. This phenomenon motivates the development of robust PCA methods which recover the PCs of the uncontaminated data. This problem received a lot of attention in the statistical community and recently became a problem of high interest in machine learning.

In the statistical community, two main approaches to robust PCA have been proposed. The first one is based on the robust estimation of the covariance matrix, e.g.,~\cite{HamEtAl1986},~\cite{HubRon2009}. Indeed, having found a robust covariance matrix one can determine robust PCs by performing the eigenvalue decomposition of this matrix. However, it has been shown that robust covariance matrix estimators with desirable properties, such as positive semidefiniteness and affine equivariance, have a breakdown point\footnote{The breakdown point~\cite{HubRon2009} of a statistical estimator is informally speaking the fraction of points which can be arbitrarily changed and the estimator is still well defined.} upper bounded by the inverse of the dimensionality~\cite{HamEtAl1986}. The second approach is the so called projection-pursuit~\cite{Hub1985},~\cite{LiChe1985}, where one maximizes a robust scale measure, instead of the standard deviation,  over all possible directions. Although, these methods have the best possible breakdown 
point of~$0.5$, they lead to non-convex, typically, non-smooth problems and current state-of-the-art are greedy search algorithms~\cite{CroEtAl2007}, which show poor performance in high dimensions. Another disadvantage is that robust PCs are computed one by one using deflation techniques~\cite{Mac2009}, which often leads to poor results for higher PCs.

In the machine learning and computer vision communities, matrix factorization approaches to robust PCA were mostly considered, where one looks for a decomposition of a data matrix into a low-rank part and a sparse part, e.g.,~\cite{CanEtAl2009},~\cite{MatGia2012},~\cite{MacTro2011},~\cite{XuCara2012}. The sparse part is either assumed to be scattered uniformly~\cite{CanEtAl2009} or it is assumed to be row-wise sparse corresponding to the model where an entire observation is corrupted and discarded. While some of these methods have strong theoretical guarantees, in practice, they depend on a regularization parameter which is non-trivial to choose as robust PCA is an unsupervised problem and default choices, e.g.,~\cite{CanEtAl2009},~\cite{MacTro2011}, often do not perform well as we discuss in Section~\ref{sec:exp}. Furthermore, most of these methods are slow as they have to compute the SVD of a matrix of the size of the data matrix at each iteration.

As we discuss in Section~\ref{sec:rpca}, our formulation of robust PCA is based on the minimization of a robust version of the reconstruction error over the Stiefel manifold, which induces orthogonality of robust PCs. This formulation has multiple advantages. First, it has the maximal possible breakdown point of $0.5$ and the interpretation of the objective is very simple and requires no parameter tuning in the default setting. In Section~\ref{sec:trpca}, we propose a new fast TRPCA algorithm for this optimization problem. Our algorithm computes both orthogonal PCs and a robust center, hence, avoiding the deflation procedure and preliminary robust centering of data. While our motivation is similar to the one of~\cite{MatGia2012}, our optimization scheme is completely different. In particular, our formulation requires no additional parameter.

%%%%%%%%%%%%%%%%%%%%%%%%%%%%%%%%%%%%%%%%%%%%%%%%%%%%
\section{Robust PCA}\label{sec:rpca}
\textit{Notation.} All vectors are column vectors and $I_p \in \rr^{p\times p}$ denotes the identity matrix. We are given data $X\in\rr^{n\times p}$ with $n$ observations in $\rr^p$ (rows correspond to data points). We assume that the data contains $t$ true observations $T\in\rr^{t\times p}$ and $n-t$ outliers $O\in\rr^{n-t \times p}$ such that $X=T\cup O$ and $T\cap O\ne\emptyset$. To be able to distinguish true data from outliers, we require the standard  in robust statistics assumption, that is $t\ge \ceil{\frac{n}{2}}$. The Stiefel manifold is denoted as  $\scal_k =\cbra{U\in\rr^{p \times k}\;|\; U^{\top} U=I}$ (the set of orthonormal $k$-frames in $\rr^p$). 

\textit{PCA.} 
Standard PCA~\cite{Jol2002} has two main interpretations. One can either see it as finding the $k$-dimensional subspace of maximum variance in the data or the $k$-dimensional affine subspace 
with minimal reconstruction error. In this paper we are focusing on the second interpretation. Given data $X\in\rr^{n\times p}$, the goal is to find the offset $m\in\rr^p$ and $k$ principal components $(u_1,\ldots,u_k)=U \in \scal_k$, which describe $\acal(m,U)=\cbra{z\in\rr^p\;\big|\; z= m+\sum_{j=1}^k s_j u_j,\;s_j\in\rr}$, the $k$-dimensional affine subspace, so that they minimize the reconstruction error
\begin{equation}\label{pca}
\cbra{\hat m, \hat U}=\argmin_{m\in\rr^p,\;U \in \scal_k,\;z_i\in\acal(m,U)} \;\frac{1}{n}\sum_{i=1}^n\norm{z_i - x_i}_2^2.
\end{equation}
It is well known that $\hat m=\frac{1}{n}\sum_{i=1}^n x_i$,  and the optimal matrix $\hat U \in \scal_k$ is generated by the top $k$ eigenvectors of the empirical covariance matrix. As $U \in \scal_k$ is an orthogonal projection, an equivalent formulation of~\eqref{pca} is given by 
\begin{equation}\label{pcare}
\cbra{\hat m,\hat U}=\argmin_{m\in\rr^p,\; U\in\scal_k}\frac{1}{n}\sum_{i=1}^n\norm{\rbra{UU^{\top}-I}\rbra{x_i-m}}_2^2.
\end{equation}

\textit{Robust PCA.} 
When the data $X$ does not contain outliers ($X=T$), we refer to the outcome of standard PCA, e.g.,~\eqref{pcare}, computed for the true data $T$ as $\{\hat m_T,\hat U_T\}$. When there are some outliers in the data $X$, i.e. $X=T\cup O$, the result $\{\hat m,\hat U\}$ of PCA can be significantly different from $\{\hat m_T,\hat U_T\}$ computed for the true data $T$. The reason is the non-robust squared $\ell_2$-norm involved in the formulation, e.g.,~\cite{HamEtAl1986},~\cite{HubRon2009}. It is well known that PCA has a breakdown point of zero, that is a single outlier can already distort the components arbitrarily. As outliers are frequently present in applications, robust versions of PCA are crucial for data analysis with the goal of recovering the true PCA solution $\{\hat m_T,\hat U_T\}$ from the contaminated data $X$. 

As opposed to standard PCA, robust formulations of  PCA based on the maximization of the variance (the projection-pursuit approach as extension of~\eqref{pca}), eigenvectors of the empirical covariance matrix (construction of a robust covariance matrix), or the minimization of the reconstruction error (as extension of~\eqref{pcare}) are not equivalent. Hence, there is no universal approach to robust PCA and the choice can depend on applications and assumptions on outliers. Moreover, due to the inherited non-convexity of standard PCA, they lead to NP-hard problems. The known approaches for robust PCA either follow to some extent greedy/locally optimal optimization techniques, e.g.,~\cite{CroEtAl2007},~\cite{LiChe1985},~\cite{TorBla2001},~\cite{XuCara2013}, or compute convex relaxations, e.g.,~\cite{CanEtAl2009},~\cite{MatGia2012},~\cite{MacTro2011},~\cite{XuCara2012}. 

In this paper we aim at a method for robust PCA based on the minimization of a robust version of the reconstruction error and adopt the classical outlier model where entire observations (corresponding to rows in the data matrix $X$) correspond to outliers. 
In order to introduce the trimmed reconstruction error estimator for robust PCA, we employ the analogy with the least trimmed squares estimator~\cite{Rou1984} for robust regression. We denote by $r_i(m,U)=\norm{\rbra{UU^{\top}-I}\rbra{x_i-m}}_2^2$ the reconstruction error of observation $x_i$ for the given affine subspace parameterized by $(m,U)$. Then the trimmed reconstruction error is defined to be the sum of the $t$-smallest reconstruction errors $r_i(m,U)$,
\begin{equation}\label{tre}
R(m,U)=\frac{1}{t}\sum_{i=1}^t r_{(i)}(m,U),
\end{equation}
where $r_{(1)}(m,U)\le\dots\le r_{(n)}(m,U)$ are in nondecreasing order and  $t$, with $\ceil{\frac{n}{2}}\leq t\leq n$, should be a lower bound on the number of true examples $T$. If such an estimate is not available as it is common in unsupervised learning, one can set by default $t=\ceil{\frac{n}{2}}$. With the latter choice it is straightforward to see that the corresponding PCA estimator has the maximum possible breakdown point of $0.5$, that is up to $50\%$ of the data points can be arbitrarily corrupted. With the default choice our method has no free parameter except the rank $k$.

The minimization of the trimmed reconstruction error~\eqref{tre} leads then to a simple and intuitive formulation of robust PCA 
\begin{equation}\label{rpca}
\cbra{m^*,U^*}=\argmin_{{m\in\rr^p,\;U\in\scal_k}}R(m,U) =\argmin_{{m\in\rr^p,\;U\in\scal_k}}\frac{1}{t}\sum_{i=1}^tr_{(i)}(m,U).
\end{equation}
Note that the estimation of the subspace $U$ and the center $m$ is done jointly. This is in contrast to~\cite{CanEtAl2009},~\cite{CroEtAl2007},~\cite{LiChe1985},~\cite{MacTro2011},~\cite{XuCara2013},~\cite{XuCara2012}, where the data has to be centered by a separate robust method which can lead to quite large errors in the estimation of the true PCA components. The same criterion~\eqref{rpca} has been proposed by~\cite{MatGia2012}, see also~\cite{XuYui1995} for a slightly different version. While both papers state that the direct minimization of~\eqref{rpca} would be desirable,~\cite{MatGia2012} solve a relaxation of~\eqref{rpca} into a convex problem while~\cite{XuYui1995} smooth the problem and employ deterministic annealing. Both approaches introduce an additional regularization parameter controlling the number of outliers. It is non-trivial to choose this parameter.

%%%%%%%%%%%%%%%%%%%%%%%%%%%%%%%%%%%%%%%%%%%%%%%%%%%%
\section{TRPCA: Minimizing Trimmed Reconstruction Error on the Stiefel Manifold}\label{sec:trpca}
In this section, we introduce TRPCA, our algorithm for the minimization of the trimmed reconstruction error~\eqref{rpca}. We first reformulate the objective of~\eqref{rpca} as it is neither convex, nor concave, nor smooth, even if $m$ is fixed. While the resulting optimization problem is still non-convex, we propose an efficient optimization scheme on the Stiefel manifold with monotonically decreasing objective.
Note that all proofs of this section can be found in the supplementary material~\cite{Suppl}.
%%%%%%%%%%%%%%%%%%%%%%%%%%%%%%%%%%%%%%%%%%%%%%%%%%%%
\subsection{Reformulation and First Properties}\label{sec:relax}
The reformulation of~\eqref{rpca} is based on the following simple identity. Let $\widetilde{x}_i = x_i-m$ and $U \in \scal_k$, then
\begin{equation}
 r_i(m,U)=\norm{\rbra{UU^{\top}-I}\rbra{x_i-m}}_2^2 = -\norm{U^{\top} \widetilde{x}_i}^2_2 + \norm{\widetilde{x}_i}^2_2 := \wtrs_{i}(m,U).
\end{equation}
The equality holds only on the Stiefel manifold. Let $\wtrs_{(1)}(m,U)\leq  \ldots \leq \widetilde{r}_{(n)}(m,U)$, then we
get the alternative formulation of~\eqref{rpca},
\begin{equation}\label{rpca2}
\cbra{m^*,U^*}=\argmin_{{m\in\rr^p,\;U\in\scal}} \wtr(m,U) = \frac{1}{t}\sum_{i=1}^t \widetilde{r}_{i}(m,U).
\end{equation}

While~\eqref{rpca2} is still non-convex, we show in the next proposition that for fixed $m$ the function $\wtr(m,U)$ is concave on $\rr^{p \times k}$. This will allow us to employ a simple optimization technique based on linearization of this concave function.
%%%%%%%%%%%%%%%%%%%%%%%%%%%%%%%%%%%%%
\begin{proposition}\label{concavity}
For fixed $m \in \rr^p$ the function $\wtr(m,U): \rr^{p \times k} \rightarrow \rr$ defined in~\eqref{rpca2} is concave in $U$.
\end{proposition}
%%%%%%%%%%%%%%%%%%%%%%%%%%%%%%%%%%%%%
%%%%%%%%%%% BEGIN LONG VERSION %%%%%%%%%%%%%
%%%%%%%%%%%%%%%%%%%%%%%%%%%%%%%%%%%%%
%\begin{comment}
\begin{proof}
We have $\wtrs_i(m,U)=-\norm{U^{\top}\widetilde{x}_i}^2_2+\norm{\widetilde{x}_i}_2^2$. As $\norm{U^{\top} \widetilde{x}_i}^2$ is convex, we deduce that $\wtrs_i(m,U)$ is concave in $U$.
The sum of the $t$ smallest concave functions out of $n\geq t$ concave functions is concave, as it can be seen as the pointwise minimum of all possible $\binom{n}{t}$ sums of $t$ of the concave functions, e.g.,~\cite{BoyVan2004}.
\end{proof}
%\end{comment}
%%%%%%%%%%%%%%%%%%%%%%%%%%%%%%%%%%%%%
%%%%%%%%%%%% END LONG VERSION %%%%%%%%%%%%%
%%%%%%%%%%%%%%%%%%%%%%%%%%%%%%%%%%%%%

The iterative scheme uses a linearization of $\wtr(m,U)$ in $U$. For that we need to characterize the superdifferential of the concave function $\wtr(m,U)$.
%%%%%%%%%%%%%%%%%%%%%%%%%%%%%%%%%%%%%
\begin{proposition}\label{pro:superdifferential}
Let $m$ be fixed. The superdifferential $\partial \wtr(m,U)$ of $\wtr(m,U): \rr^{p \times k} \rightarrow \rr$ is given as
\begin{equation}
 \partial \wtr(m,U)  = \Big\{ \sum_{i \in I} \alpha_i (x_i-m)(x_i-m)^{\top} U\,\Big|\, \sum_{i=1}^n \alpha_i = t,\; 0\leq \alpha_i \leq 1 \Big\},
\end{equation}
where $I=\{ i \,|\, \wtrs_i(m,U) \leq \wtrs_{(t)}(m,U)\}$
with $\wtrs_{(1)}(m,U)\leq \ldots \leq \wtrs_{(n)}(m,U)$.
\end{proposition}
%%%%%%%%%%%%%%%%%%%%%%%%%%%%%%%%%%%%%
%%%%%%%%%%% BEGIN LONG VERSION %%%%%%%%%%%%%
%%%%%%%%%%%%%%%%%%%%%%%%%%%%%%%%%%%%%
%\begin{comment}
\begin{proof}
We reduce it to a well known case. We can write $\wtr(m,U)$ as
\begin{equation}
\wtr(m,U) = \min_{0\leq \alpha_i\leq 1, \; i=1,\ldots,n, \; \sum\limits_{i=1}^n \alpha_i=t} \quad\sum_{i=1}^n \alpha_i \wtrs_i(m,U),
\end{equation}
that is a minimum of a parameterized set of concave functions. As the parameter set is compact and continuous (see Theorem 4.4.2 in~\cite{HirLem2001}), we have
\begin{equation}
\partial \wtr(m,U) = \mathrm{conv}\Big(\bigcup_{\alpha^j \in I(U)} \partial \big(\sum_{i=1}^n \alpha^{j}_i \wtrs_i(m,U)\big)\Big) \\
                = \mathrm{conv}\Big(\bigcup_{\alpha^j \in I(U)} \sum_{i=1}^n \alpha^{j}_i \partial \wtrs_i(m,U)\Big),
\end{equation}
where $I(U)=\{ \alpha \,|\, \sum_{i=1}^n \alpha_i \wtrs_i(m,U)=\wtr(m,U),\,\sum_{i=1}^n \alpha_i = t,\, 0\leq \alpha_i \leq 1, \,i=1,\ldots,n\}$
and $\mathrm{conv}(S)$ denotes the convex hull of $S$. Finally, using
that $\wtrs_i(m,U)$ is differentiable with $\partial \wtrs_i(m,U) = \{(x_i-m)(x_i-m)^{\top} U\}$ yields the result.
\end{proof}
%\end{comment}
%%%%%%%%%%%%%%%%%%%%%%%%%%%%%%%%%%%%%
%%%%%%%%%%%% END LONG VERSION %%%%%%%%%%%%%
%%%%%%%%%%%%%%%%%%%%%%%%%%%%%%%%%%%%%

%%%%%%%%%%%%%%%%%%%%%%%%%%%%%%%%%%%%%%%%%%%%%%%%%%%
\subsection{Minimization Algorithm}
Algorithm~\ref{alg:trpca} for the minimization of~\eqref{rpca2} is based on block-coordinate descent in $m$ and $U$. For the minimization in $U$ we use that $\wtr(m,U)$ is concave for fixed $m$. Let $G \in \partial \wtr(m,U^k)$, then by definition of the supergradient of a concave function,
\begin{equation} \label{eq:supergrad}
\wtr\rbra{m,U^{k+1}}  \leq \wtr\rbra{m,U^k} + \inner{ G, U^{k+1}-U^k}.
\end{equation}
The minimization of the linear upper bound on the Stiefel manifold can be done in closed form, see Lemma~\ref{le:polar} below. For that we use a modified version of a result of~\cite{JouEtAll2010}. 
Before giving the proof, we introduce the polar decomposition of a matrix $G \in\rr^{p\times k}$ which is defined to be $G=QP$, where $Q\in\scal$ is an orthonormal matrix of size $p\times k$ and $P$ is a symmetric positive semidefinite matrix of size $k\times k$. We denote the factor $Q$ of $G$ by $\polar(G)$. The polar can be computed in ${\cal O}(p k^2)$ for $p\geq k$~\cite{JouEtAll2010} as  $\polar(G)=UV^{\top}$ (see Theorem 7.3.2. in~\cite{MatrAnal}) using the SVD of $G$, $G=U \Sigma V^{\top}$.
However, faster methods have been proposed, see~\cite{HigSch1990}, which do not even require the computation of the SVD. 
\begin{lemma}\label{le:polar}Let $G\in\rr^{p\times k}$, with $k\le p$, and denote by $\sigma_i(G)$, $i=1,\dots,k$, the singular values of $G$. Then $\mathrm{min}_{U\in\scal_k}\inner{G,U}=-\sum_{i=1}^k\sigma_i(G)$, with minimizer $U^*=-\polar(G)$. If $G$ is of full rank, then $\polar(G)=G(G^{\top} G)^{-1/2}$.
\end{lemma}
%%%%%%%%%%%%%%%%%%%%%%%%%%%%%%%%%%%%%
%%%%%%%%%%% BEGIN LONG VERSION %%%%%%%%%%%%%
%%%%%%%%%%%%%%%%%%%%%%%%%%%%%%%%%%%%%
%\begin{comment}
\begin{proof}
Let $G=U\Sigma V^{\top}$ be the SVD of $G$, that is $U \in O(p)$, $V \in O(k)$, where $O(m)$ denotes the set of orthogonal matrices in $\rr^m$,
\begin{equation}
 \min_{O \in \scal_k} \inner{G,O}
= \min_{O \in \scal_k} \inner{\Sigma,U^{\top} O V}\\
                                  = \min_{W \in \scal_k} \sum_{i=1}^k \sigma_i(G) W_{ii} \geq - \sum_{i=1}^k \sigma_i(G).
\end{equation}
The lower bound is realized by $-UV^{\top} \in \scal_k$ which is equal to $-\polar(G)$. We have, $
- \inner{U\Sigma V^{\top}, UV^{\top}}=-\mathrm{trace}(\Sigma)=-\sum_{i=1}^k \sigma(G)_i.$
The final statement follows from the proof of Theorem 7.3.2. in~\cite{MatrAnal}.
\end{proof}
%\end{comment}
%%%%%%%%%%%%%%%%%%%%%%%%%%%%%%%%%%%%%
%%%%%%%%%%%% END LONG VERSION %%%%%%%%%%%%%
%%%%%%%%%%%%%%%%%%%%%%%%%%%%%%%%%%%%%

%%%%%% ALGORITHM %%%%%%
\begin{algorithm}
\caption{TRPCA}
   \label{alg:trpca}
\begin{algorithmic}
   \State {\bfseries Input:} $X$, $t$, $d$, $U^0\in\scal$, and $m^0$ median of $X$, tolerance $\varepsilon$
   \State {\bfseries Output:} robust center $m^k$ and robust PCs $U^k$
   \Repeat \; for $k = 1,2,\dots$
     \State Center data $\widetilde{X}^{k}=\cbra{\widetilde{x}_i^{k}=x_i-m^{k},\;i=1,\dots,n}$
     \State Compute supergradient $\gcal(U^k)$ of $\wtr(m^k,U^k)$ for fixed $m^k$
     \State Update $U^{k+1}=-\polar\rbra{\gcal(U^k)}$
     \State Update $m^{k+1}=\frac{1}{t}\sum_{i\in\ical^{k'}}x_i$, where $\ical^{k'}$ are the indices of the $t$ smallest \\ \hspace{0.4cm} $\wtrs_i(m^k,U^{k+1})$, $i=1,\ldots,n$
   \Until{relative descent below $\varepsilon$}
\end{algorithmic}
\end{algorithm}
%%%%%% ALGORITHM %%%%%%

Given that $U$ is fixed, the center $m$ can be updated simply as the mean of the points realizing the current objective of~\eqref{rpca2}, that is the points realizing the $t$-smallest reconstruction error.
Finally, although the objective of~\eqref{rpca2} is neither convex nor concave in $m$, we prove monotonic descent of Algorithm~\ref{alg:trpca}.
%%%%%%%%%%%%%%%%%%%%%%%
\begin{theorem}\label{thm:descent}
The following holds for Algorithm~\ref{alg:trpca}. At every iteration, either $\wtr(m^{k+1},U^{k+1})<\wtr(m^k,U^k)$ or the algorithm terminates.
\end{theorem}
%%%%%%%%%%%%%%%%%%%%%%%%%%%%%%%%%%%%%
%%%%%%%%%%% BEGIN LONG VERSION %%%%%%%%%%%%%
%%%%%%%%%%%%%%%%%%%%%%%%%%%%%%%%%%%%%
%\begin{comment}
\begin{proof} 
Let $m^k$ be fixed and $G(U^k) \in \partial \wtr(m,U^k)$, then from \eqref{eq:supergrad} we have
\begin{equation}
 \wtr(m^k,U) \leq \wtr(m,U^k) - \inner{G(U^k),U^k} + \inner{G(U^k),U}.
\end{equation}
The minimizer $U^{k+1}=\argmin_{U \in \scal_k} \inner{G(U^k),U}$, over the Stiefel manifold can be computed via Lemma~\ref{le:polar}
as $U^{k+1}=-\polar(G(U^k))$. Thus we get immediately,
\begin{equation*}
\wtr(m^k,U^{k+1}) \leq \wtr(m^k,U^k).
\end{equation*}
After the update of $U^{k+1}$ we compute $\ical^{k'}$ which are the indices of the $t$ smallest $\wtrs_i(m^k,U^{k+1})$, $i=1,\ldots,n$.
If there are ties, then they are broken randomly.
For fixed $U^{k+1}$ and fixed $\ical^{k'}$ the minimizer of the objective
\begin{equation}
\sum_{i \in \ical^{k'}} -\norm{(U^{k+1})^{\top} (x_i - m)}^2_2 + \norm{x_i-m}^2_2,
\end{equation}
is given by $m^{k+1}=\frac{1}{t}\sum\limits_{i \in \ical^{k'}} x_i$, which yields,
$\sum\limits_{i \in \ical^{k'}} \wtrs_i(m^{k+1},U^{k+1}) \leq \wtr(m^k,U^{k+1})$.
After the computation of $m^{k+1}$, $\ical^{k'}$ need no longer correspond to the $t$ smallest
reconstruction errors $\wtrs_i(m^{k+1},U^{k+1})$. However, taking the $t$ smallest ones only further reduces the objective, $\wtr(m^{k+1},U^{k+1}) \leq \sum_{i \in \ical^{k'}} \wtrs_i(m^{k+1},U^{k+1})$.
This yields finally the result, $ \wtr(m^{k+1},U^{k+1}) \leq \wtr(m^k,U^k)$.
\end{proof}
%\end{comment}
%%%%%%%%%%%%%%%%%%%%%%%%%%%%%%%%%%%%%
%%%%%%%%%%%% END LONG VERSION %%%%%%%%%%%%%
%%%%%%%%%%%%%%%%%%%%%%%%%%%%%%%%%%%%%

The objective is non-smooth and neither convex nor concave. The Stiefel manifold is a non-convex constraint set. These facts make the formulation of critical points conditions challenging. Thus, while potentially stronger convergence results like convergence to a critical point are appealing, they are currently out of reach. However, as we will see in Section~\ref{sec:exp}, Algorithm~\ref{alg:trpca} yields good empirical results, even beating state-of-the-art methods based on convex relaxations or other non-convex formulations.

%%%%%%%%%%%%%%%%%%%%%%%%%%%%%%%%%%%%%%%%%%%%%%%%%%%%
\subsection{Complexity and Discussion}
The computational cost of each iteration of Algorithm~\ref{alg:trpca} is dominated by ${\cal O}(pk^2)$ for computing the polar and ${\cal O}(pkn)$ for a supergradient of $\wtr(m,U)$ and, thus, has  total cost ${\cal O}(pk(k+n))$. We compare this to the cost of the proximal method in~\cite{CanEtAl2009},~\cite{WriEtAl2009}
for minimizing $\mathrm{min}_{X=A+E} \norm{A}_* + \lambda \norm{E}_1$. In each iteration, the dominating cost is ${\cal O}(\min\{pn^2,np^2\})$ for the SVD of a matrix of size $p \times n$. If the natural condition $k \ll \min\{p,n\}$ holds, we observe that the computational cost of TRPCA is significantly better. Thus even though we do $10$ random restarts with different starting vectors, our TRPCA is still faster than all competing methods, which can also be seen from the runtimes in Table~\ref{tab:runtime}.

In~\cite{MatGia2012}, a relaxed version of the trimmed reconstruction error is minimized:
\begin{equation}
\min_{m \in \rr^p,\,U \in S_k\, ,s \in \rr^k} \norm{X -  \mathbf{1}_n m^{\top} - Us - O}_F^2 + \lambda \norm{O}_{2,1},
\end{equation}
where $\norm{O}_{2,1}$ is added in order to enforce row-wise sparsity of $O$. The optimization is done via an alternating scheme. However, the disadvantage of this formulation is that it is difficult to adjust the number of outliers via the choice of $\lambda$ and thus requires multiple runs of the algorithm to find a suitable range, whereas in our formulation the number of outliers $n-t$ can be directly  controlled by the user or $t$ can be set to the default value $\ceil{\frac{n}{2}}$.

%%%%%%%%%%%%%%%%%%%%%%%%%%%%%%%%%%%%%%%%%%%%%%%%%
\section{Experiments}\label{sec:exp}
We compare our TRPCA (the code is available for download at ~\cite{Suppl}) algorithm with the following robust PCA methods: ORPCA~\cite{MatGia2012}, LLD\footnote{Note, that the LLD algorithm~\cite{MacTro2011} and the OPRPCA algorithm~\cite{XuCara2012} are equivalent.}~\cite{MacTro2011}, HRPCA~\cite{XuCara2013}, standard PCA, and true PCA on the true data $T$ (ground truth). For background subtraction, we also compare our algorithm with PCP~\cite{CanEtAl2009} and RPCA~\cite{TorBla2001}, although the latter two algorithms are developed for a different outlier model. 

To get the best performance of LLD and ORPCA, we run both algorithms with different values of the regularization parameters to set the number of zero rows (observations) in the outlier matrix equal to $\tilde t$ (which increases runtime significantly). The HRPCA algorithm has the same parameter $t$ as our method. 
 
We write $(0.5)$ in front of an algorithm name if the default value $\tilde t=\ceil{\frac{n}{2}}$ is used, otherwise, we use the ground truth information $\tilde t=|T|$.  As performance measure we use the reconstruction error relative to the reconstruction error of the true data (which is achieved by PCA on the true data only): 
\begin{equation}
 \begin{aligned}
\mathrm{tre}(U,m)=\frac{1}{t}\sum\nolimits_{\cbra{i\;|\;x_i\in T}}r_i(m,U) - r_i(\hat m_T,\hat U_T),
\end{aligned}
\end{equation}
where $\{\hat m_T, \hat U_T\}$ is the true PCA of $T$ and it holds that $tre(U,m)\ge0$. The smaller $tre(U,m)$, i.e., the closer the estimates $\cbra{m,U}$ to  $\{\hat m_T, \hat U_T\}$, the better. We choose datasets which are computationally feasible for all methods.

\begin{figure}
\centering
\begin{tabular}{ccc}
\includegraphics[width=.33\textwidth]{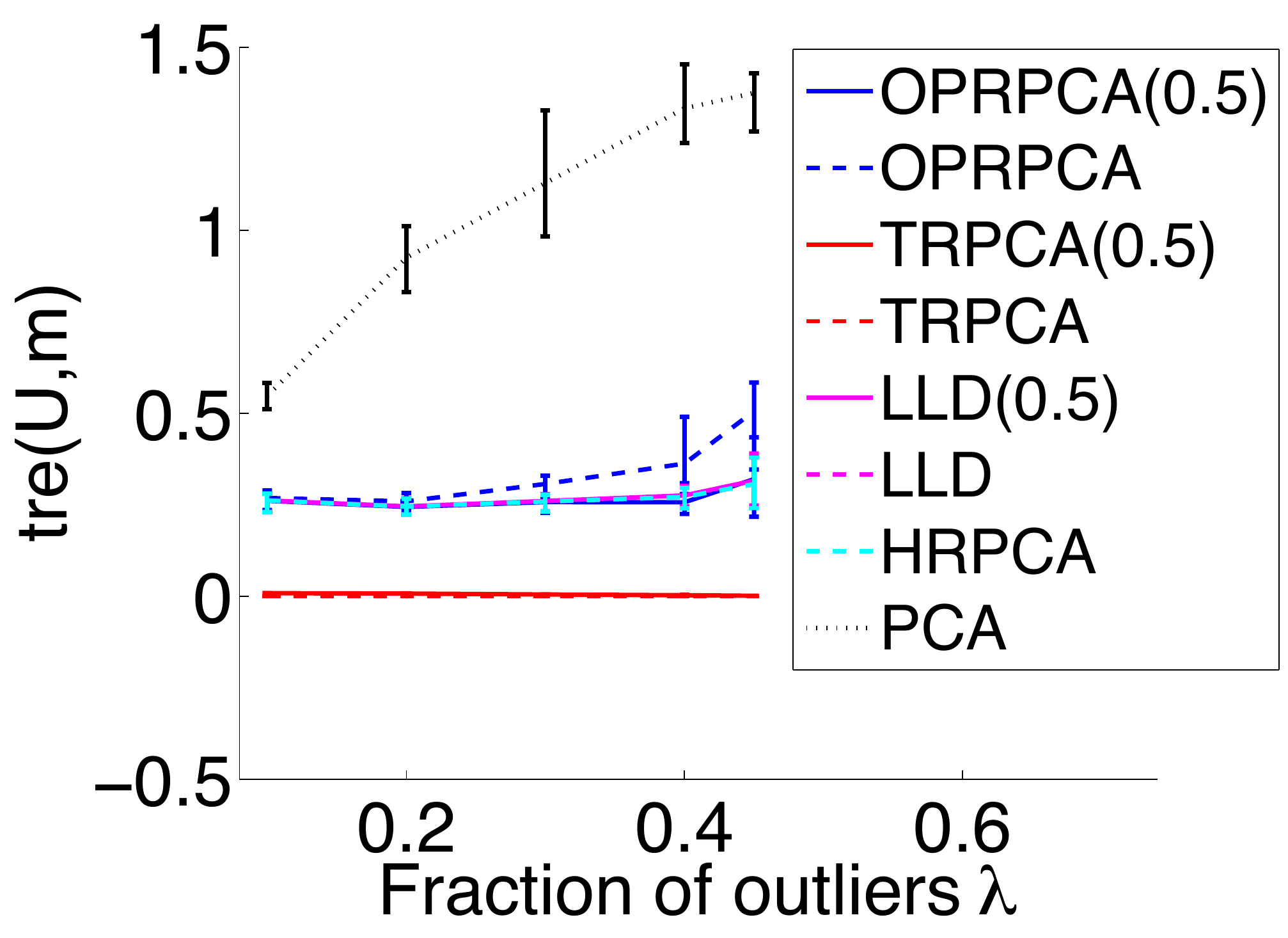} &
\includegraphics[width=.33\textwidth]{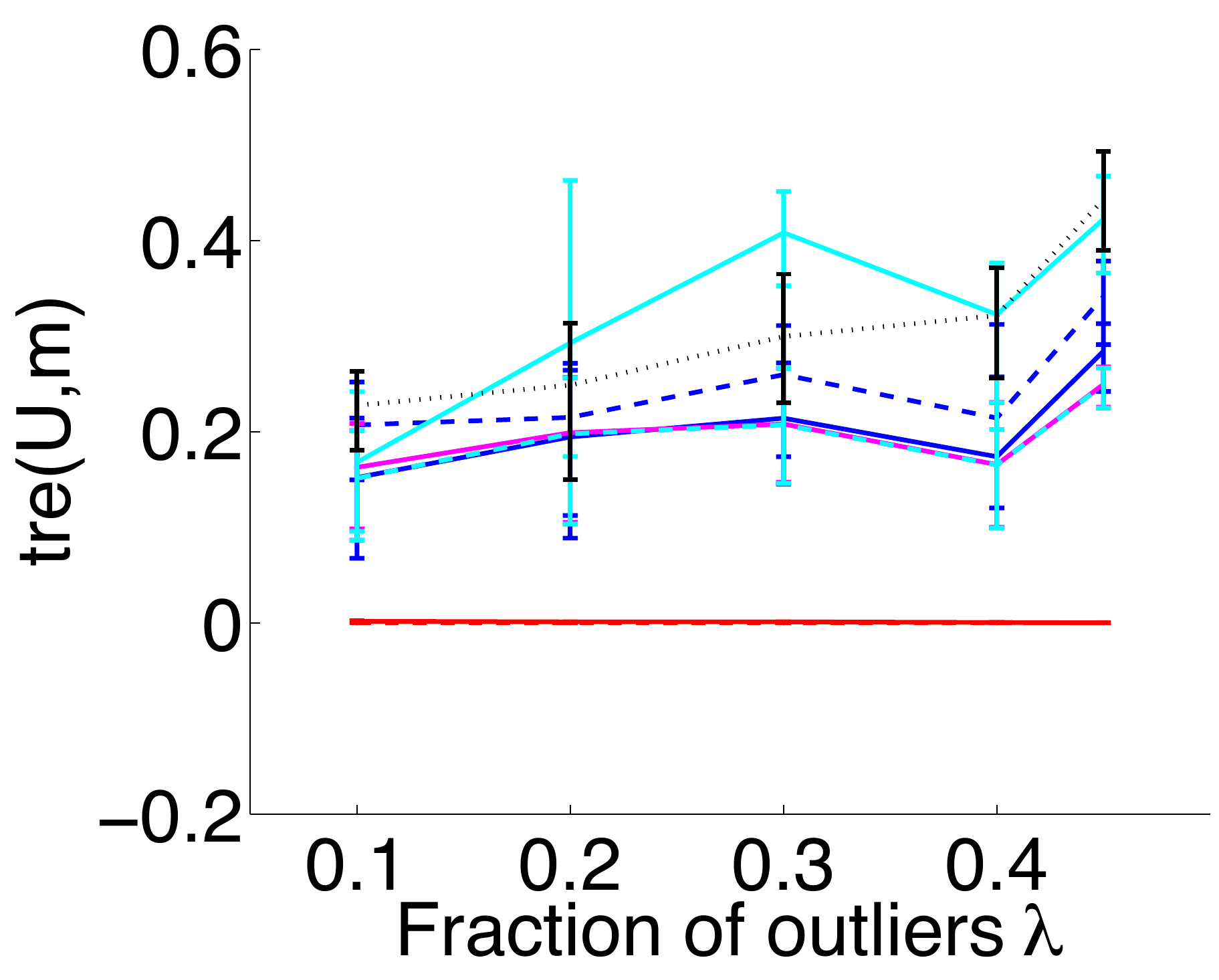} &
\includegraphics[width=.33\textwidth]{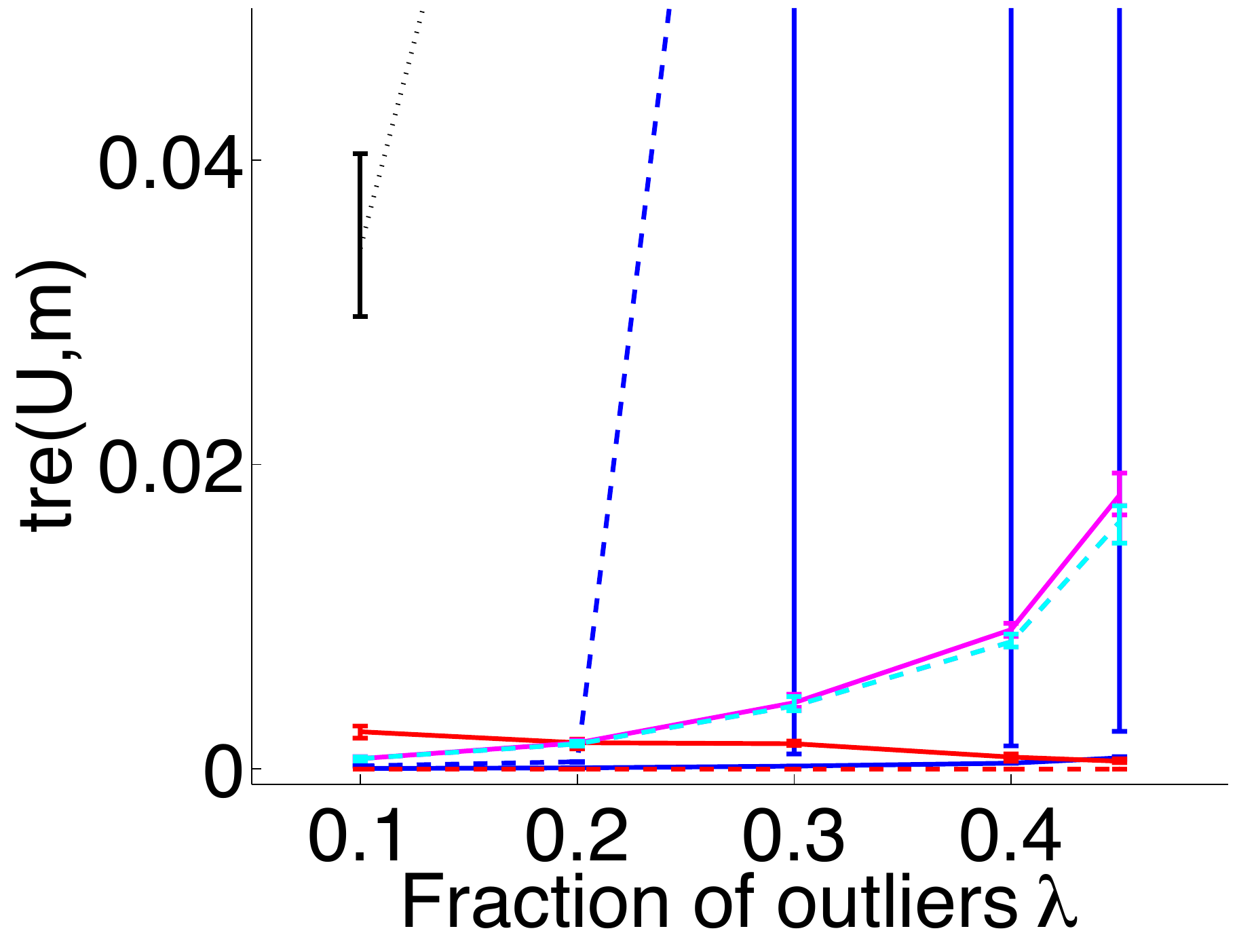} \\
\includegraphics[width=.33\textwidth]{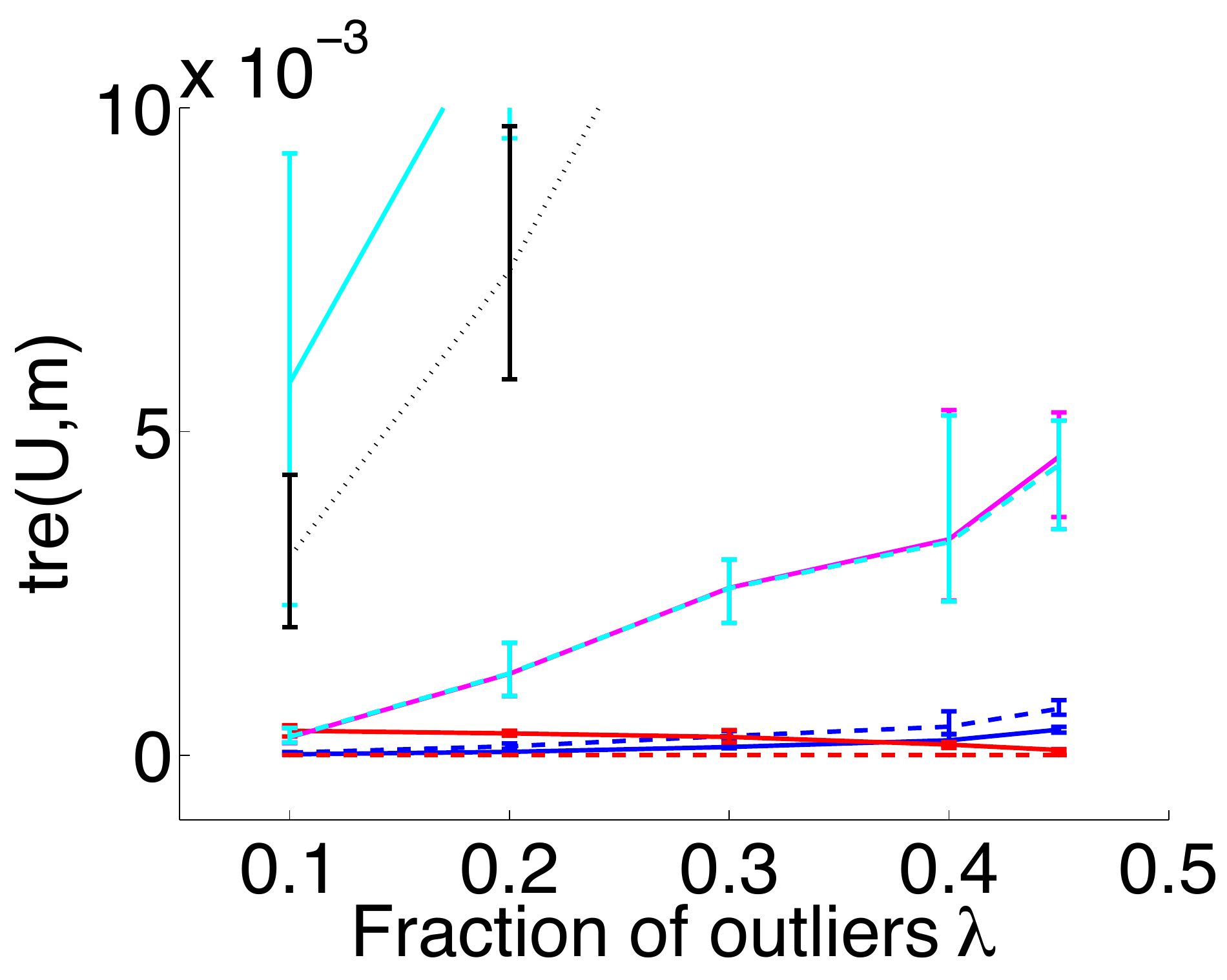} &
\includegraphics[width=.33\textwidth]{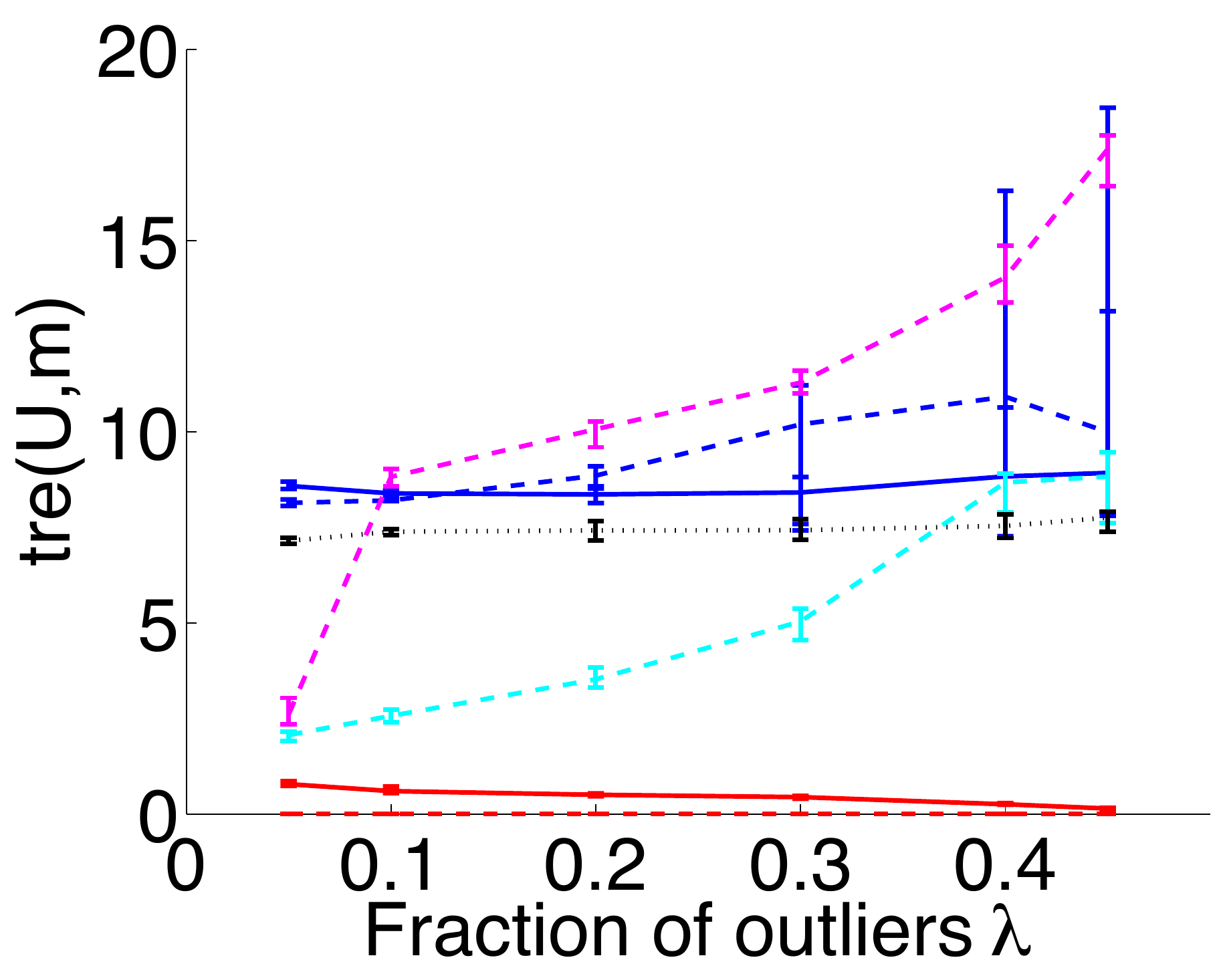} &
\includegraphics[width=.33\textwidth]{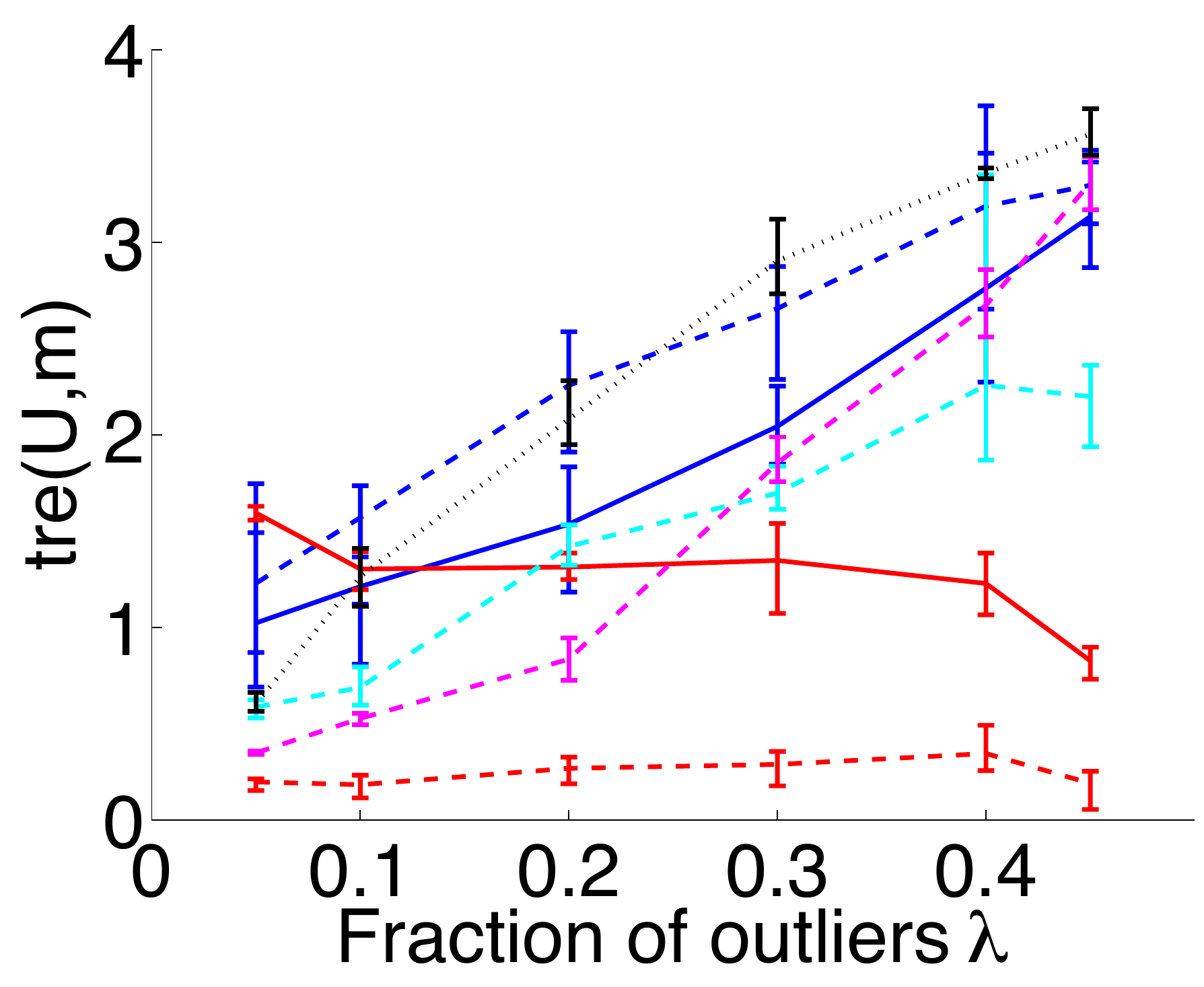}
\end{tabular}
\caption{ First row left to right: 1) Data1, $p=100$, $\sigma_o=2$; 2) Data1, $p=20$,
$\sigma_o=2$; 3) Data2, $p=100$, $\sigma_o=0.35$ ; Second row left to right: 1)
Data2, $p=20$, $\sigma_o=0.35$; 2) USPS10, $k=1$; 3) USPS10, $k=10$.
}
\label{fig:toy}
\end{figure} 

%%%%%%%%%%%%%%%%%%%%%%%%%%%%%%%%%%%%%%%%%%%%%%%%%%
\subsection{Synthetic Data Sets}\label{sec:synthetic}
We sample uniformly at random a subspace of dimension $k$ spanned by $U \in S_k$ and generate the true data $T \in \rr^{t \times p}$ as $T=AU^{\top}+E$ where the entries of $A \in \rr^{t \times k}$ are sampled uniformly on $[-1,1]$ and the noise $E \in \rr^{t \times p}$ has Gaussian entries distributed as $\mathcal{N}(0,\sigma_T)$. We consider two types of outliers: (Data1) the outliers $O \in \rr^{o \times p}$ are uniform samples from $[0,\sigma_o]^p$, (Data2) the outliers are samples from a random half-space, let $w$ be sampled uniformly at random from the unit sphere and let $x \sim \mathcal{N}(0,\sigma_0\mathds{1})$  then an outlier $o_i \in \rr^p$ is generated as $o_i= x - \max\{\inner{x,w},0\}w$. For Data2, we also downscale true data by $0.5$ factor. We always set $n=t+o=200$, $k=5$, and $\sigma_T=0.05$ and construct data sets for different fractions of outliers $\lambda=\frac{o}{t+o}\in \cbra{0.1,\,0.2,\,0.3,\,0.4,\,0.45}$. For every $\lambda$ we sample 5 data sets and report mean and standard deviation of the relative true reconstruction error $\mathrm{tre}(U,m)$.

%%%%%%%%%%%%%%%%%%%%%%%%%%%%%%%%%%%%%%%%%%%%%%%%%%%
\subsection{Partially Synthetic Data Set}
We use USPS, a dataset of $16 \times 16$ images of handwritten digits. We use digits 1 as true observations $T$ and digits 0 as outliers $O$ and mix them in different proportions. 
We refer to this data set as USPS10 and the results can be found in Fig.~\ref{fig:toy}. 
%%%%%%%%%%%%%%%%%%%%%%%
%%%%% SUPPLEMENT %%%%%%%%%%
Another similar experiment is on the MNIST data set of $28 \times 28$ images of handwritten digits. We use digits $1$ (or $7$) as true observations $T$ and all other digits $0,2,3,\dots,9$ as outliers $O$ (each taken in equal proportion). We mix true data and outliers in different proportions and the results can be found in Fig.~\ref{fig:mnist} (or Fig.~\ref{fig:mnist7}), where we excluded LLD due to its low computational time, see Tab.~\ref{tab:runtime}.
%%%%% SUPPLEMENT %%%%%%%%%%
%%%%%%%%%%%%%%%%%%%%%%%
We notice that TRPCA algorithm with the parameter value $\tilde t=t$ (ground truth information) performs almost perfectly and outperforms all other methods, while the default version of TRPCA with parameter $\tilde t=\ceil{\frac{n}{2}}$ shows slightly worse performance. 
The fact that TRPCA estimates simultaneously the robust center $m$ influences positively the overall performance of the algorithm, see, e.g., the experiments for background subtraction and modeling in Section~\ref{sec:bms} and additional ones in the supplementary material. 
%%%%%%%%%%%%%%%%%%
%%%%% SUPPLEMENT %%%%%
That is Fig.~\ref{fig:wsorig}-\ref{fig:rmores}.

\begin{figure}
\centering
\begin{tabular}{cc}
\includegraphics[width=.48\columnwidth]{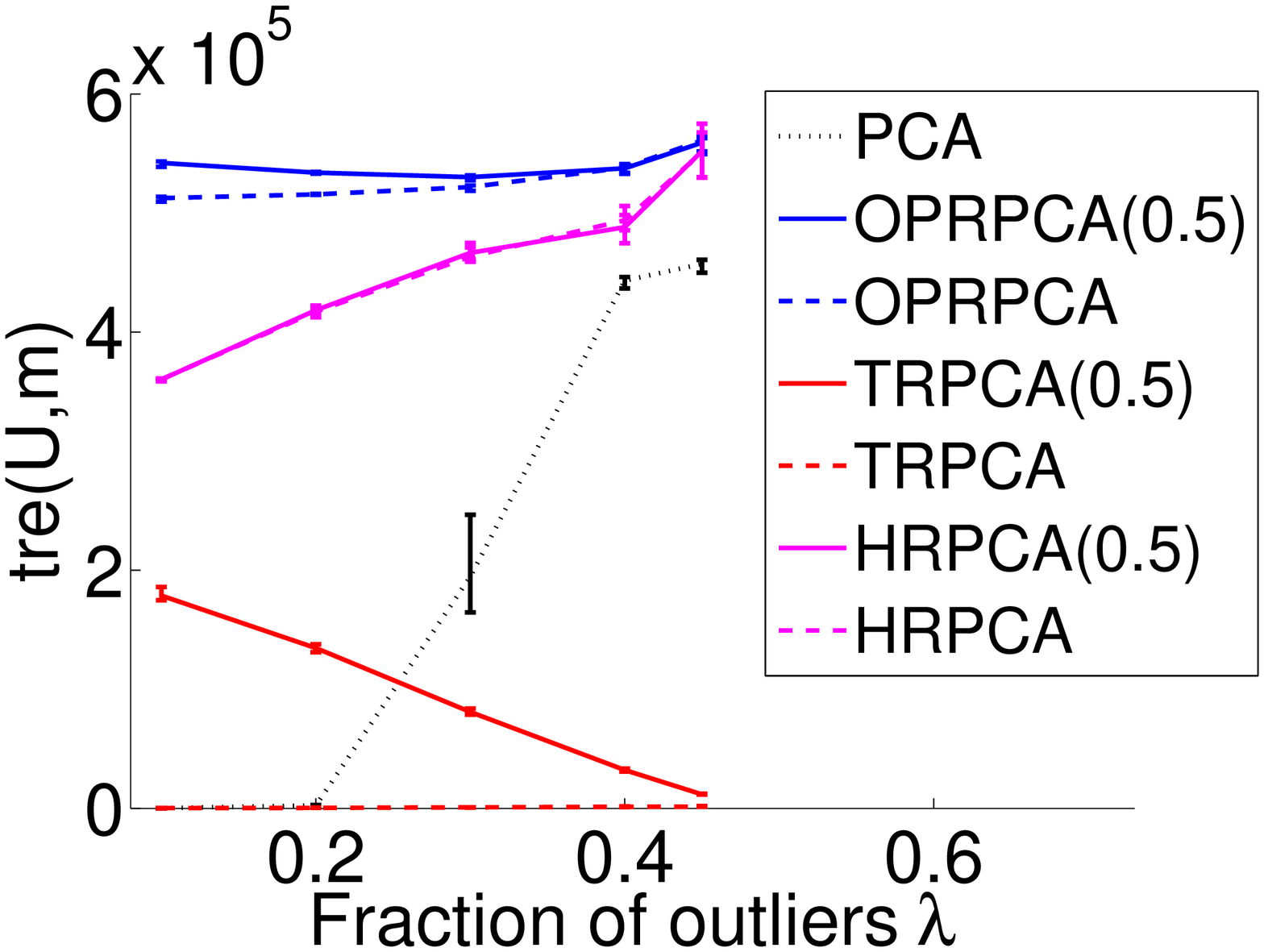} &
\includegraphics[width=.48\columnwidth]{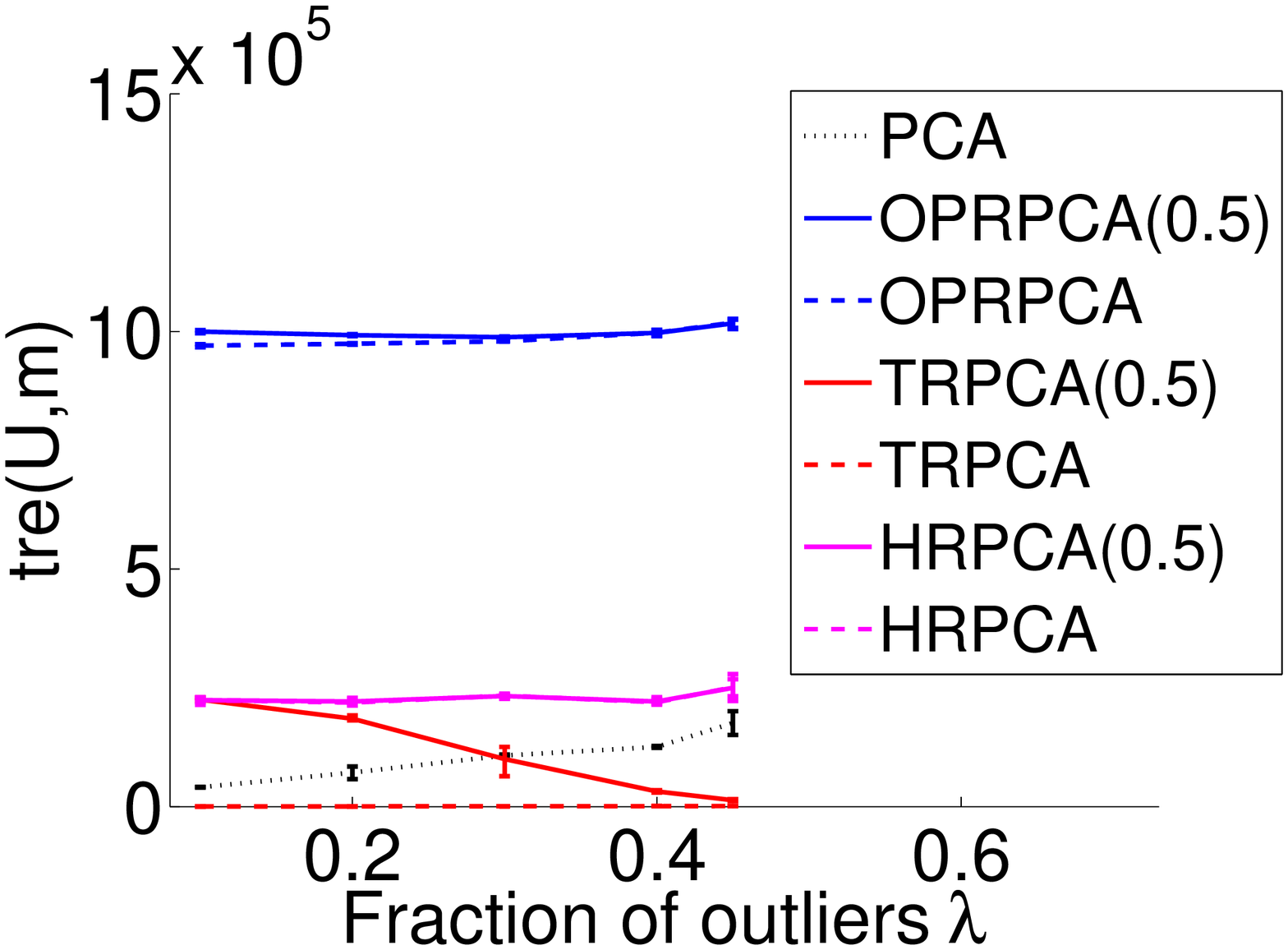}
\end{tabular}
\caption{
Experiment on the MNIST data set with digits 1 as true observations $T$ and all other digits $0,2,3,\dots,9$ as outliers. Number of recovered PCs is $k=1$ (left) and $k=5$ (right).
}
\label{fig:mnist}
\end{figure}
\begin{figure}
\centering
\begin{tabular}{cc}
\includegraphics[width=.48\columnwidth]{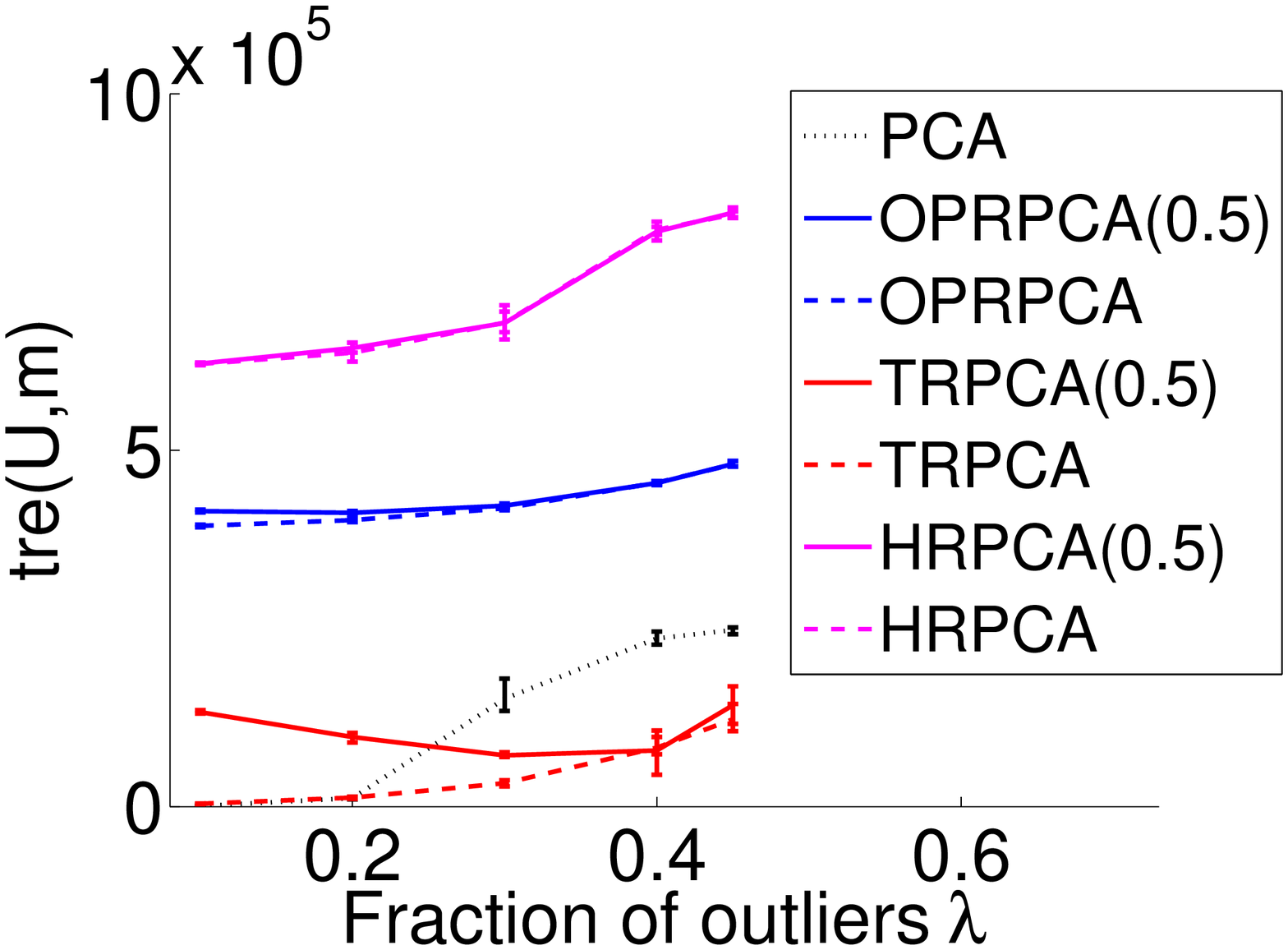} &
\includegraphics[width=.48\columnwidth]{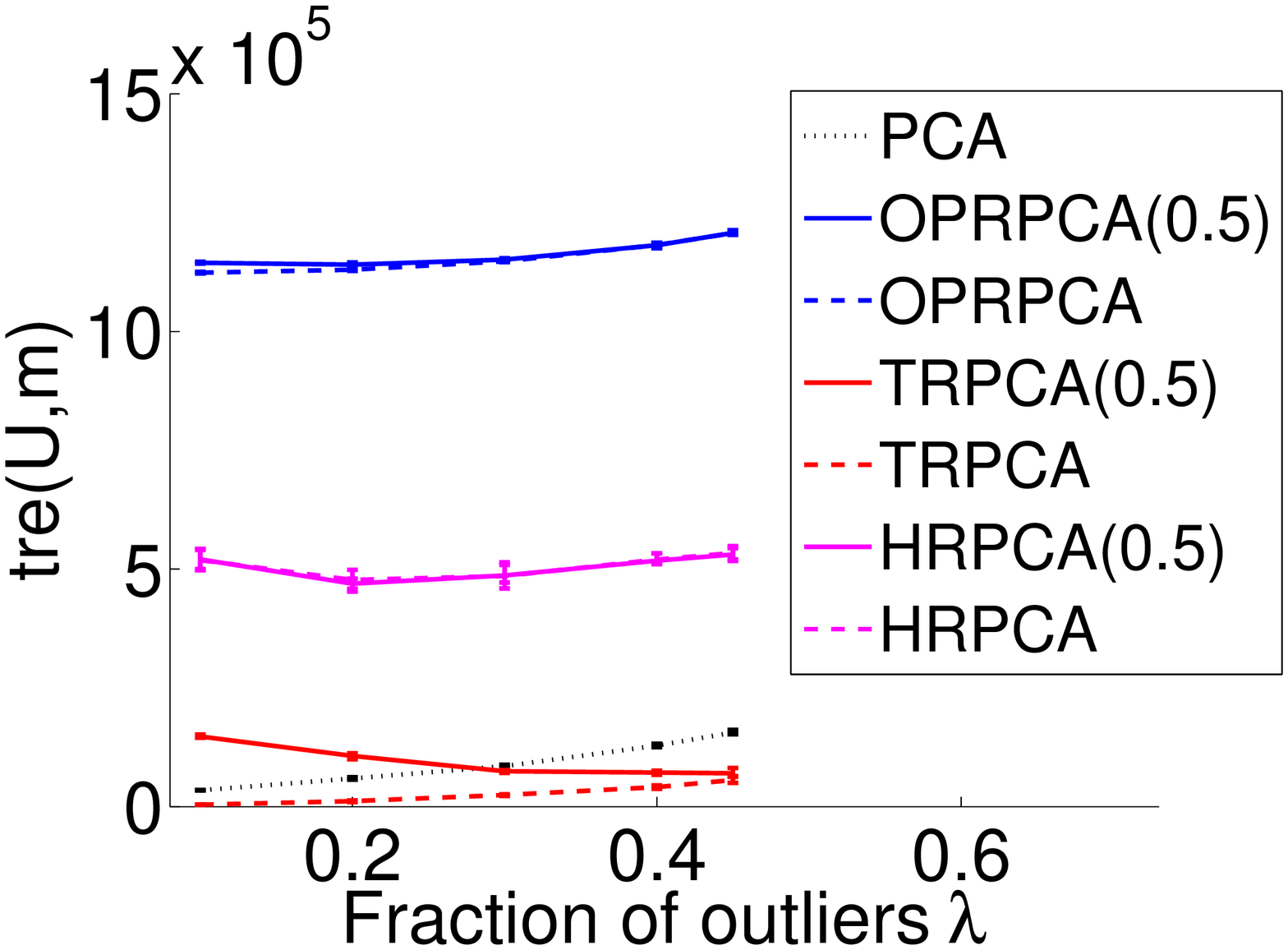}
\end{tabular}
\caption{
Experiment on the MNIST data set with digits 7 as true observations $T$ and all other digits $0,2,3,\dots,9$ as outliers. Number of recovered PCs is $k=1$ (left) and $k=5$ (right).
}
\label{fig:mnist7}
\end{figure}
%%%%%%%%%%%%%%
%%% SUPPLEMENT %%%
%%%%%%%%%%%%%%

%%%%%%%%%%%%%%%%%%%%%%%%%%%%%%%%%%%%%%%%%%%%%%%%%%%%
\begin{figure}
\centering
\begin{tabular}{ccccc}
\includegraphics[width=.194\columnwidth]{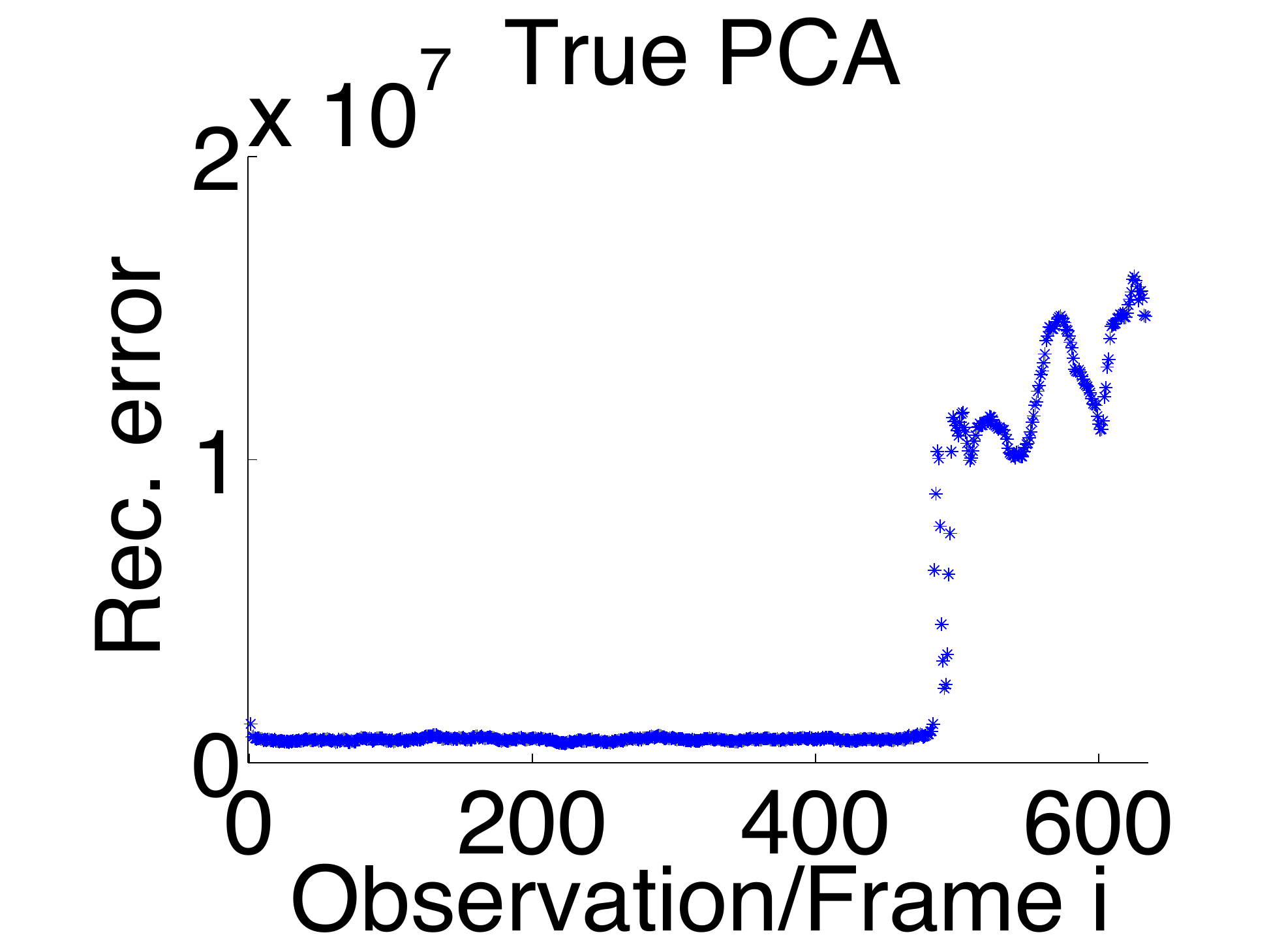} &
\includegraphics[width=.194\columnwidth]{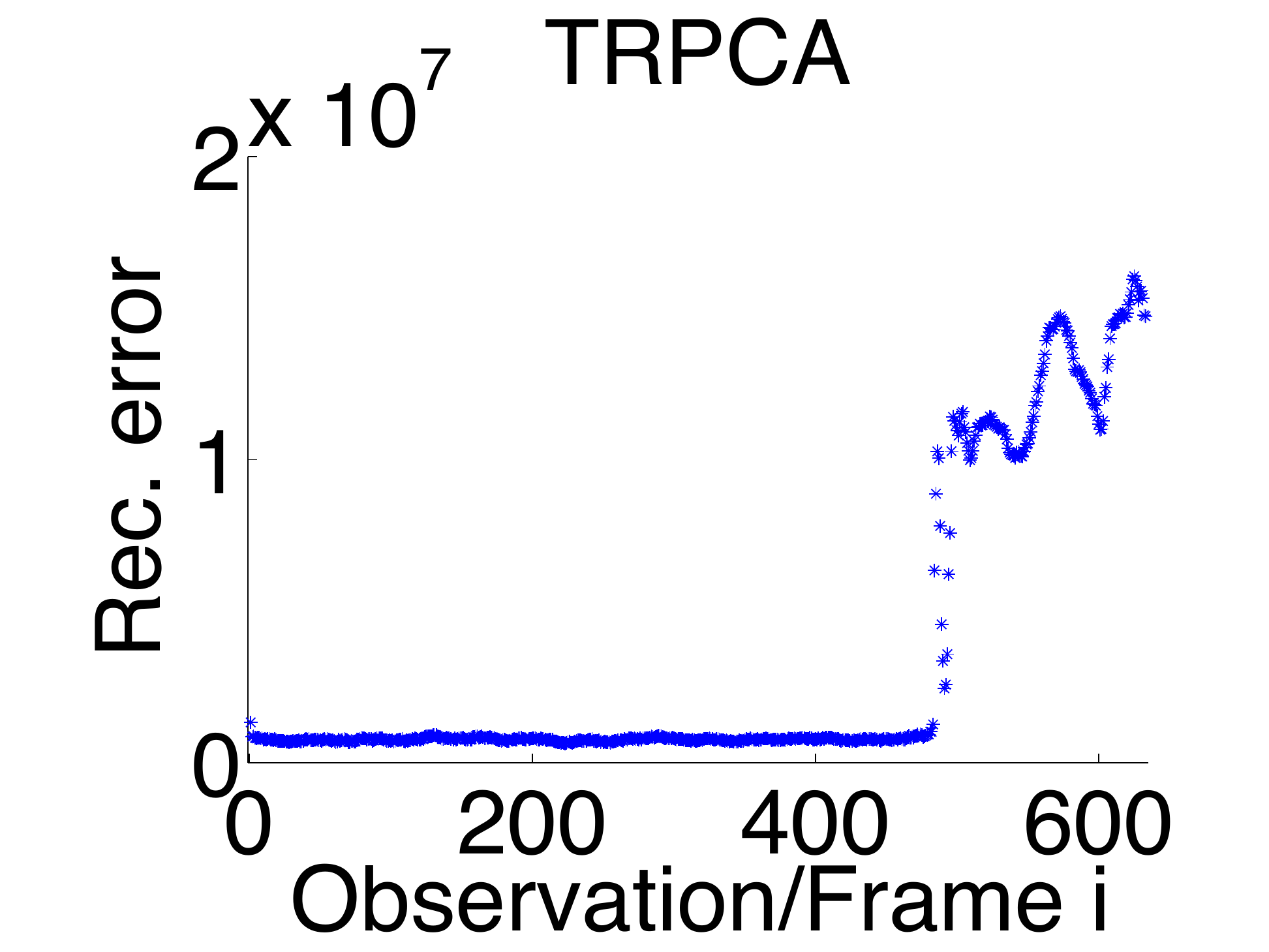} &
\includegraphics[width=.194\columnwidth]{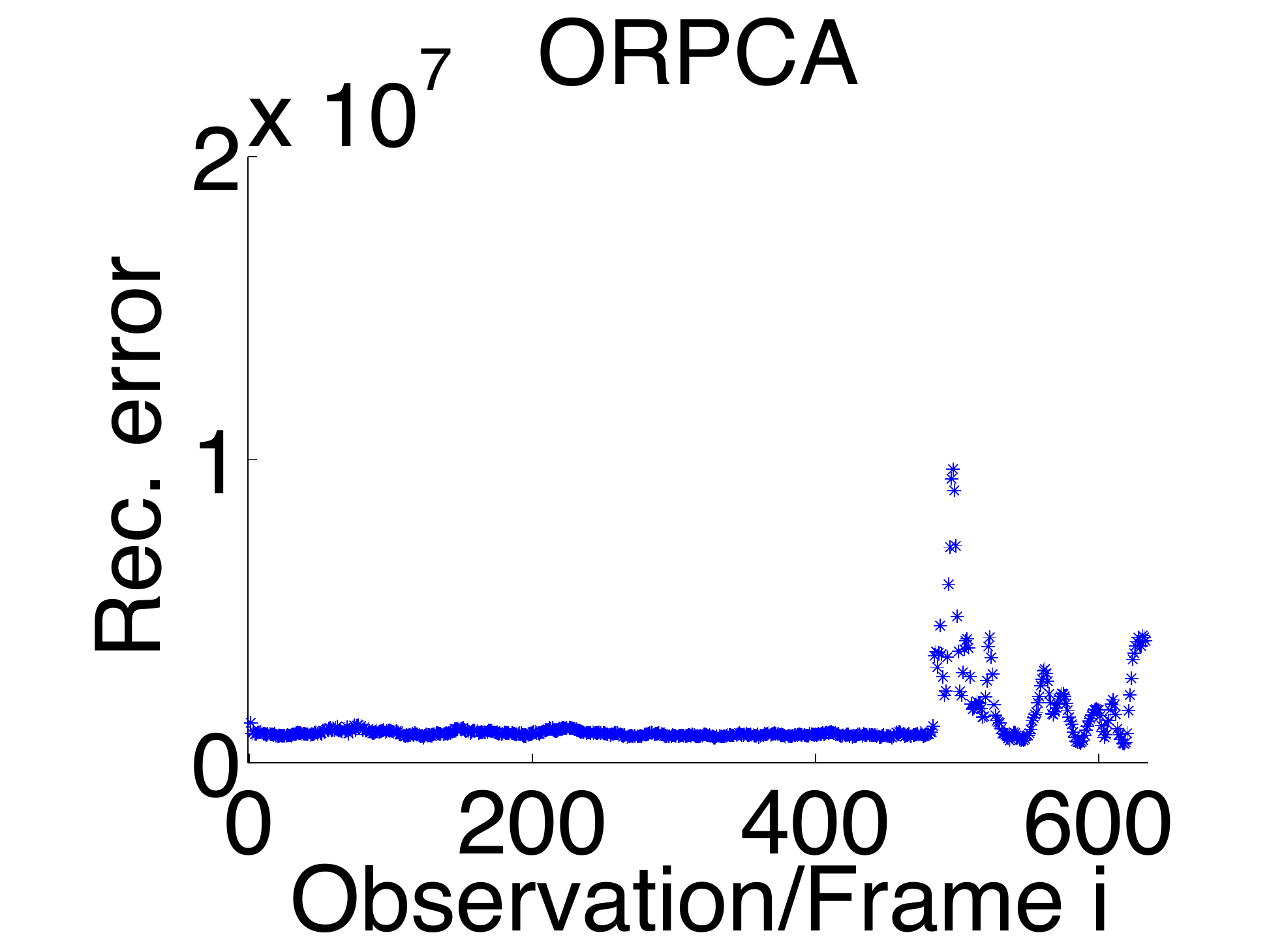} &
\includegraphics[width=.194\columnwidth]{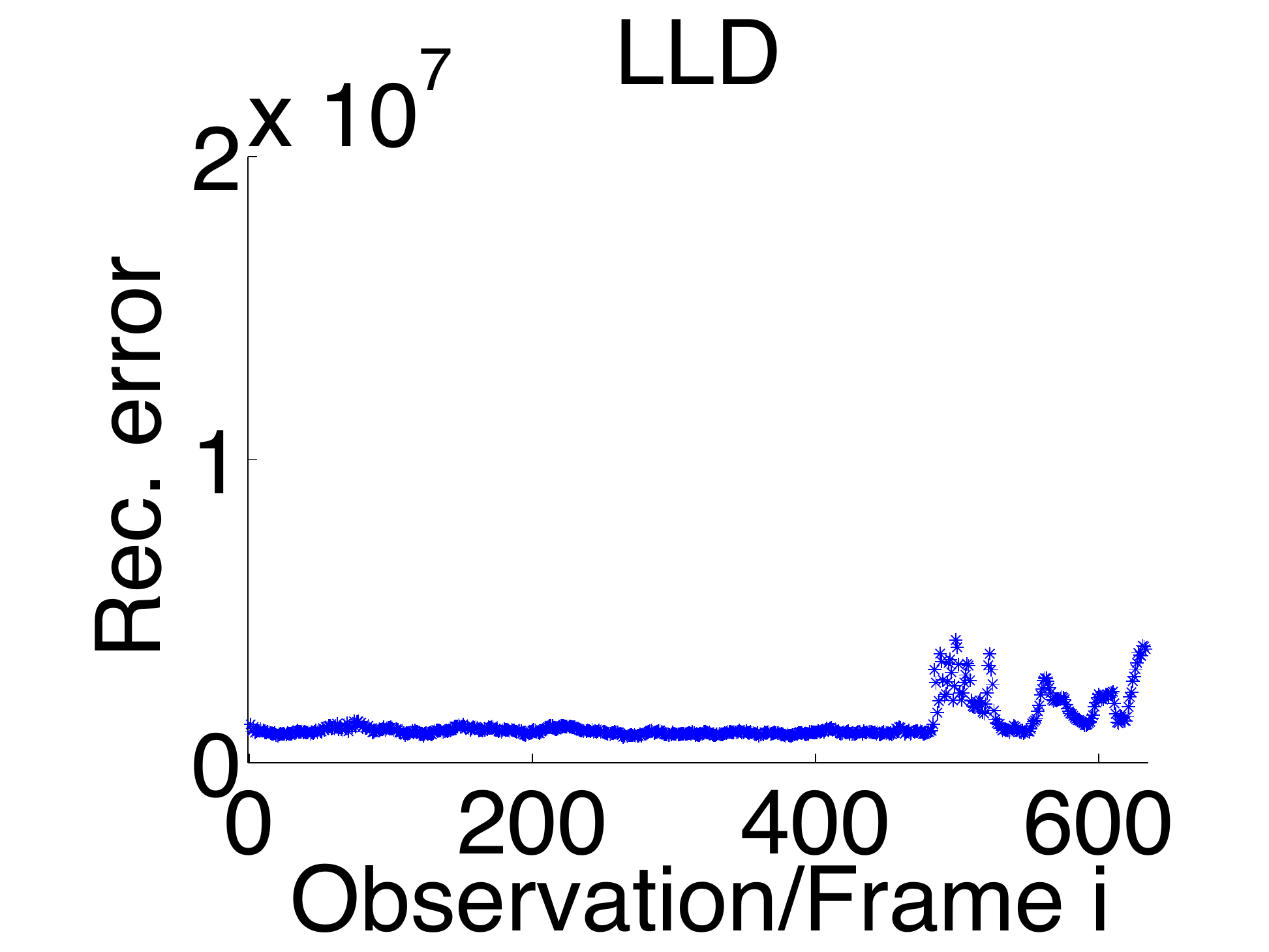} &
\includegraphics[width=.194\columnwidth]{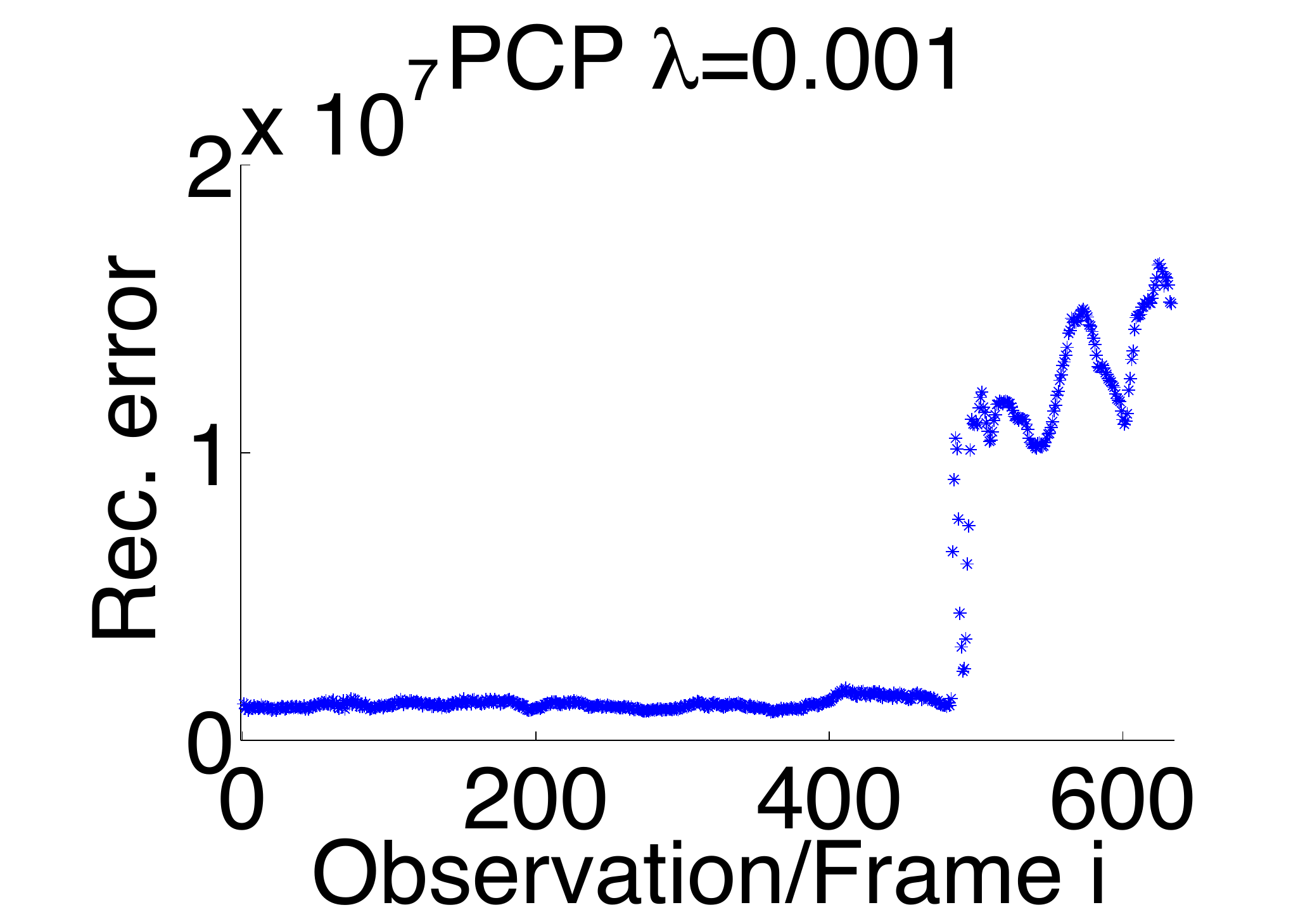} \\
\includegraphics[width=.194\columnwidth]{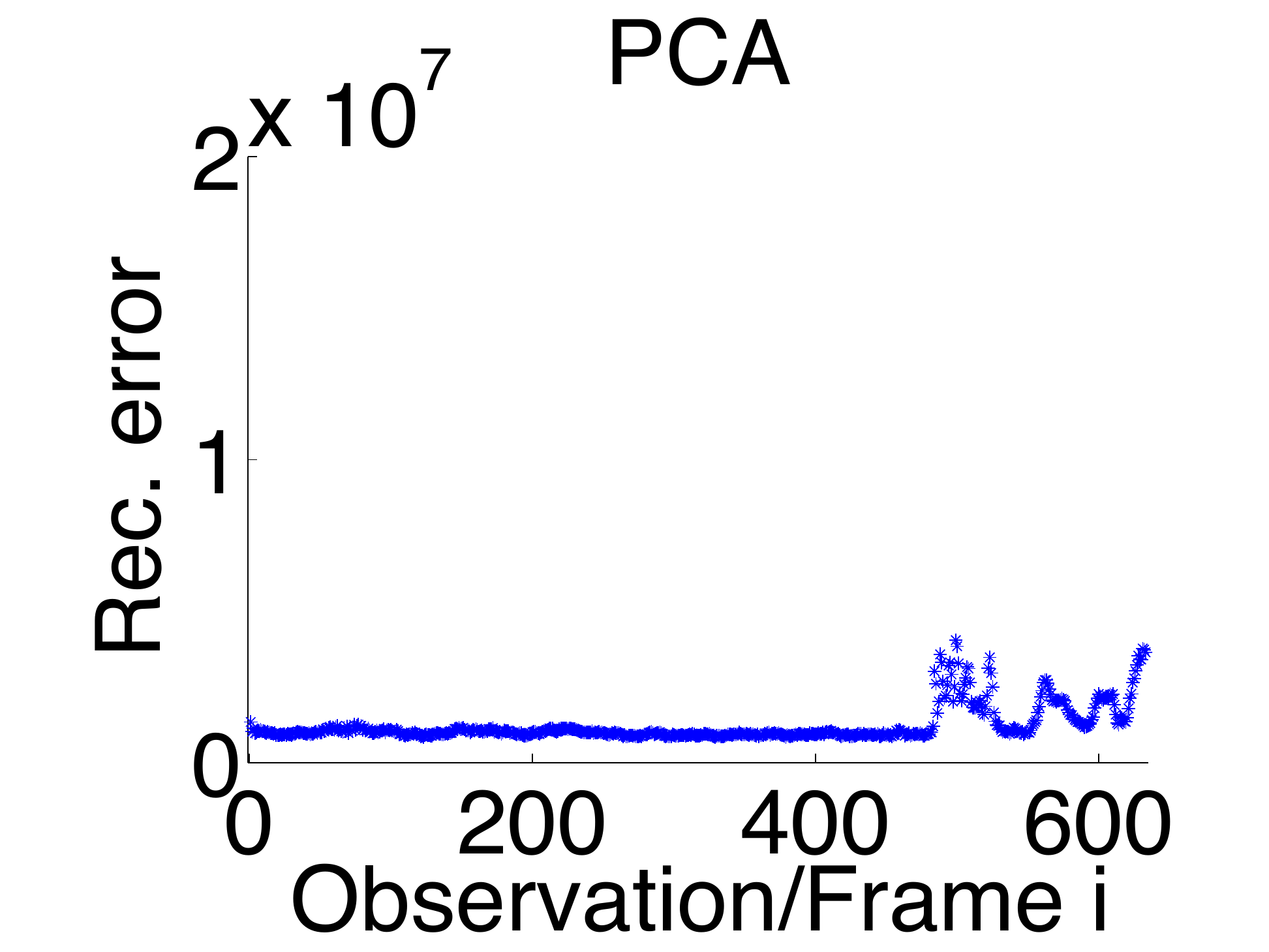} &
\includegraphics[width=.194\columnwidth]{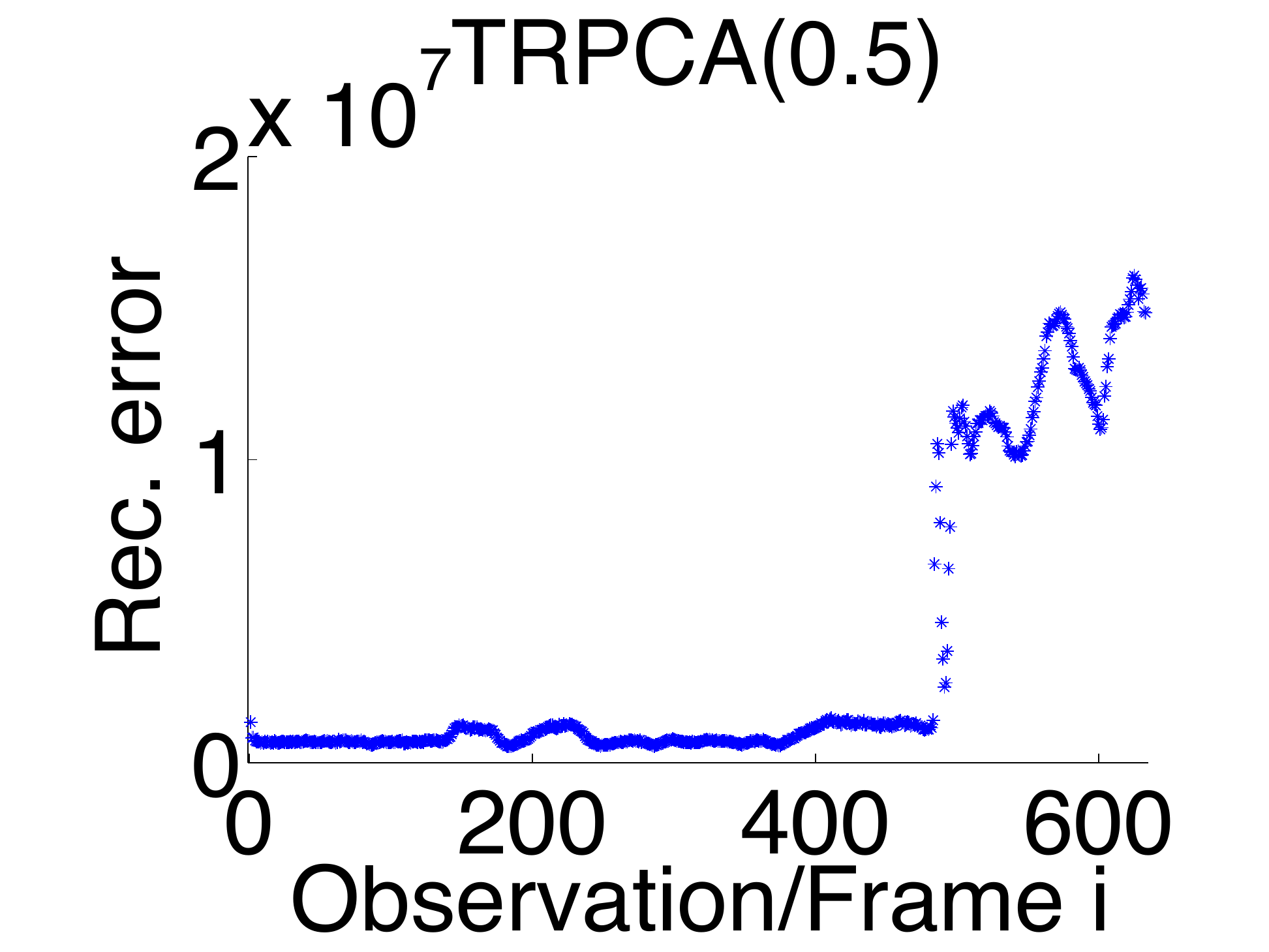} &
\includegraphics[width=.194\columnwidth]{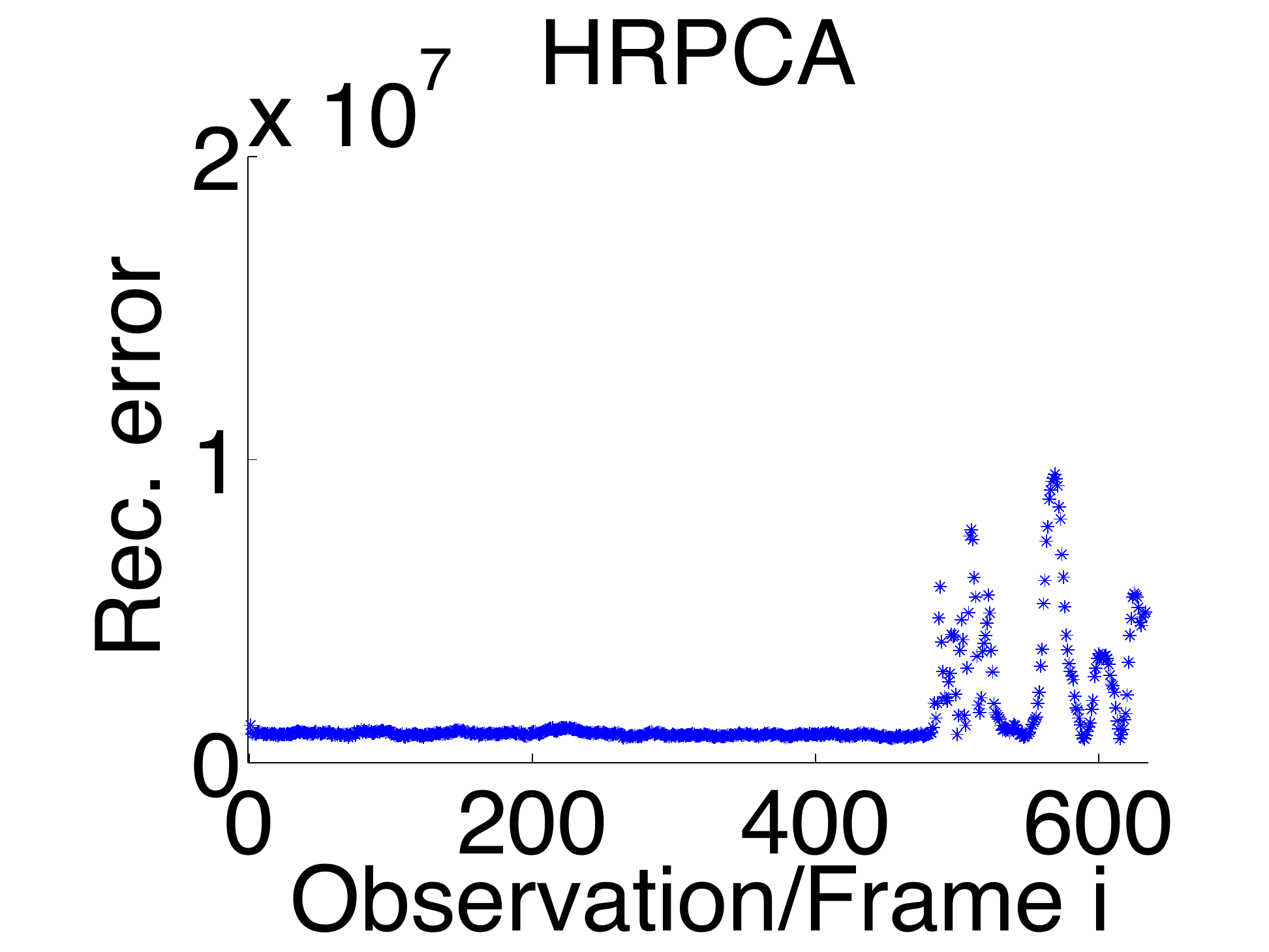}  &
\includegraphics[width=.194\columnwidth]{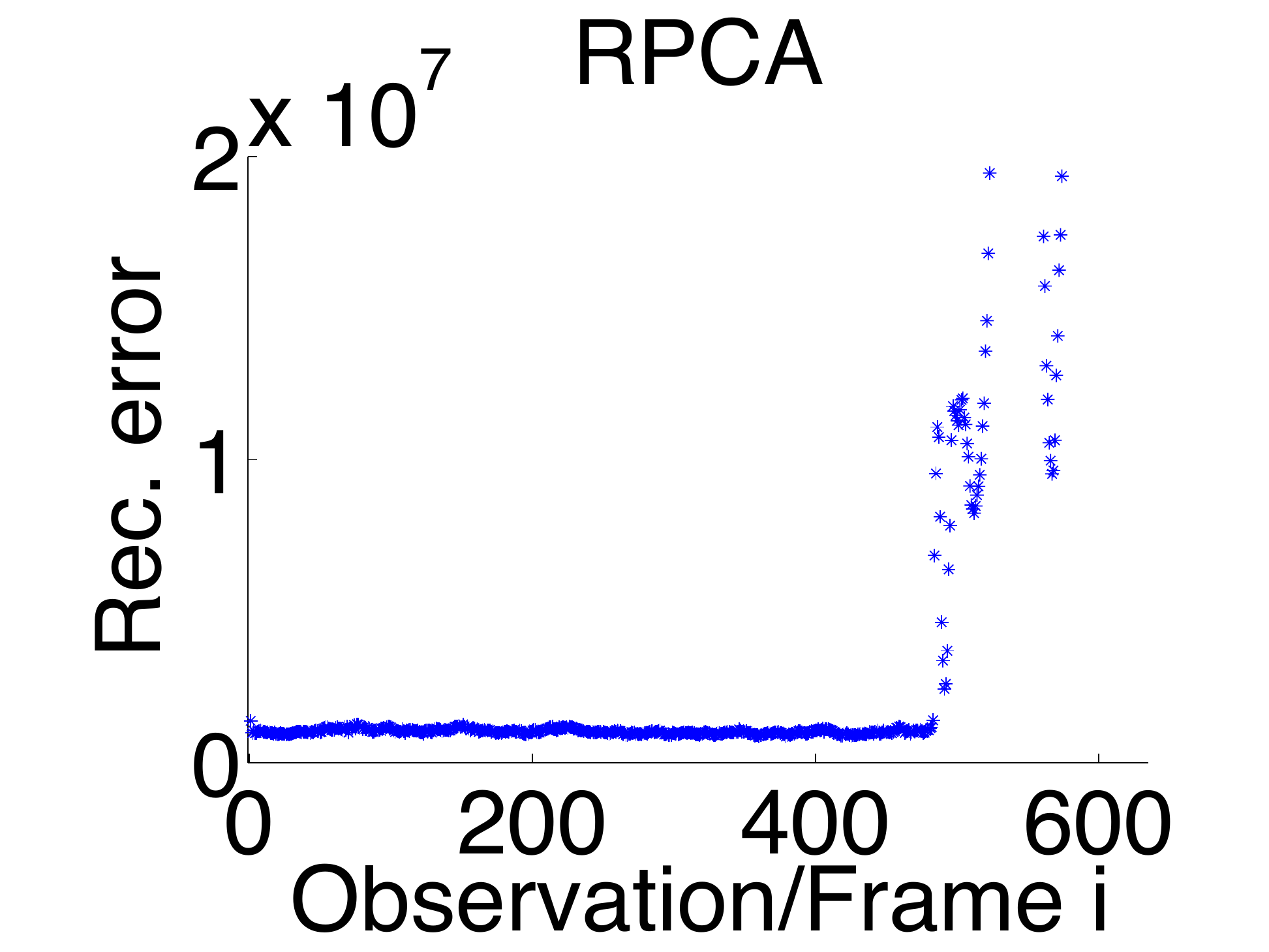} &
\includegraphics[width=.194\columnwidth]{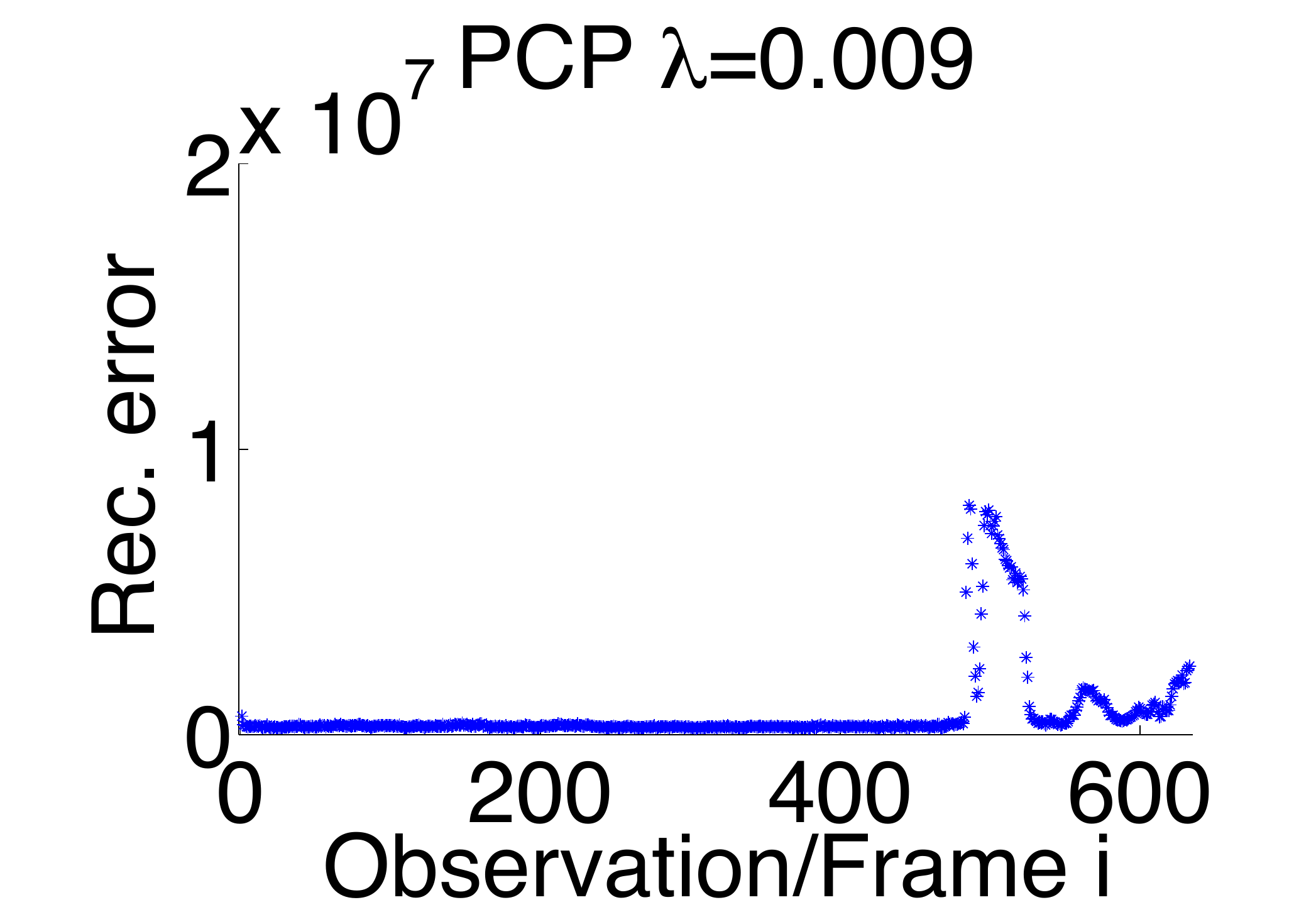}
\end{tabular}
\caption{Reconstruction errors, i.e., $||(x_i-m^*)-U^*\rbra{U^*}^{\top}(x_i-m^*)||_2^2$, on the y-axis, for each frame on the x-axes for $k=10$. Note that the person is visible in the scene from frame 481 until the end. We consider the background images as true data and, thus, the reconstruction error should be high after frame 481 (when the person enters).}
\label{fig:imagesresWS}
\end{figure}

%%%%%%%%%%%%%%%%%%%%%%%%%%%%%%%%%%%%%%%%%%%%%%%%%%
\subsection{Background Modeling and Subtraction} \label{sec:bms}
In~\cite{TorBla2001} and~\cite{CanEtAl2009} robust PCA has been proposed as a method for background modeling and subtraction. While we are not claiming that robust PCA is the best method to do this, it is an interesting test for robust PCA. The data $X$ are the image frames of a video sequence. The idea is that slight change in the background leads to a low-rank variation of the data whereas the foreground changes cannot be modeled by this and can be considered as outliers. Thus with the estimates $m^*$ and $U^*$ of the robust PCA methods, the solution of the background subtraction and modeling problem is given as
\begin{equation}\label{foreground}
x_i^{b}=m^*+U^*(U^*)^{\top}(x_i-m^*)
\end{equation}
where $x_i^b$ is the background of frame $i$ and its foreground is simply $x_i^f=x_i-x_i^b$.

We experimentally compare the performance of all robust PCA methods on the water surface data set~\cite{WaterSurf}, which has moving water in its background. We choose this dataset of $n=633$ frames each of size $p=128\times 160=20480$ as it is computationally feasible for all the methods. In Fig.~\ref{fig:bfwsi560}, we show the background subtraction results of several robust PCA algorithms. We optimized the value $\lambda$ for PCP of~\cite{CanEtAl2009},~\cite{WriEtAl2009} by hand to obtain a good decomposition, see the bottom right pictures of Fig.~\ref{fig:bfwsi560}. How crucial the choice of $\lambda$ is for this method can be seen from the bottom right pictures. Note that the reconstruction error of both the default version of TRPCA and TRPCA(0.5) with ground truth information provide almost perfect reconstruction errors with respect to the true data, cf., Fig.~\ref{fig:imagesresWS}. Hence, TRPCA is the only method which recovers the foreground and background without mistakes. We refer to the supplementary material for more explanations regarding this experiment as well as results for another background subtraction data set. The runtimes of all methods for the water surface data set are presented in Table~\ref{tab:runtime}, which shows that TRPCA is the fastest of all methods.

%%%%%%%%%%%%%%%%%%%%%%%%%%%%%%%%%%%%%%%%%%%%%%%%%%%%
\begin{table}
\footnotesize{
\centering
\caption{
Runtimes for the water surface data set for the algorithms described in Section~\ref{sec:exp}. For TRPCA/TRPCA(0.5) we report the average time of one initialization (in practice, $5-10$ random restarts are sufficient). For PCP we report the runtime for the employed parameter $\lambda=0.001$.
For all others methods, it is the time of one full run of the algorithm including the search for regularization parameters.
}

\begin{center}
\begin{tabular}{| l | r | r | r | r | r | r | r | r | r |} \hline
           & trpca  & trpca(.5) & orpca & orpca(.5) & hrpca & hrpca(.5) & lld & rpca & pcp($\lambda=0.001$)  \\  \hline
$k=1$ & $7$ & ${13}$ & $3659$ & $3450$ & $45990$ & $48603$ & $-$ & $1078$ & $-$ \\ \hline
$k=3$ & ${99}$ &  ${61}$ & $8151$ & $13852$  & $50491$ & $56090$ & $-$ & $730$ & $-$ \\ \hline
$k=5$ & ${64}$ & ${78}$ & $2797$ & $3726$ & $72009$ & $77344$ & $232667$ & $3615$ & $875$ \\ \hline
$k=7$ & ${114}$ & ${62}$ &  $4138$ & $3153$ & $67174$ & $90931$ & $-$ & $4230$ & $-$ \\ \hline
$k=9$ & ${119}$ & ${92}$ & $6371$ & $8508$ & $96954$ &  $106782$ & $-$ & $4113$ & $-$ \\ \hline
\end{tabular} \label{tab:runtime}
\end{center}
}
\end{table}

\section{Conclusion}
We have presented a new method for robust PCA based on the trimmed reconstruction error. Our efficient algorithm, using fast descent on the Stiefel manifold, works in the default setting ($t=\ceil{\frac{n}{2}}$) without any free parameters and is significantly faster than other competing methods. In all experiments TRPCA performs better or at least similar to other robust PCA methods, in particular, TRPCA solves challenging background subtraction tasks.\\[.2cm]

\noindent
\textbf{Acknowledgements.} M.H. has been partially supported by the ERC Starting Grant NOLEPRO and
M.H. and S.S. have been partially supported by the DFG Priority Program 1324, ``Extraction of quantifiable information from complex systems".

%%%%%%%%%%%%%%%%%%%%%%%%%%%%%%%%%%%%%%%%%%%%%%%%%%%%
\begin{figure}
\centering
\begin{tabular}{cccc}
\includegraphics[width=.25\columnwidth]{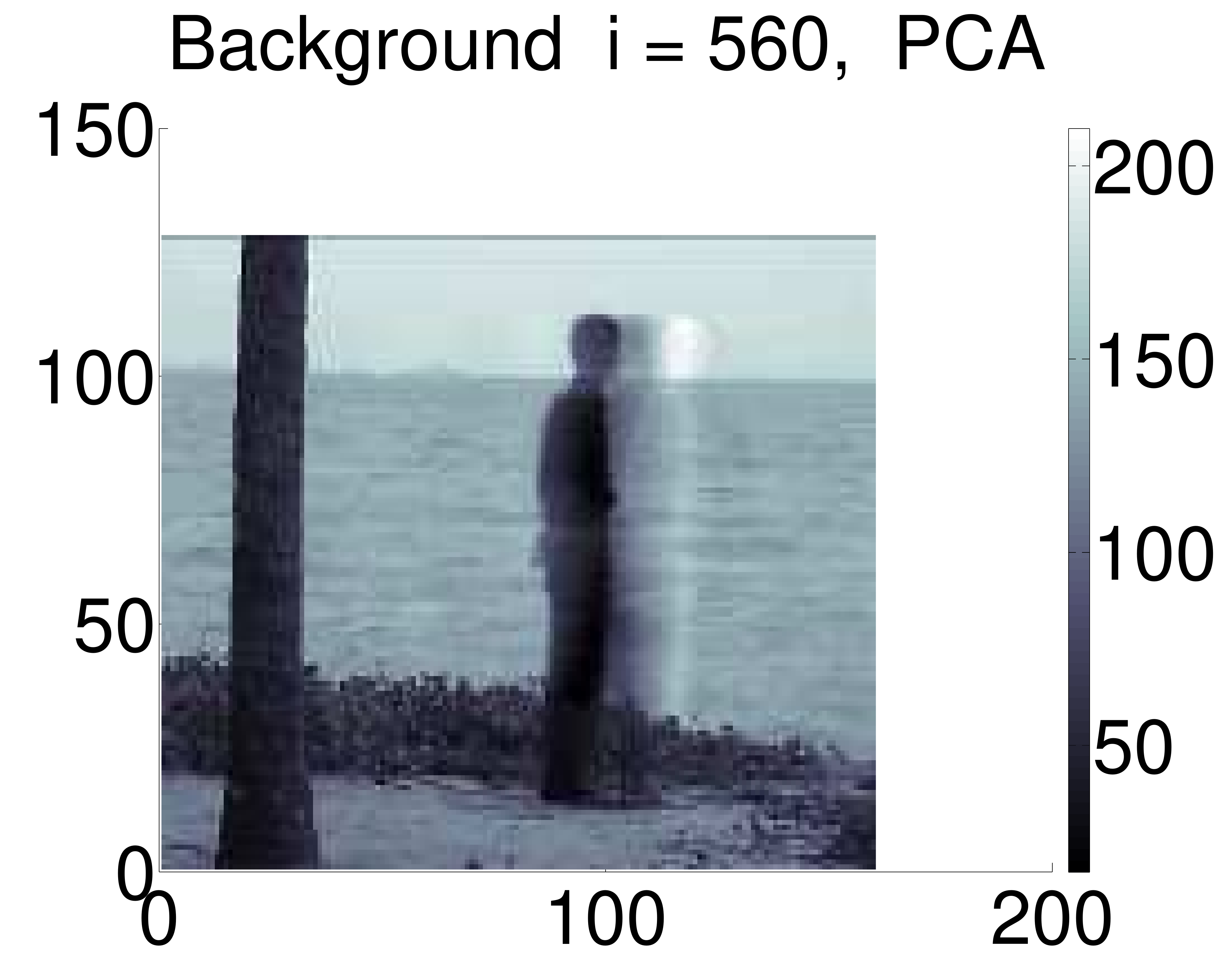} & 
\includegraphics[width=.25\columnwidth]{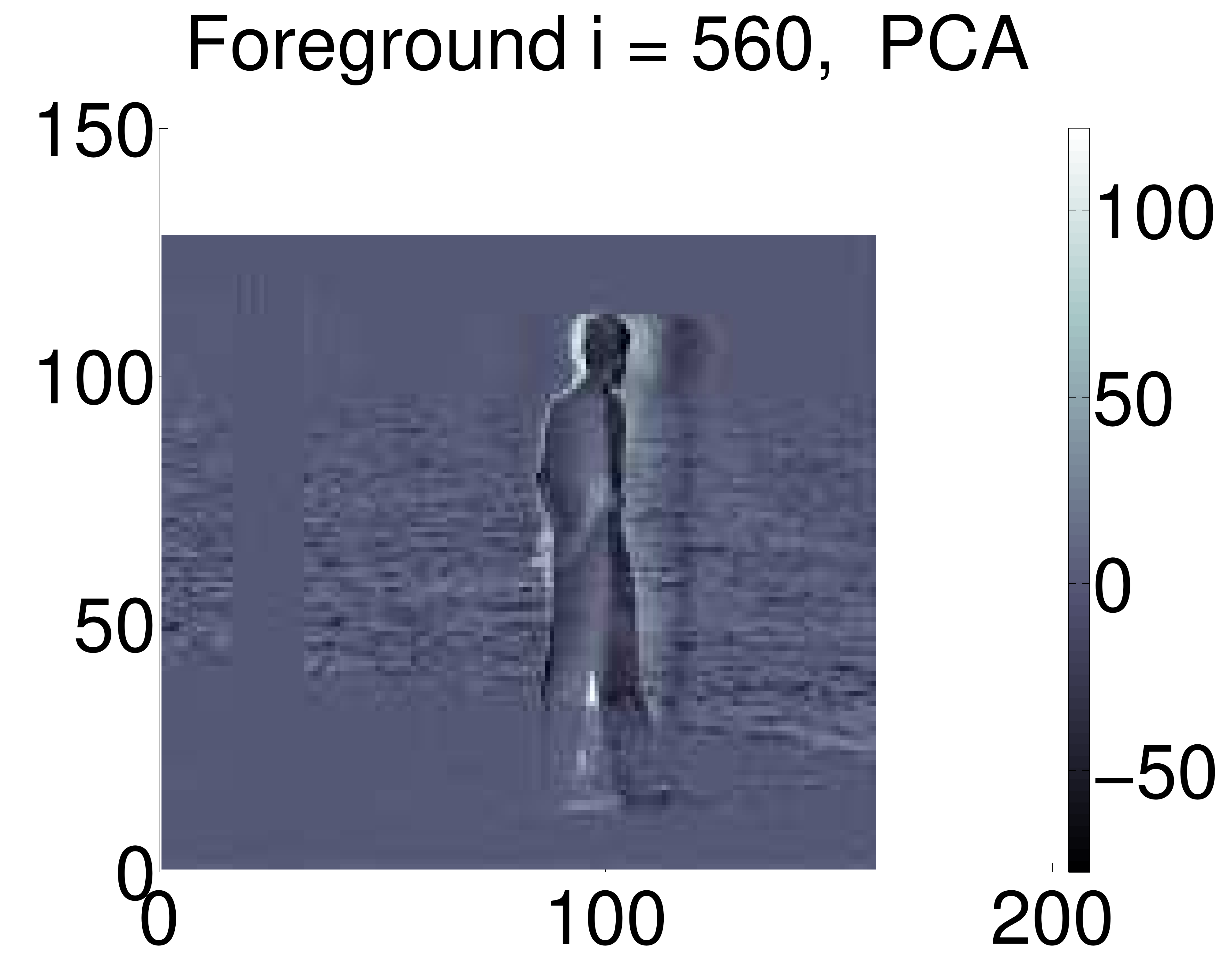} &
\includegraphics[width=.25\columnwidth]{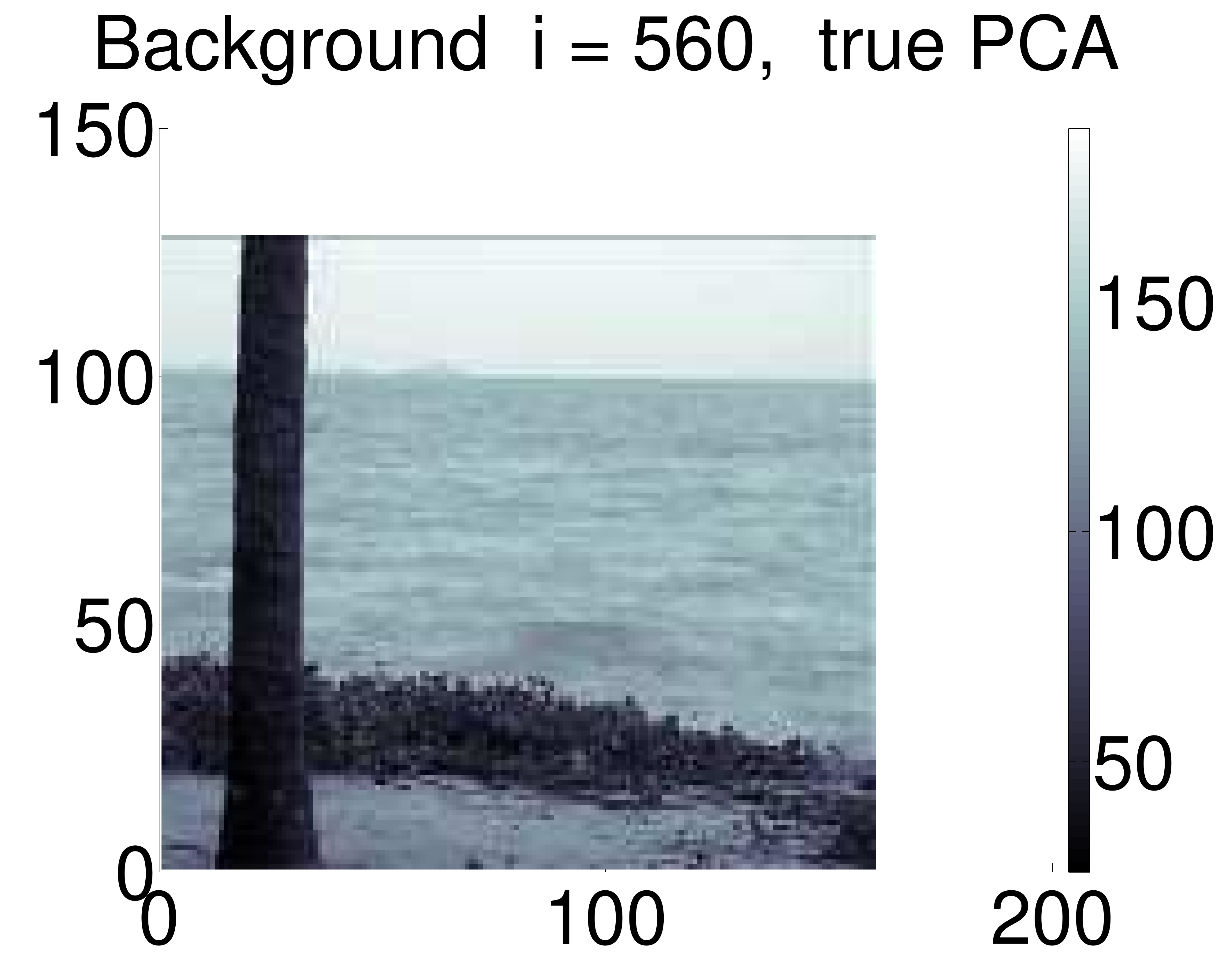} &
\includegraphics[width=.25\columnwidth]{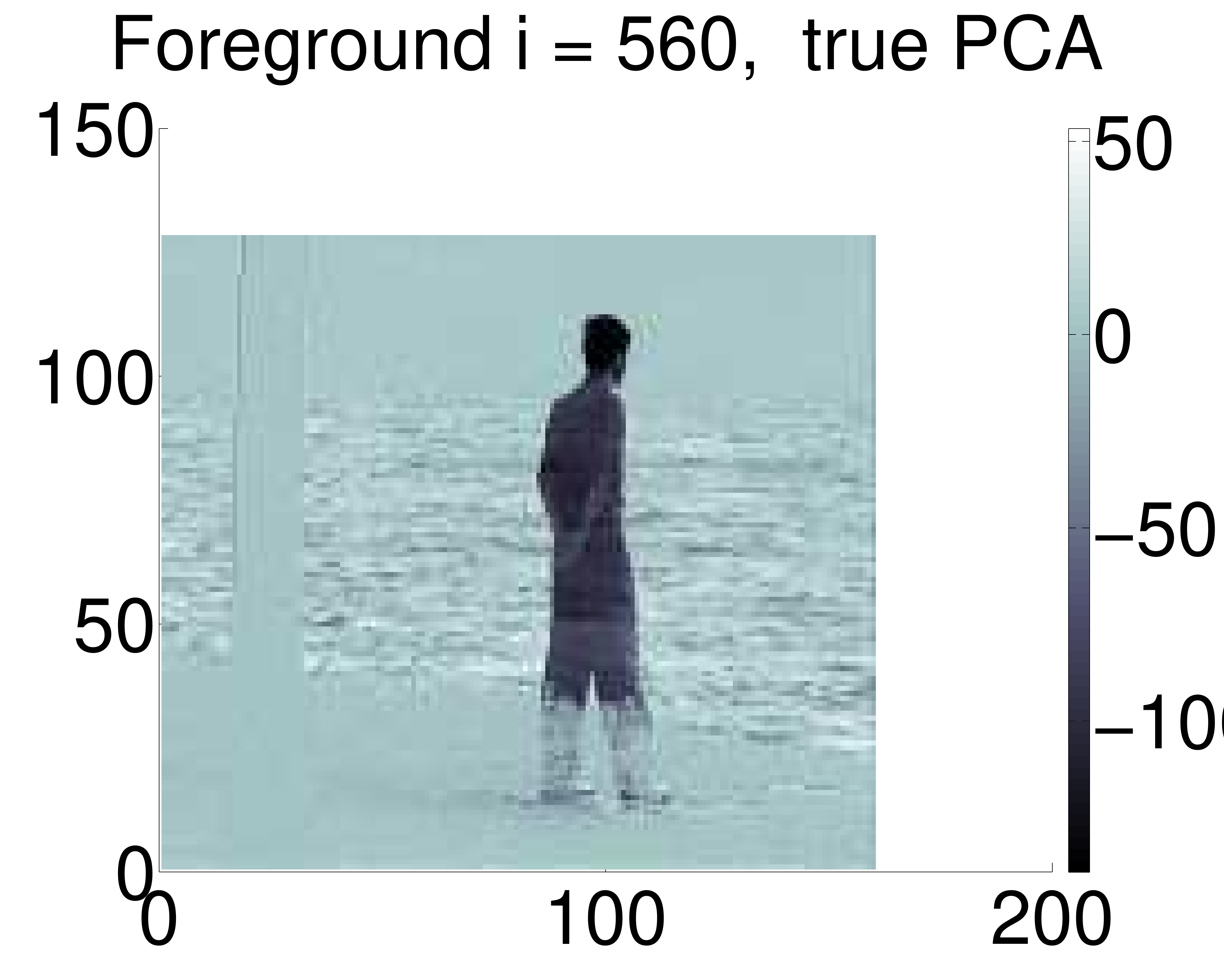} \\ 
\hline 
\includegraphics[width=.25\columnwidth]{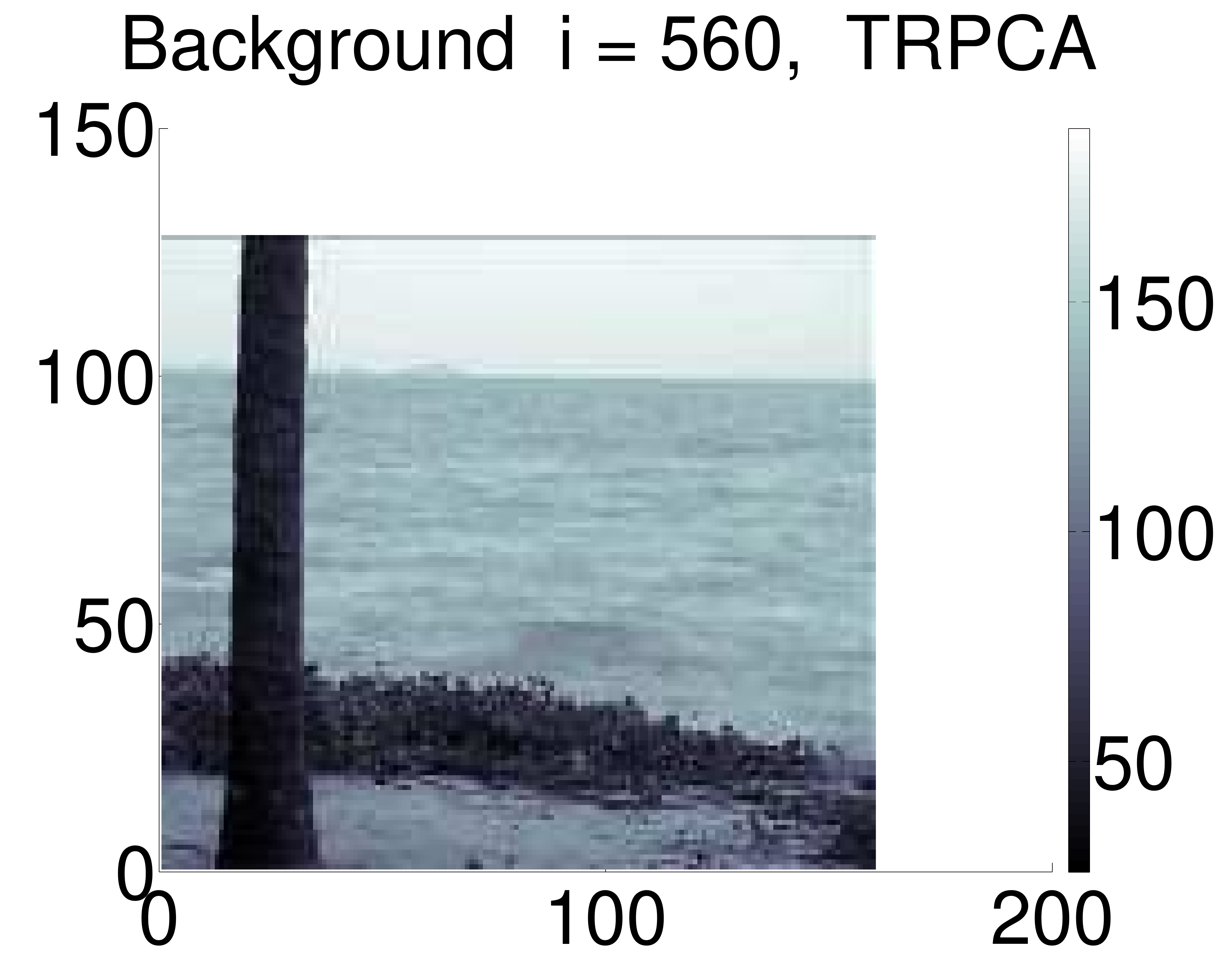} & 
\includegraphics[width=.25\columnwidth]{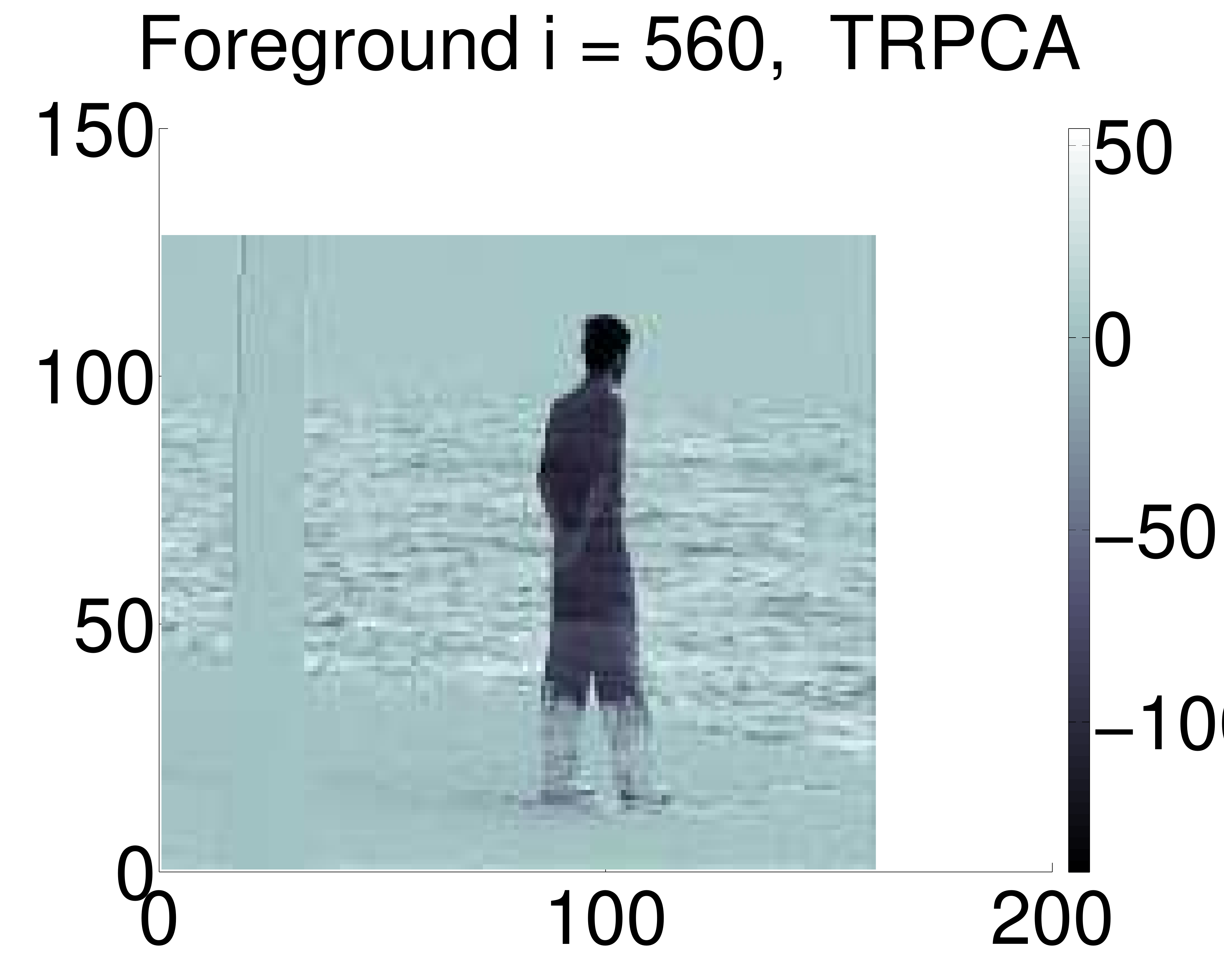} & 
\includegraphics[width=.25\columnwidth]{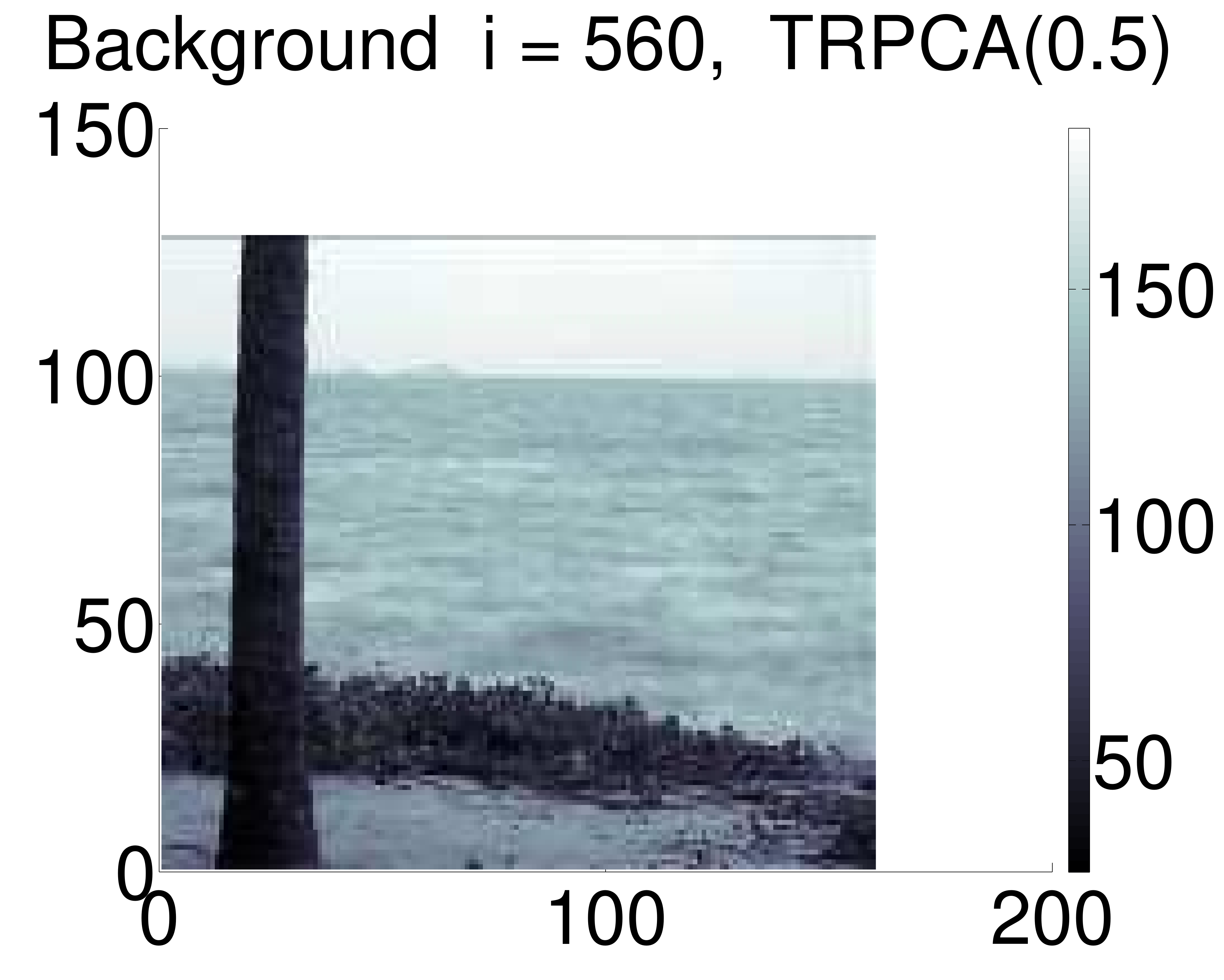} & 
\includegraphics[width=.25\columnwidth]{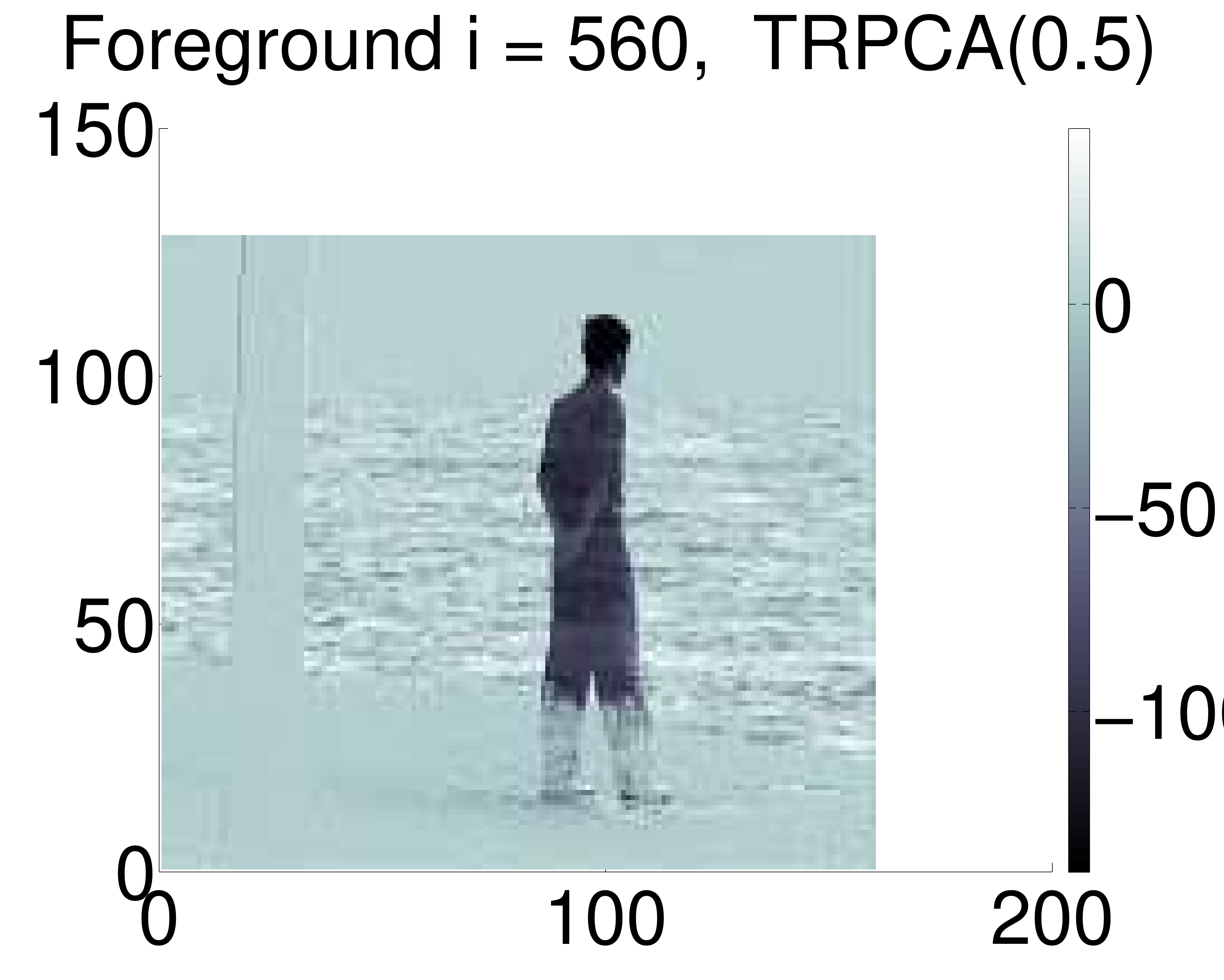} \\ 
\includegraphics[width=.25\columnwidth]{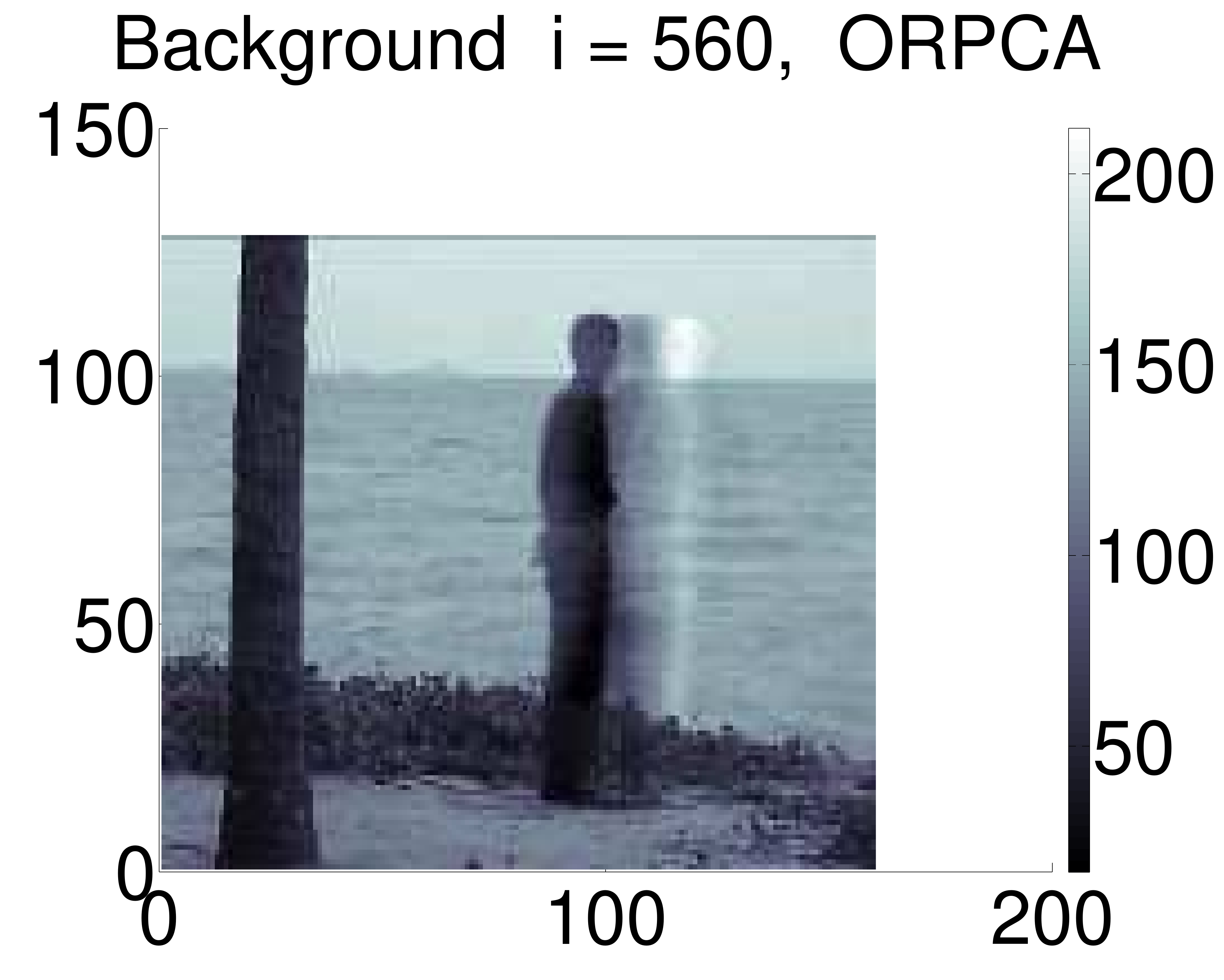}& 
\includegraphics[width=.25\columnwidth]{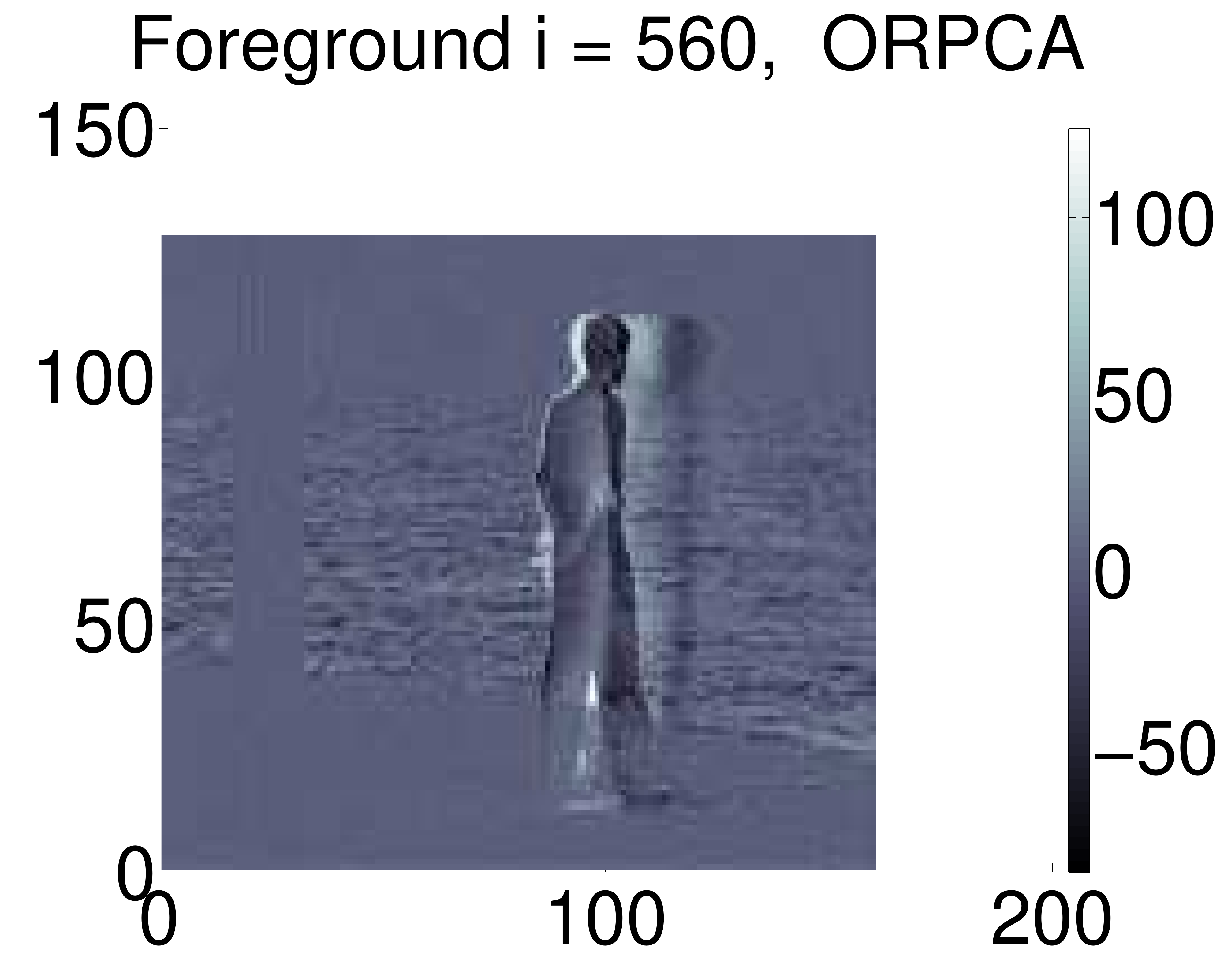} & 
\includegraphics[width=.25\columnwidth]{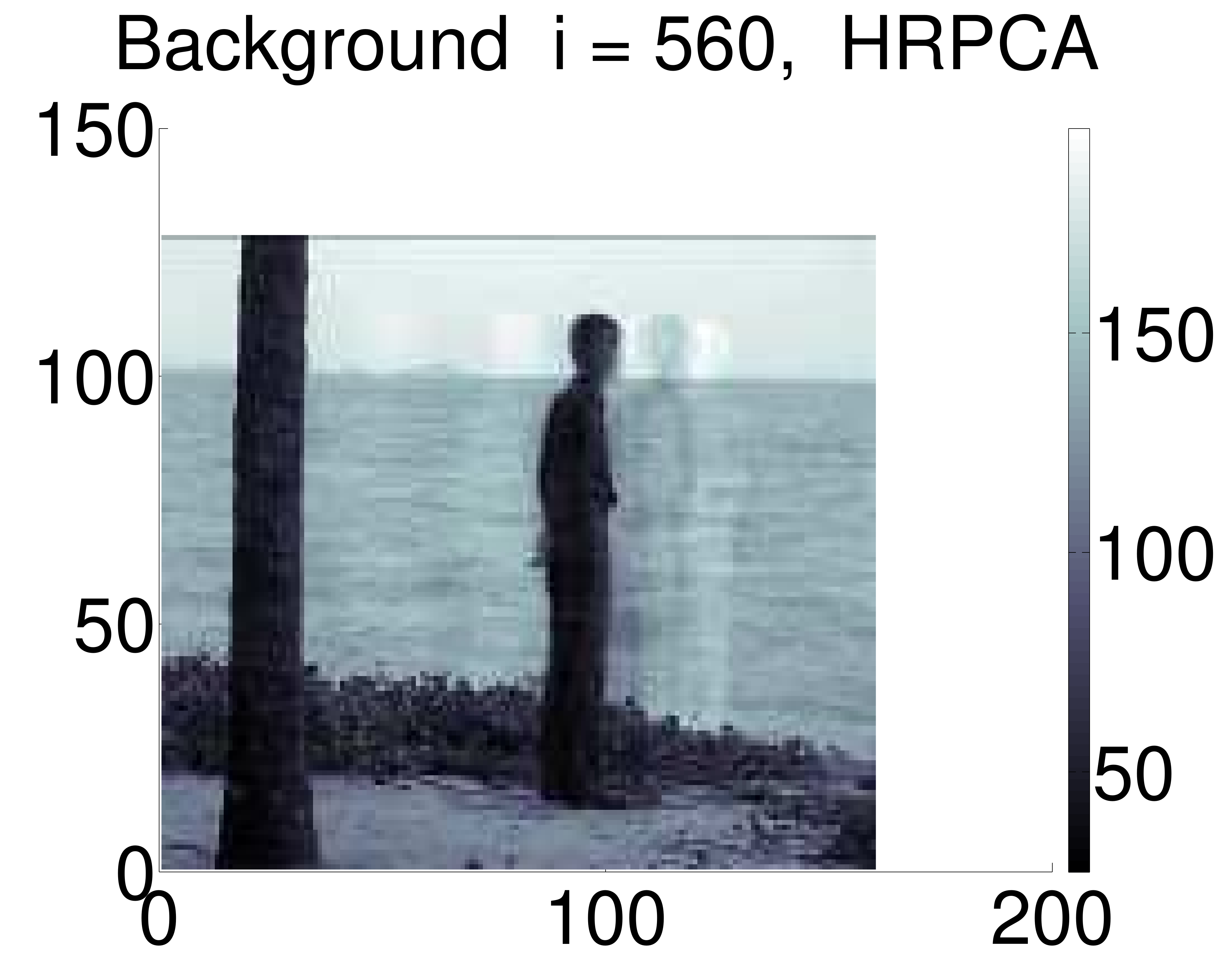} & 
\includegraphics[width=.25\columnwidth]{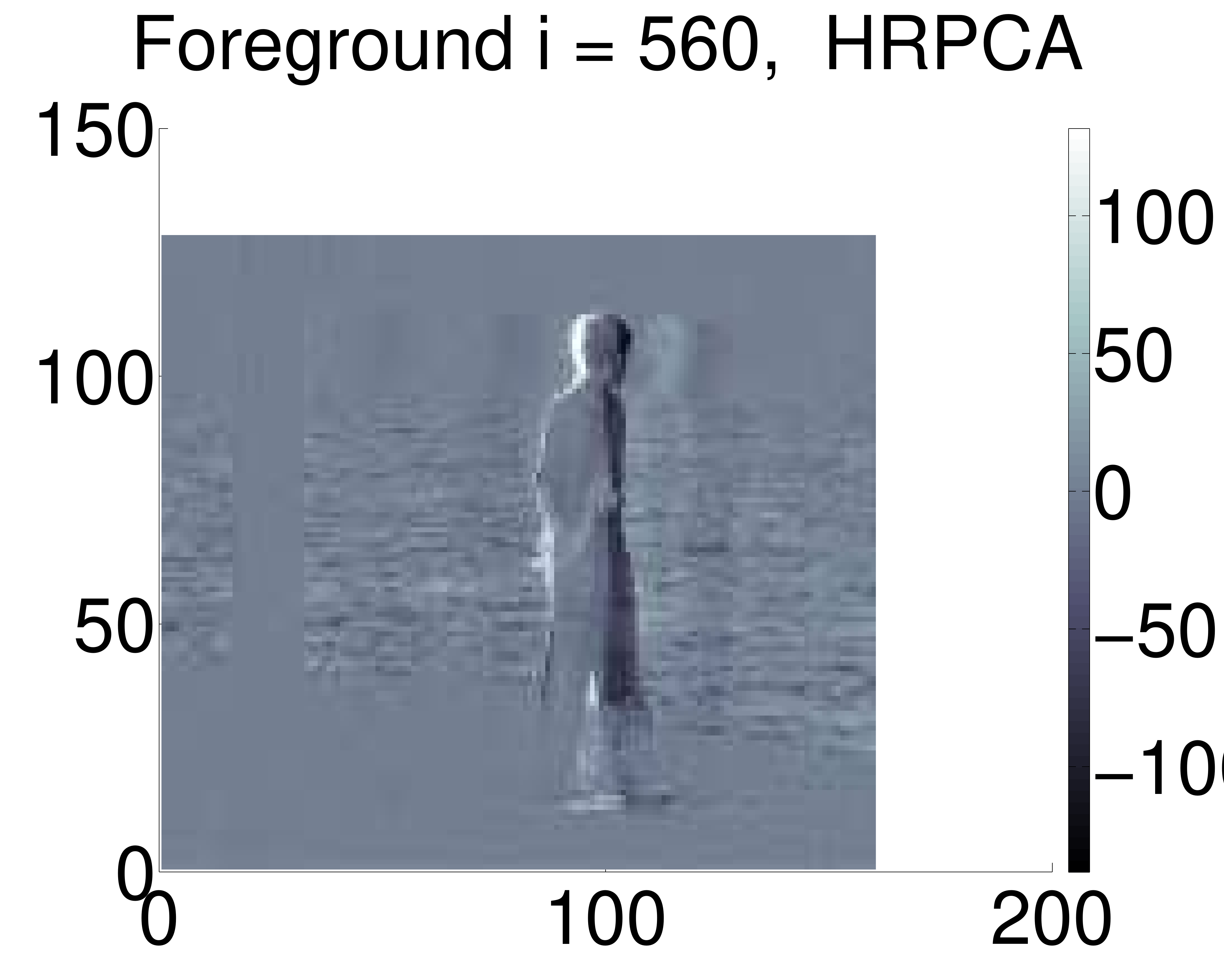} \\ 
\includegraphics[width=.25\columnwidth]{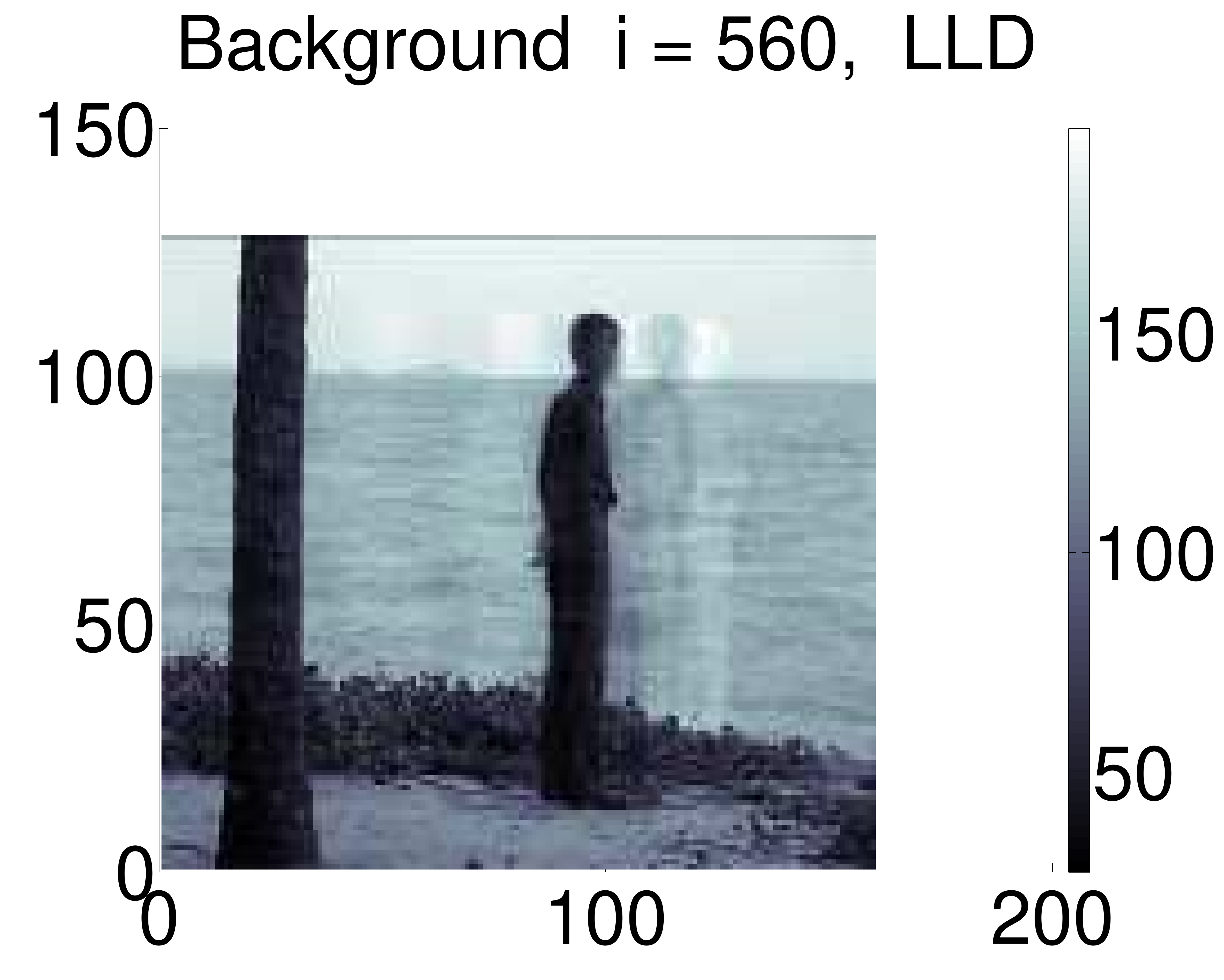} & 
\includegraphics[width=.25\columnwidth]{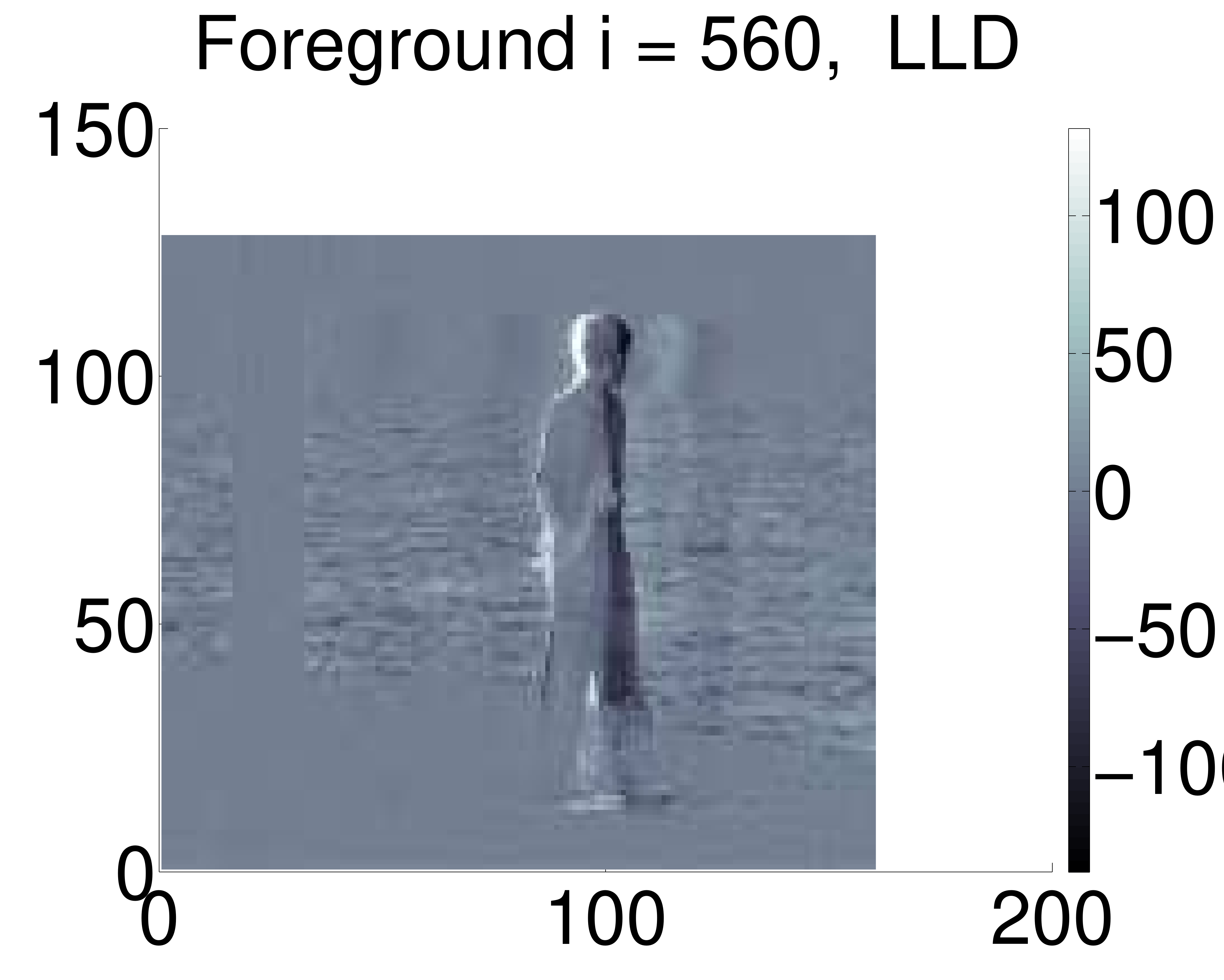}  &
\includegraphics[width=.25\columnwidth]{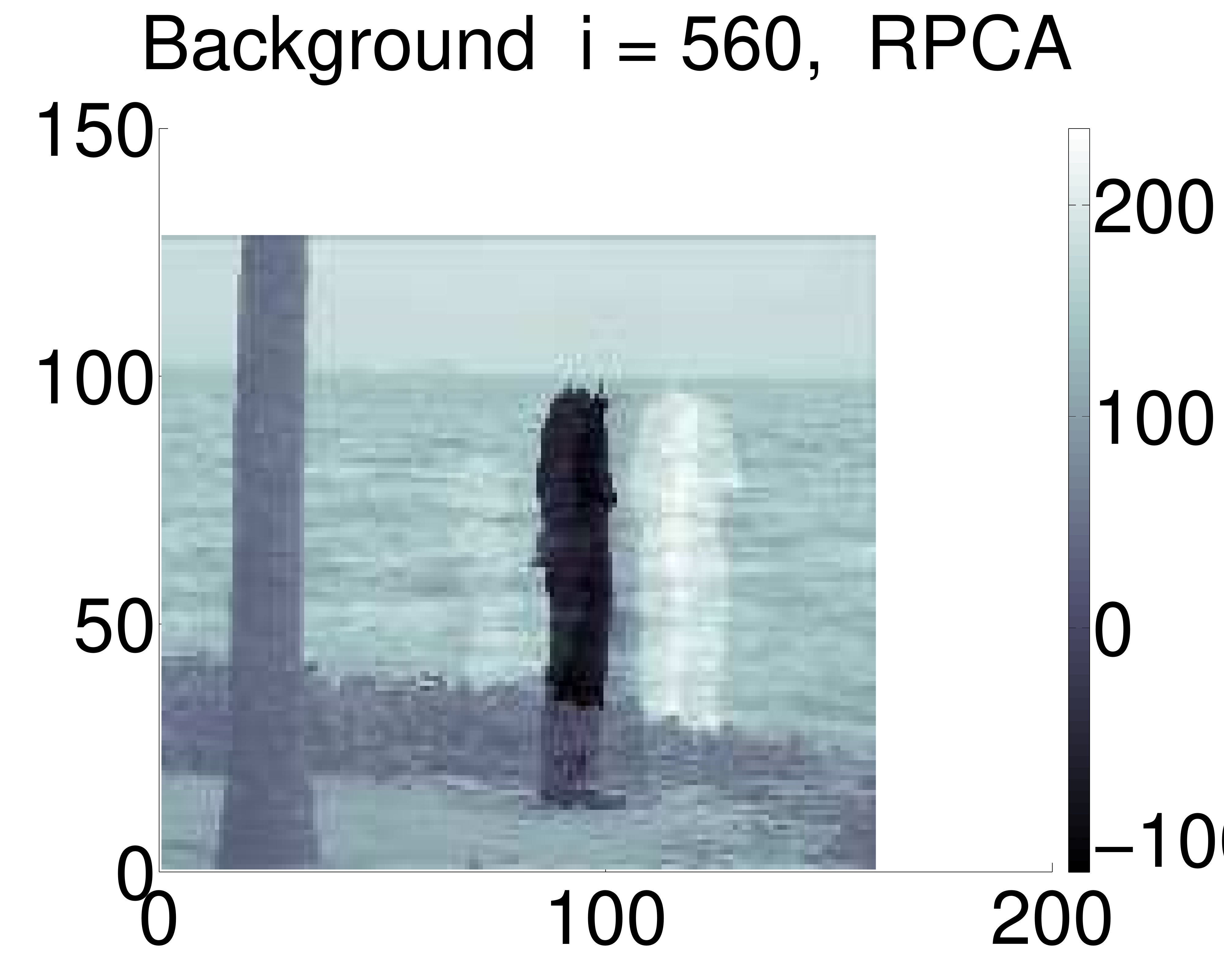} & 
\includegraphics[width=.25\columnwidth]{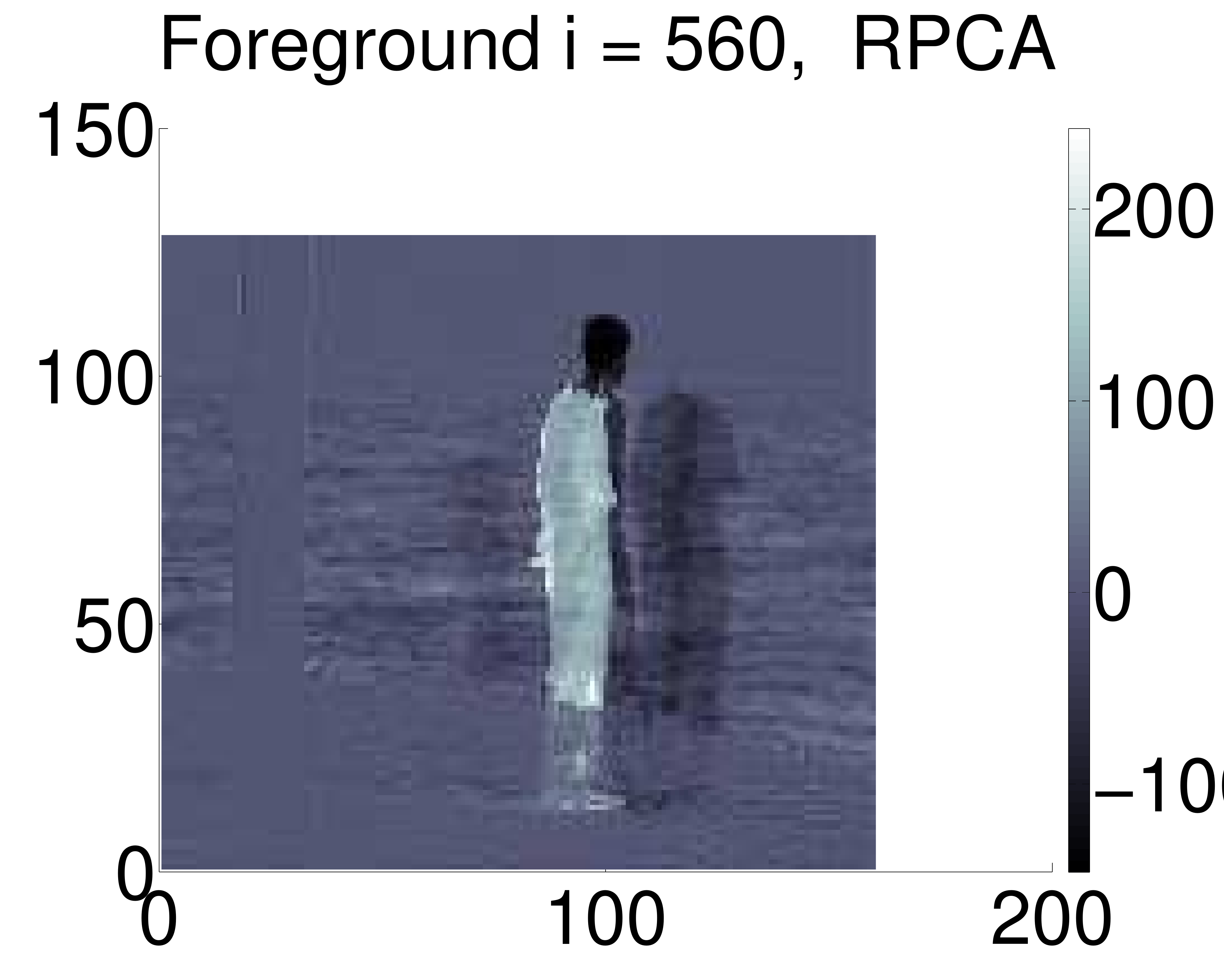} \\ 
\includegraphics[width=.25\columnwidth]{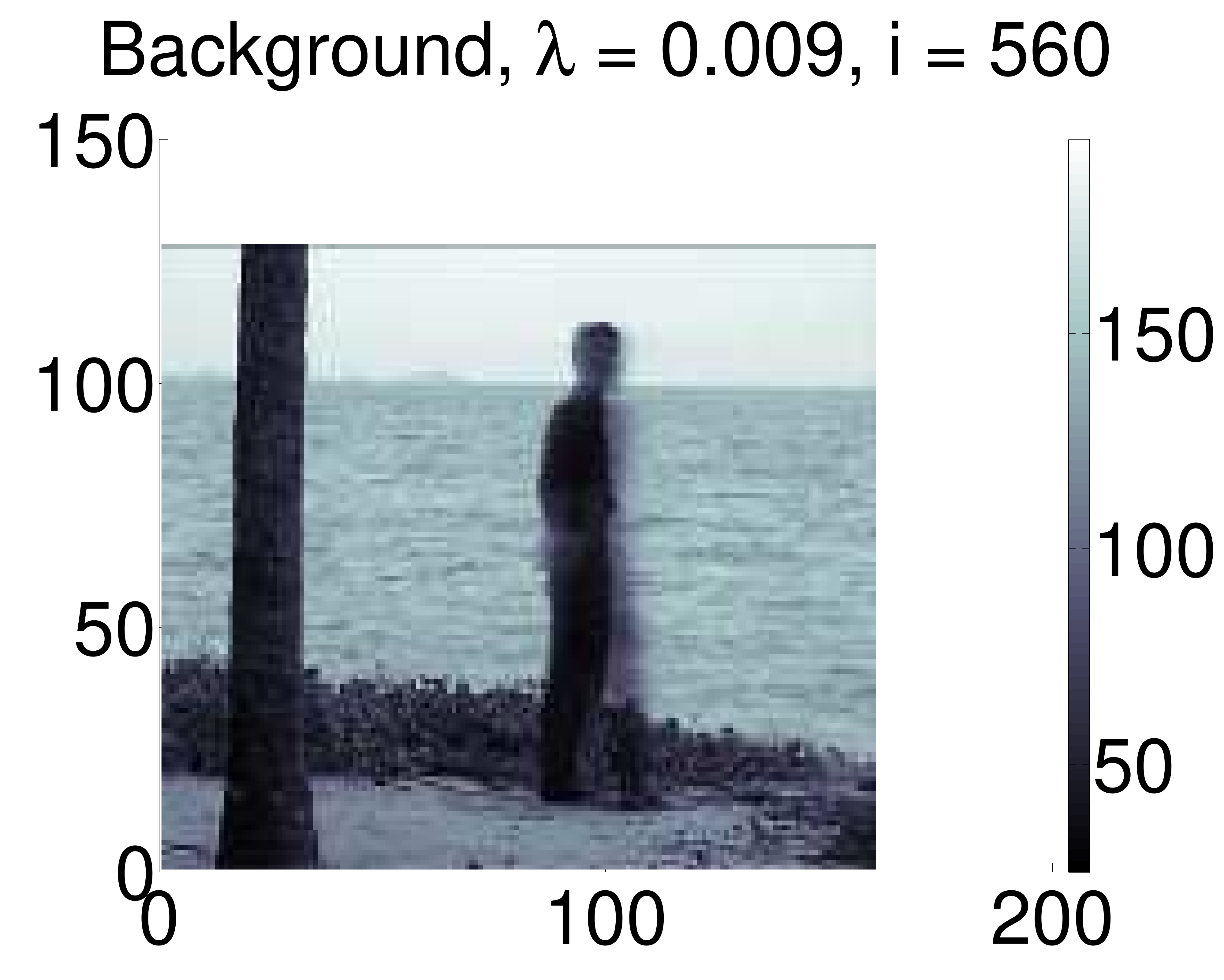} & 
\includegraphics[width=.25\columnwidth]{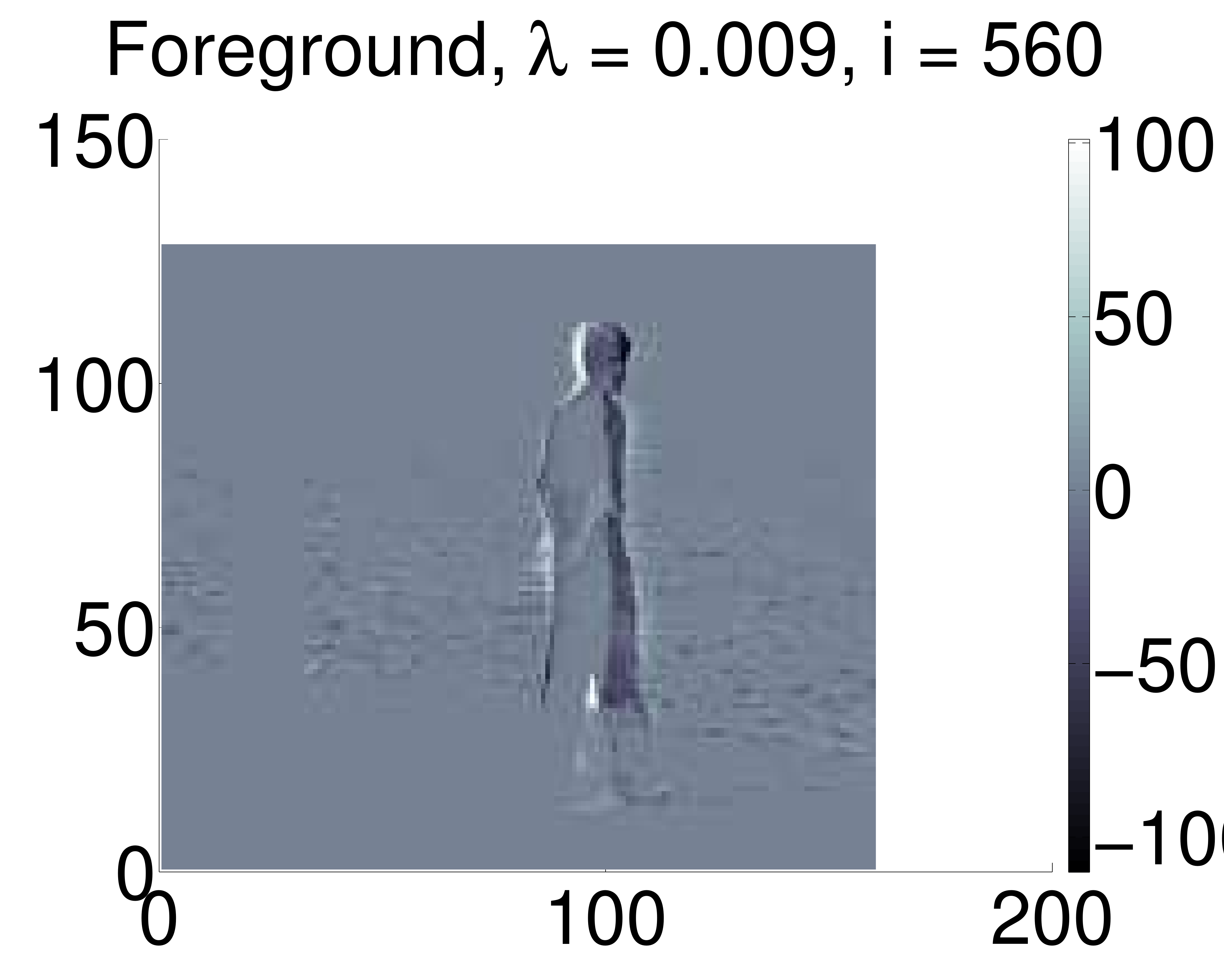} & 
\includegraphics[width=.25\columnwidth]{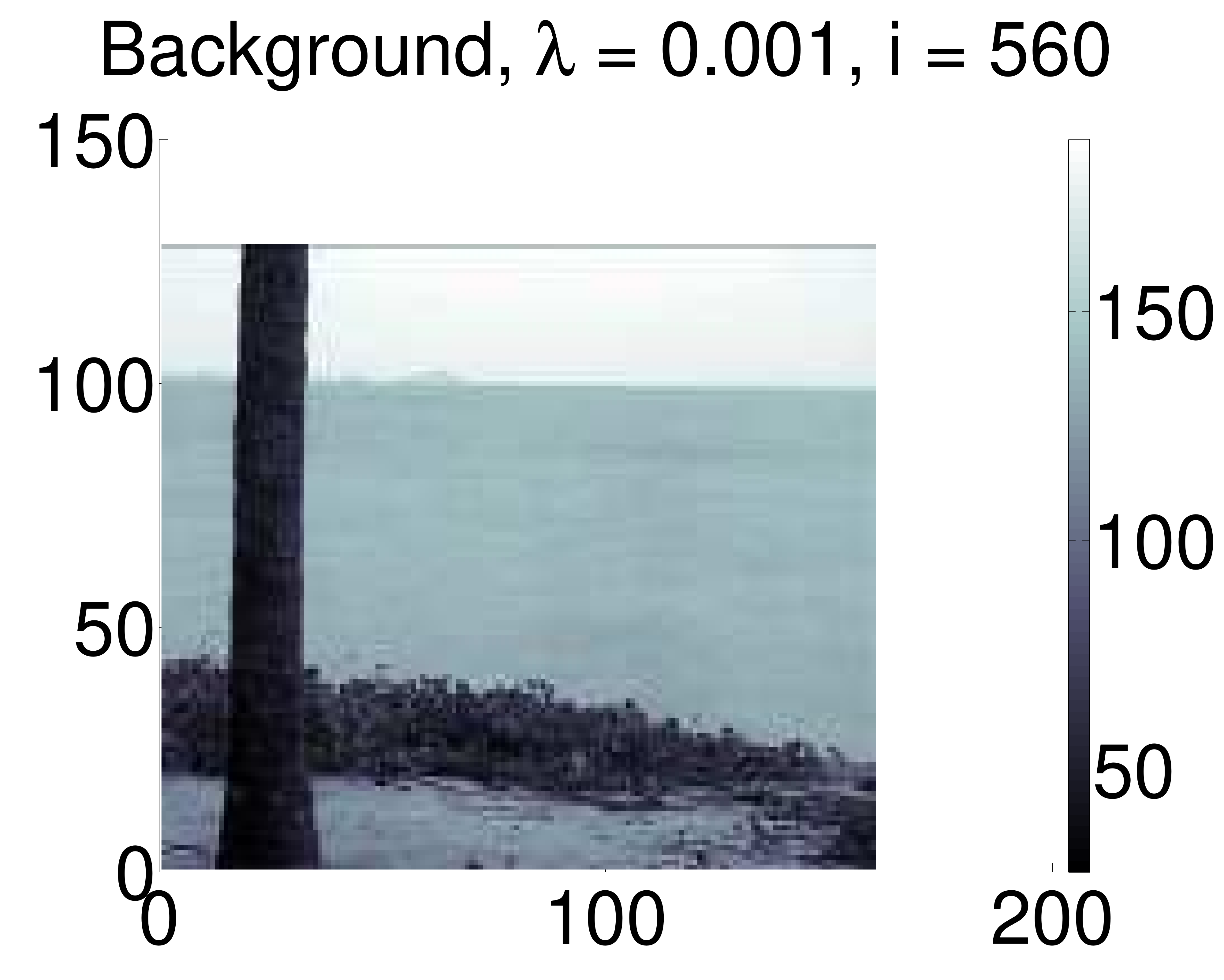} & 
\includegraphics[width=.25\columnwidth]{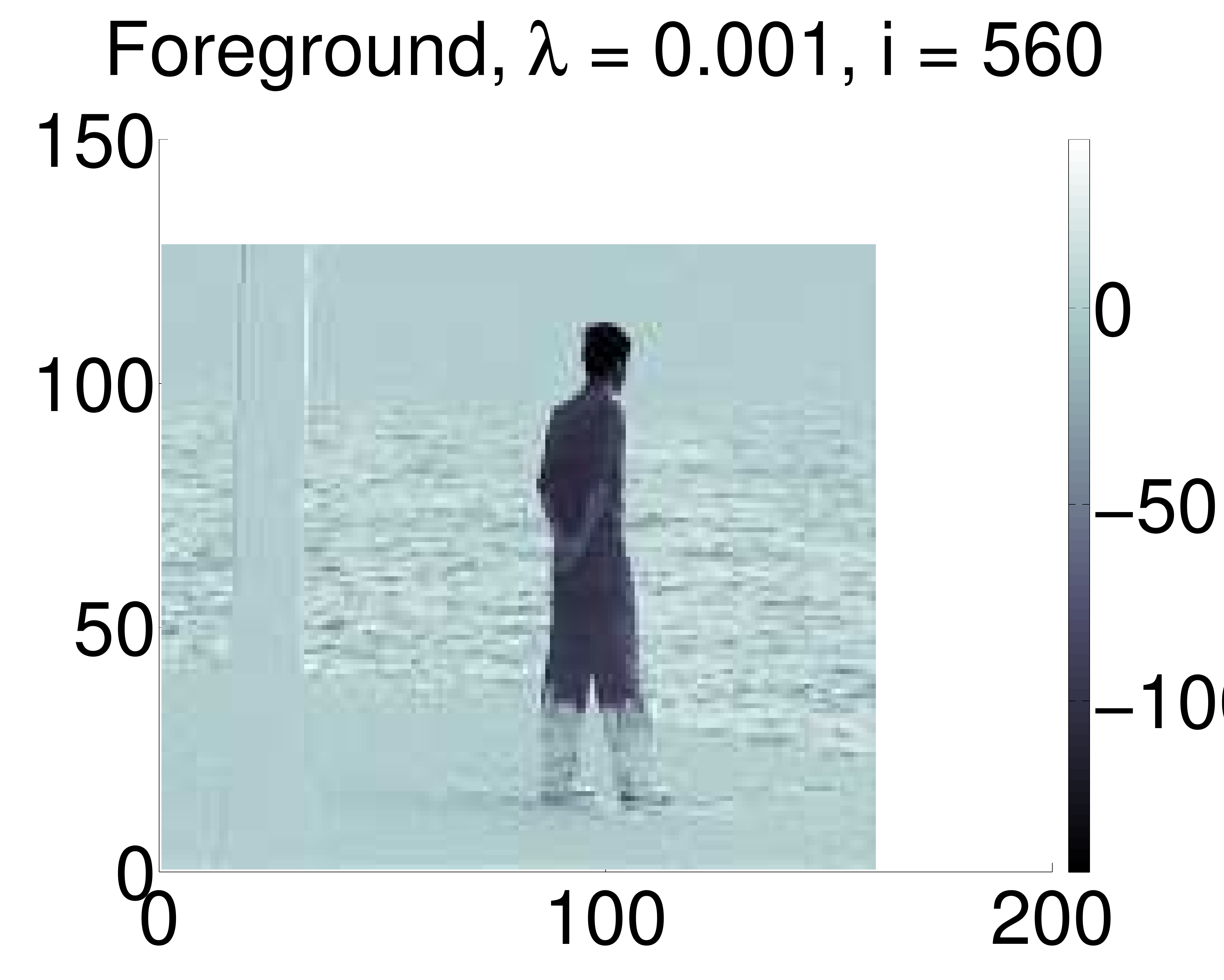}
\end{tabular}
\caption{Backgrounds and foreground for frame $i=560$ of the water surface data set. The last row corresponds to the PCP algorithm with values of $\lambda$ set by hand}
\label{fig:bfwsi560}
\end{figure}

\clearpage

\bibliographystyle{splncs03}
\bibliography{bibliography-supp}

\clearpage

%%%%%%%%%%%%%%%%%%%%%%%%%%%%%%%%%%%%%%%%%%%
%%%%%%%%%%%%%%%%%%%%%%%%%%%%%%%%%%%%%%%%%%%
%%%%%%%%%%%%%%%%%%%%%%%%%%%%%%%%%%%%%%%%%%%
%%%%%%%%%%%%%%%%%%%%%%%%%%%%%%%%%%%%%%%%%%%
%%%%%%%%%%%  SUPPLEMENTARY EXPERIMENTS  %%%%%%%%%%%%%
%%%%%%%%%%%%%%%%%%%%%%%%%%%%%%%%%%%%%%%%%%%
%%%%%%%%%%%%%%%%%%%%%%%%%%%%%%%%%%%%%%%%%%%
%%%%%%%%%%%%%%%%%%%%%%%%%%%%%%%%%%%%%%%%%%%
%%%%%%%%%%%%%%%%%%%%%%%%%%%%%%%%%%%%%%%%%%%
%\begin{comment}
\section{Supplementary material: Experiments}

In this supplementary material we present additional illustrations of the background subtraction experiments in Fig.~4-15. We consider the water surface data set and the moved object\footnote{See http://research.microsoft.com/en-us/um/people/jckrumm/wallflower/testimages.htm} data set. For both data sets the frames where no person is present represent the true data $T$ (background) and frames where the person is present are considered as outliers $O$.

%%%%%%%%%%%%%%%%%%%%%%%%%%%%%%%%%%%%%%%%%%%%%%%%%%%%
\begin{figure}
\begin{tabular}{cccc}
\includegraphics[width=.25\columnwidth]{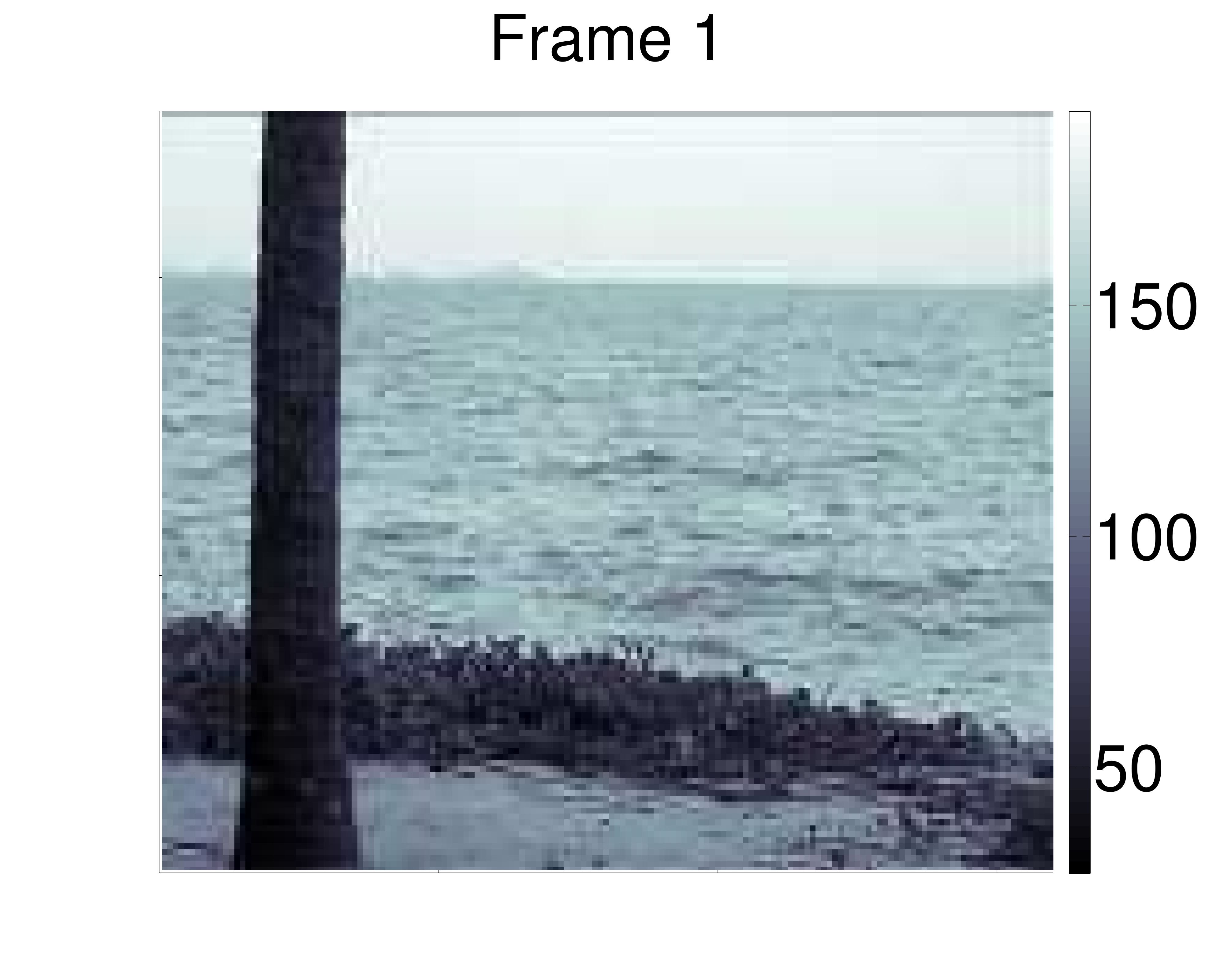} & 
\includegraphics[width=.25\columnwidth]{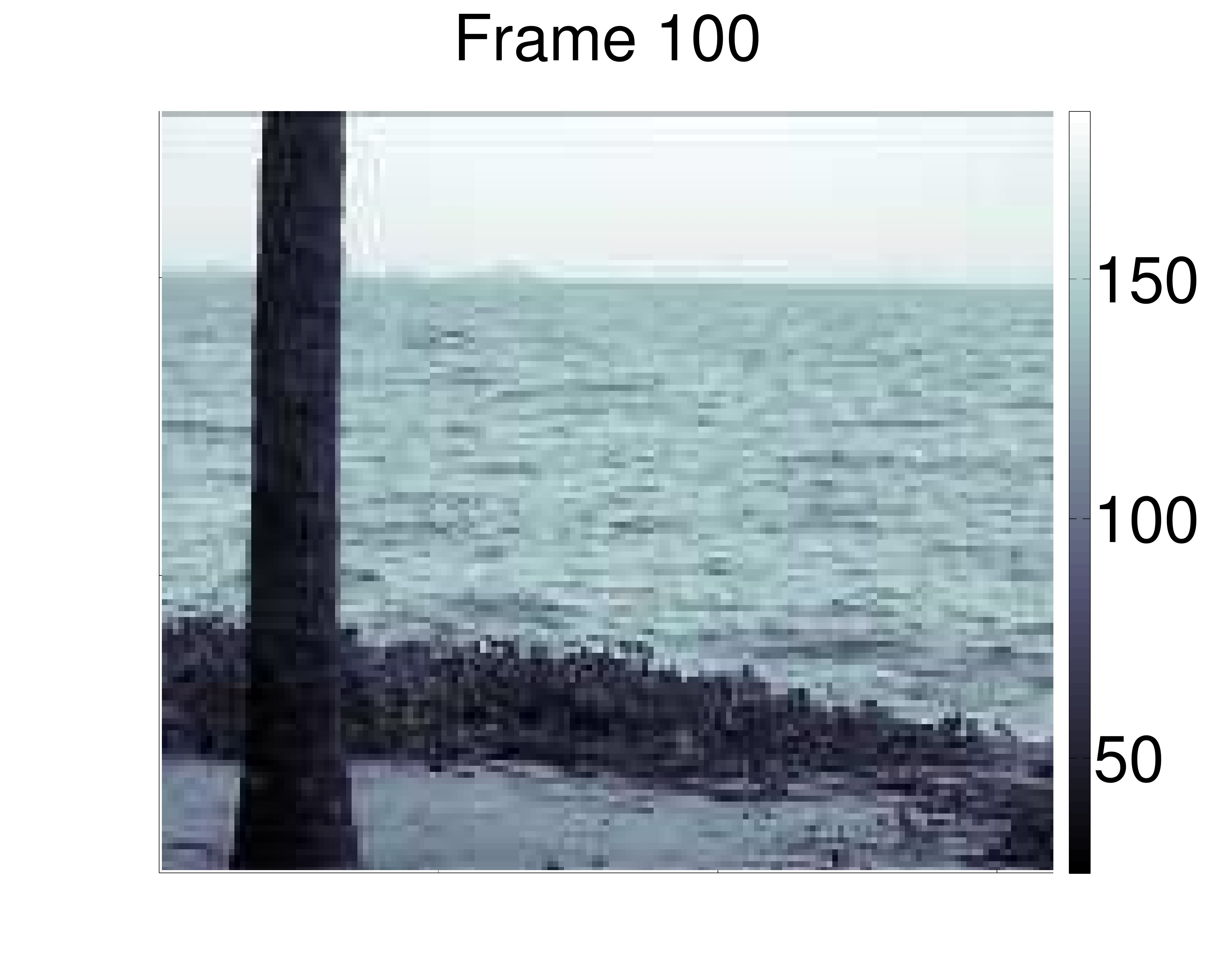} & 
\includegraphics[width=.25\columnwidth]{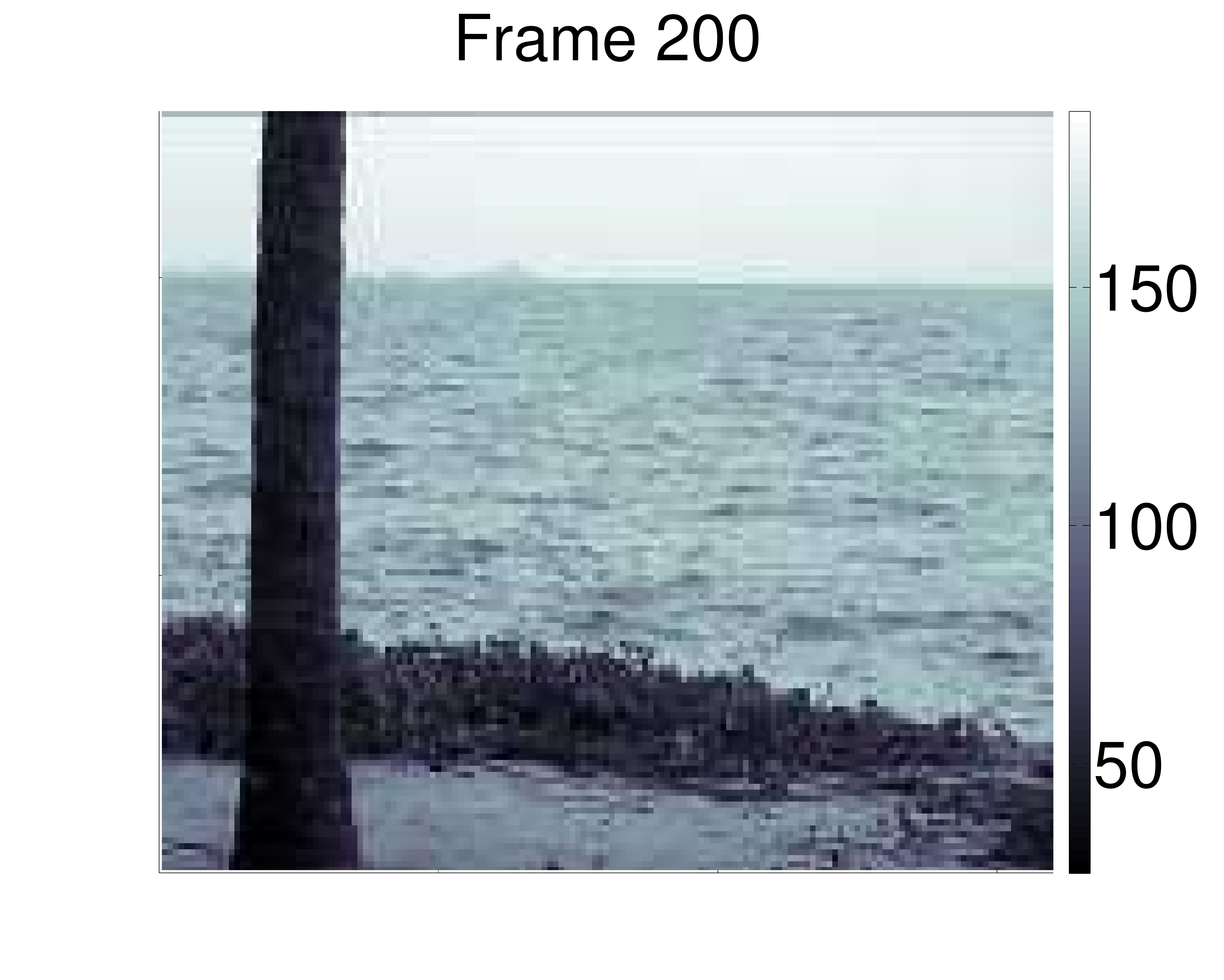} & 
\includegraphics[width=.25\columnwidth]{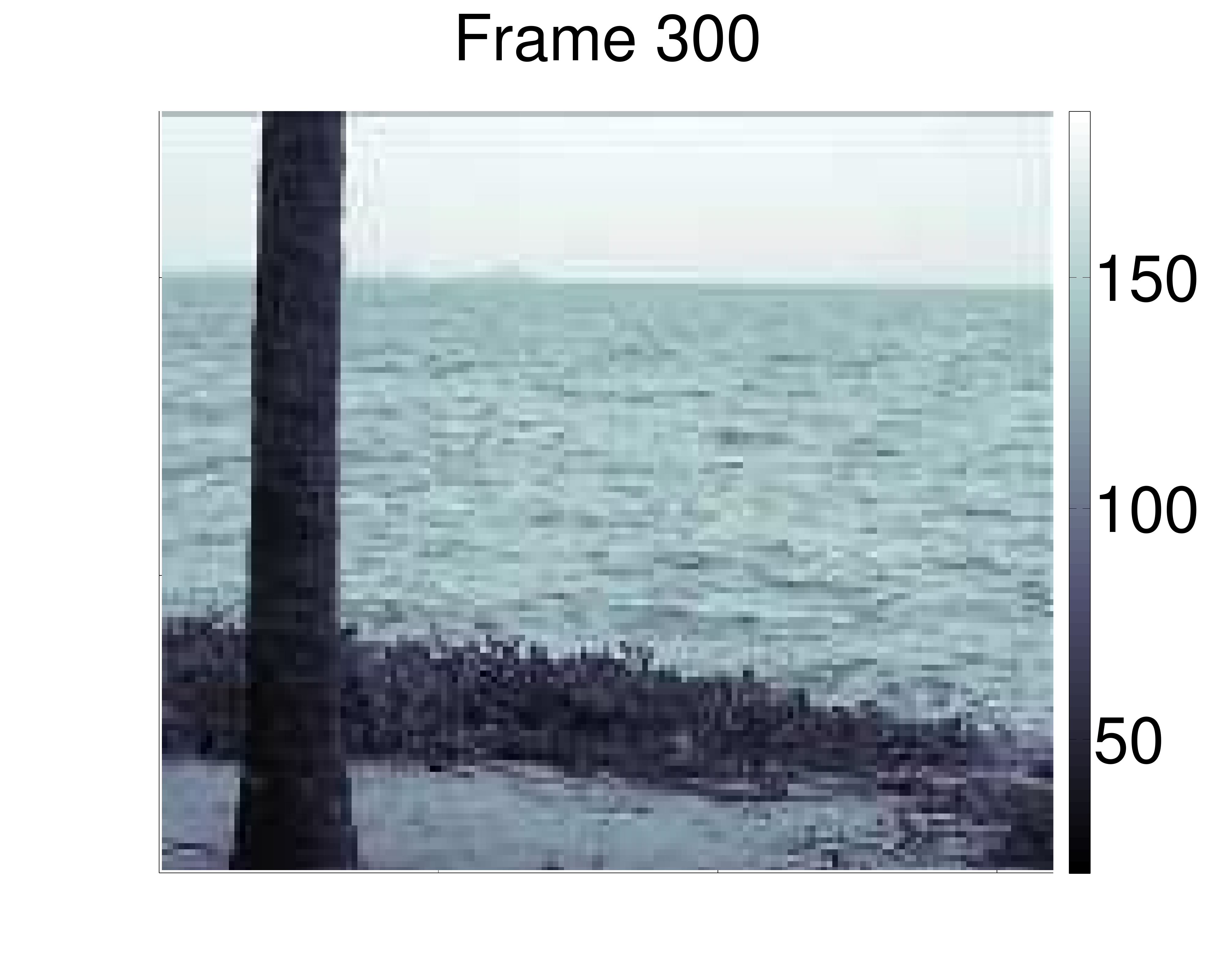} \\
\includegraphics[width=.25\columnwidth]{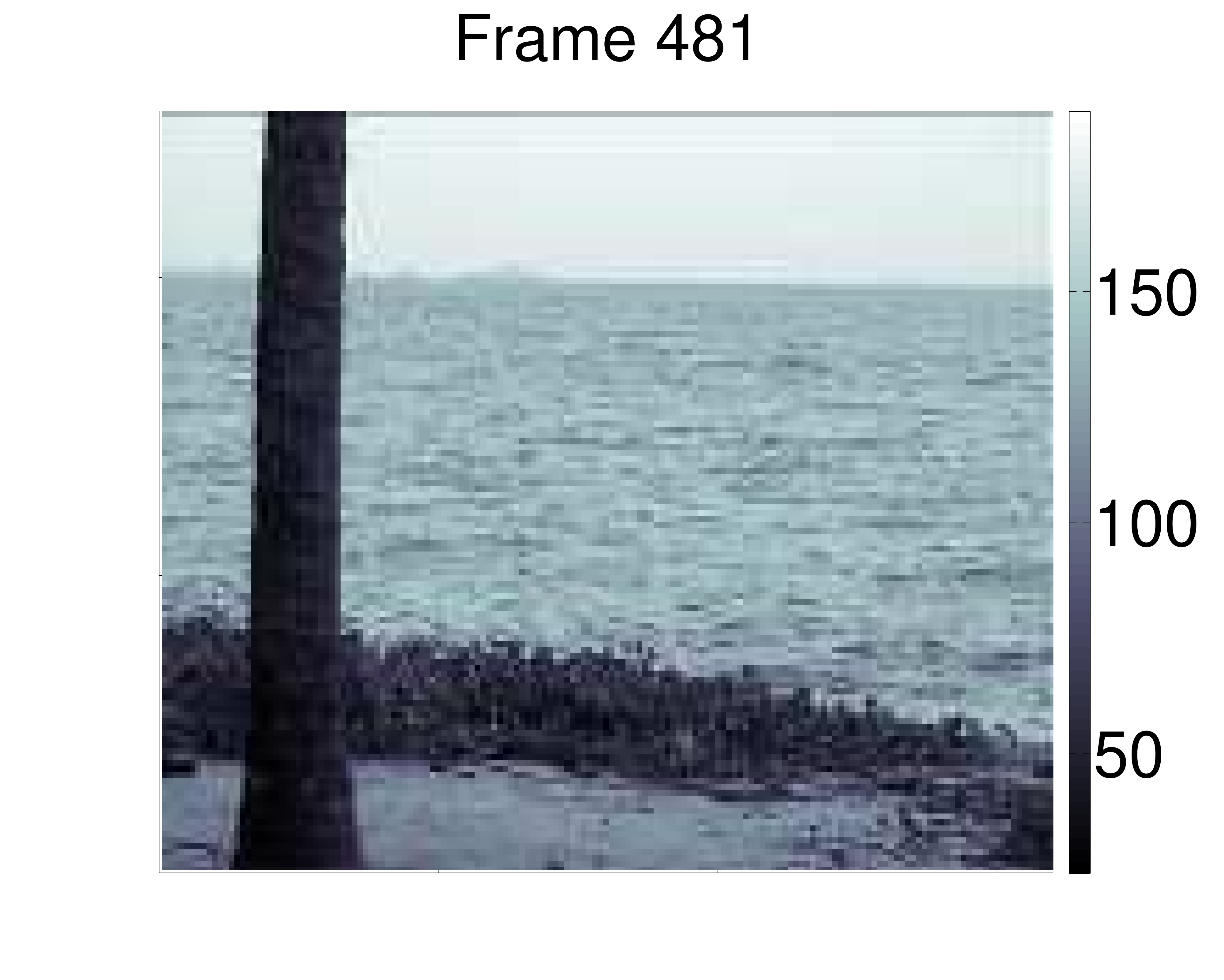} & 
\includegraphics[width=.25\columnwidth]{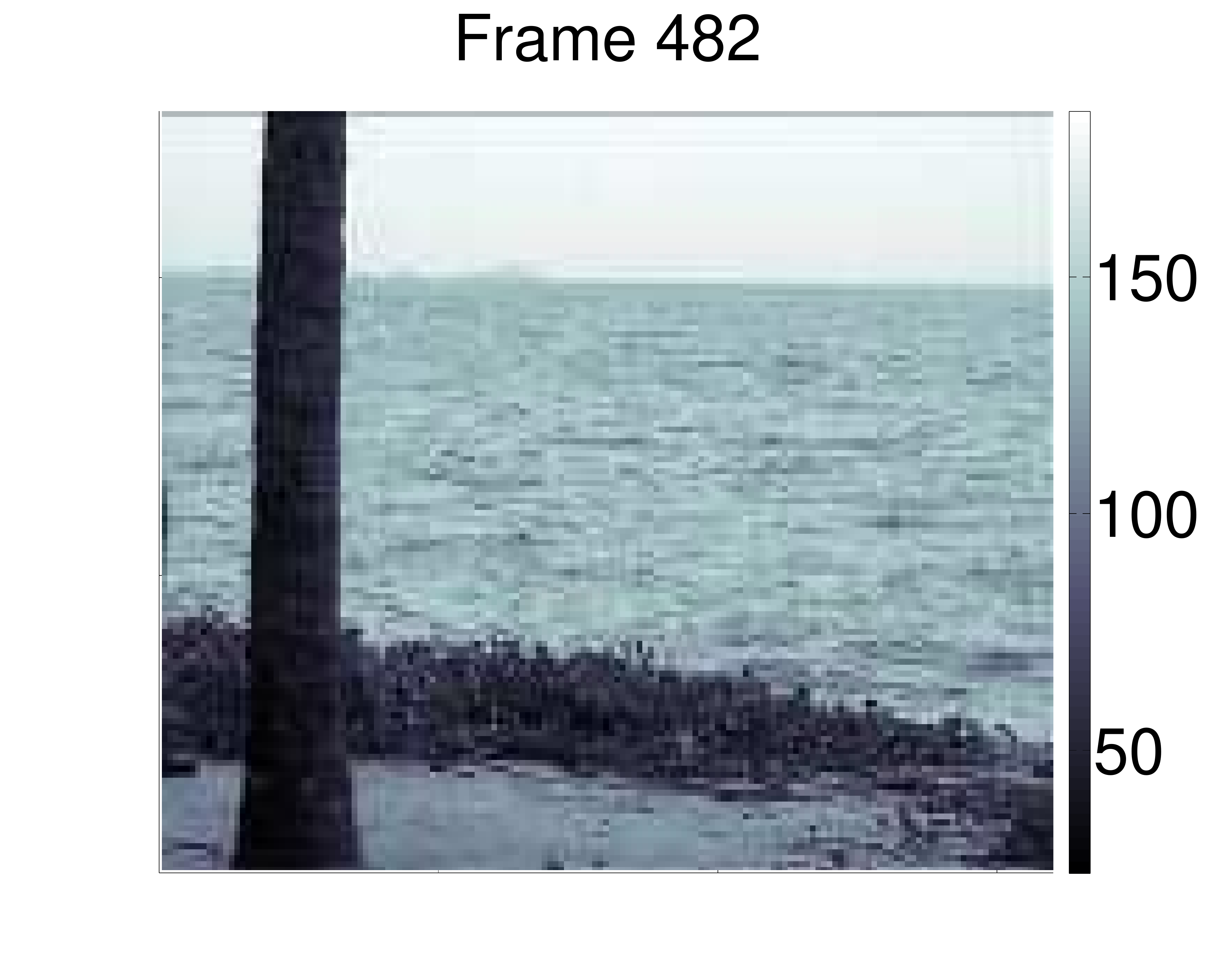} & 
\includegraphics[width=.25\columnwidth]{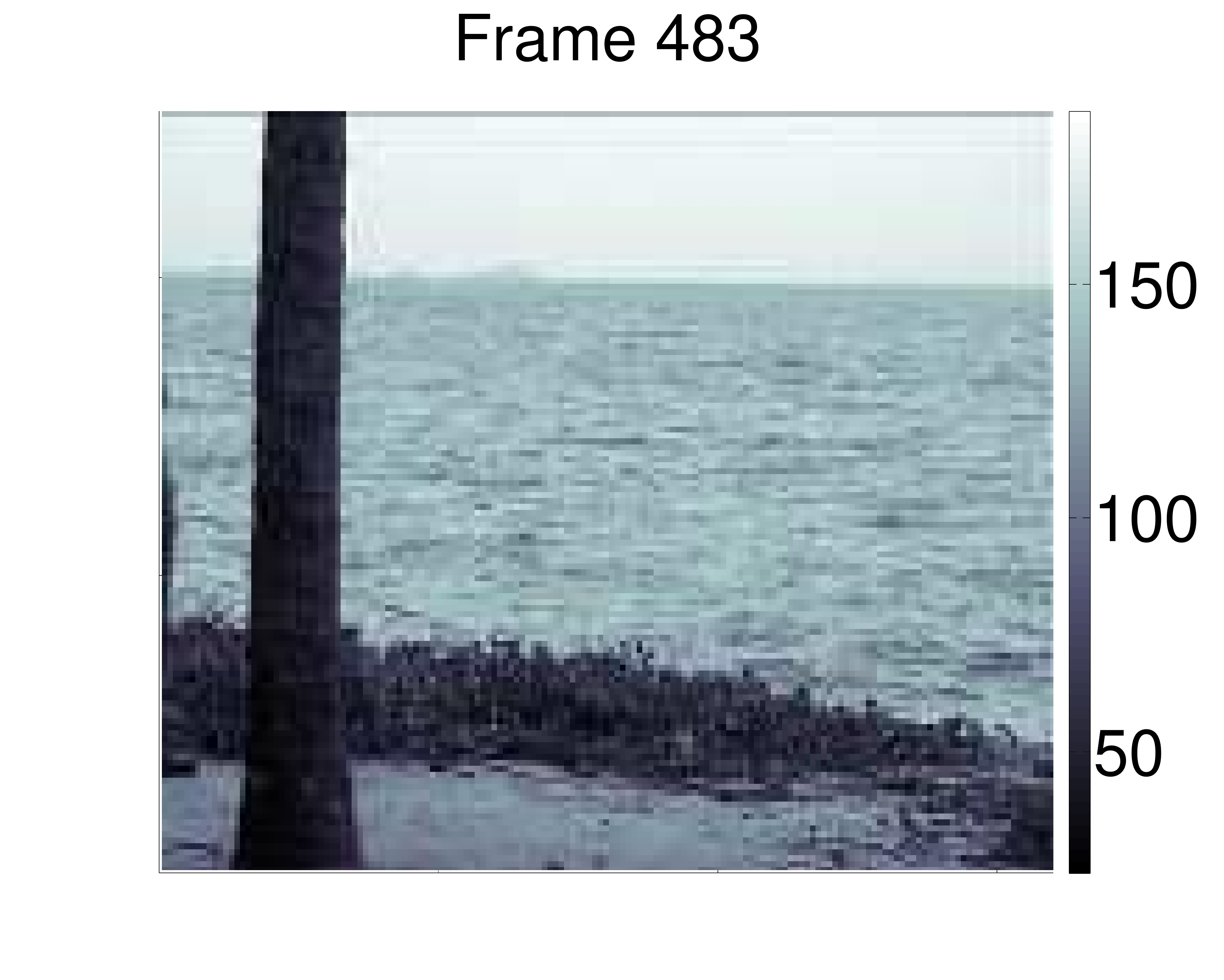} & 
\includegraphics[width=.25\columnwidth]{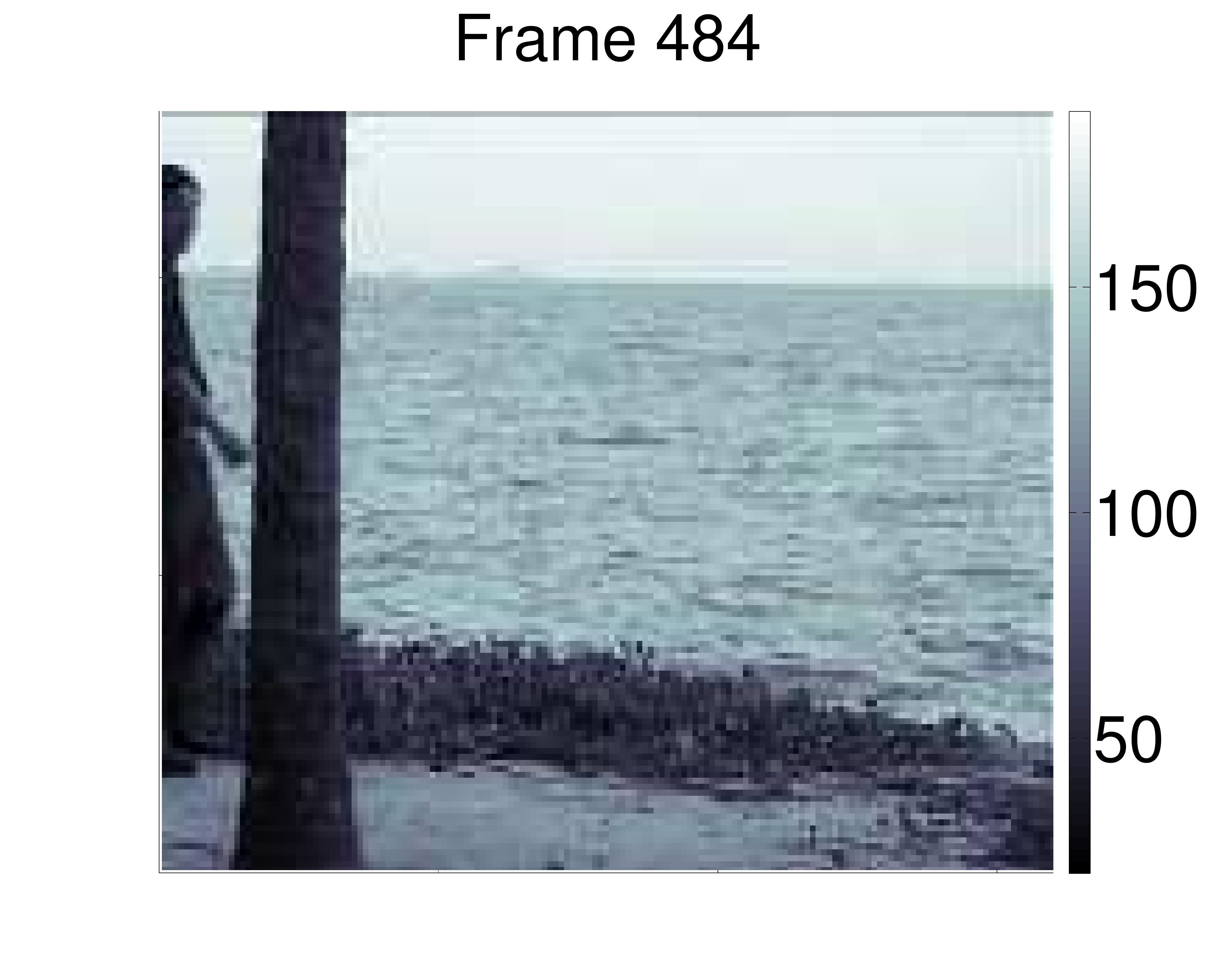} \\
\includegraphics[width=.25\columnwidth]{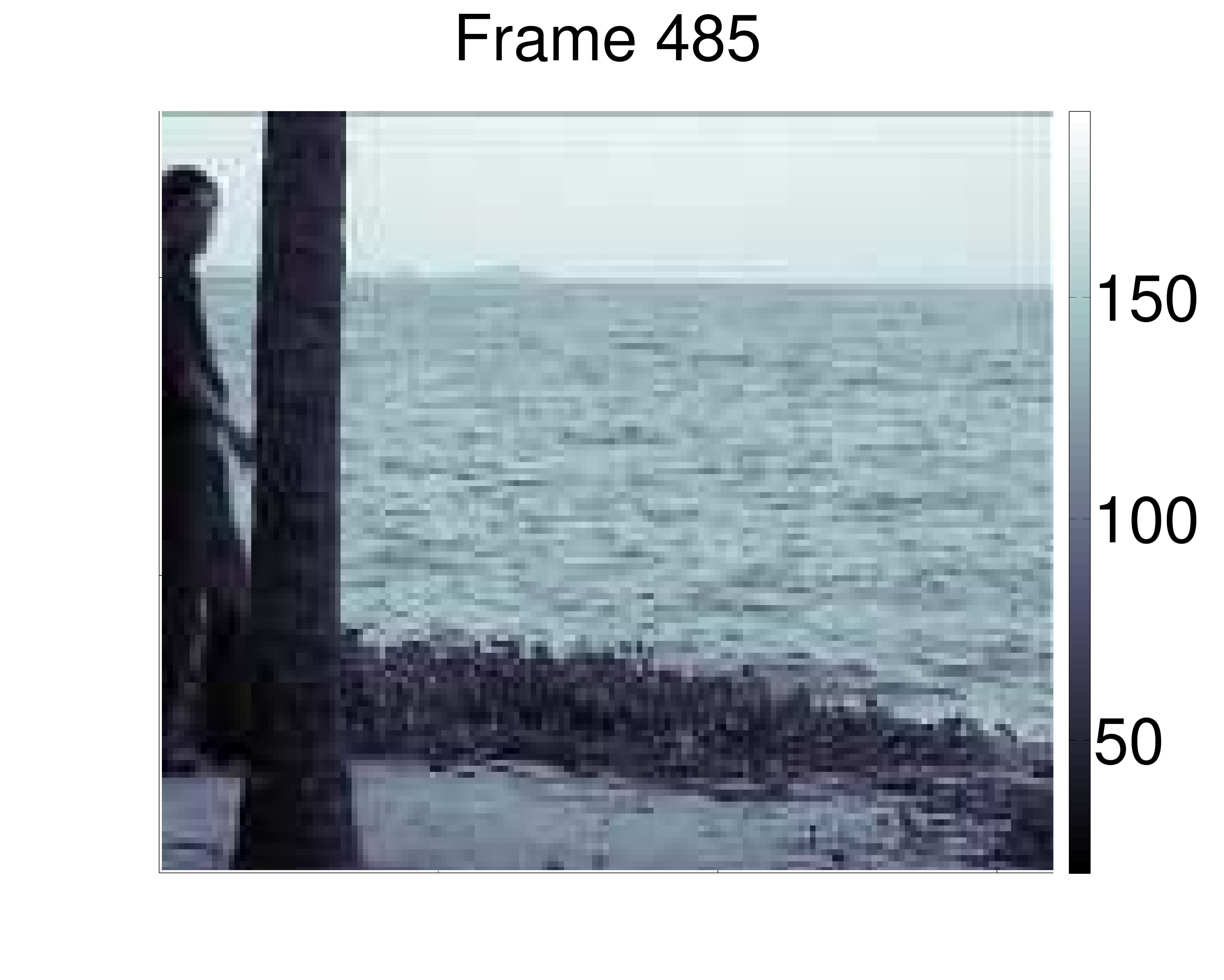} & 
\includegraphics[width=.25\columnwidth]{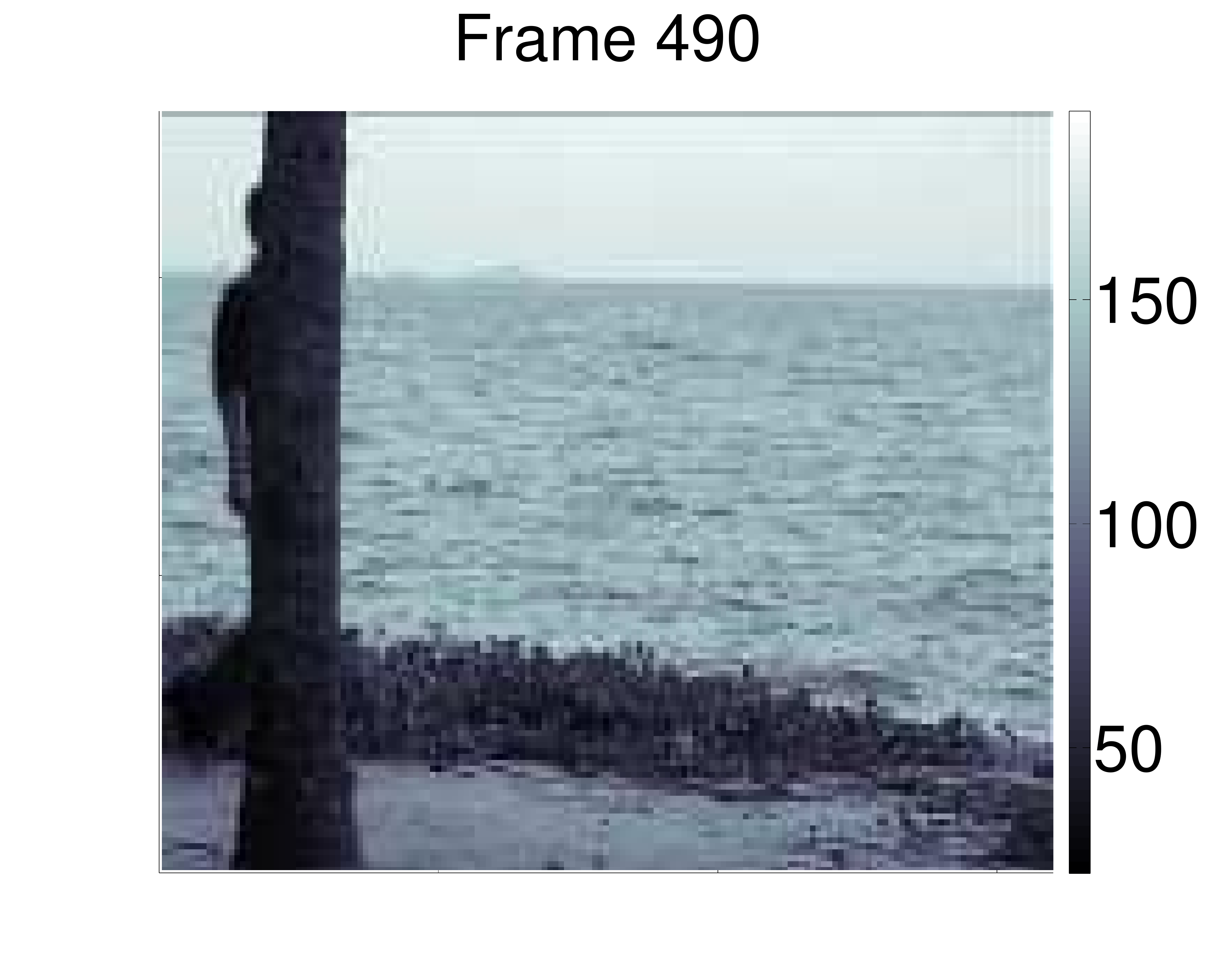} & 
\includegraphics[width=.25\columnwidth]{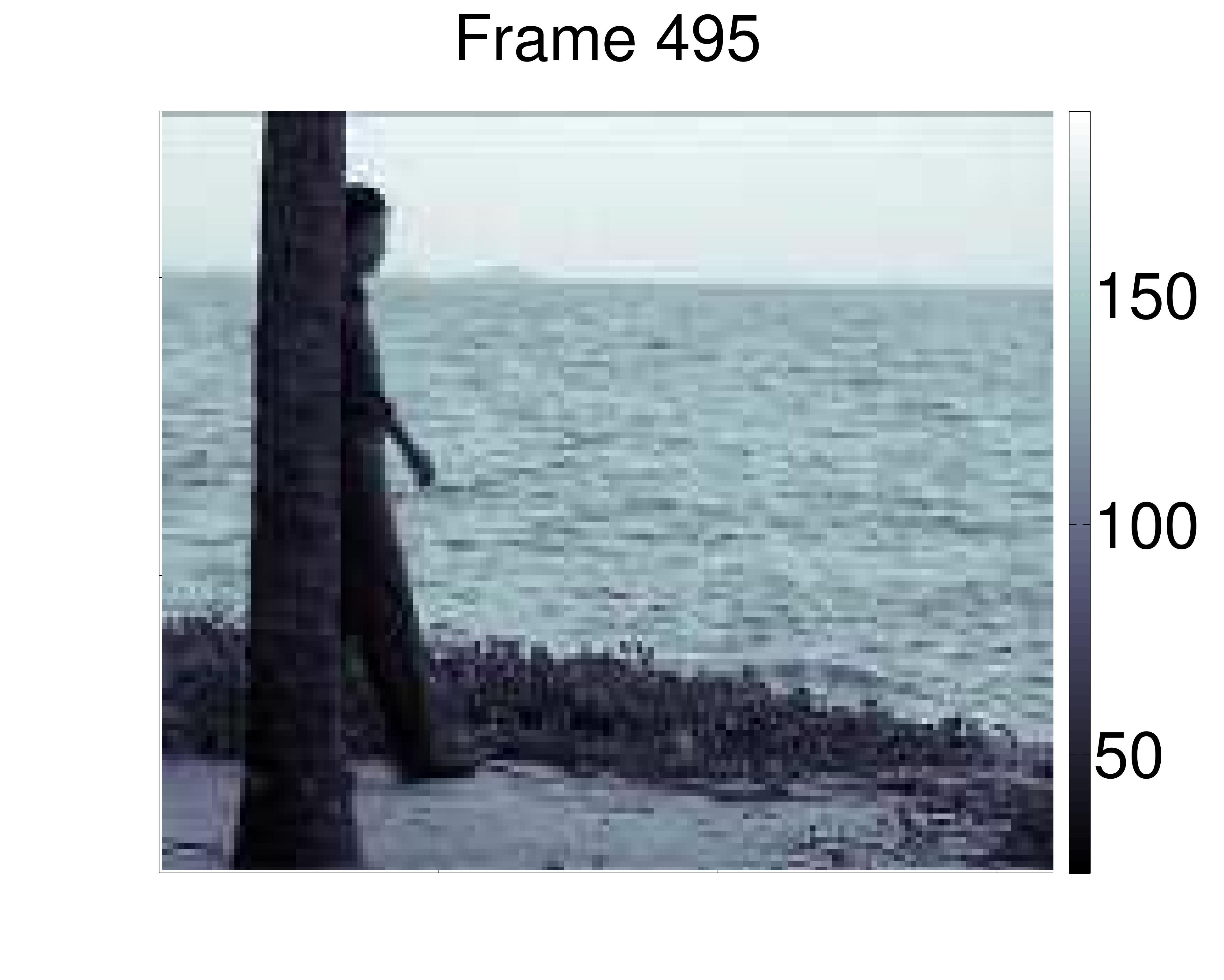} & 
\includegraphics[width=.25\columnwidth]{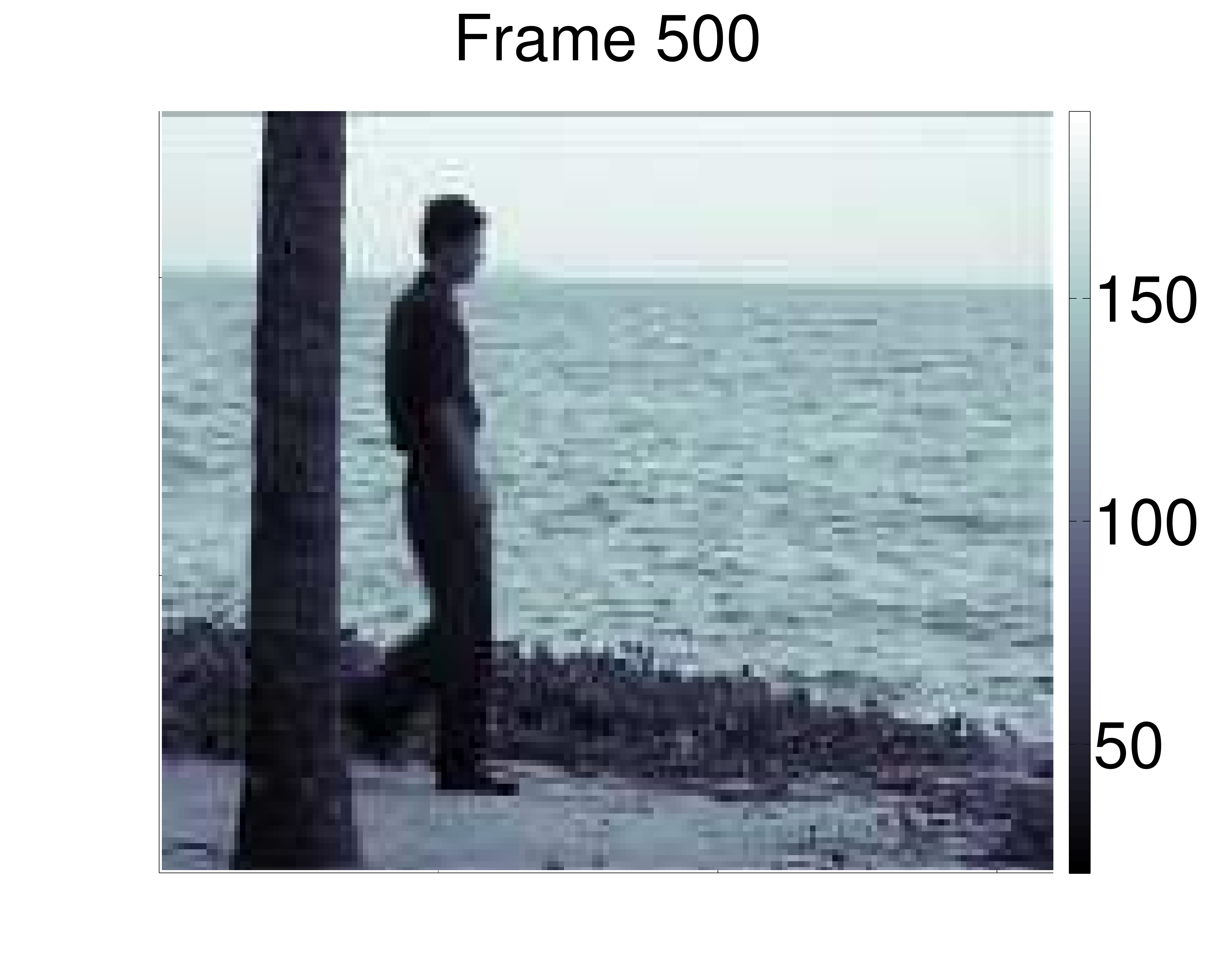} \\
\includegraphics[width=.25\columnwidth]{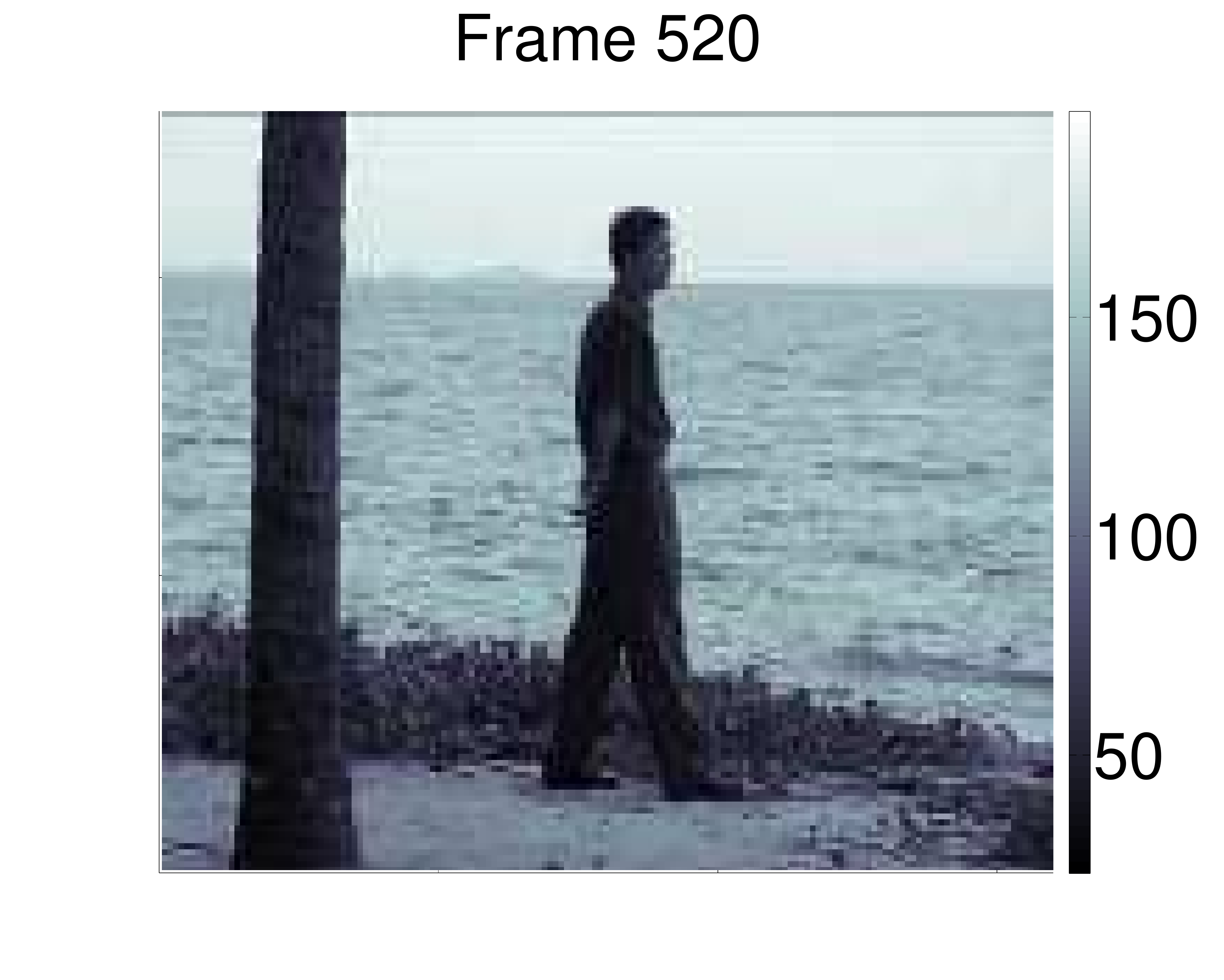} & 
\includegraphics[width=.25\columnwidth]{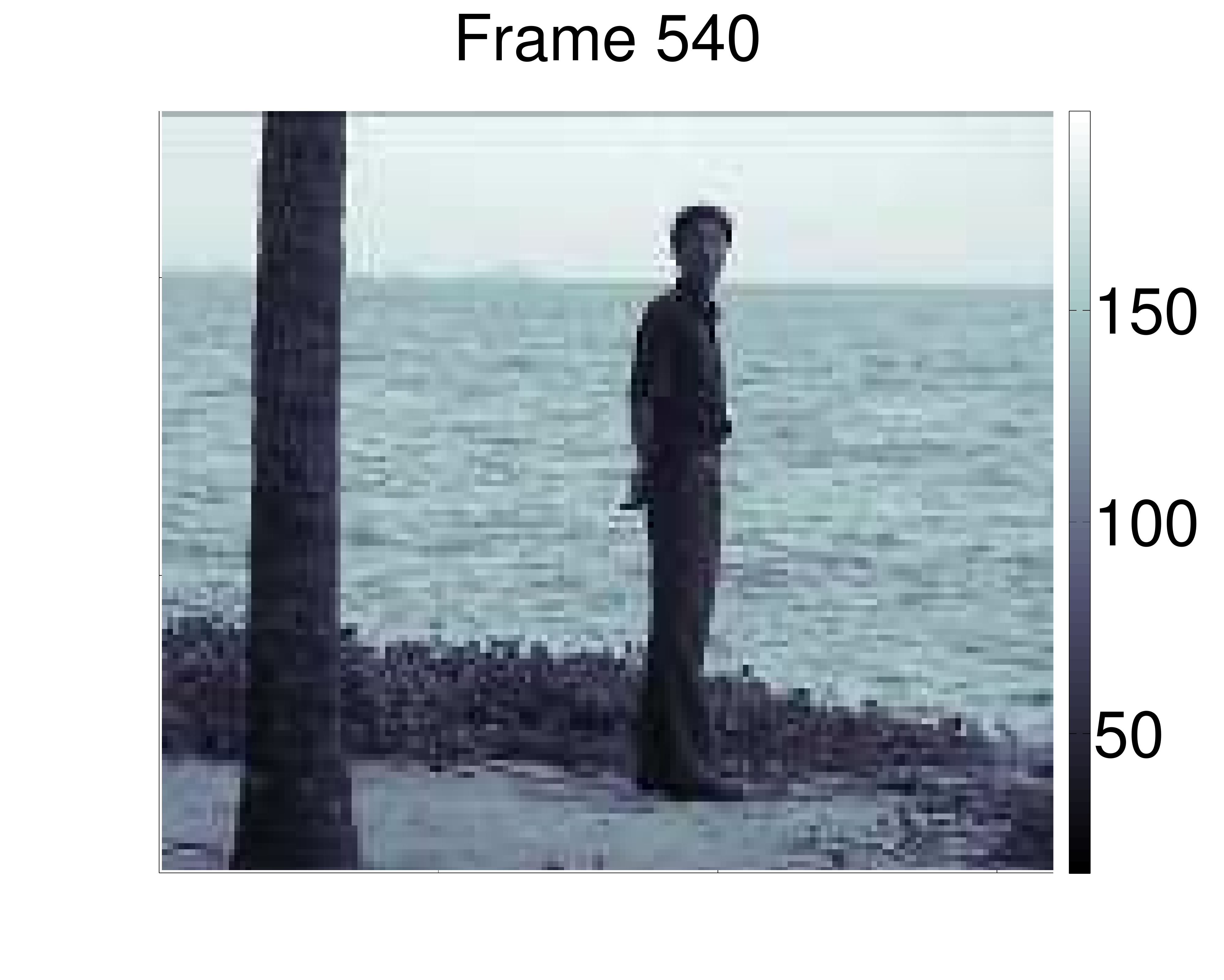} & 
\includegraphics[width=.25\columnwidth]{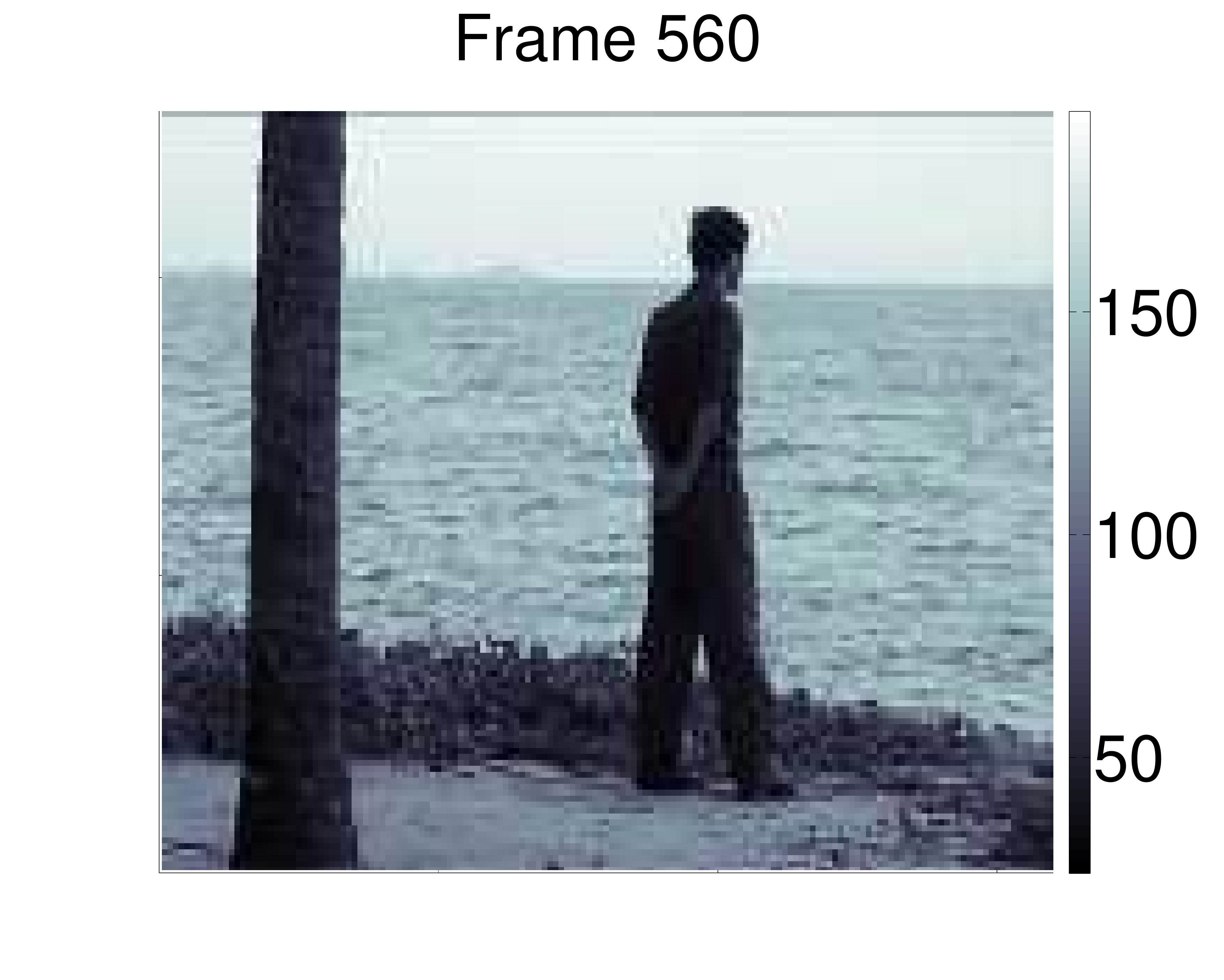} & 
\includegraphics[width=.25\columnwidth]{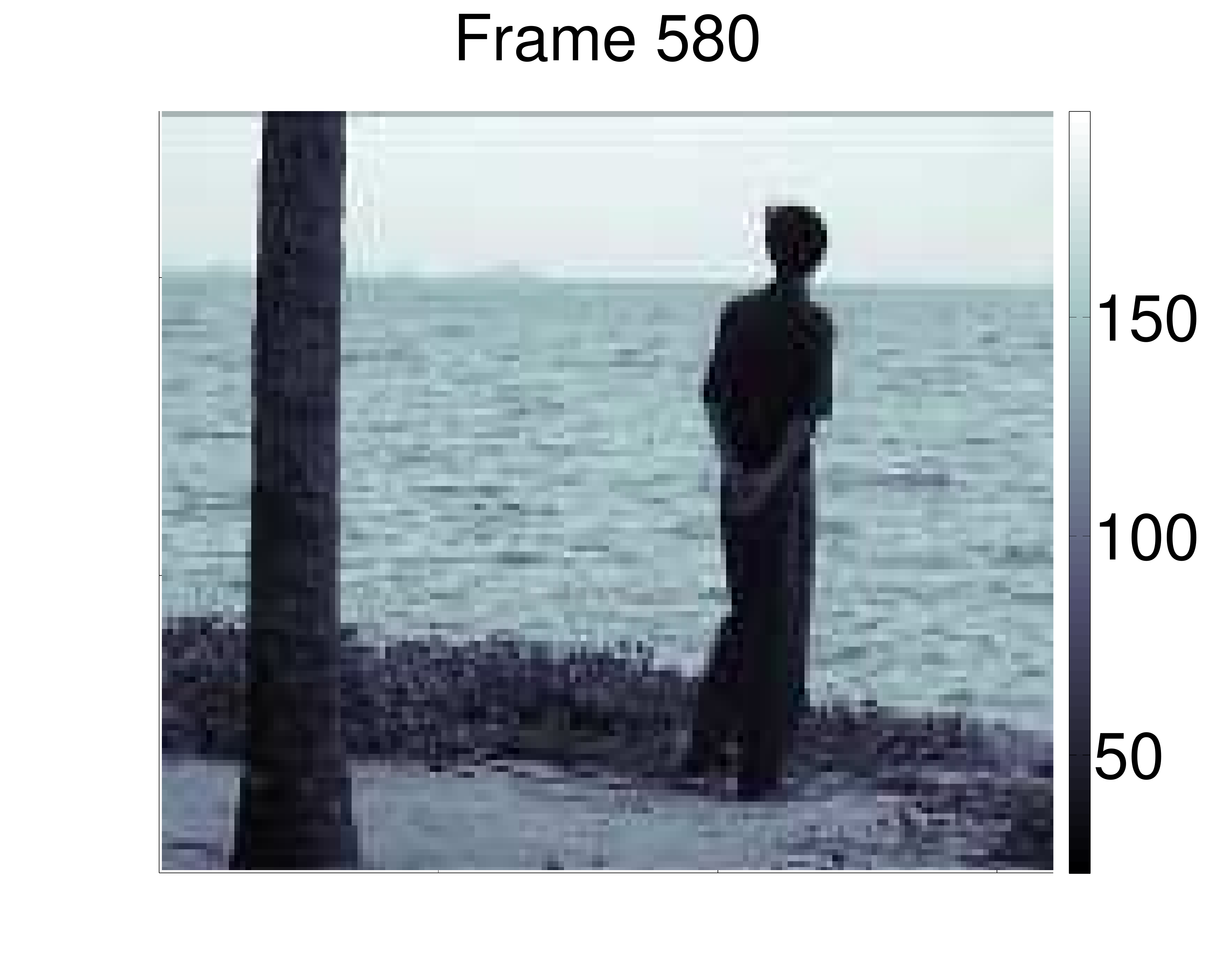}
\end{tabular}
\caption{Examples of the original frames of the water surface data set. Frames from 1 to 481 contain only background (true data) with a moving water surface. The person (considered as outlier) enters the scene in frame 482 and is present up to the last frame 633
}
\label{fig:wsorig}
\end{figure}
\clearpage

%%%%%%%%%%%%%%%%%%%%%%%%%%%%%%%%%%%%%%%%%%%%%%%%%%%%
\begin{figure}
\begin{tabular}{cccc}
\includegraphics[width=.25\columnwidth]{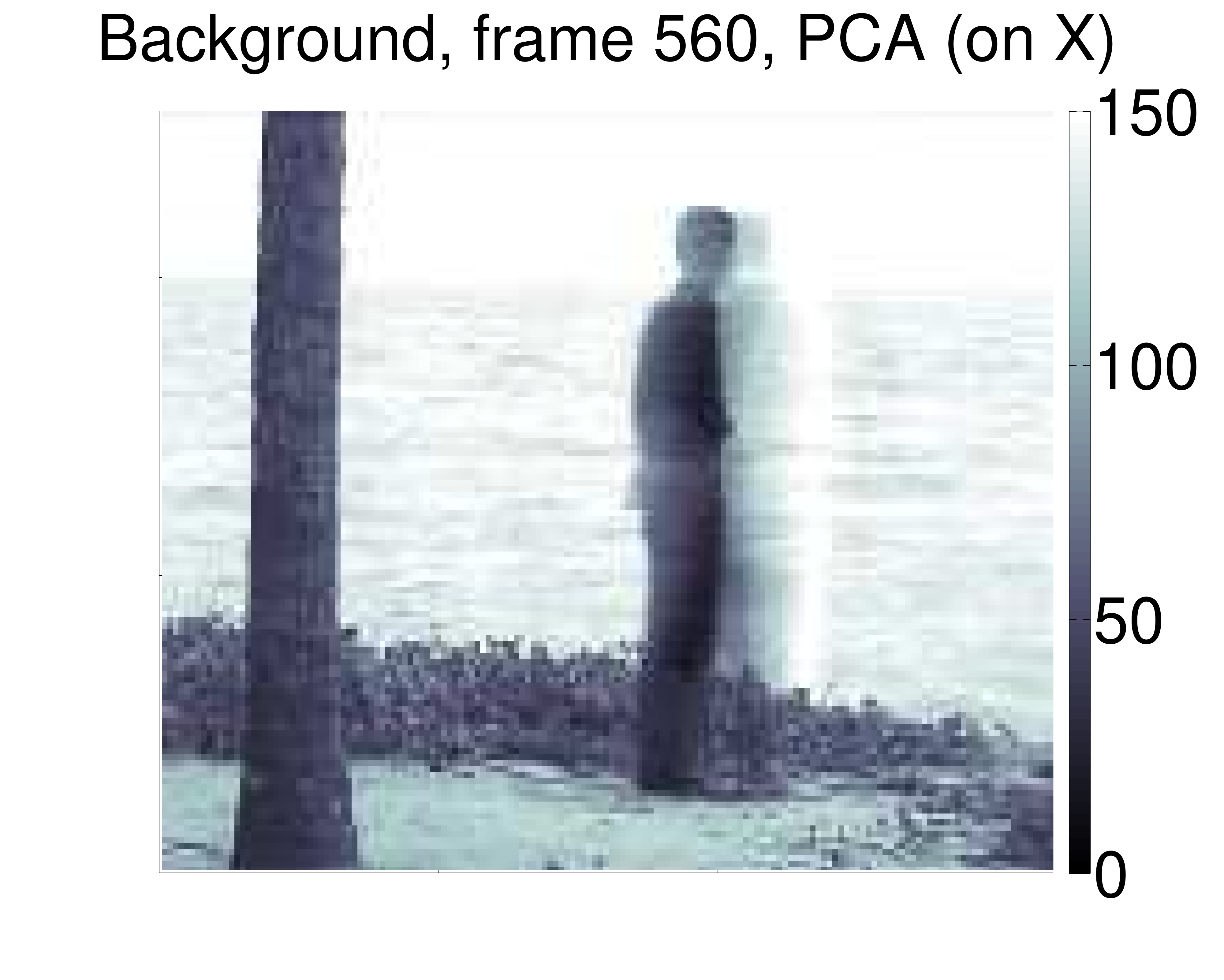} & 
\includegraphics[width=.25\columnwidth]{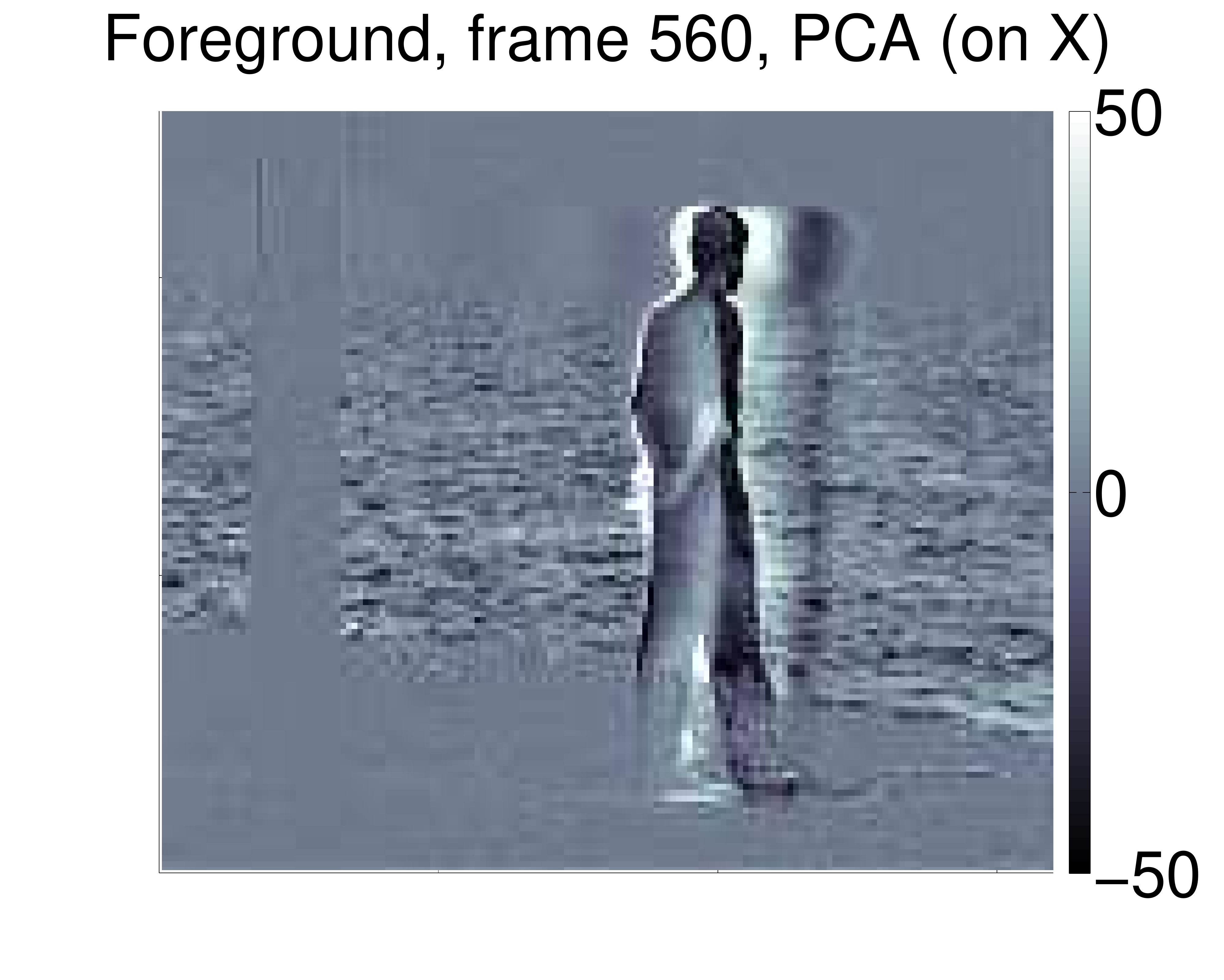} & 
\includegraphics[width=.25\columnwidth]{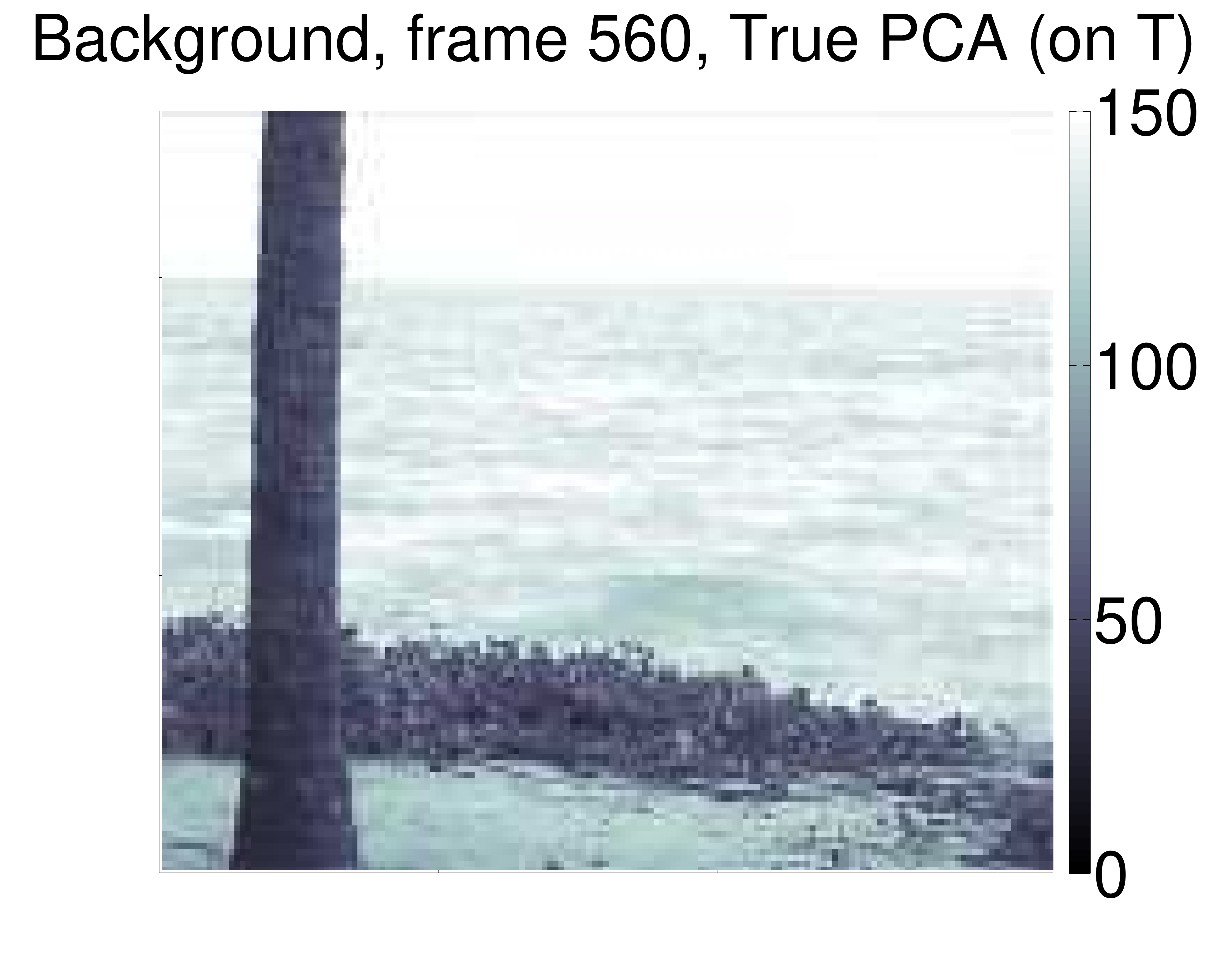} & 
\includegraphics[width=.25\columnwidth]{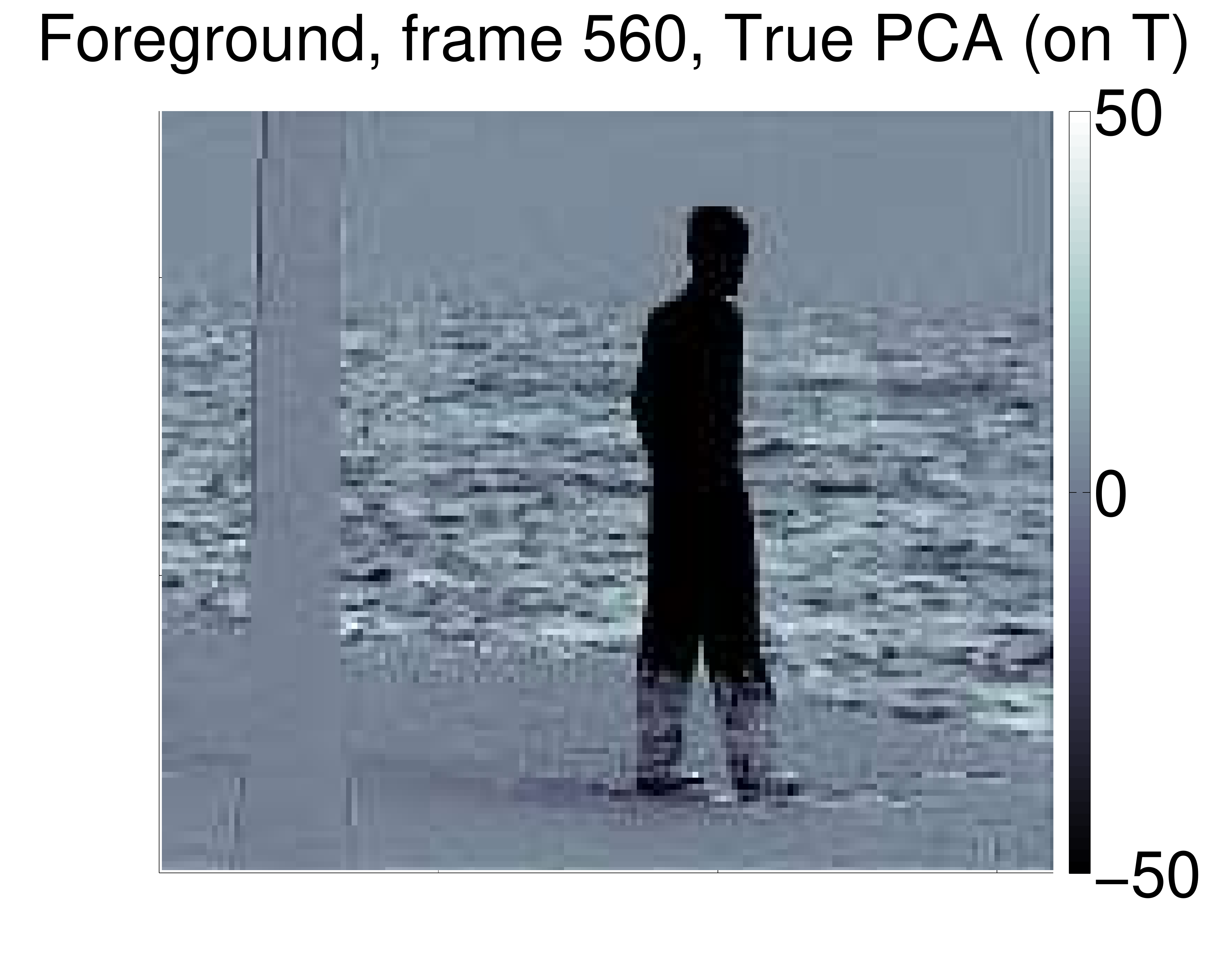} \\ \hline
\includegraphics[width=.25\columnwidth]{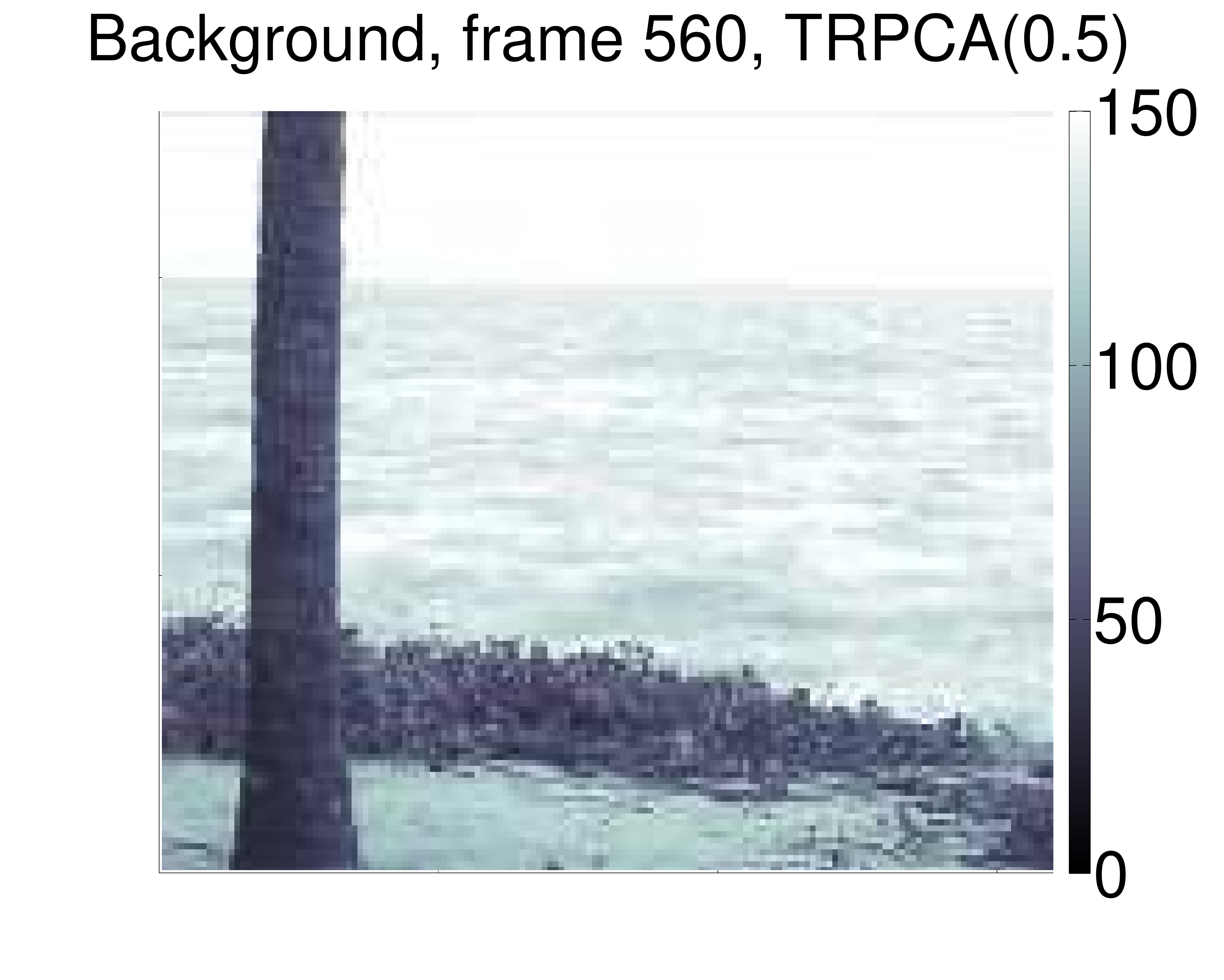} & 
\includegraphics[width=.25\columnwidth]{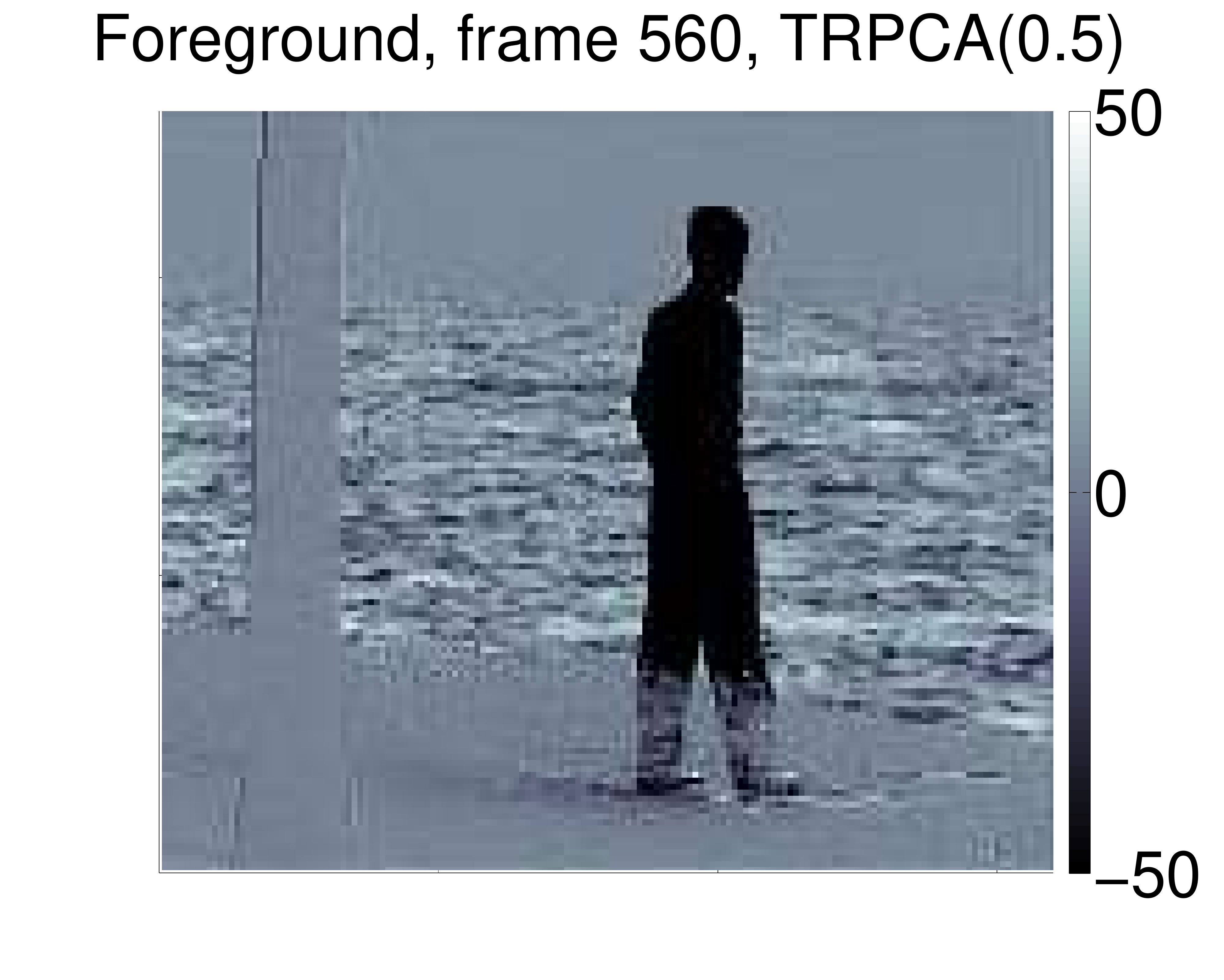} & 
\includegraphics[width=.25\columnwidth]{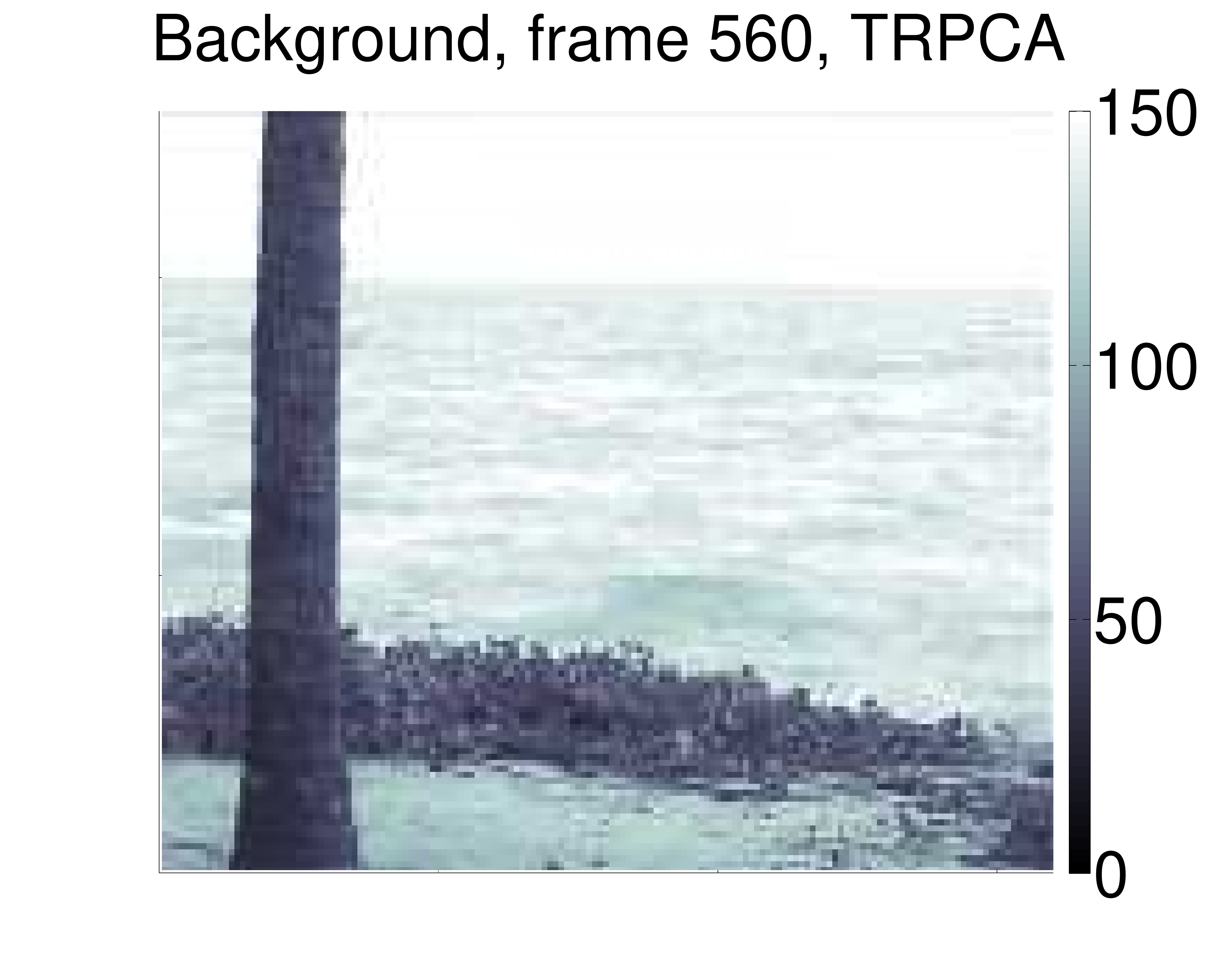} & 
\includegraphics[width=.25\columnwidth]{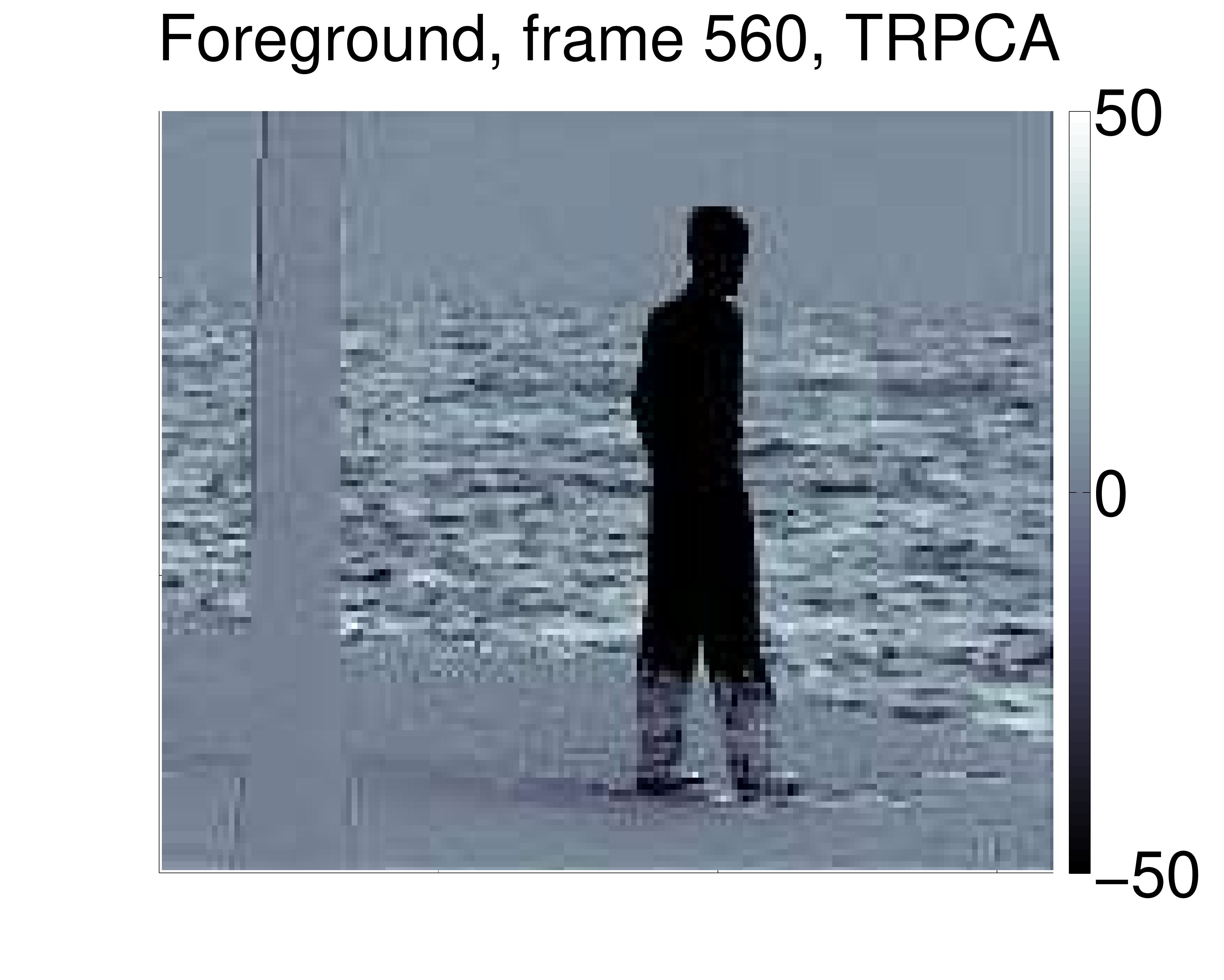} \\
\includegraphics[width=.25\columnwidth]{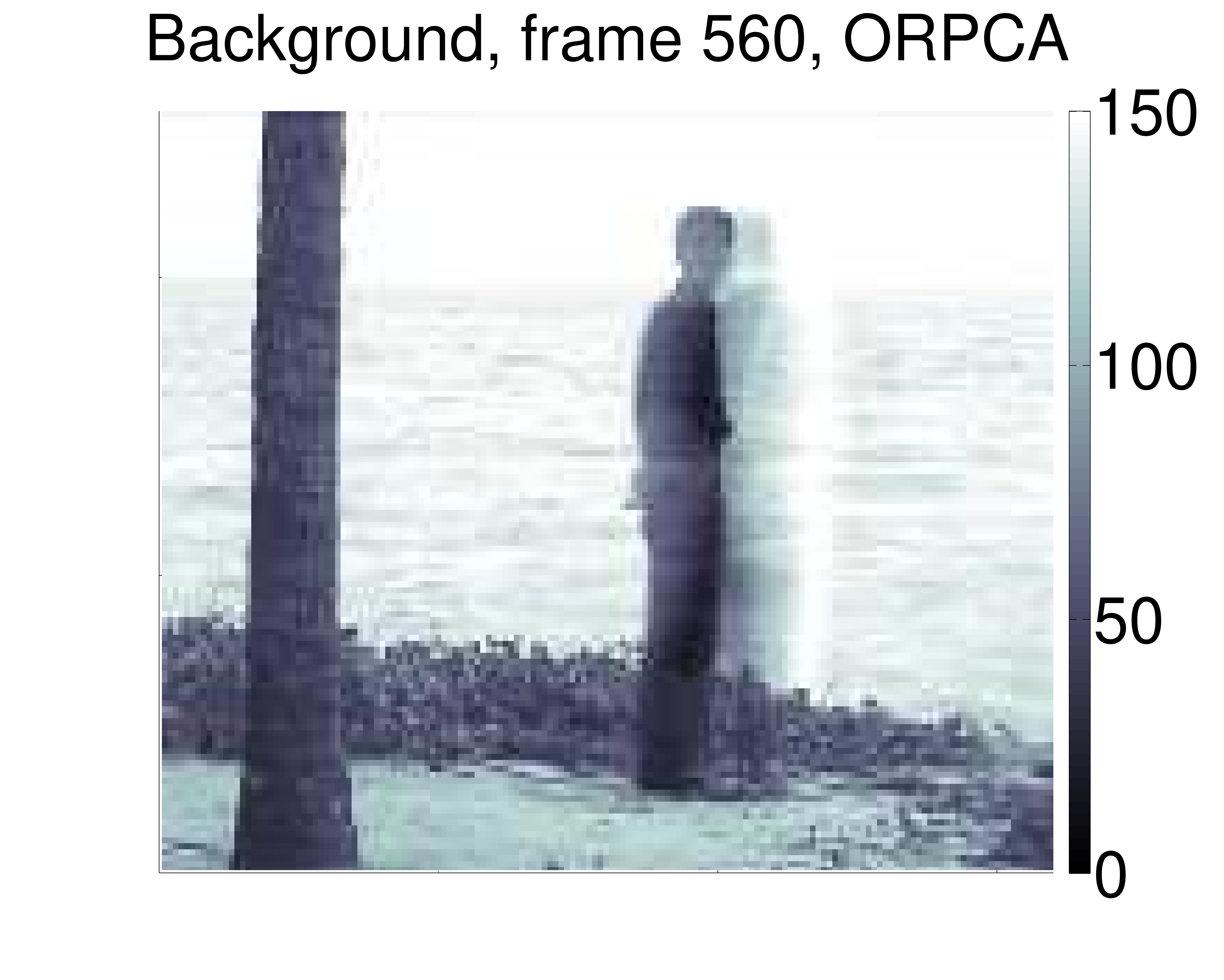} & 
\includegraphics[width=.25\columnwidth]{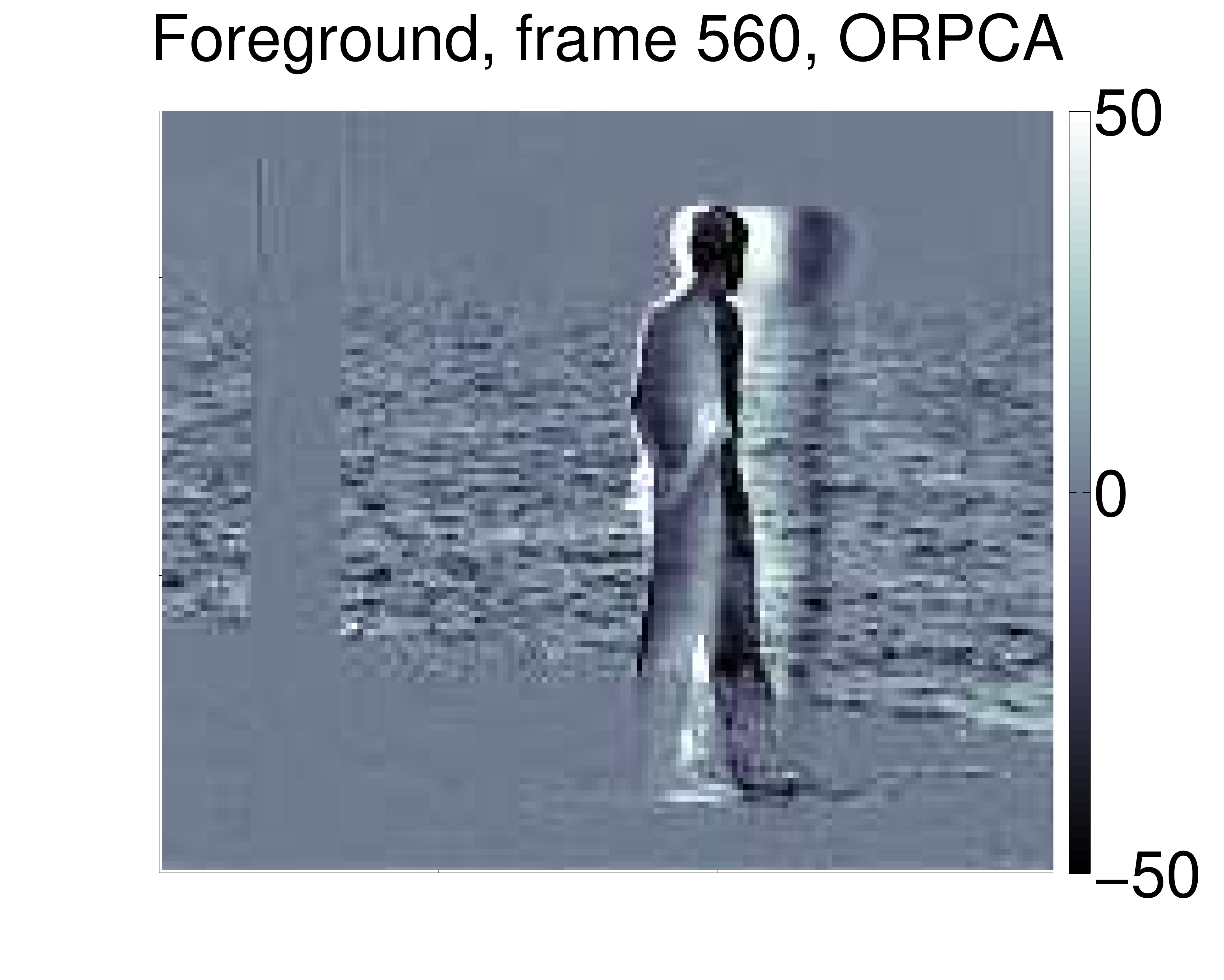} & 
\includegraphics[width=.25\columnwidth]{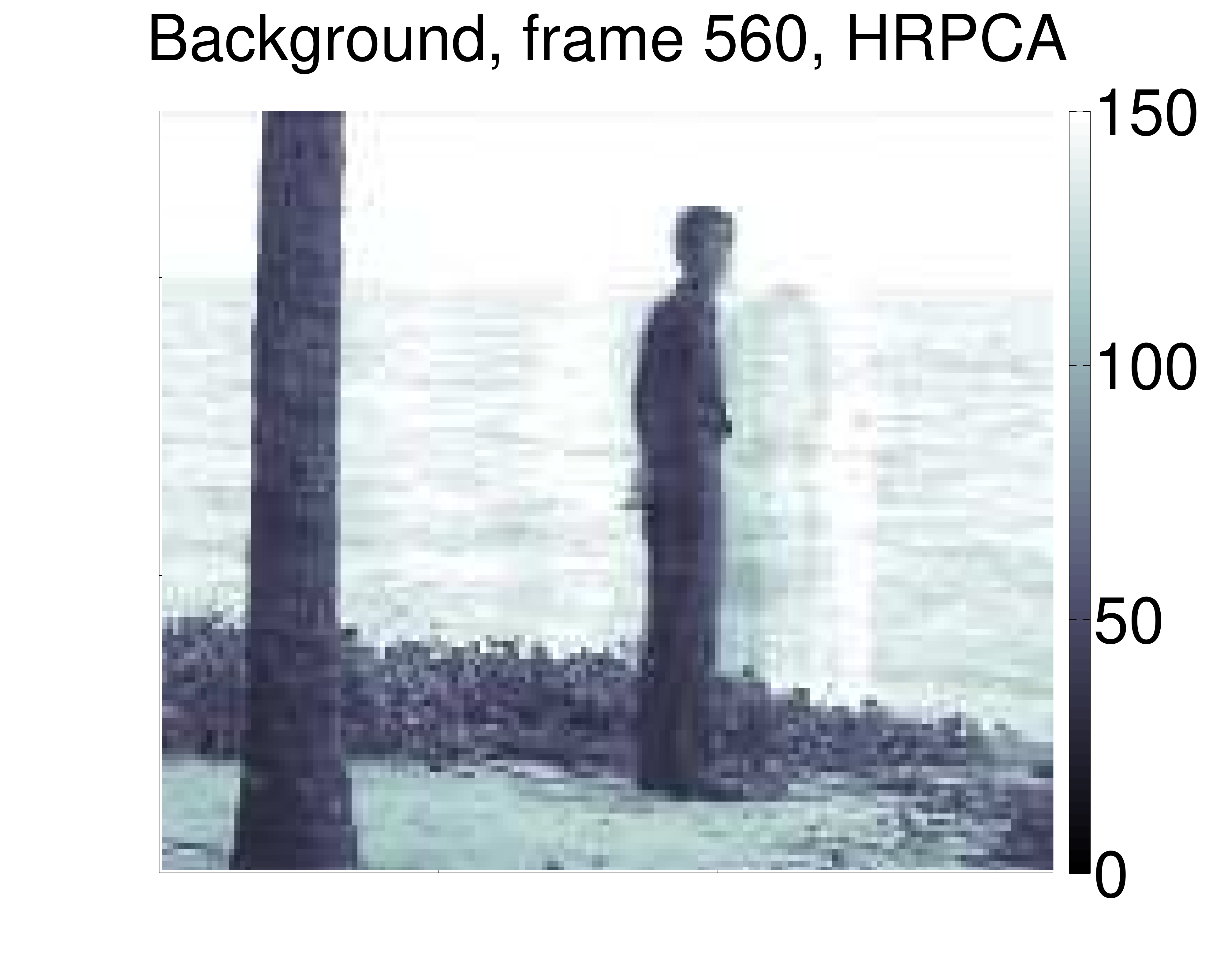} & 
\includegraphics[width=.25\columnwidth]{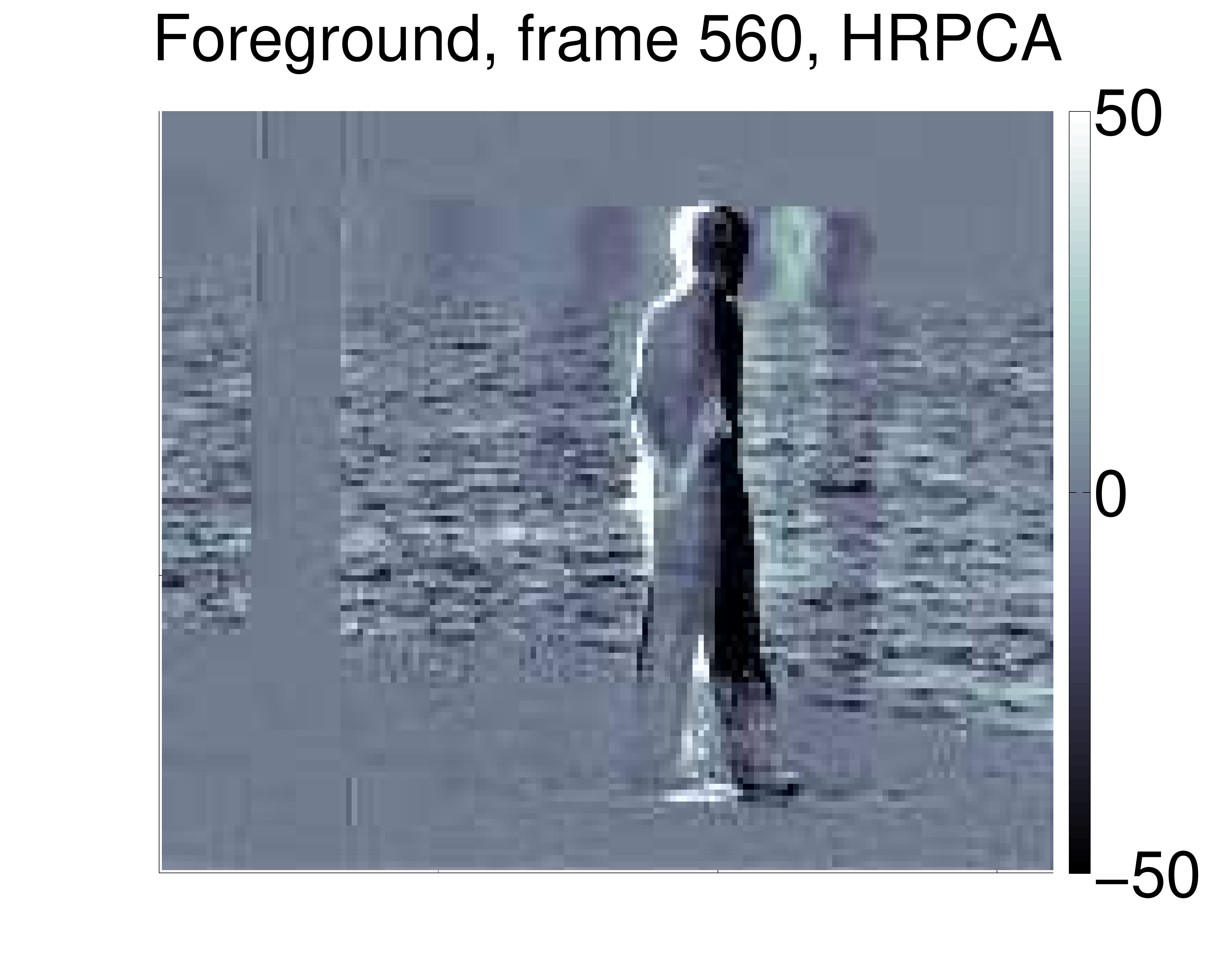} \\
\includegraphics[width=.25\columnwidth]{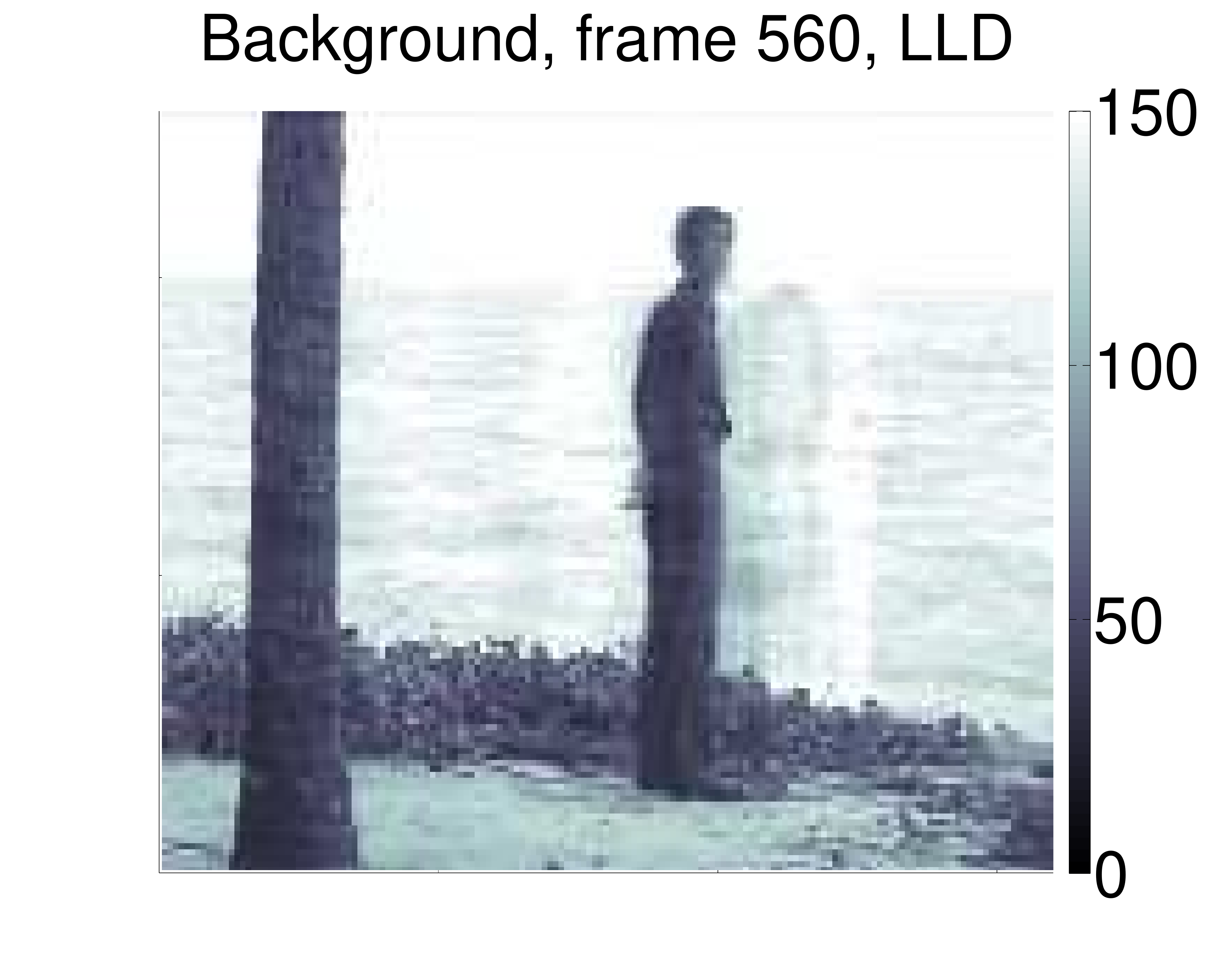} & 
\includegraphics[width=.25\columnwidth]{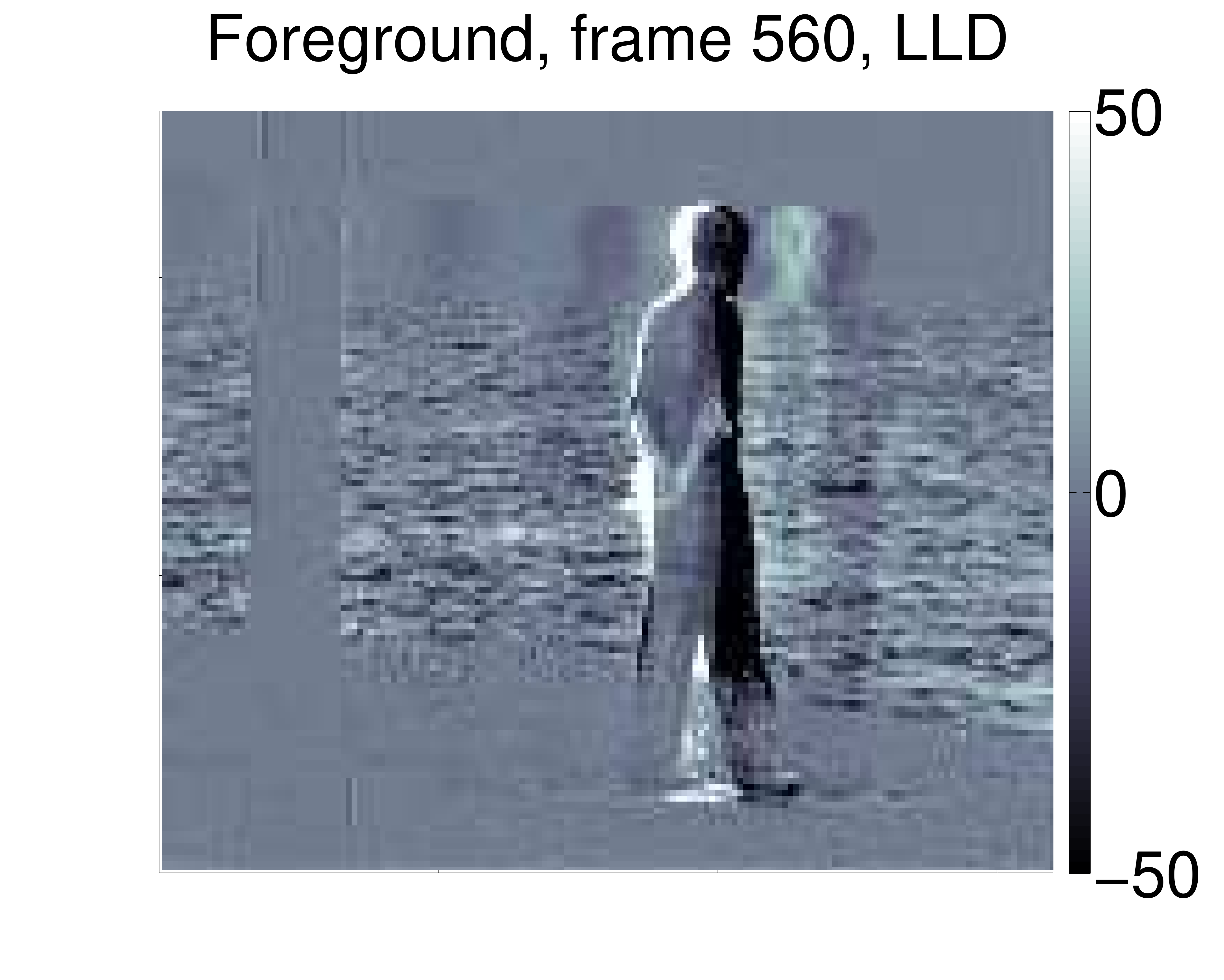} & 
\includegraphics[width=.25\columnwidth]{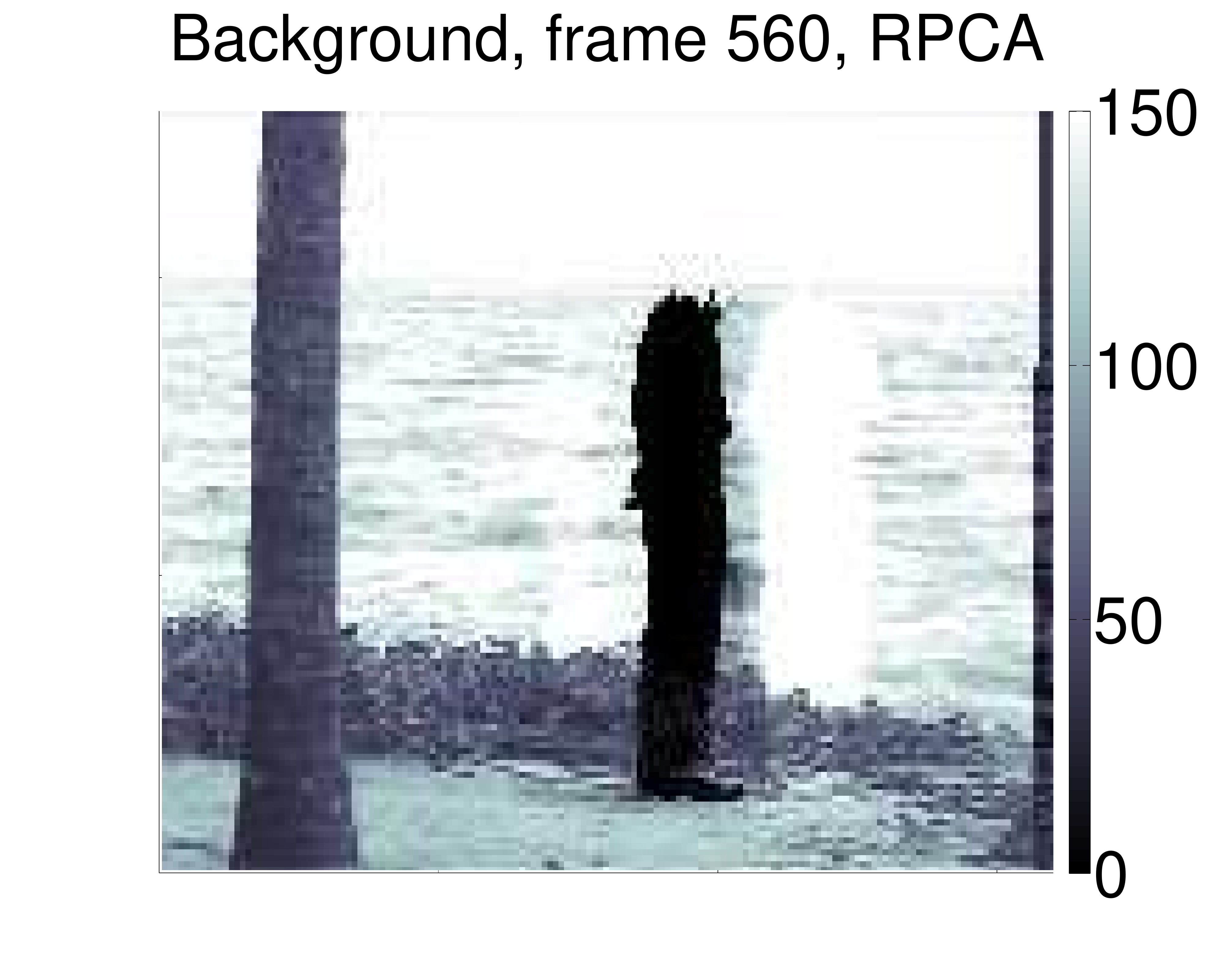} & 
\includegraphics[width=.25\columnwidth]{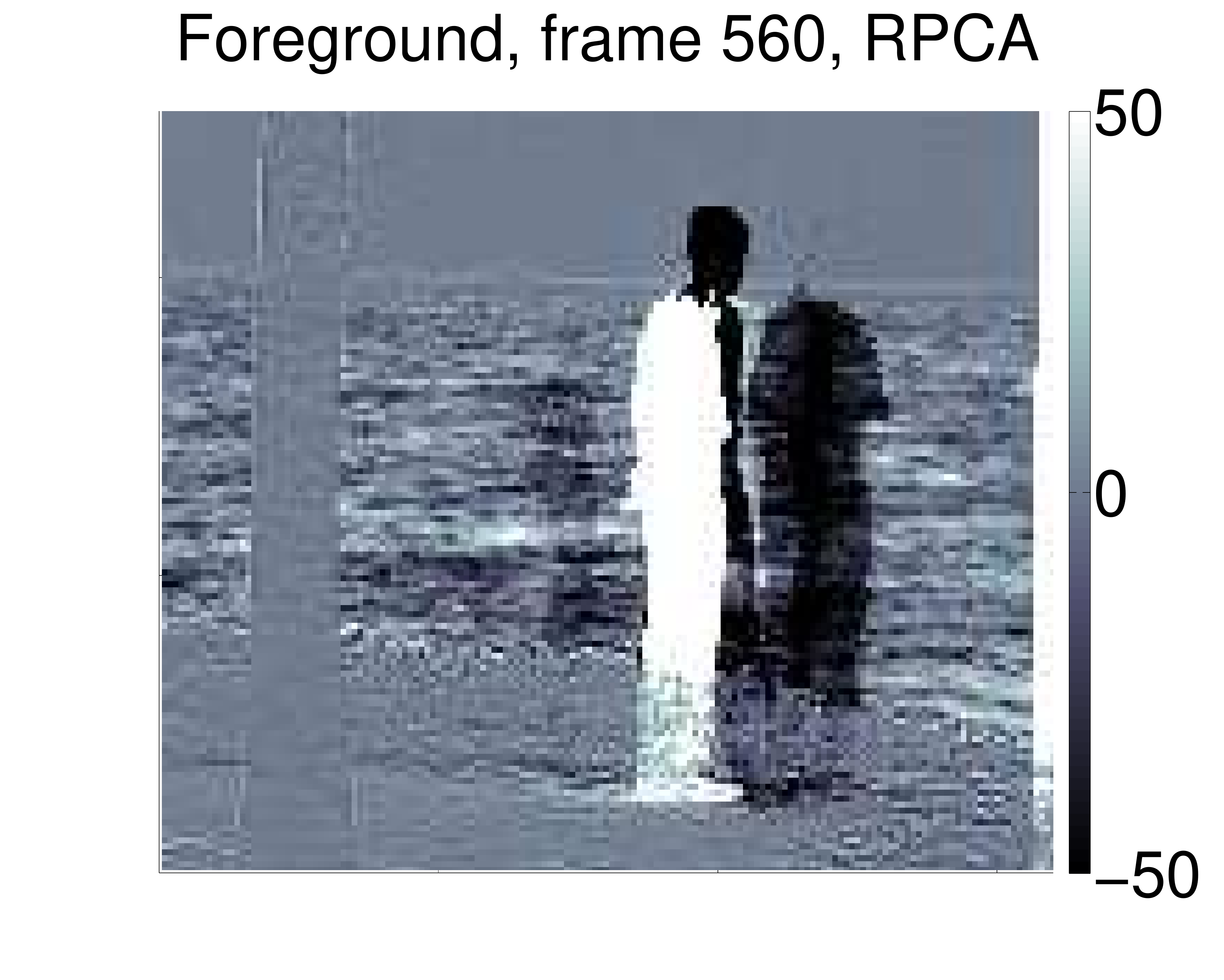}
\end{tabular}
\caption{Background $x_i^b$ and foreground $x_i^f$ recovered with different methods, using \eqref{foreground}, of frame 560 of the water surface data set, number of components $k=10$. These images correspond to the one of Fig. \ref{fig:bfwsi560}, but the scaling has been changed for better visibility. Namely, all backgrounds/foreground images are rescaled so that the maximum and minimum pixel values are the same (please, note the numbers on the color bar); results for PCP can be found in Fig. \ref{fig:wsbfpcp}
}
\label{fig:wsbf}
\end{figure}
\clearpage

%%%%%%%%%%%%%%%%%%%%%%%%%%%%%%%%%%%%%%%%%%%%%%%%%%%%
\begin{figure}
\begin{tabular}{cccc}
\includegraphics[width=.25\columnwidth]{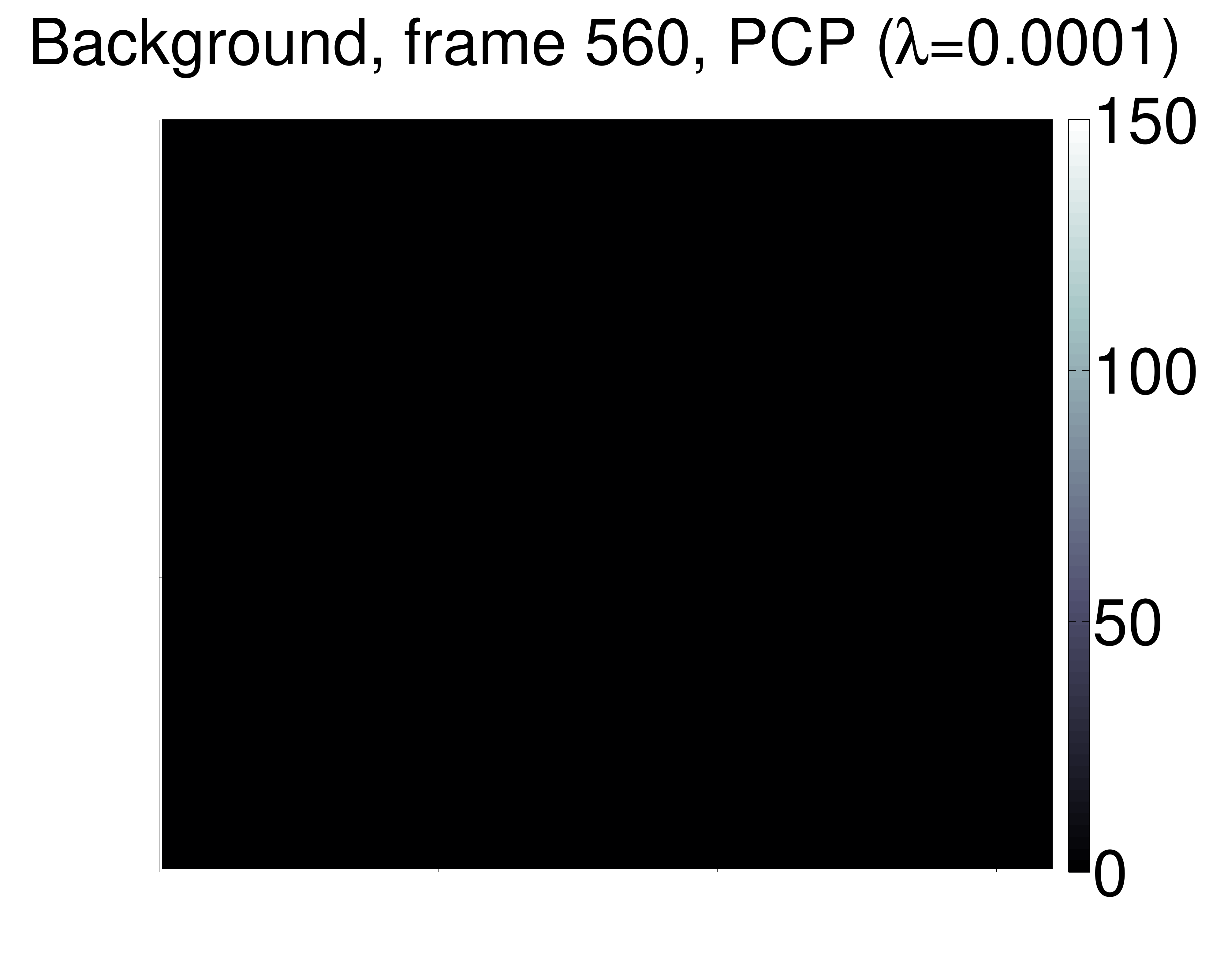} & 
\includegraphics[width=.25\columnwidth]{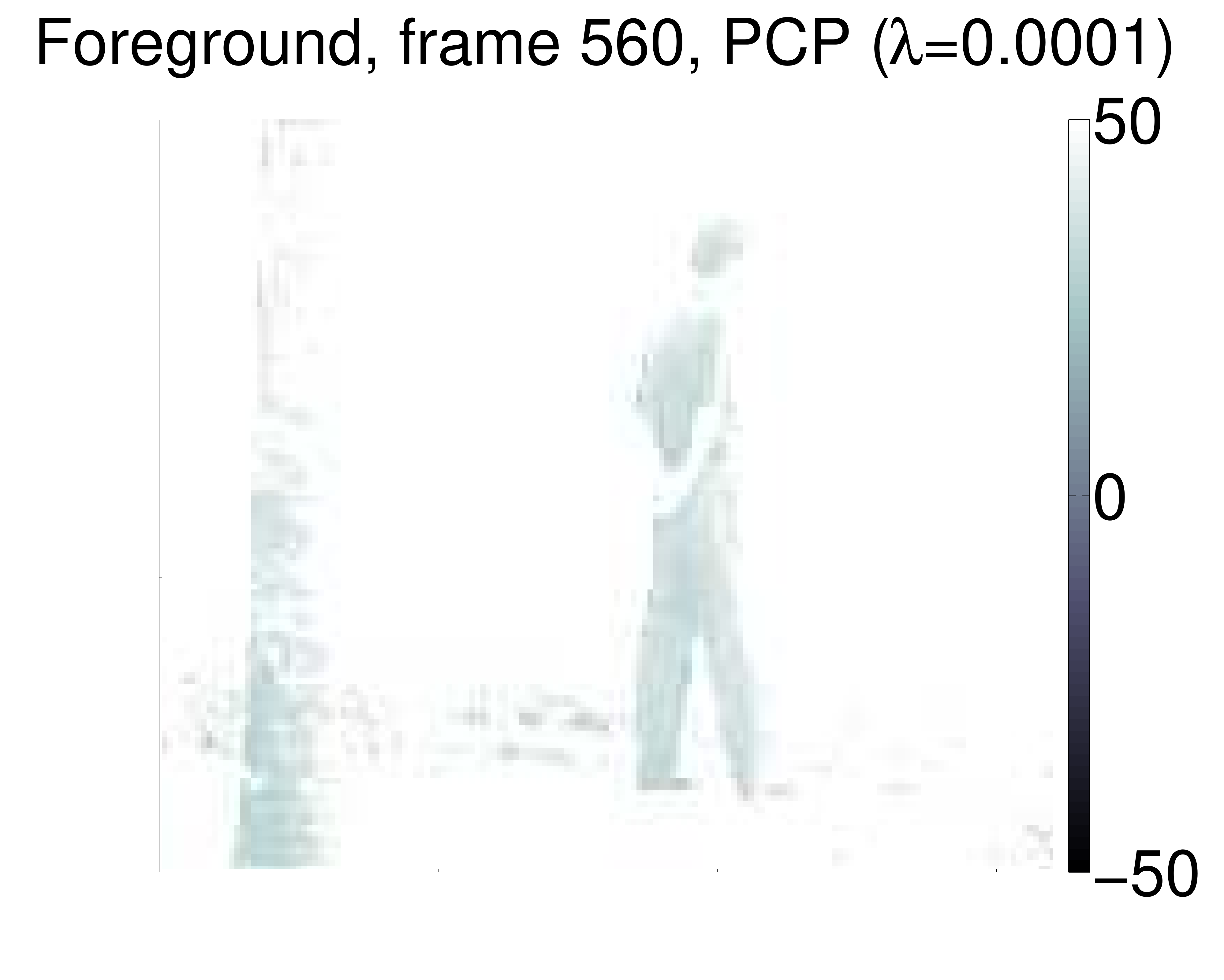} & 
\includegraphics[width=.25\columnwidth]{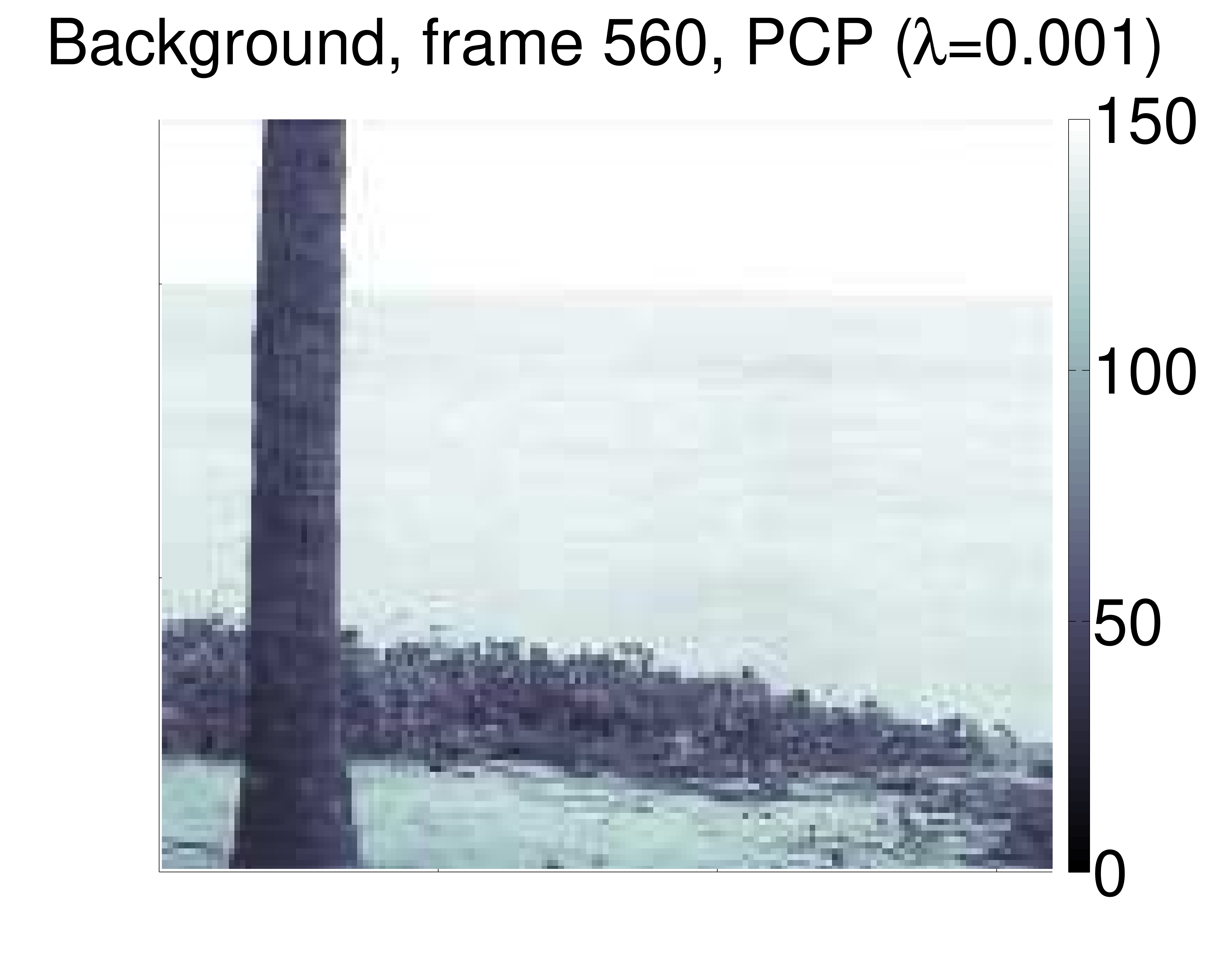} & 
\includegraphics[width=.25\columnwidth]{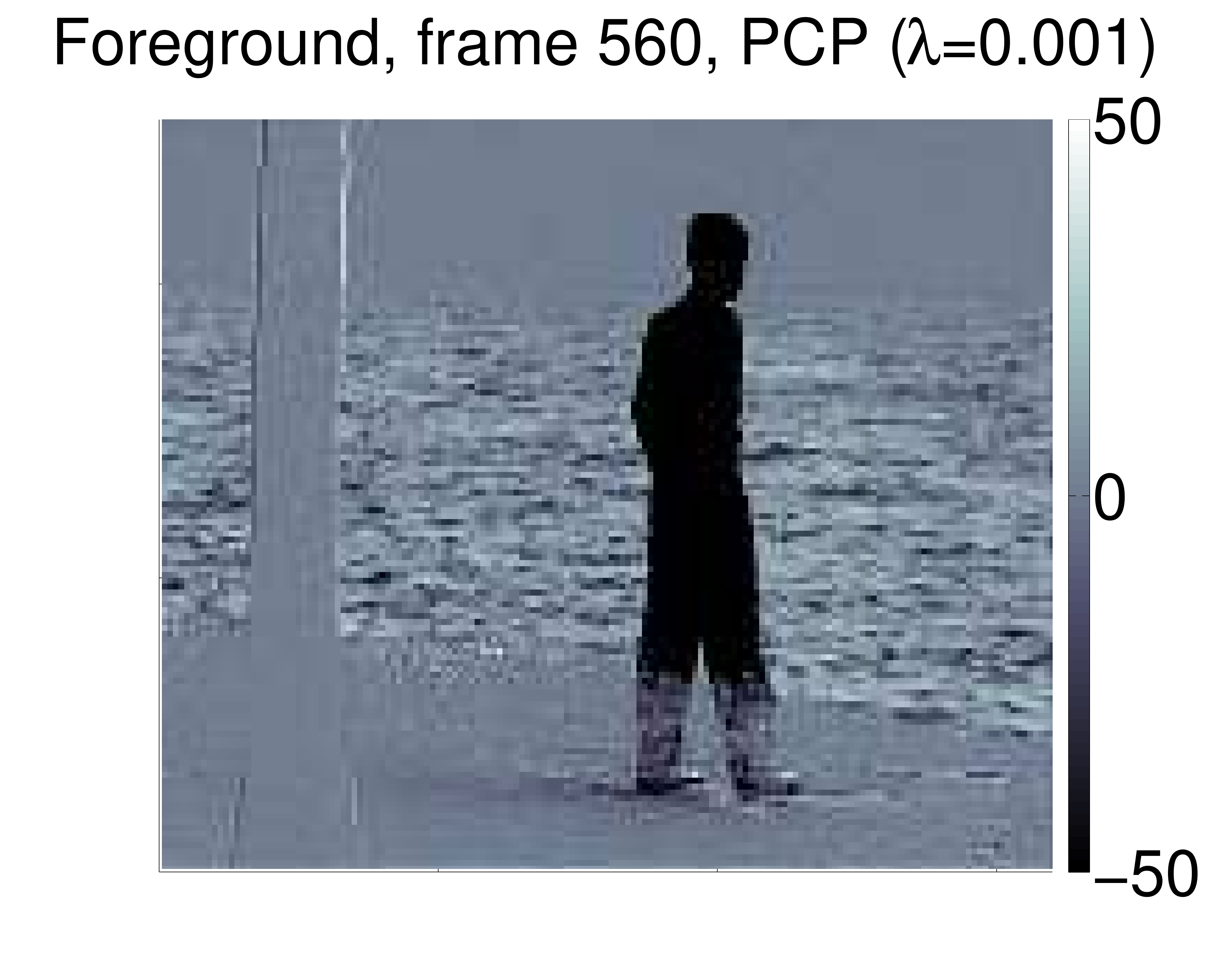} \\
\includegraphics[width=.25\columnwidth]{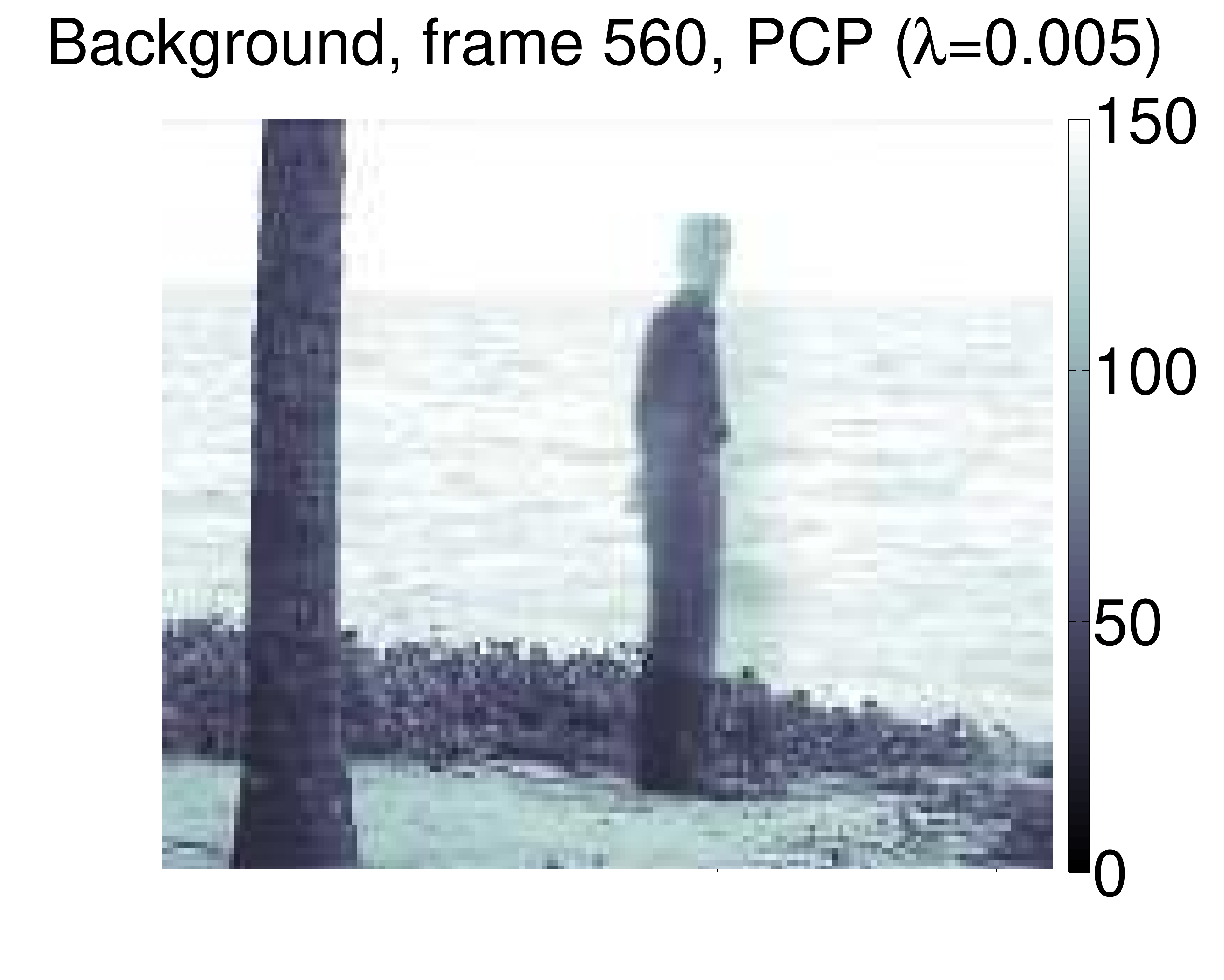} & 
\includegraphics[width=.25\columnwidth]{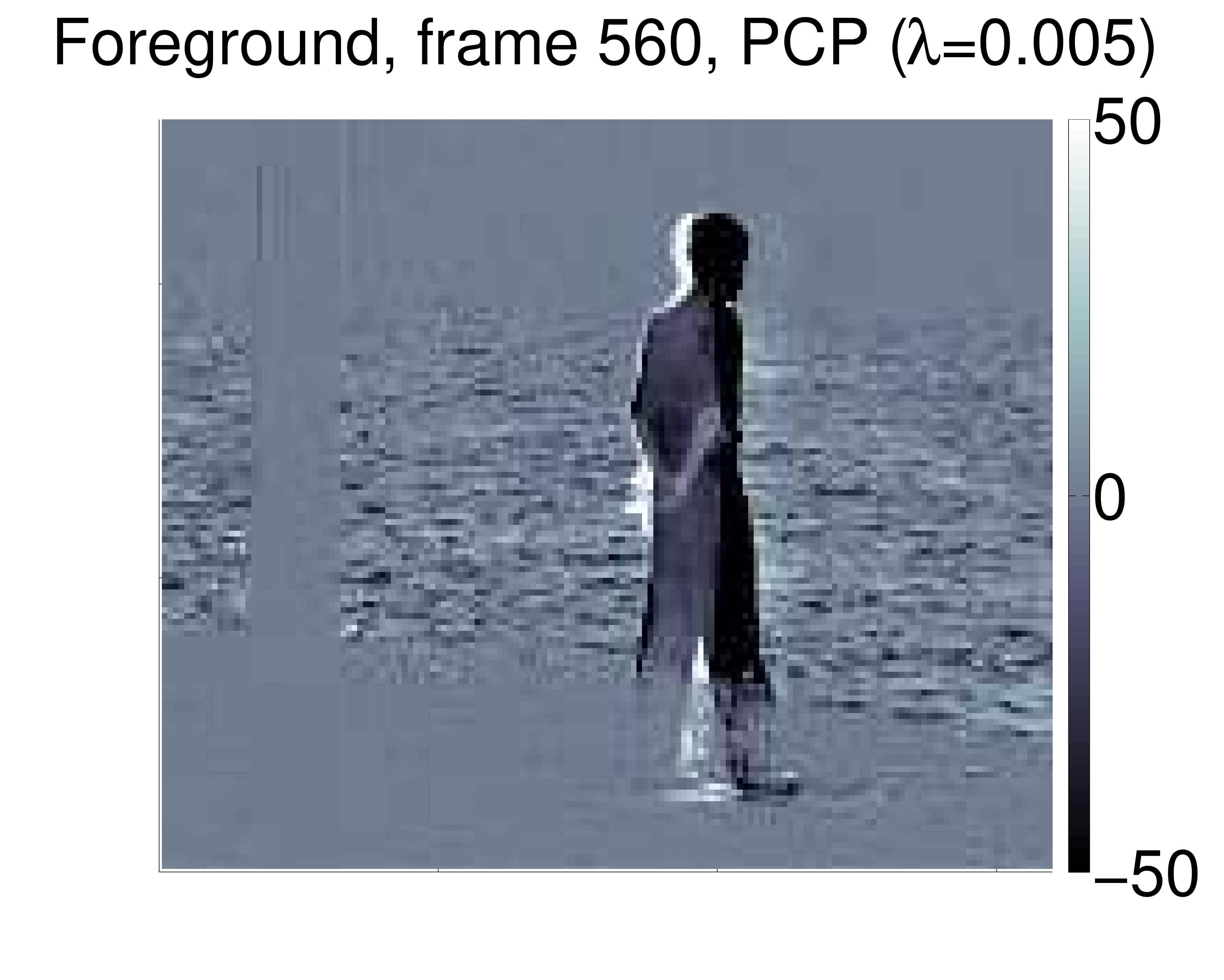} & 
\includegraphics[width=.25\columnwidth]{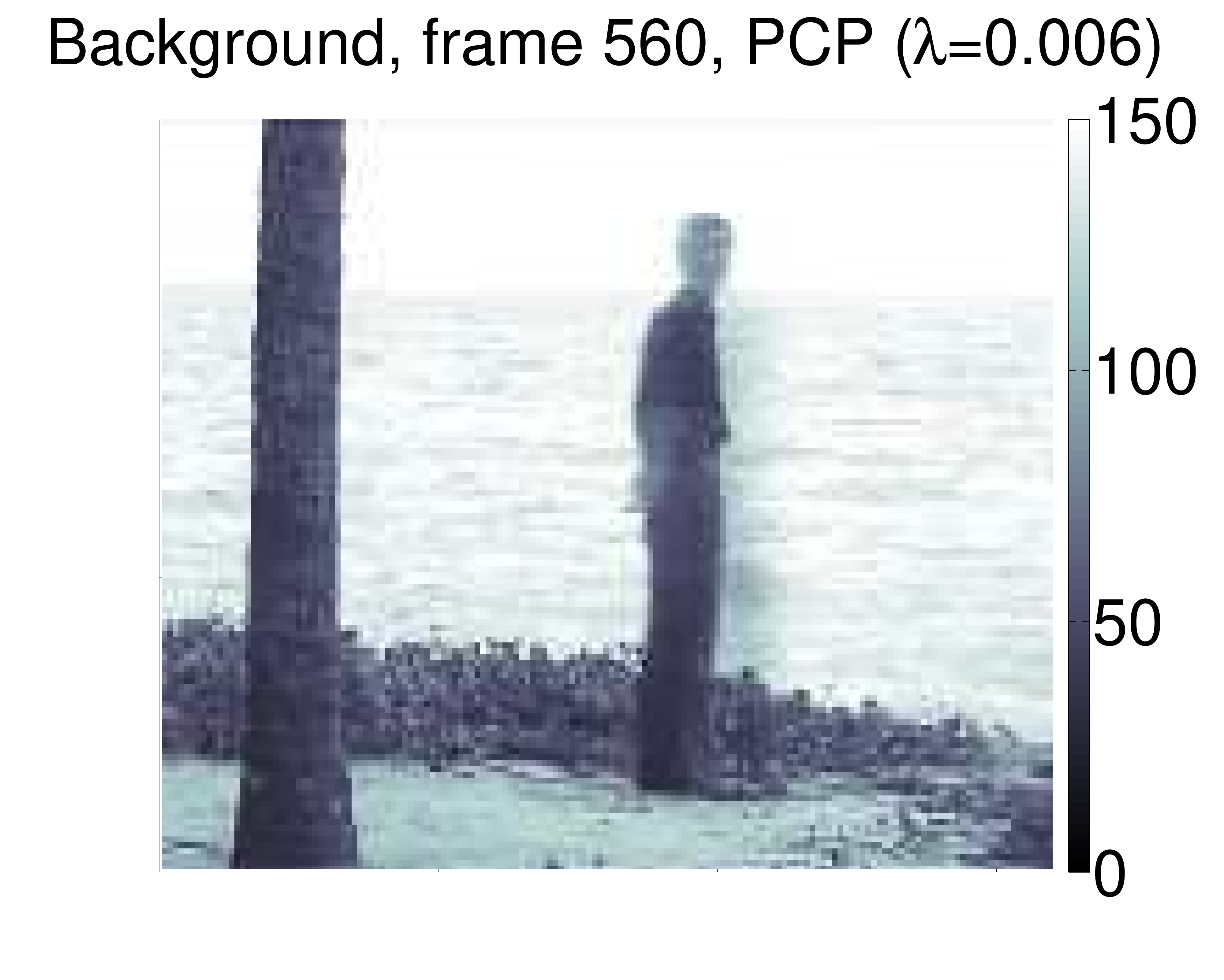} & 
\includegraphics[width=.25\columnwidth]{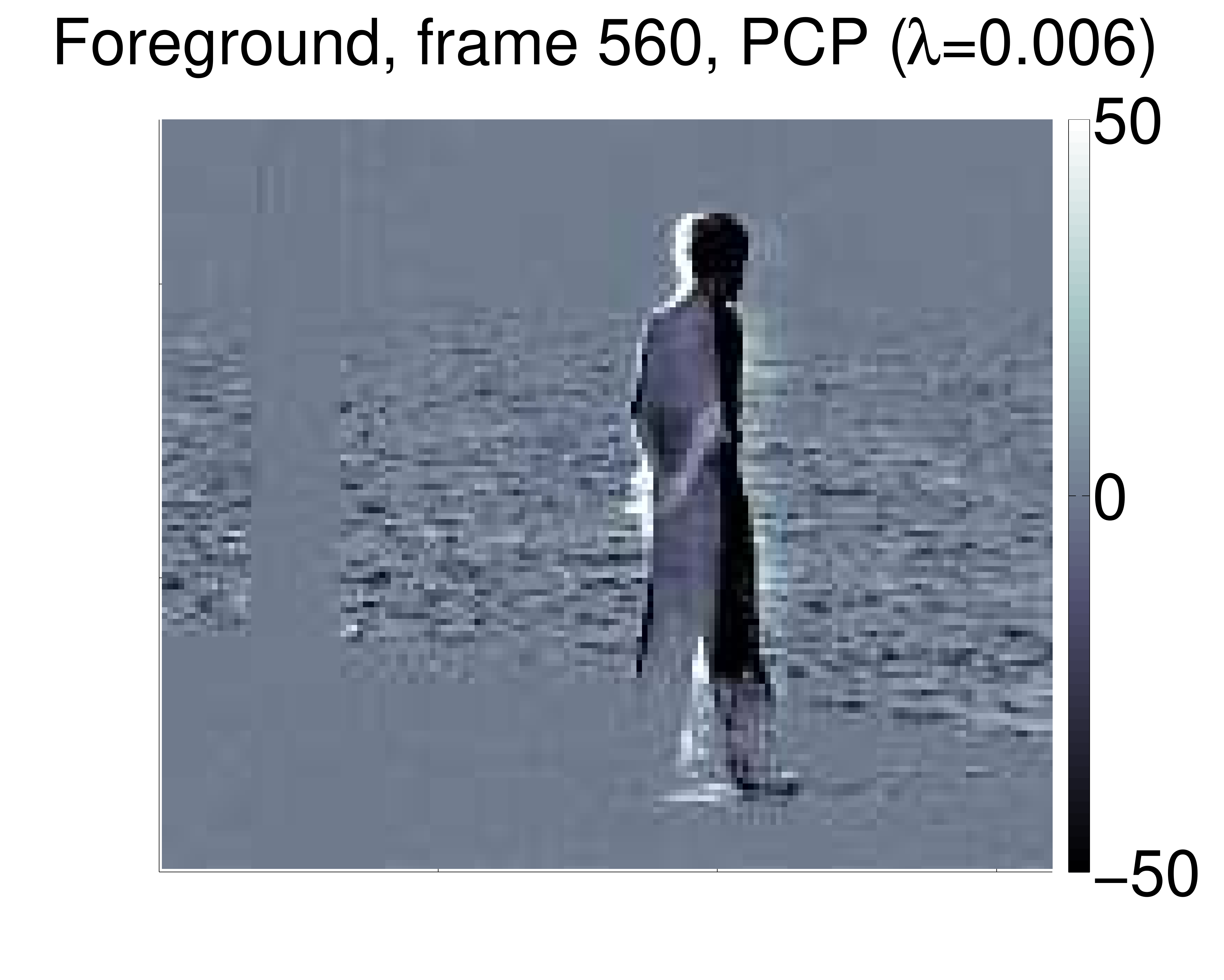} \\
\includegraphics[width=.25\columnwidth]{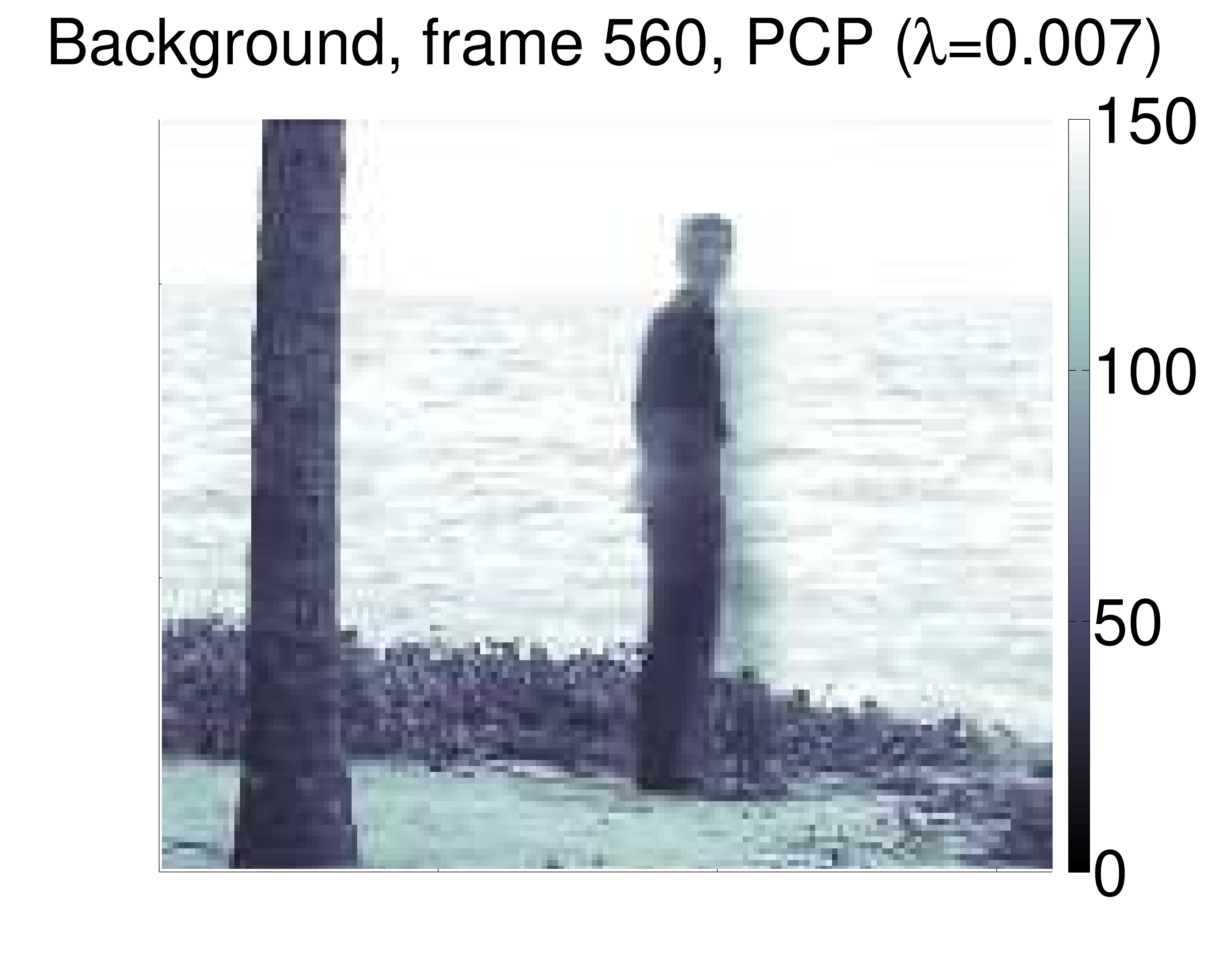} & 
\includegraphics[width=.25\columnwidth]{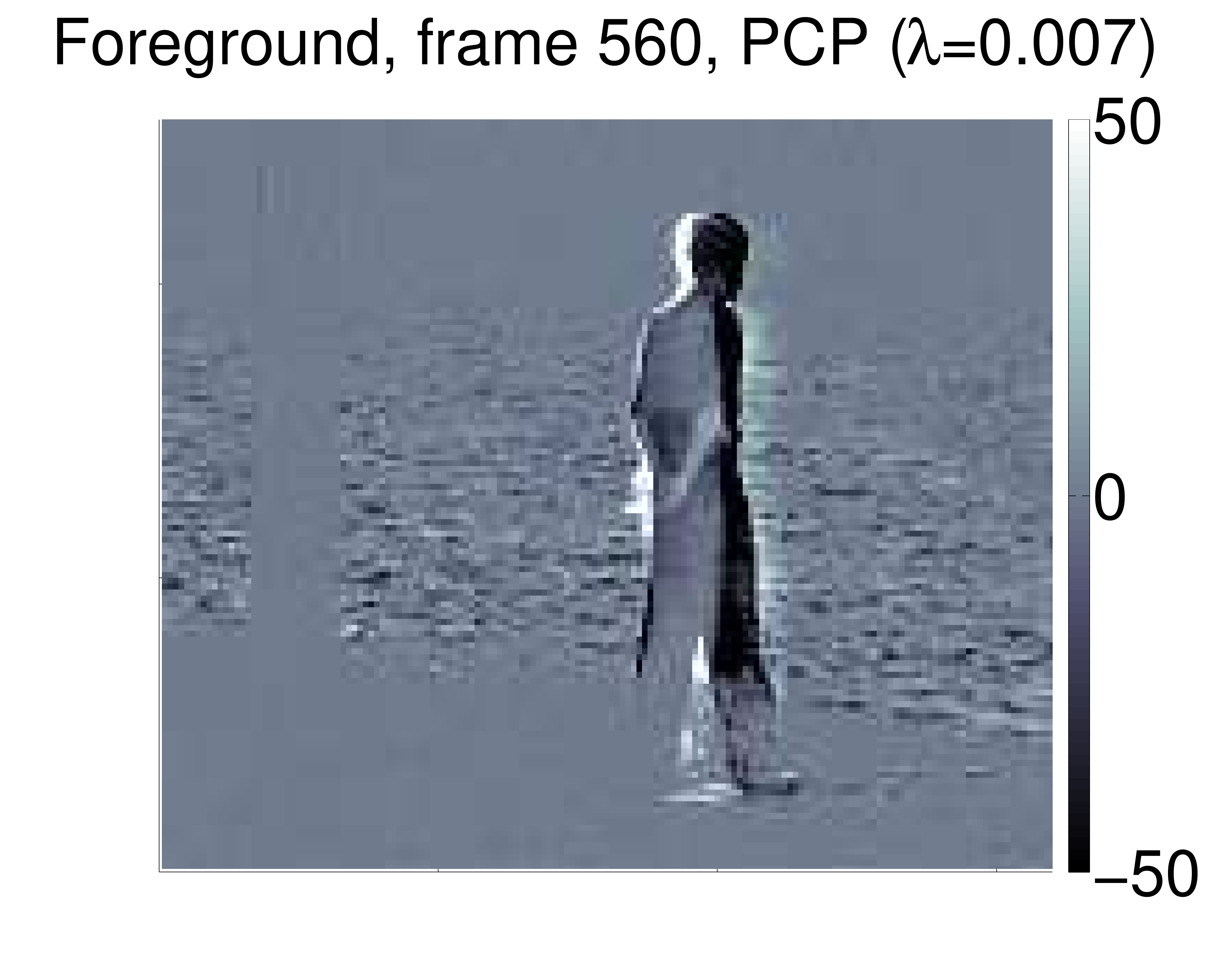} & 
\includegraphics[width=.25\columnwidth]{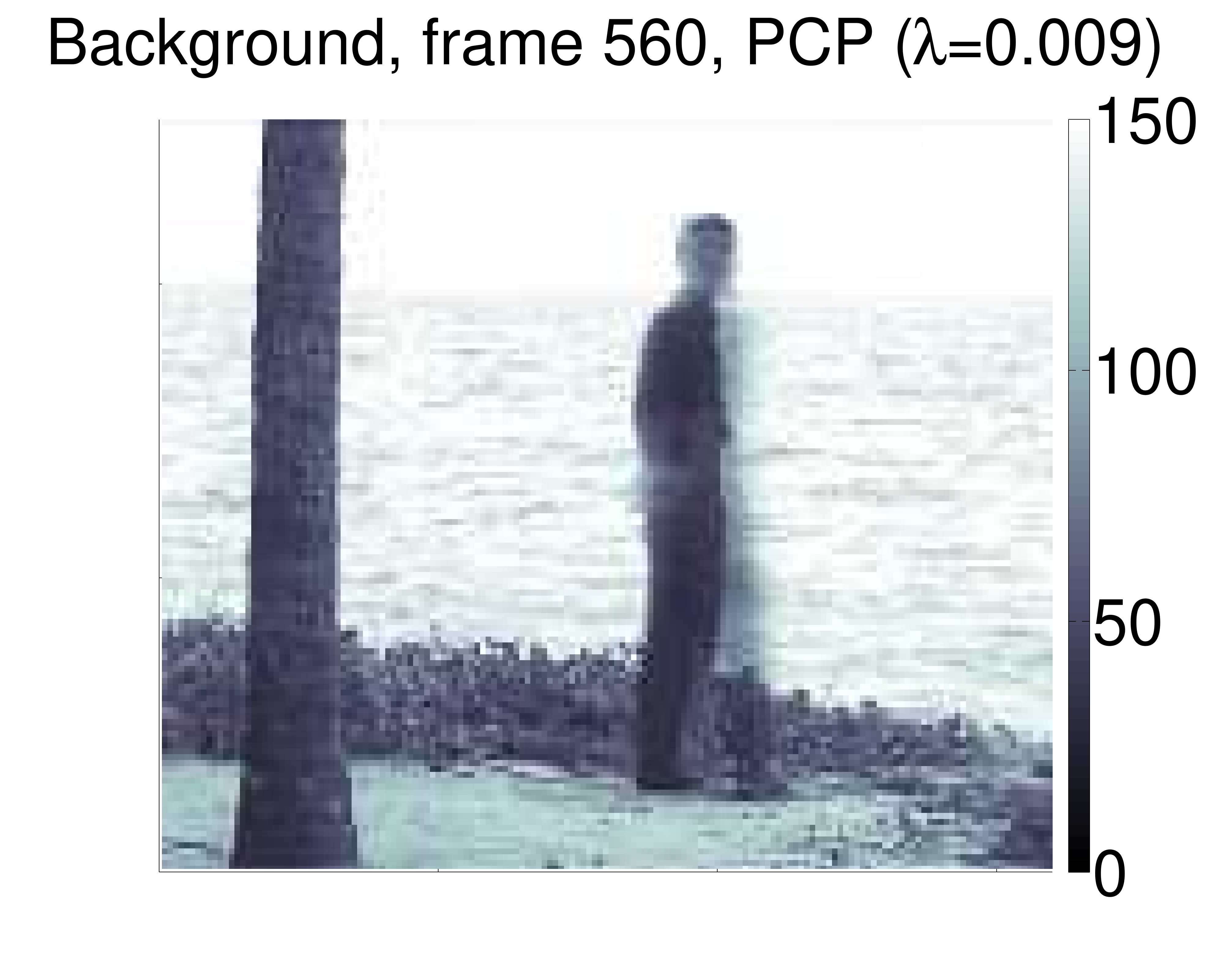} & 
\includegraphics[width=.25\columnwidth]{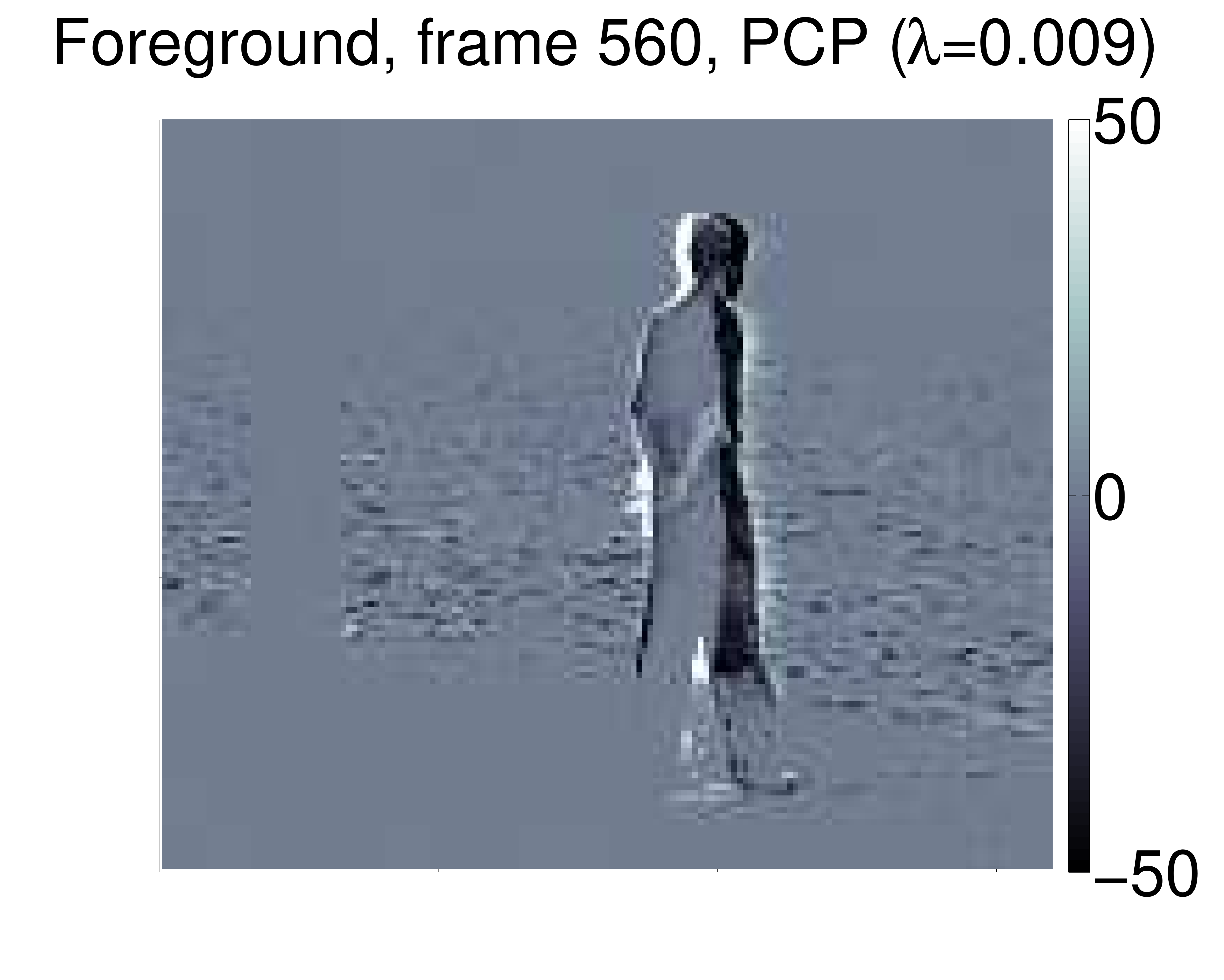} \\
\includegraphics[width=.25\columnwidth]{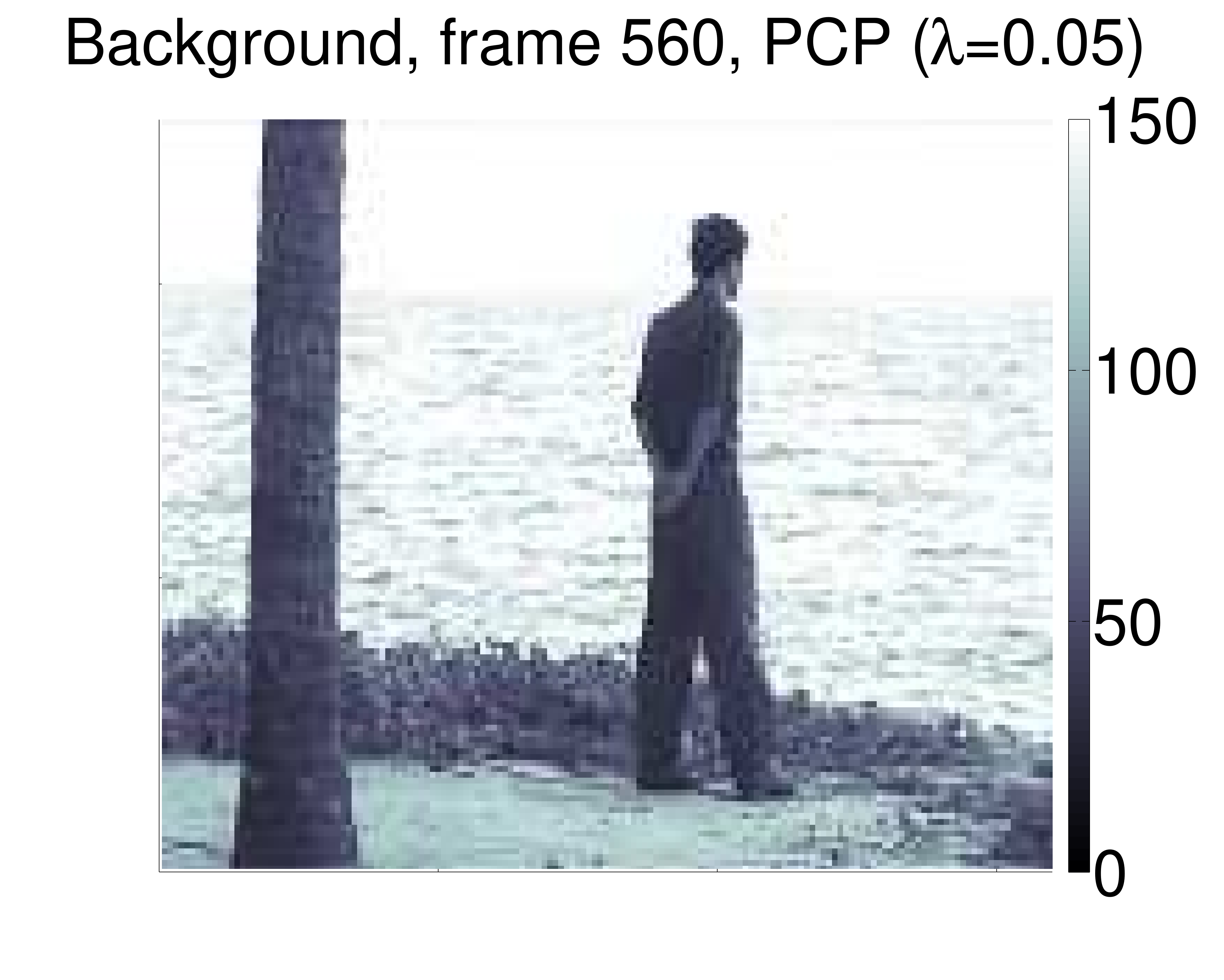} & 
\includegraphics[width=.25\columnwidth]{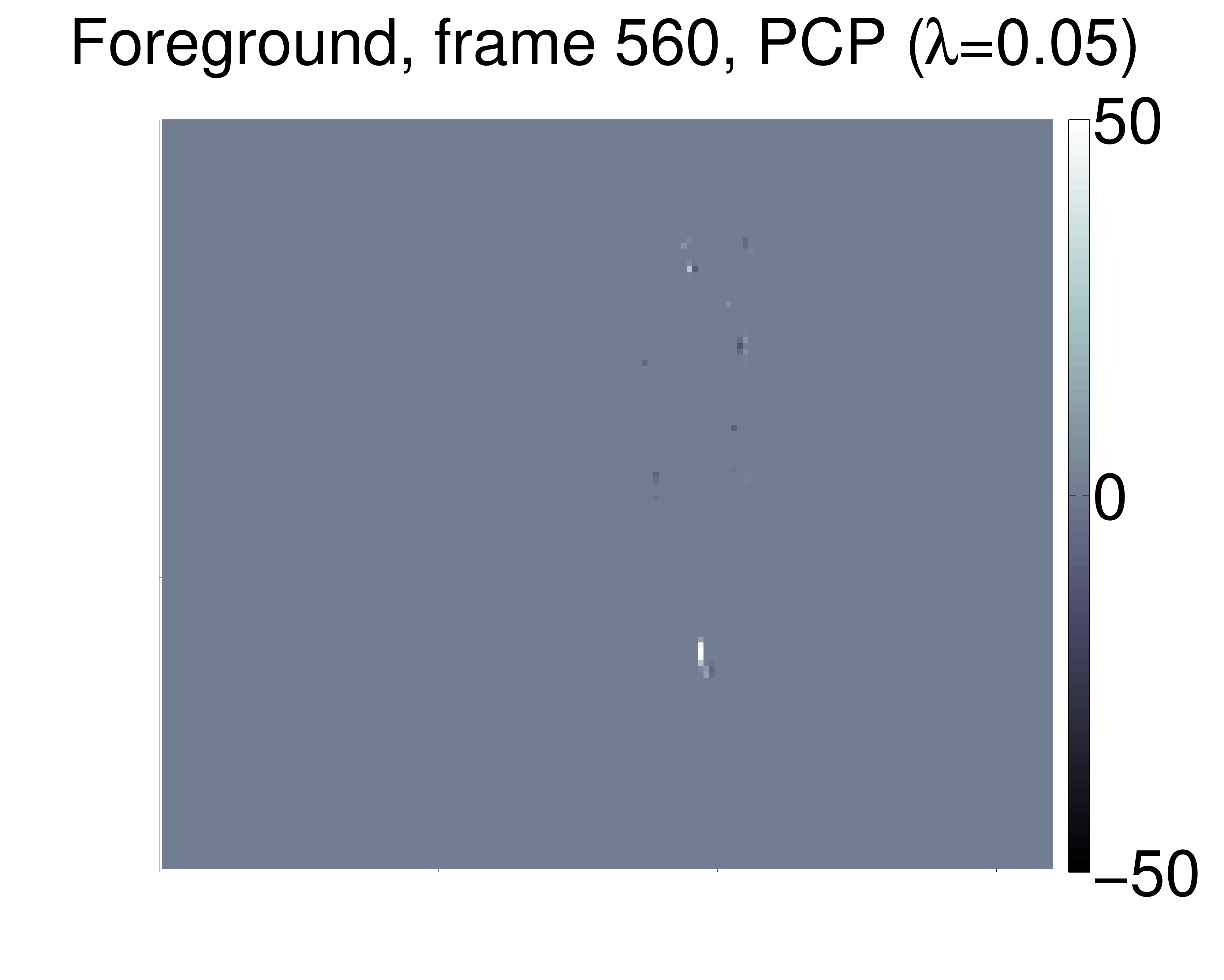} &
\includegraphics[width=.25\columnwidth]{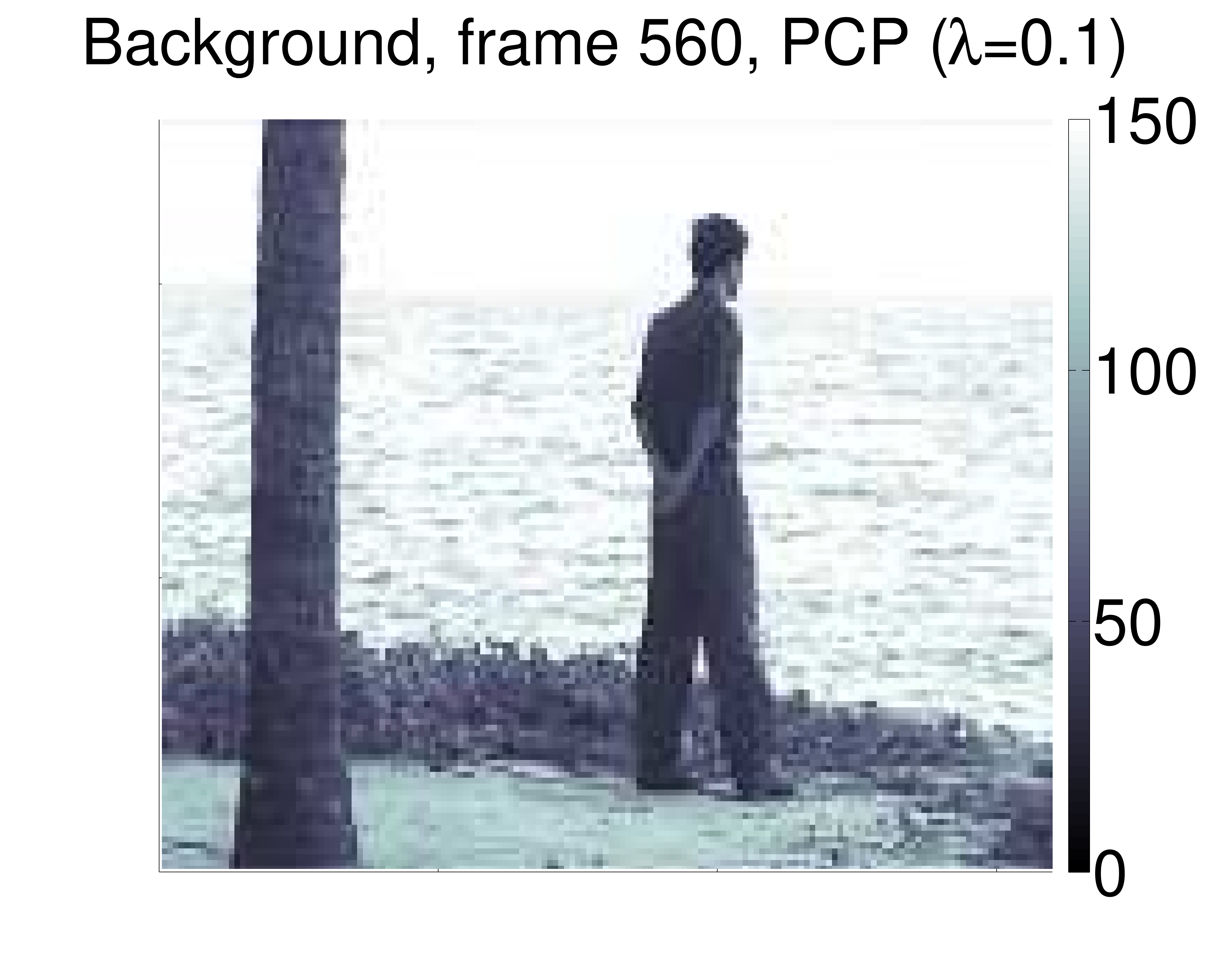} & 
\includegraphics[width=.25\columnwidth]{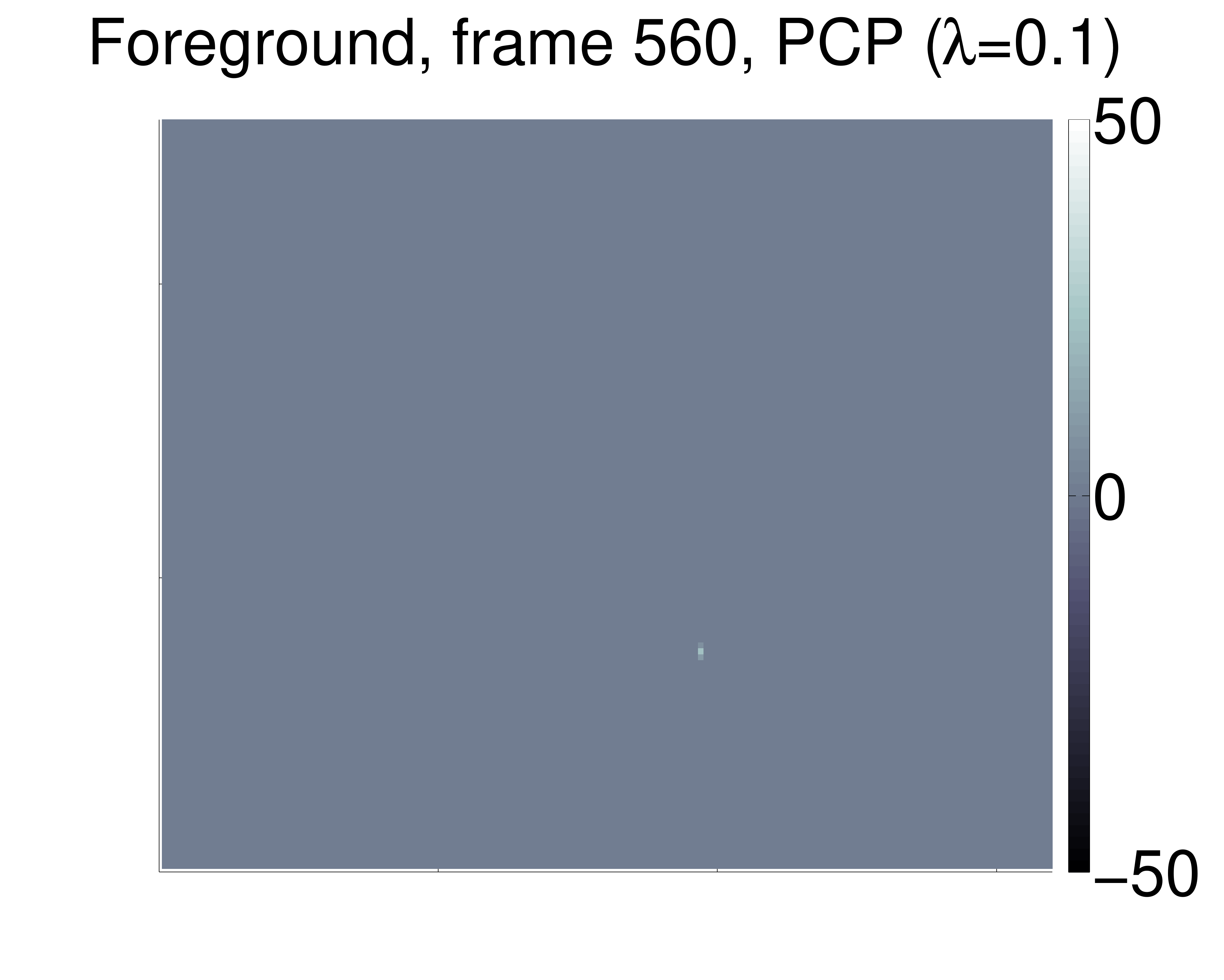}
\end{tabular}
\caption{Background $x_i^b$ and foreground $x_i^f$ recovered, using \eqref{foreground}, of frame 560 of the water surface data set with PCP using different regularization parameters. See similar results for other methods in previous Fig. \ref{fig:wsbf}
}
\label{fig:wsbfpcp}
\end{figure}
\clearpage

%%%%%%%%%%%%%%%%%%%%%%%%%%%%%%%%%%%%%%%%%%%%%%%%%%%%
\begin{figure}
\begin{tabular}{cccc}
\includegraphics[width=.25\columnwidth]{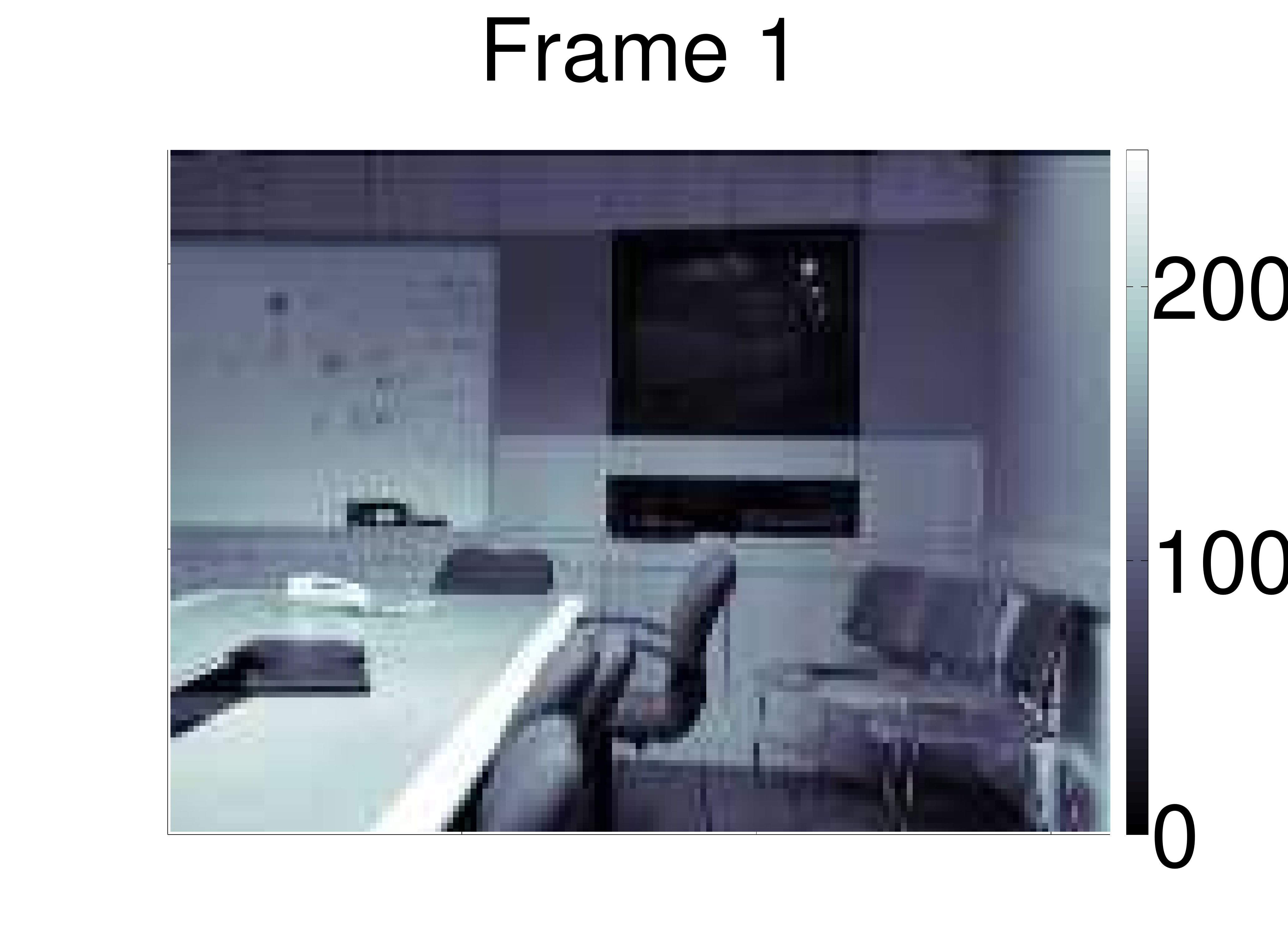} & 
\includegraphics[width=.25\columnwidth]{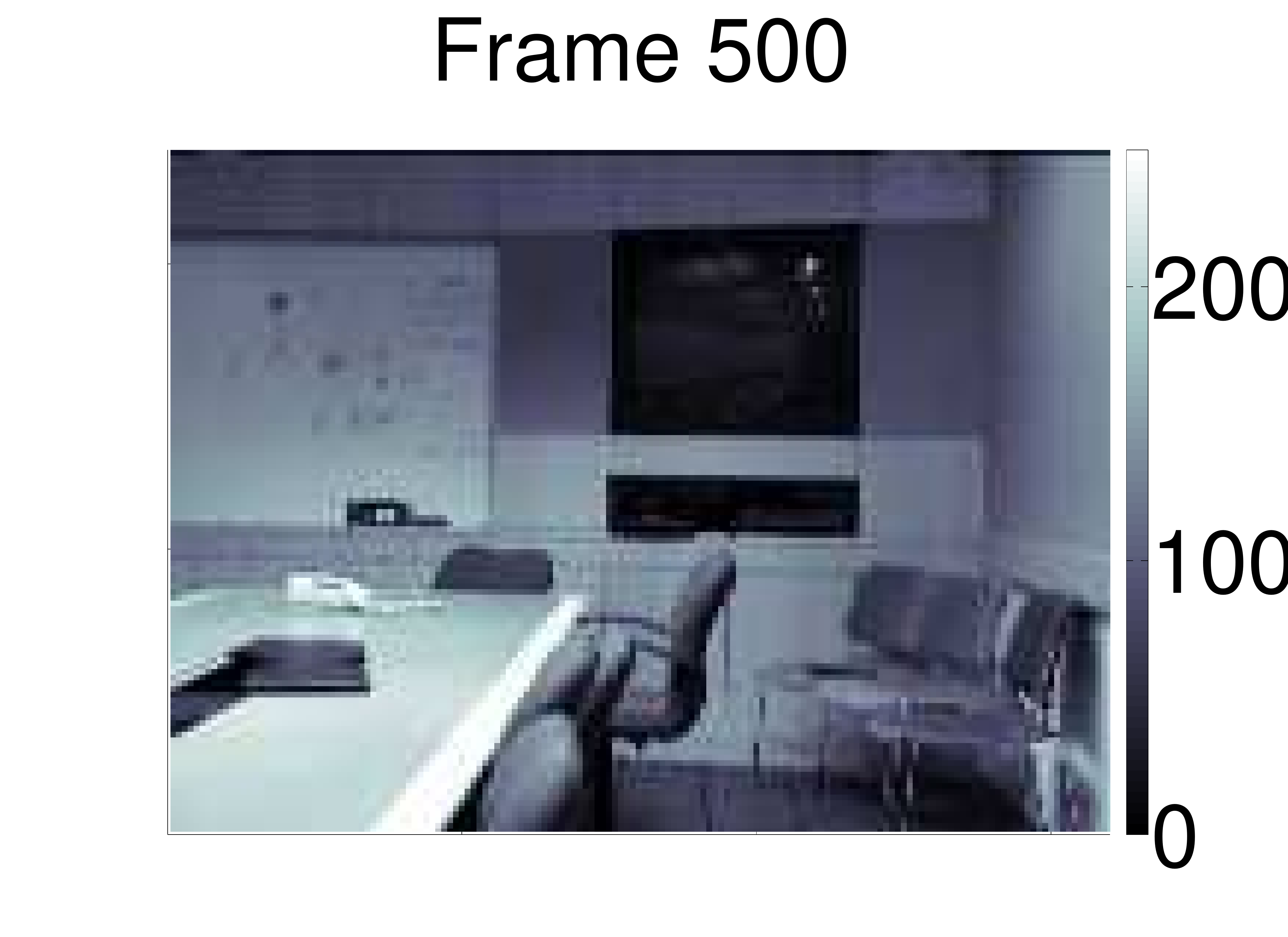} & 
\includegraphics[width=.25\columnwidth]{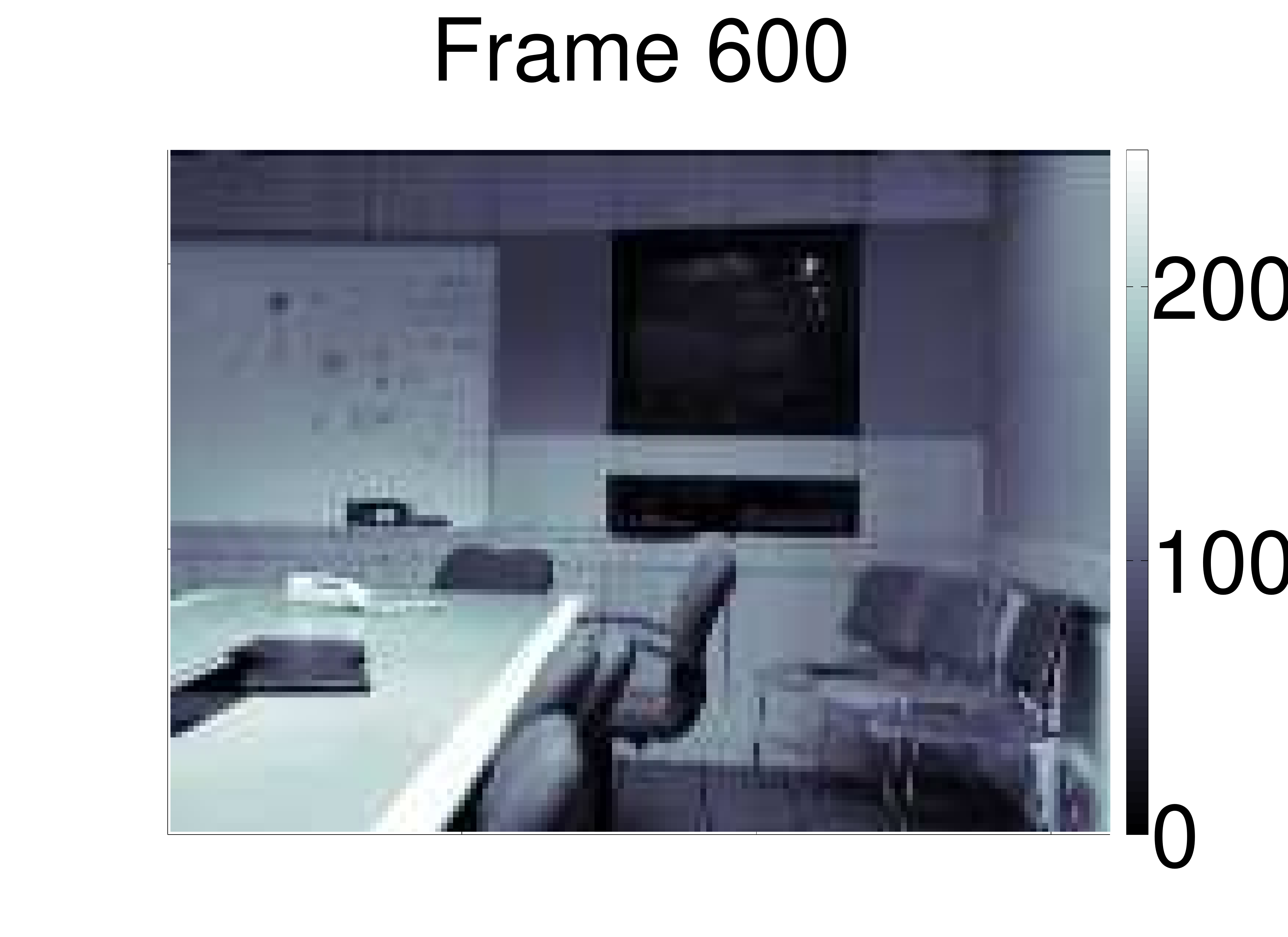} & 
\includegraphics[width=.25\columnwidth]{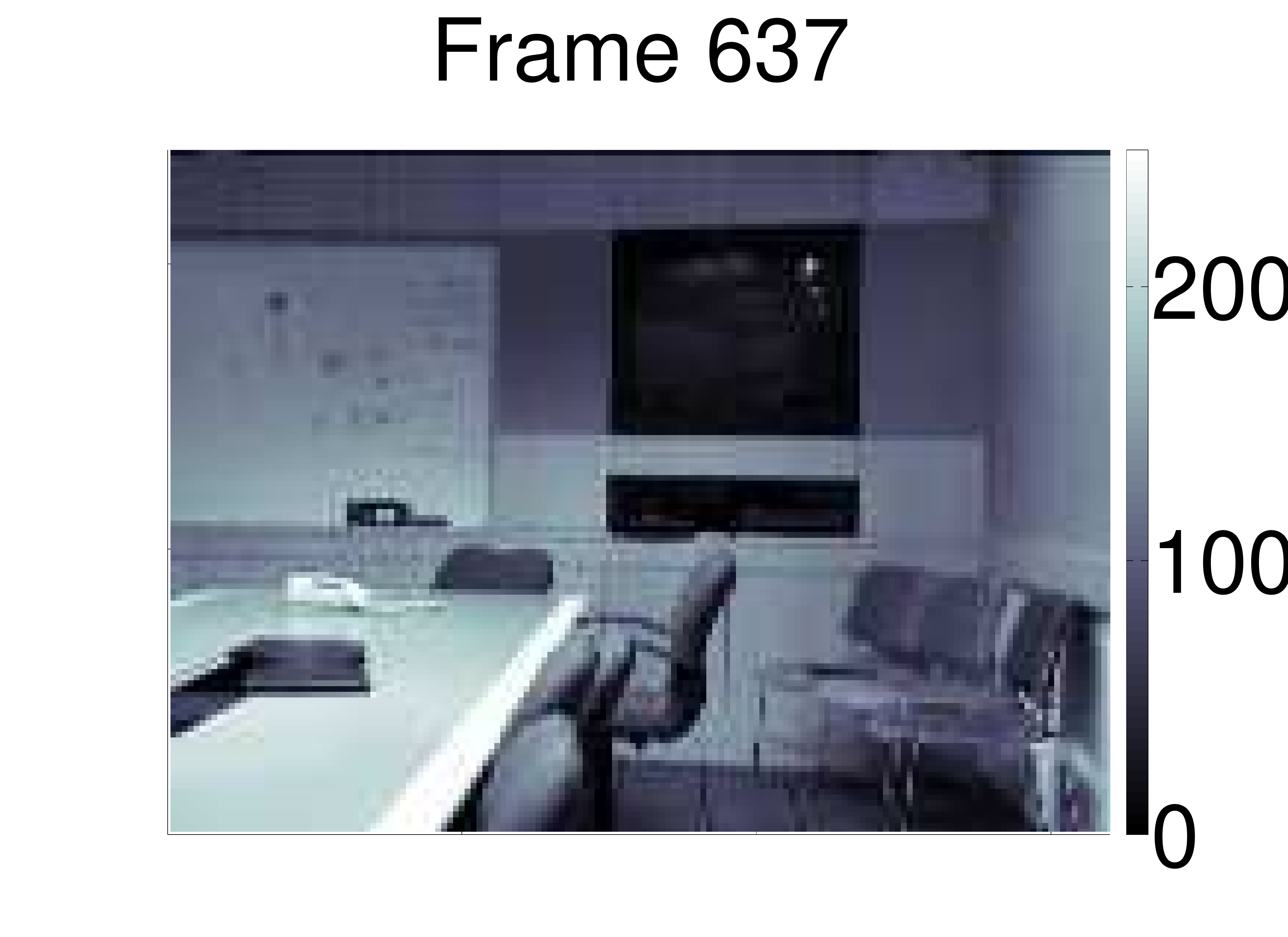} \\
\includegraphics[width=.25\columnwidth]{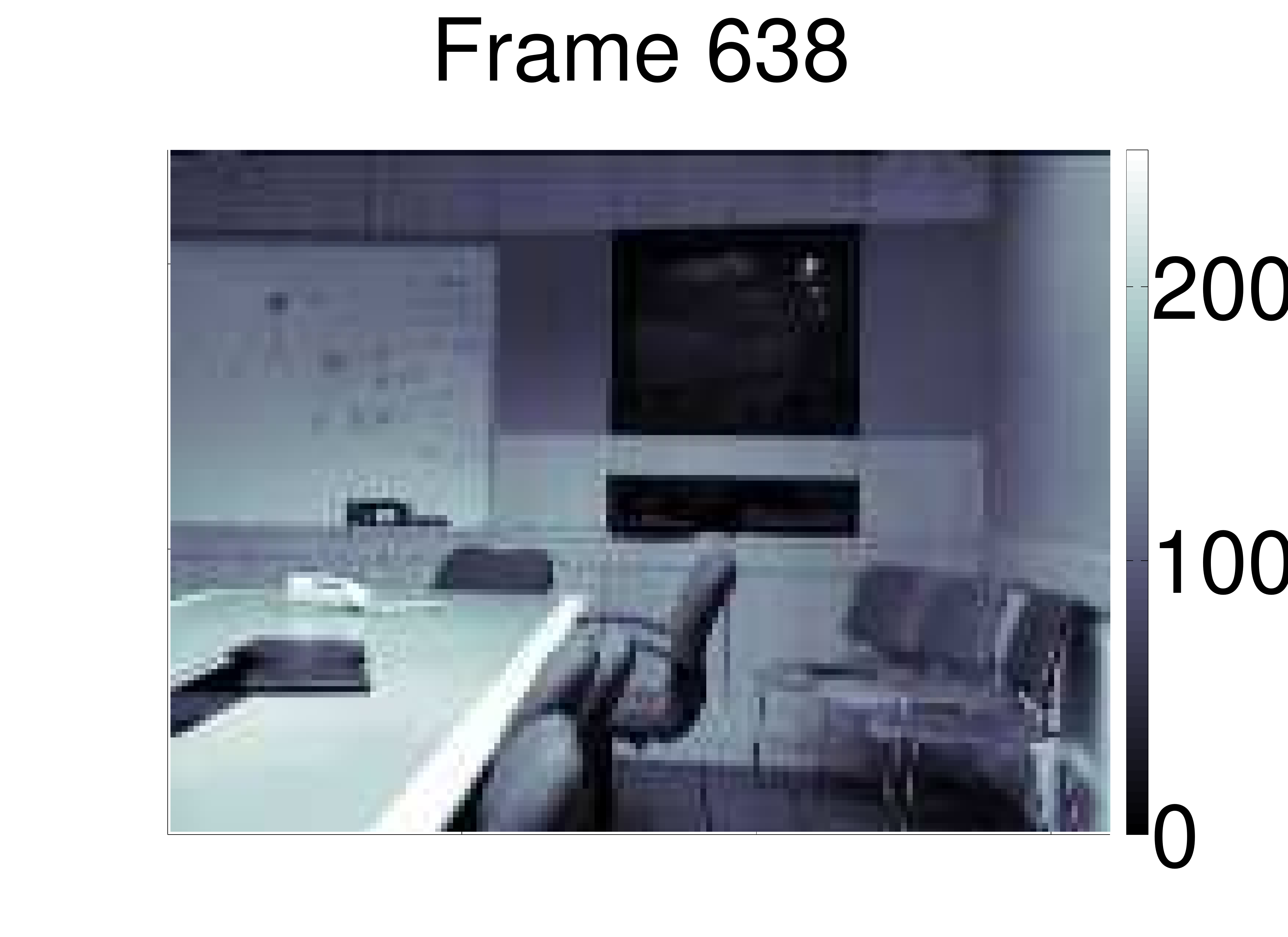} & 
\includegraphics[width=.25\columnwidth]{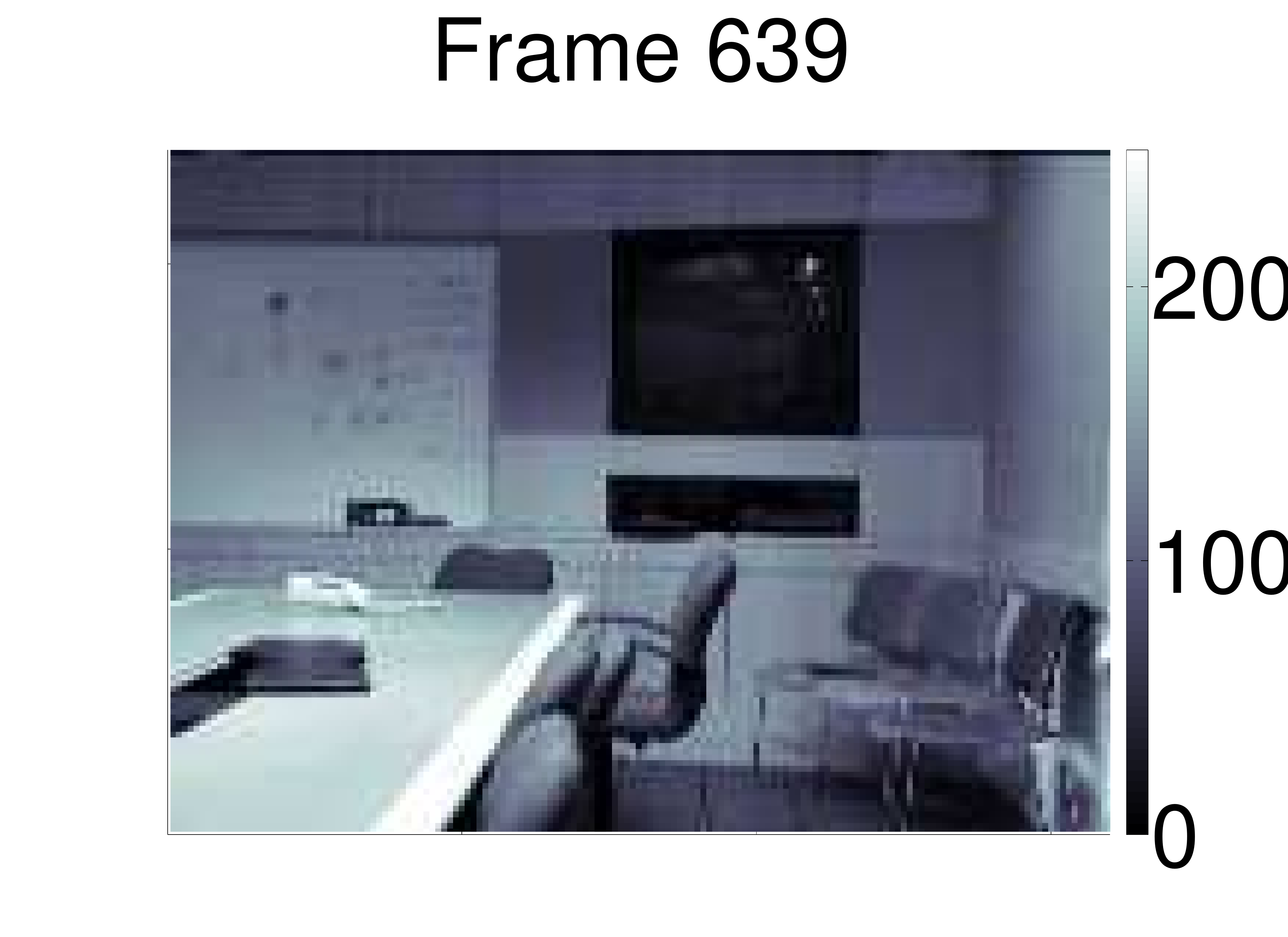} & 
\includegraphics[width=.25\columnwidth]{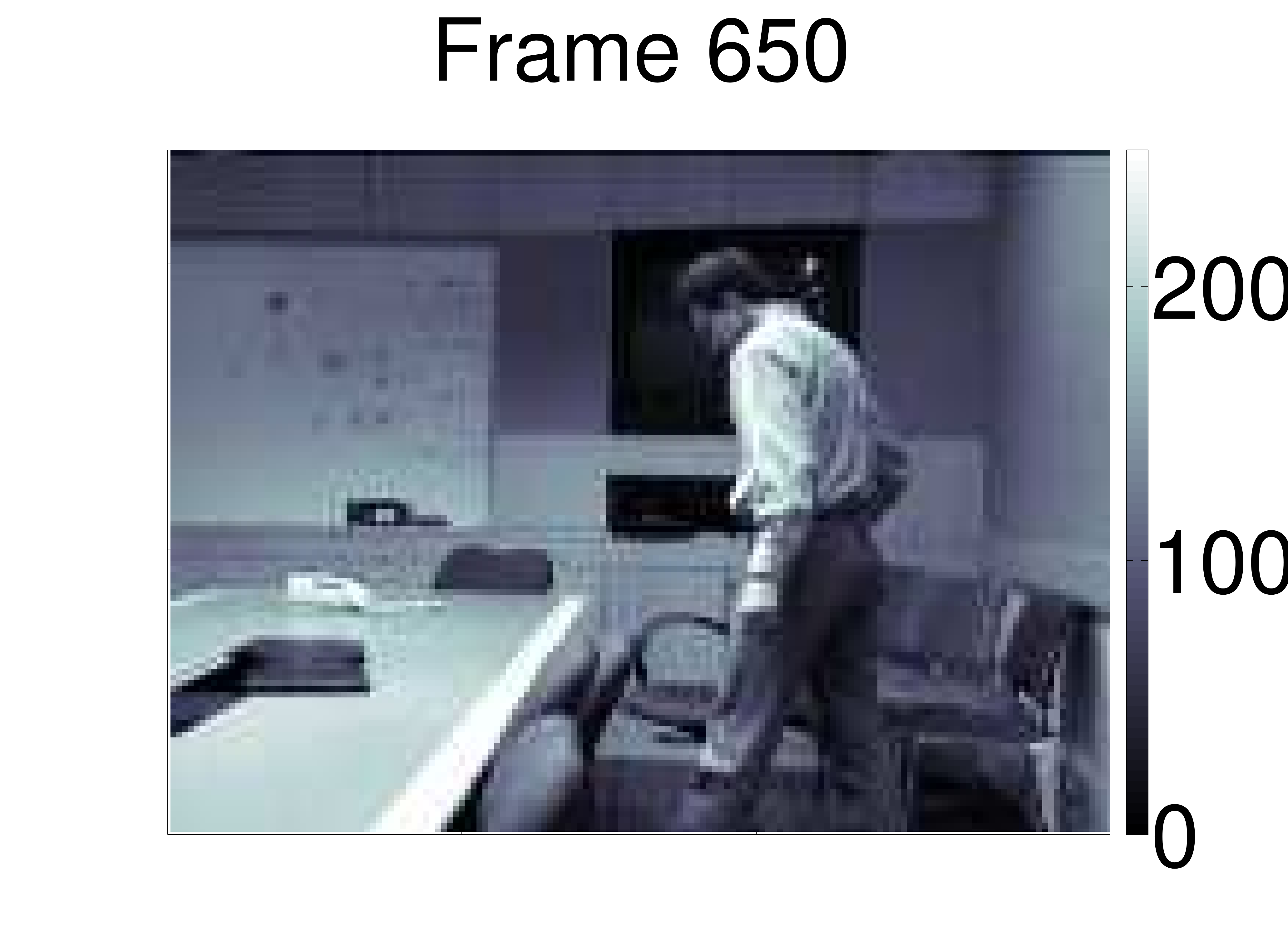} & 
\includegraphics[width=.25\columnwidth]{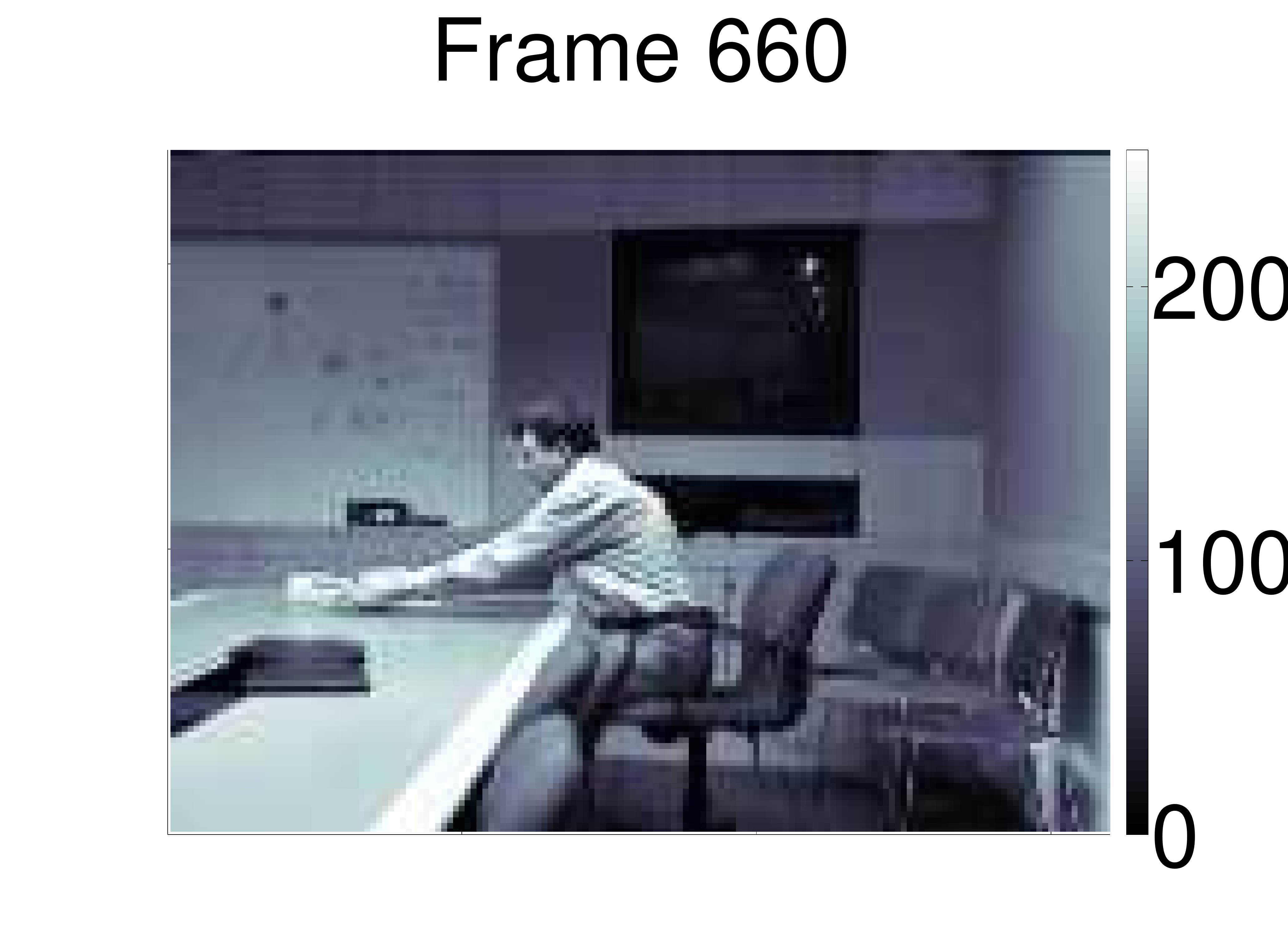} \\
\includegraphics[width=.25\columnwidth]{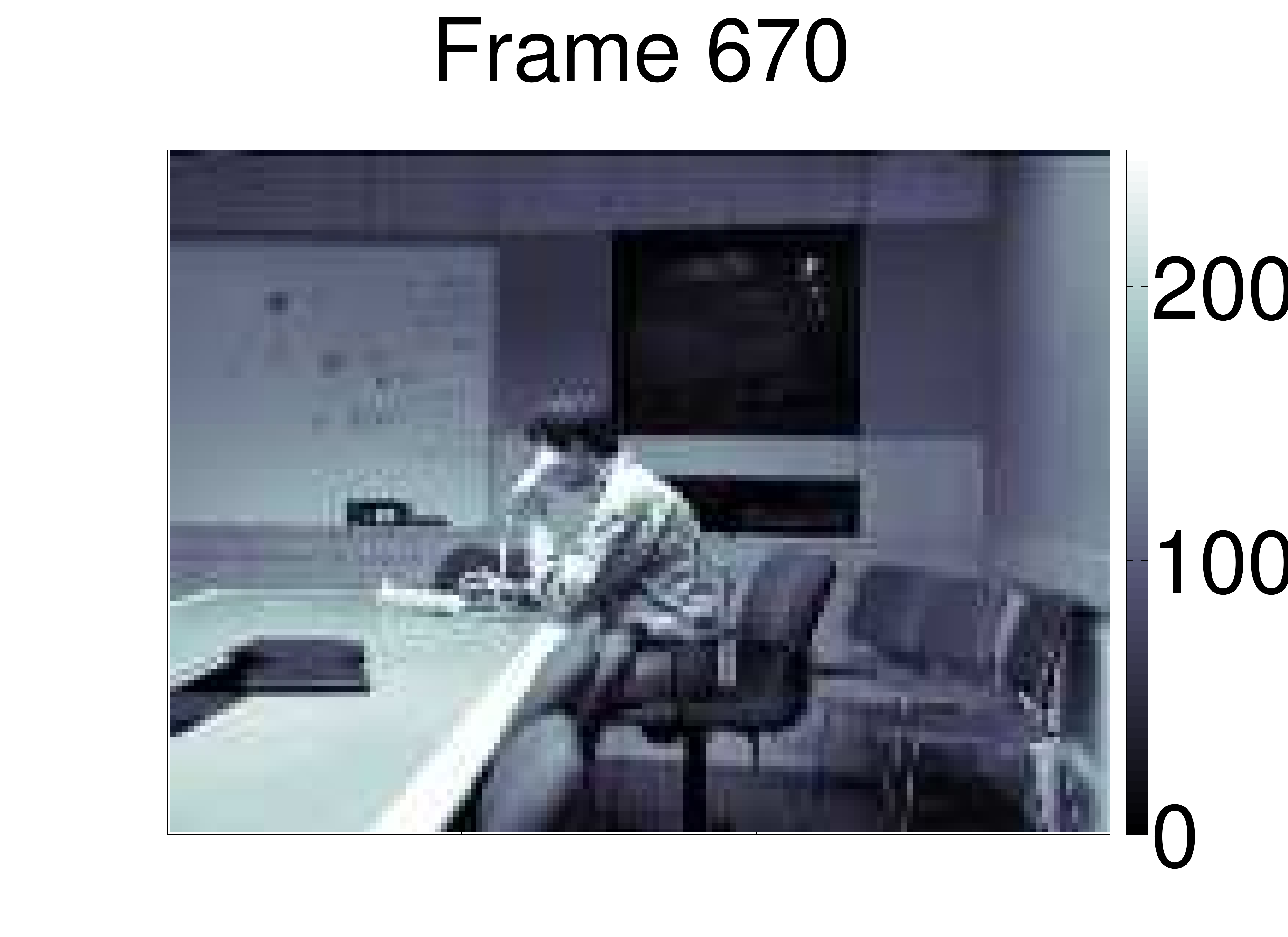} & 
\includegraphics[width=.25\columnwidth]{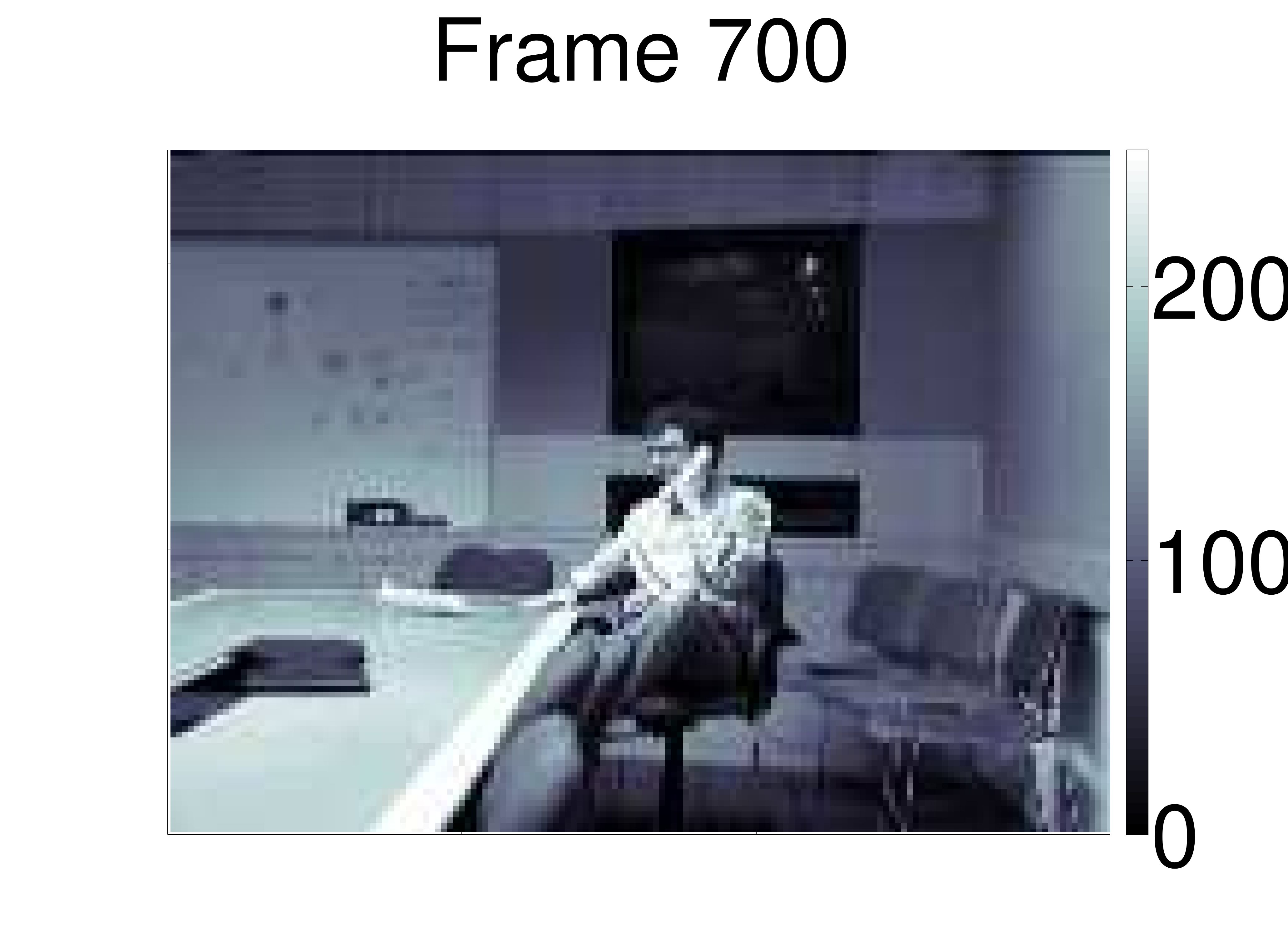} & 
\includegraphics[width=.25\columnwidth]{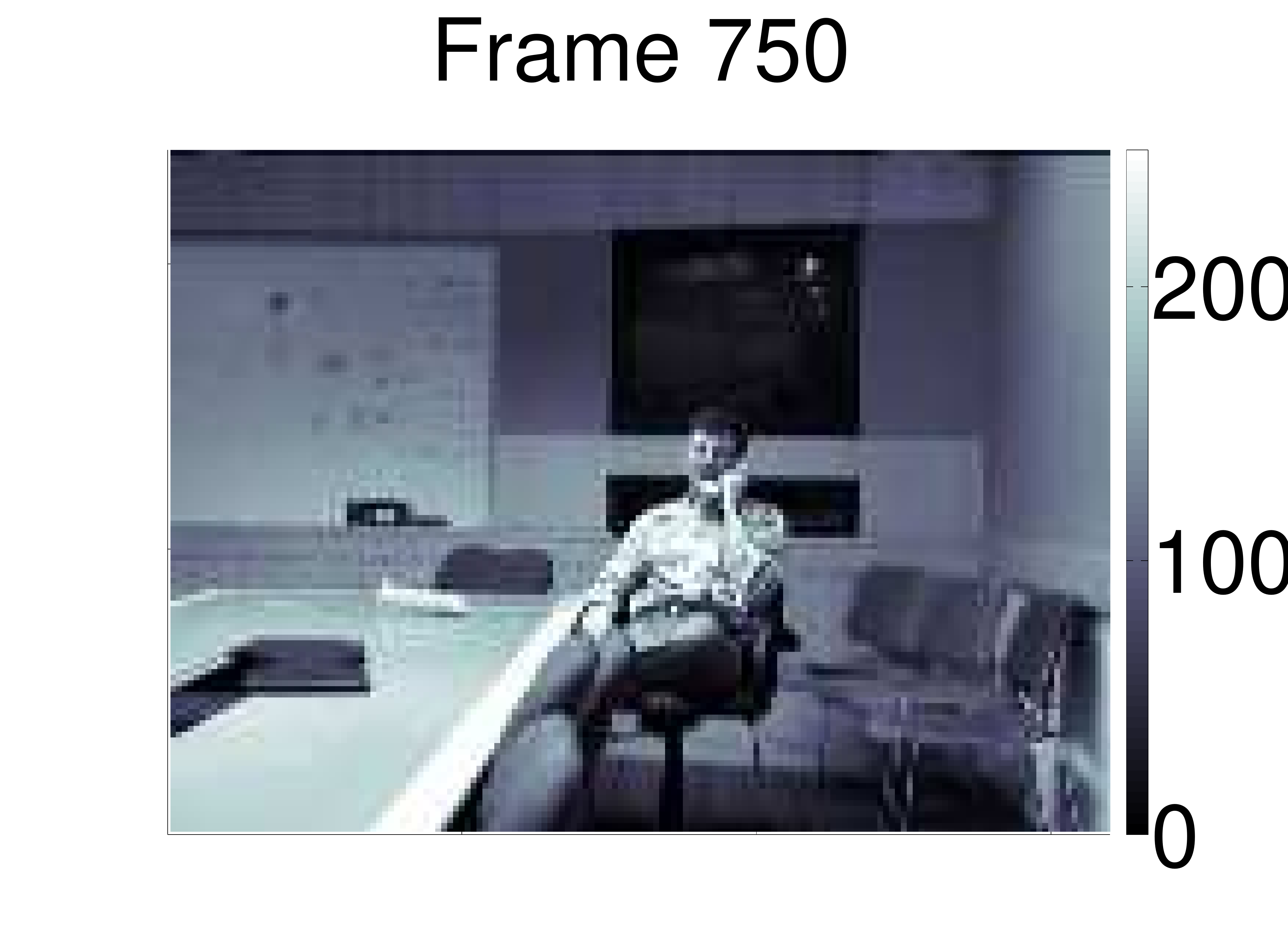} & 
\includegraphics[width=.25\columnwidth]{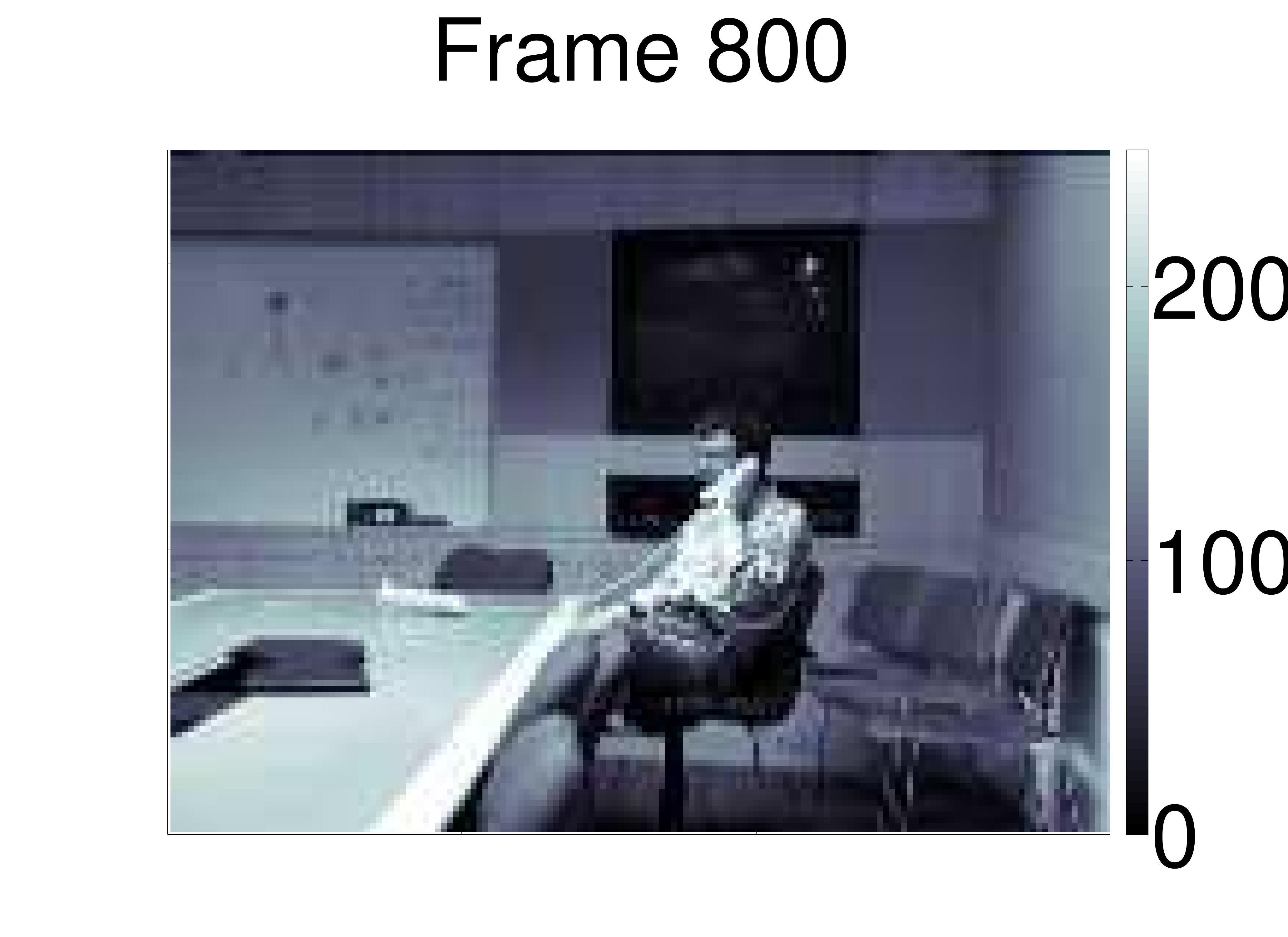} \\
\includegraphics[width=.25\columnwidth]{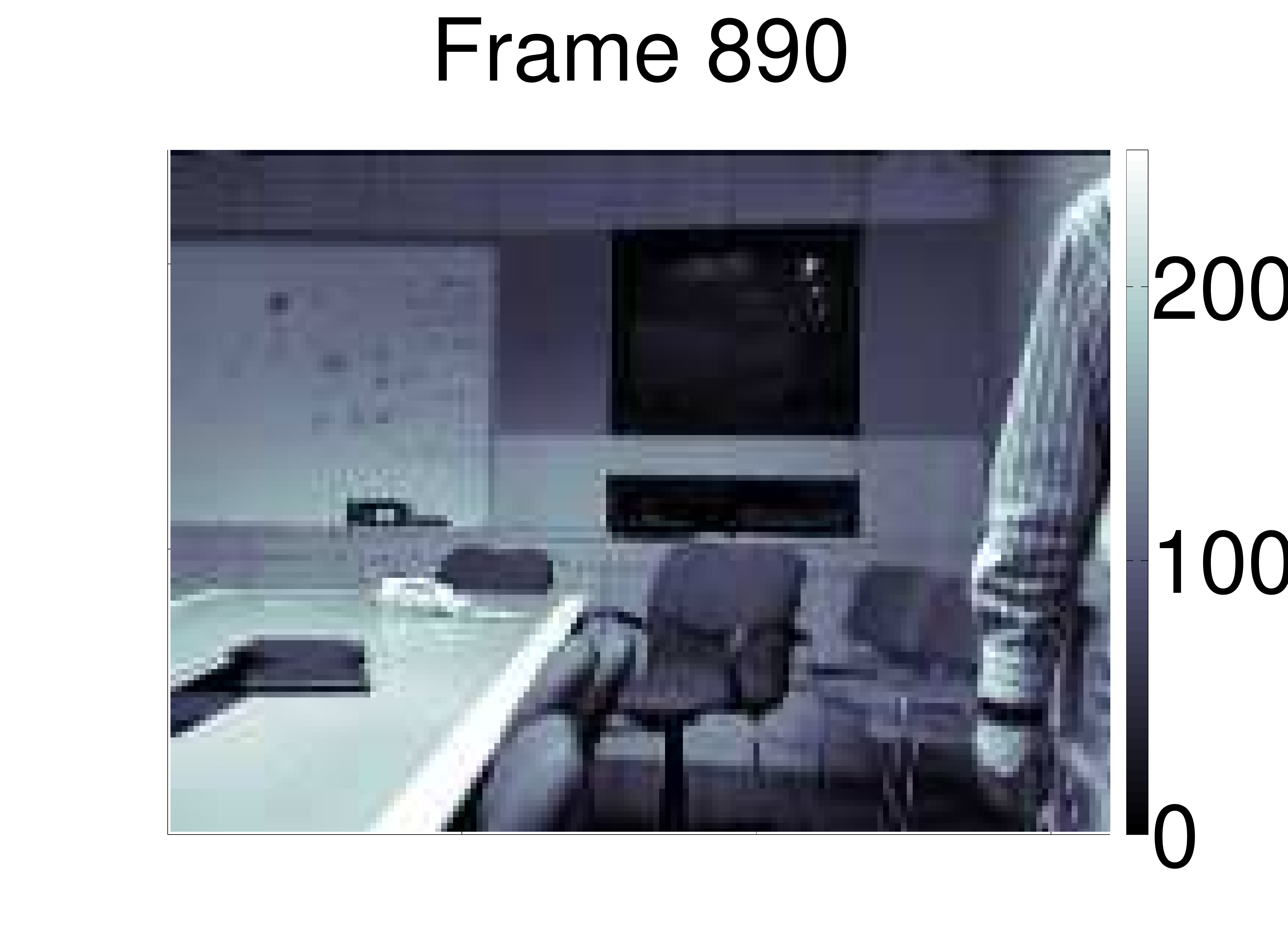} & 
\includegraphics[width=.25\columnwidth]{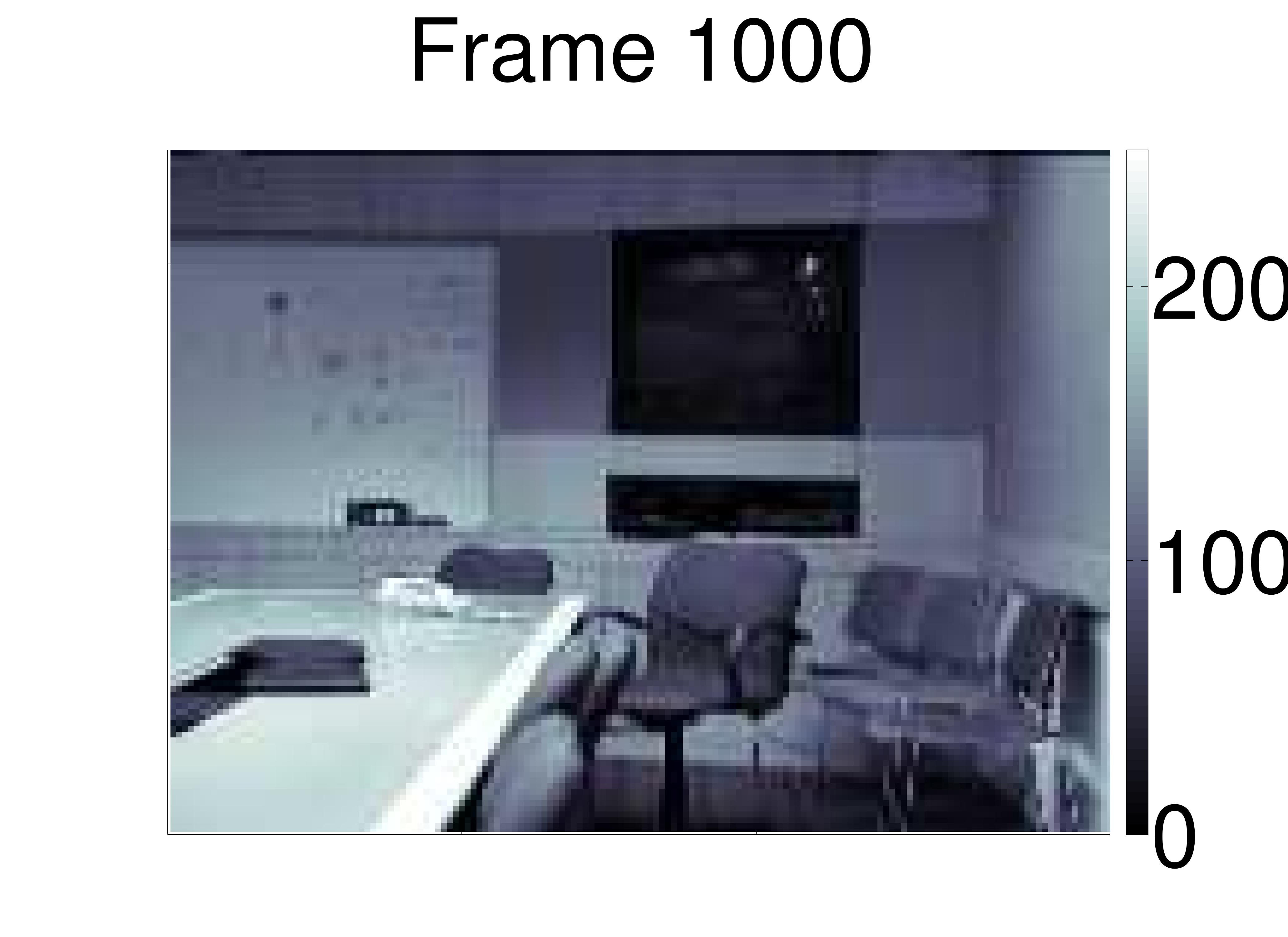} & 
\includegraphics[width=.25\columnwidth]{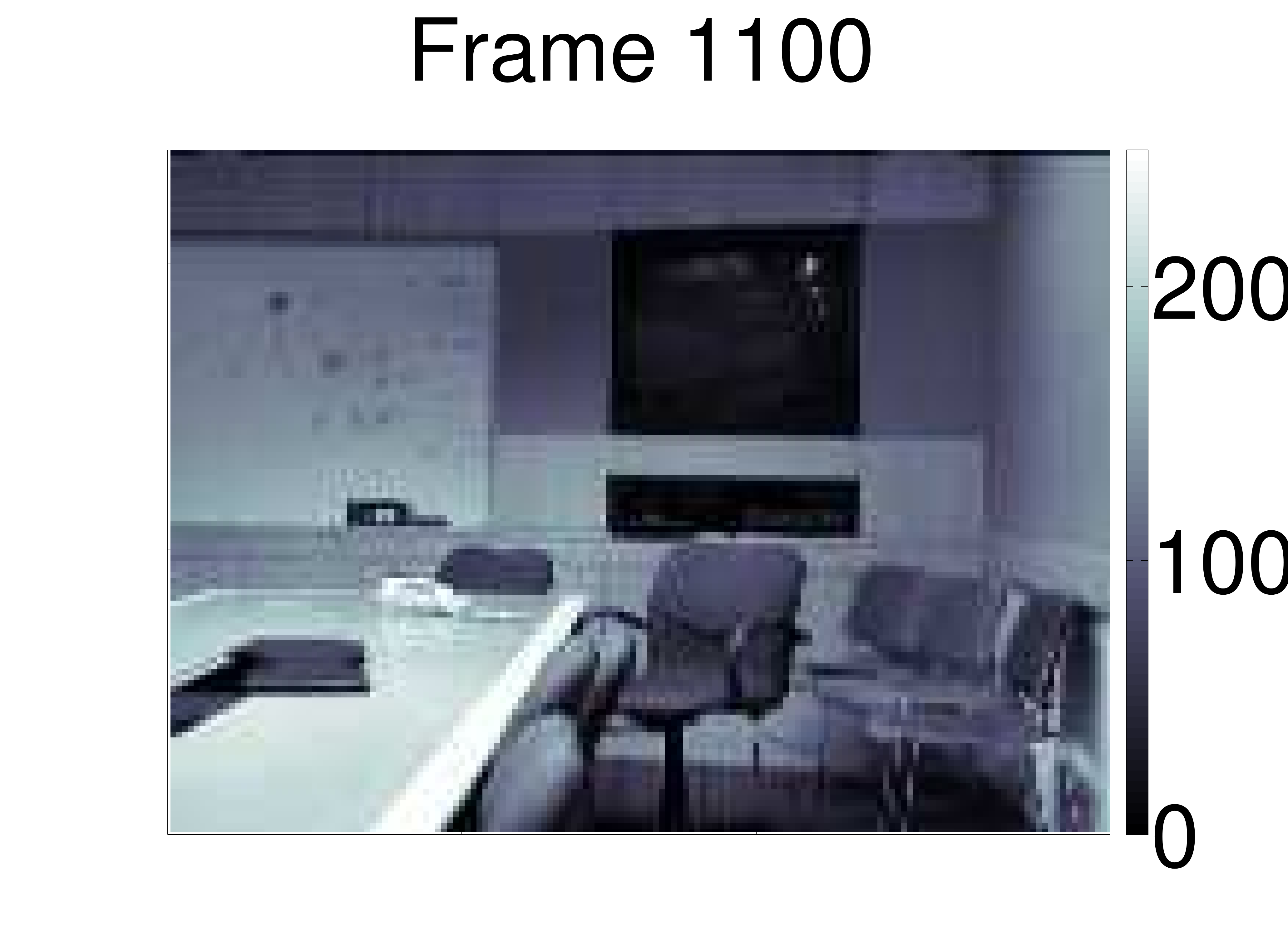} & 
\includegraphics[width=.25\columnwidth]{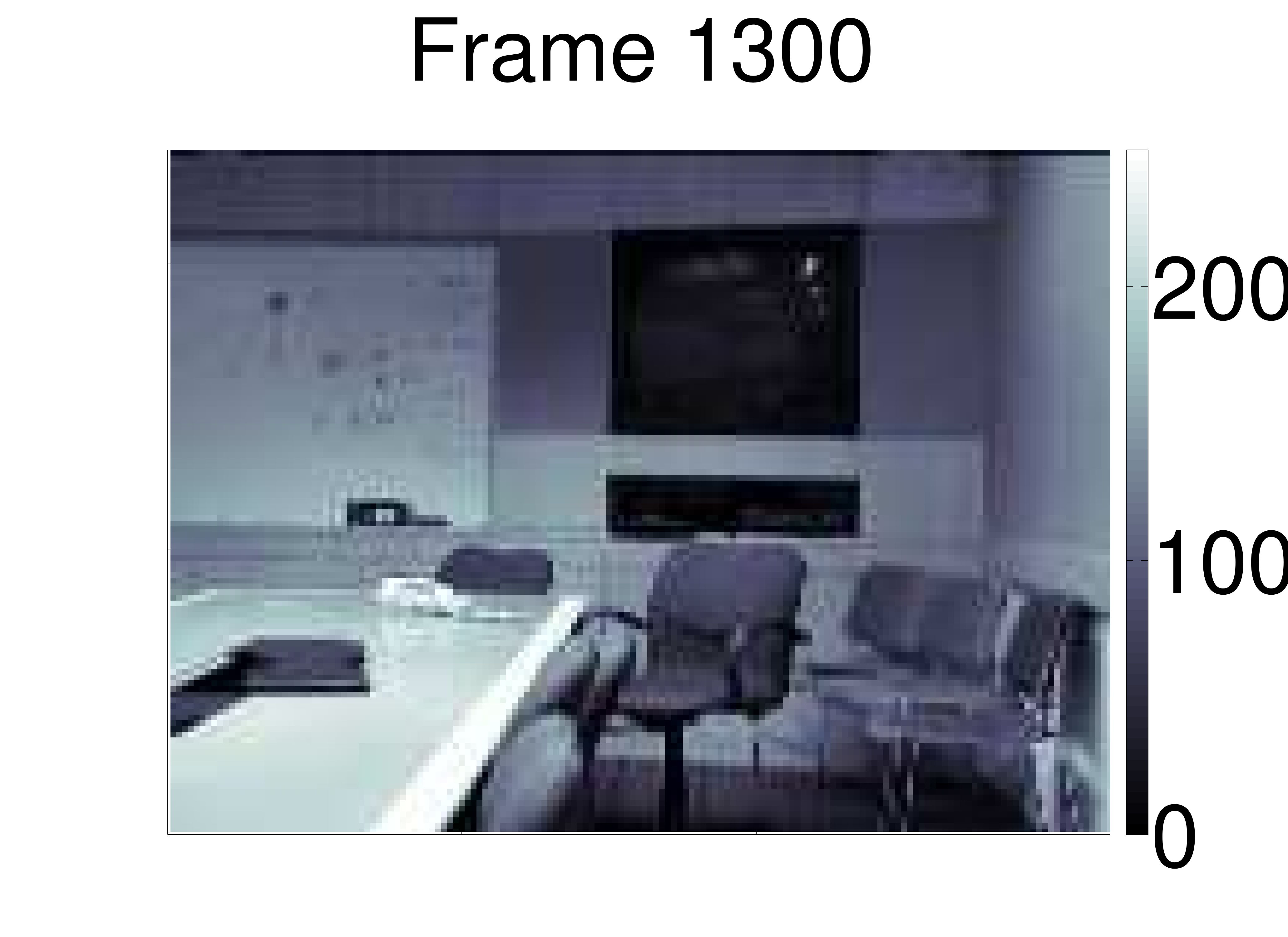} \\
\includegraphics[width=.25\columnwidth]{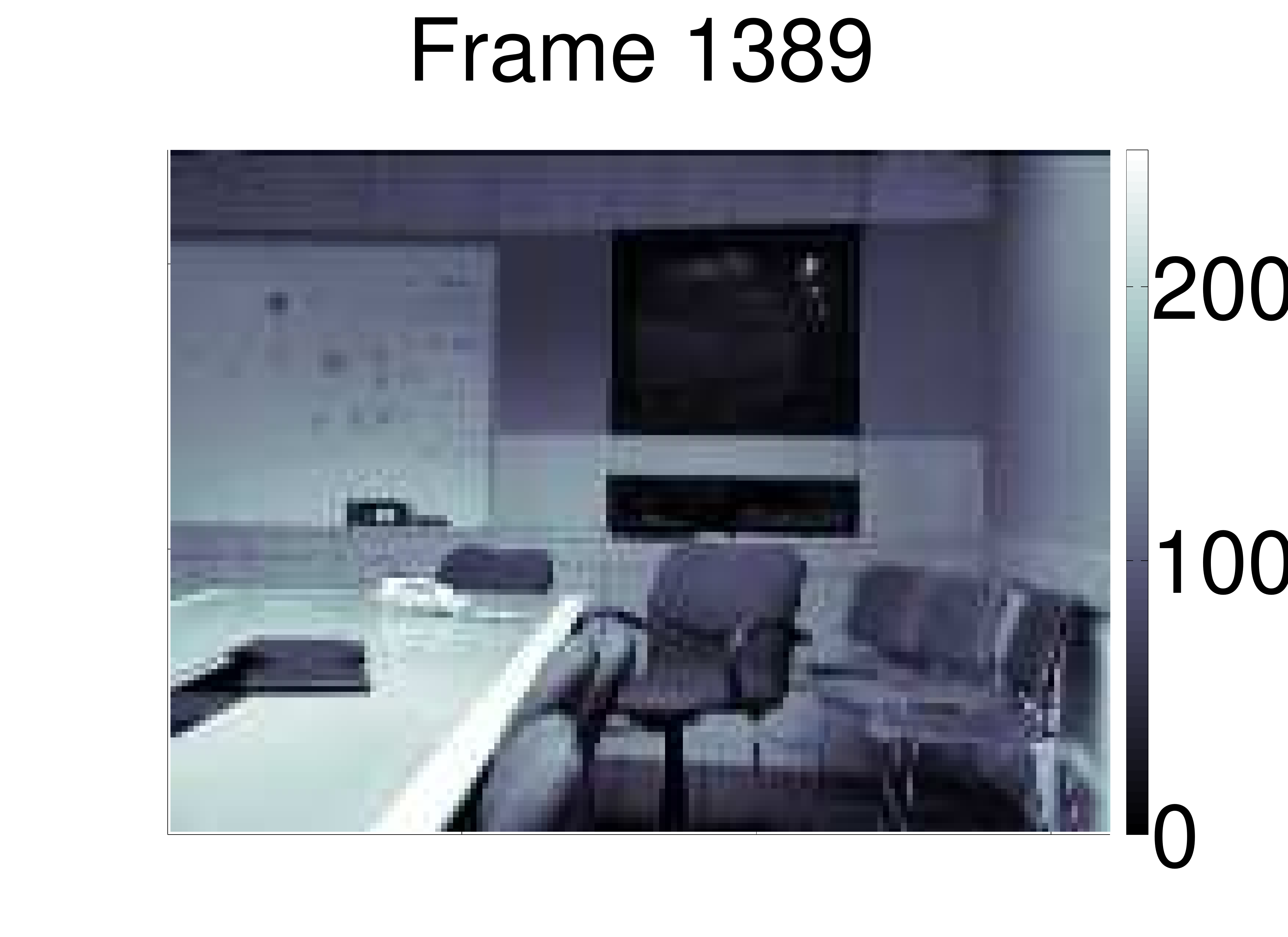} & 
\includegraphics[width=.25\columnwidth]{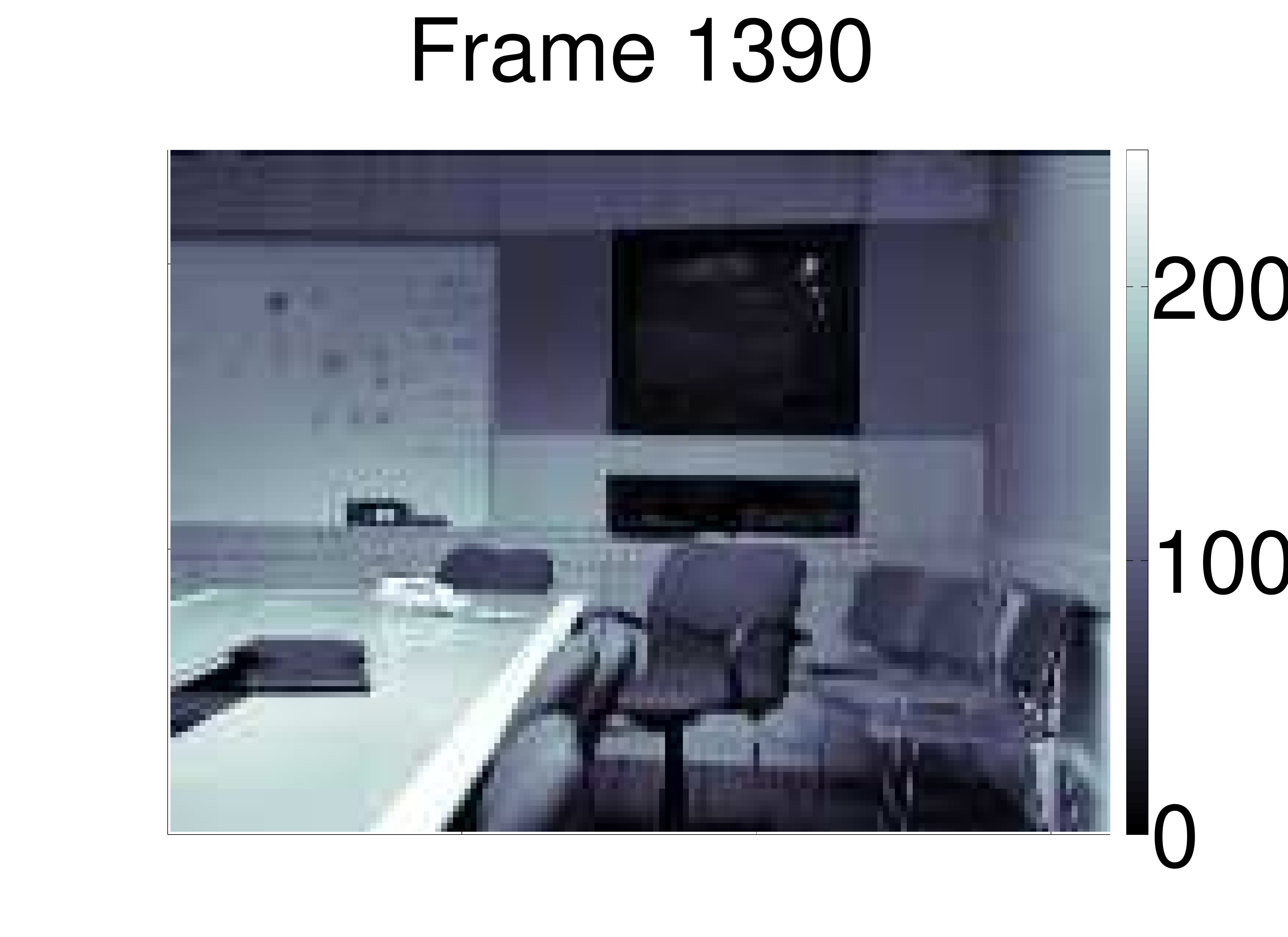} & 
\includegraphics[width=.25\columnwidth]{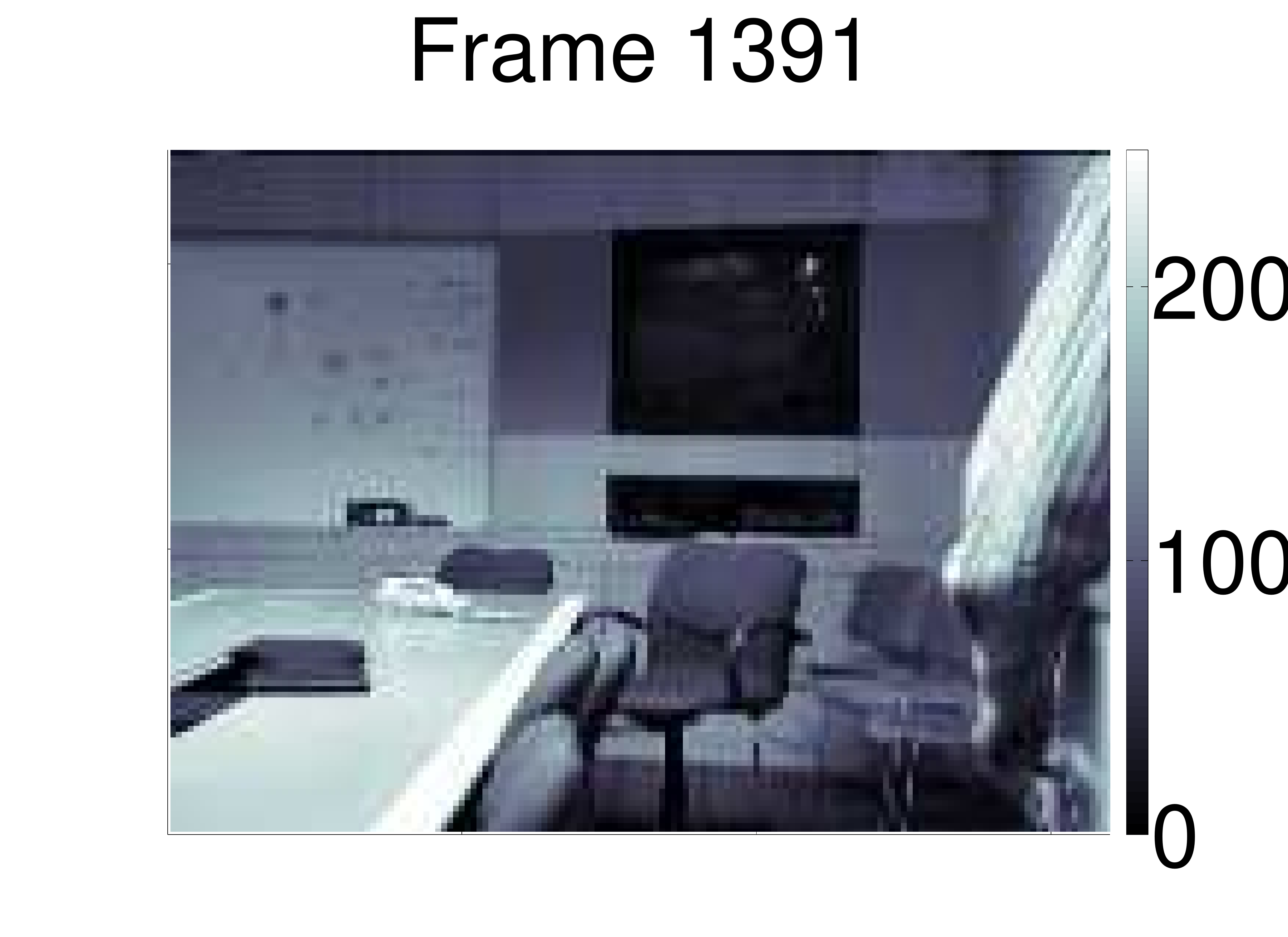} & 
\includegraphics[width=.25\columnwidth]{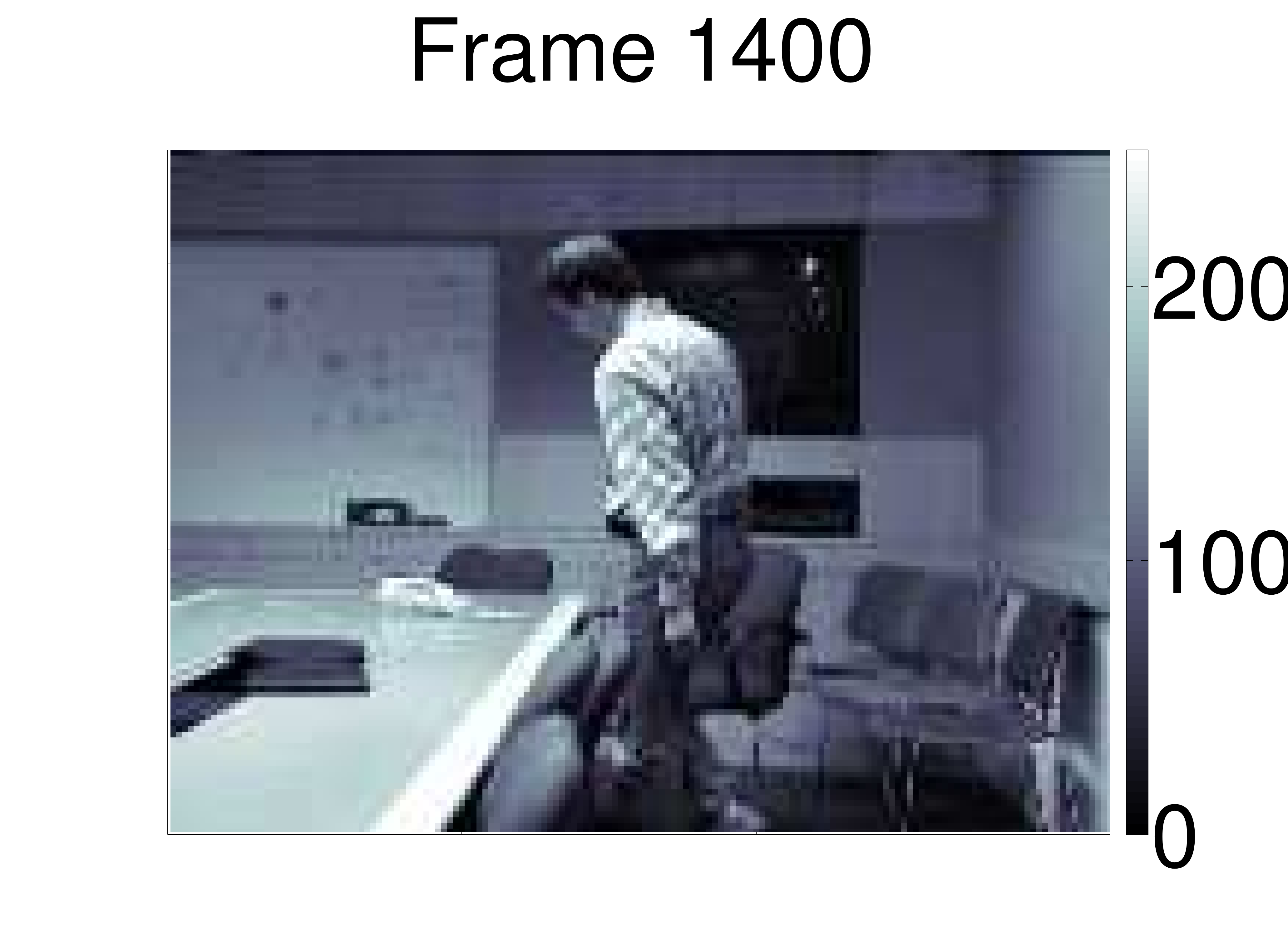} \\
\includegraphics[width=.25\columnwidth]{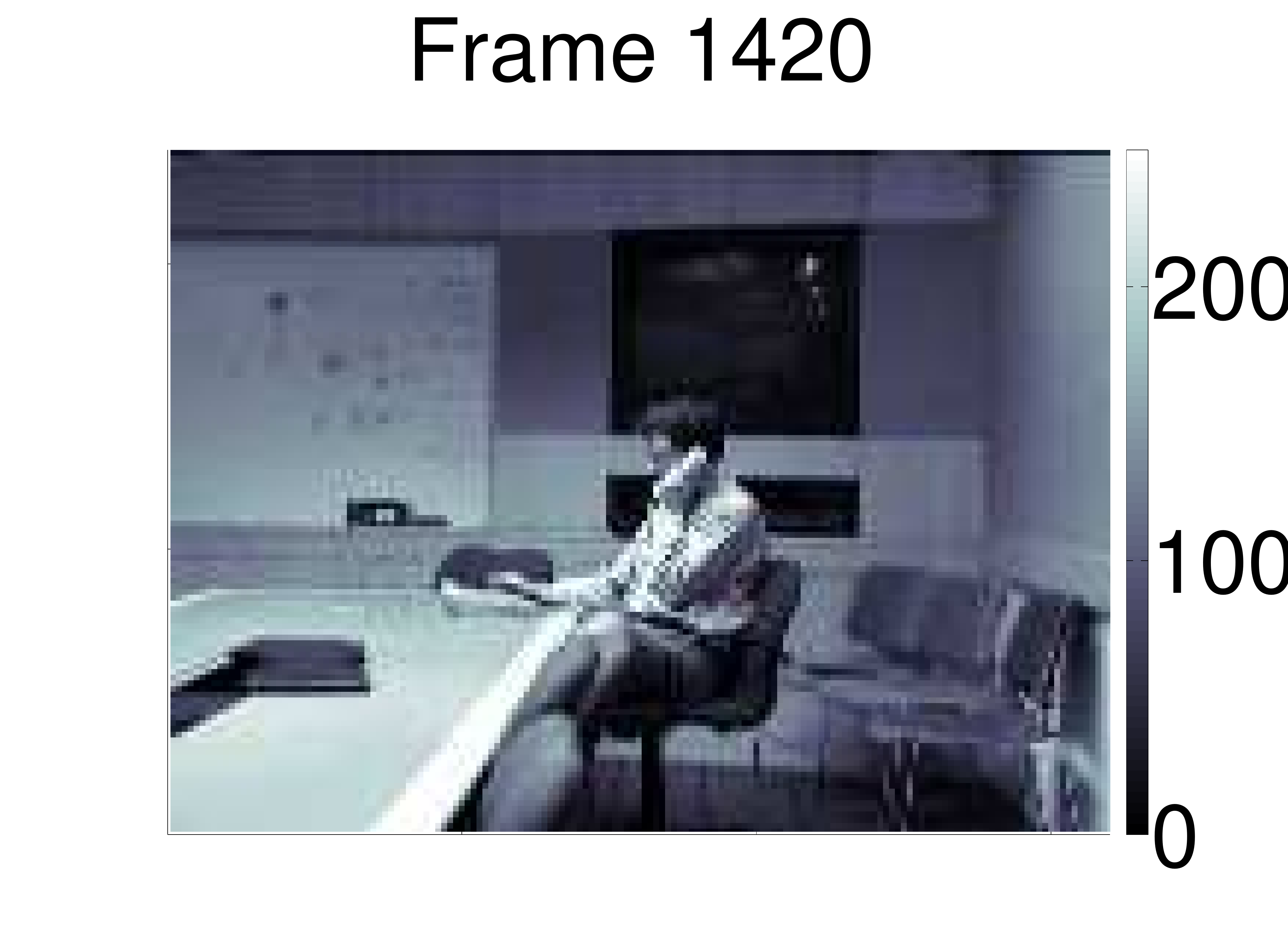} & 
\includegraphics[width=.25\columnwidth]{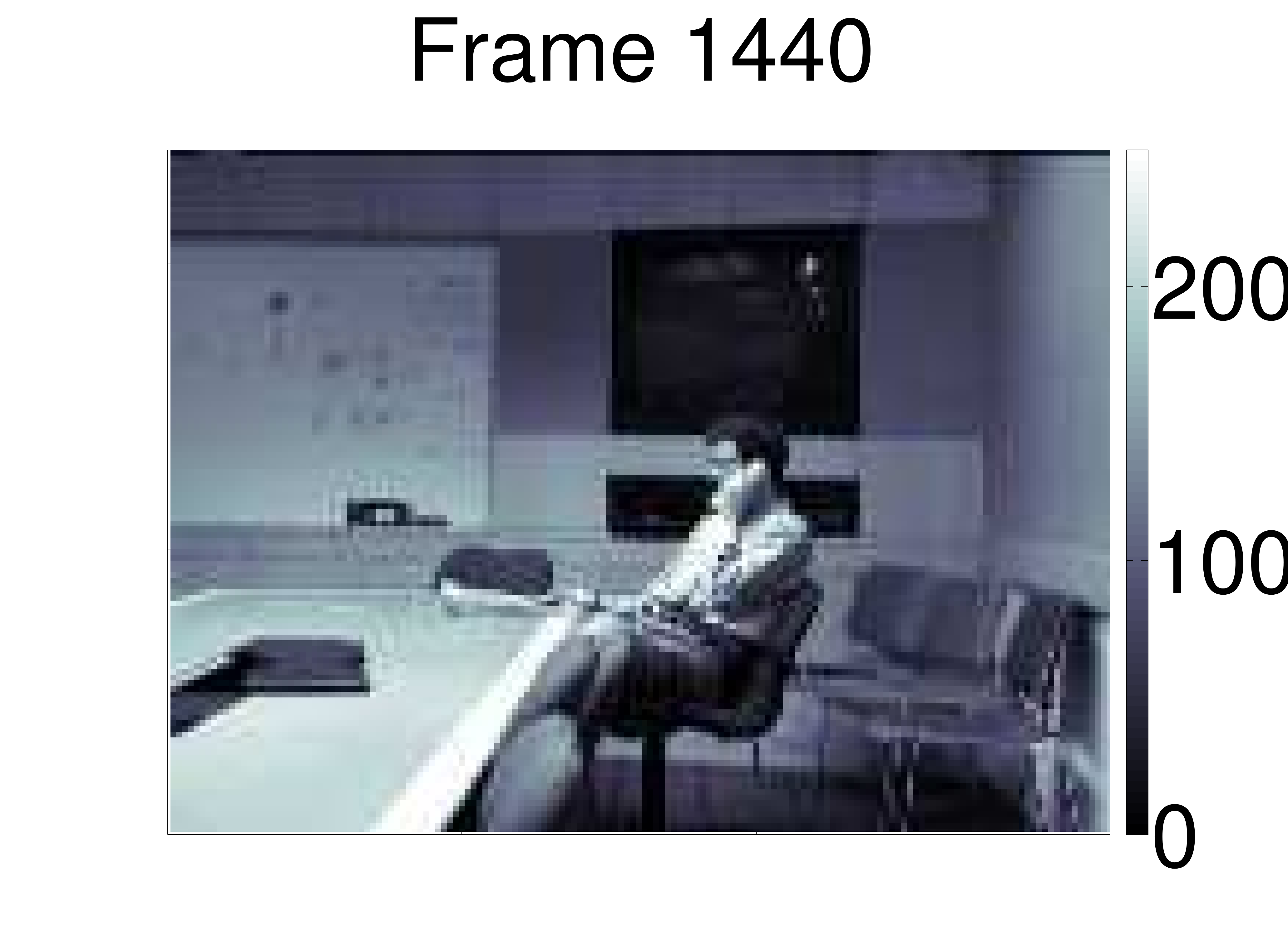} & 
\includegraphics[width=.25\columnwidth]{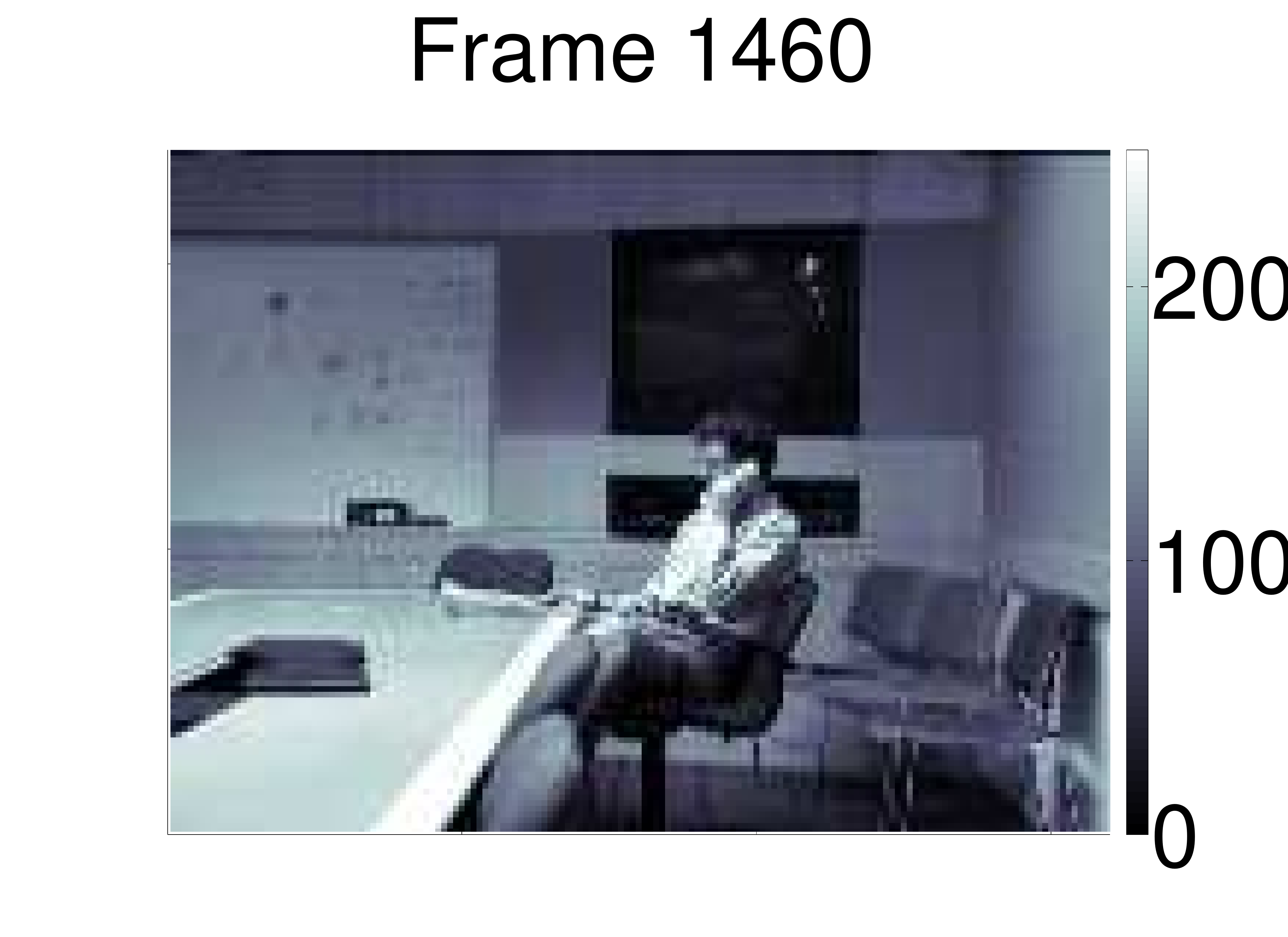} & 
\includegraphics[width=.25\columnwidth]{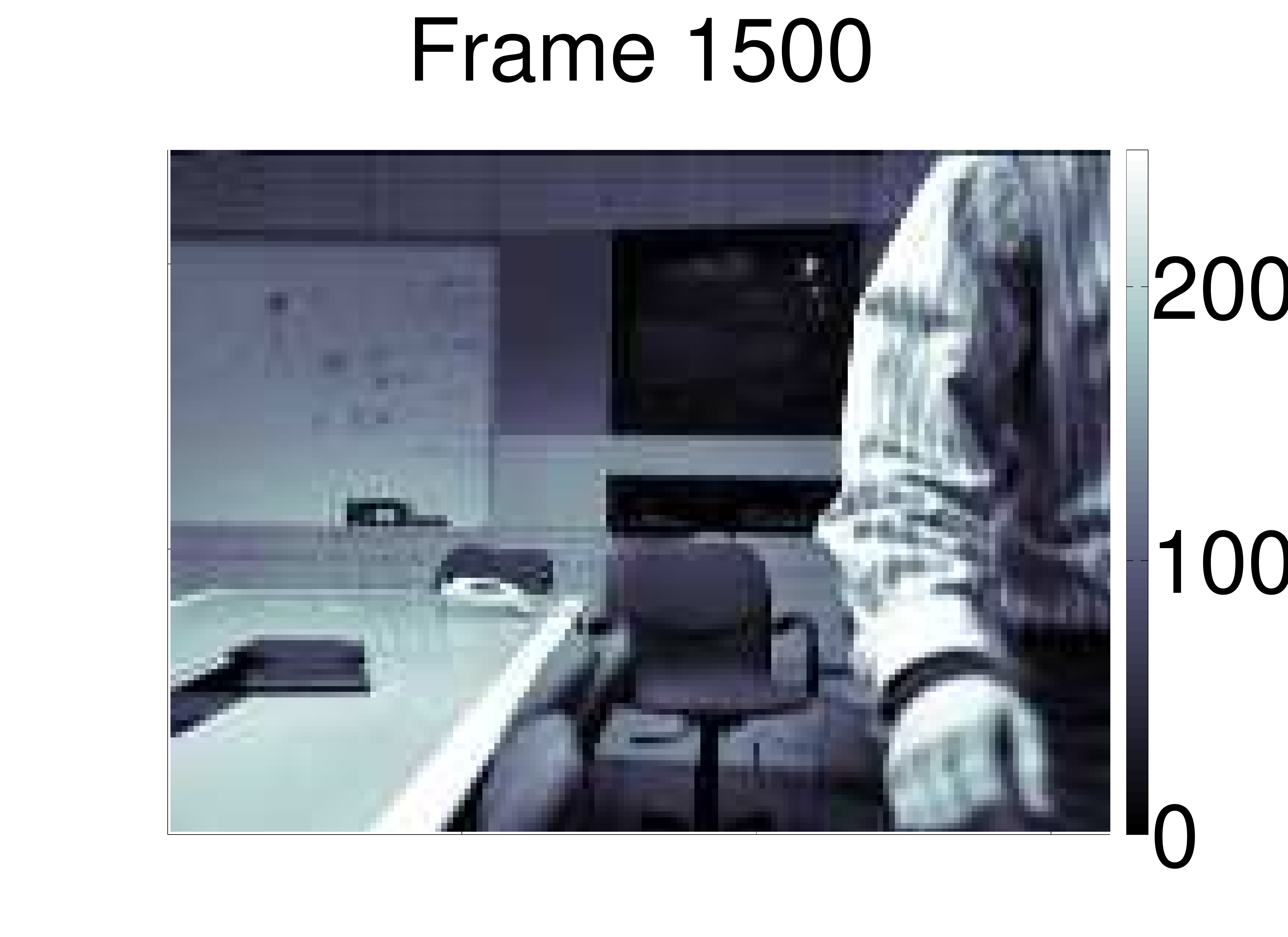} \\
\includegraphics[width=.25\columnwidth]{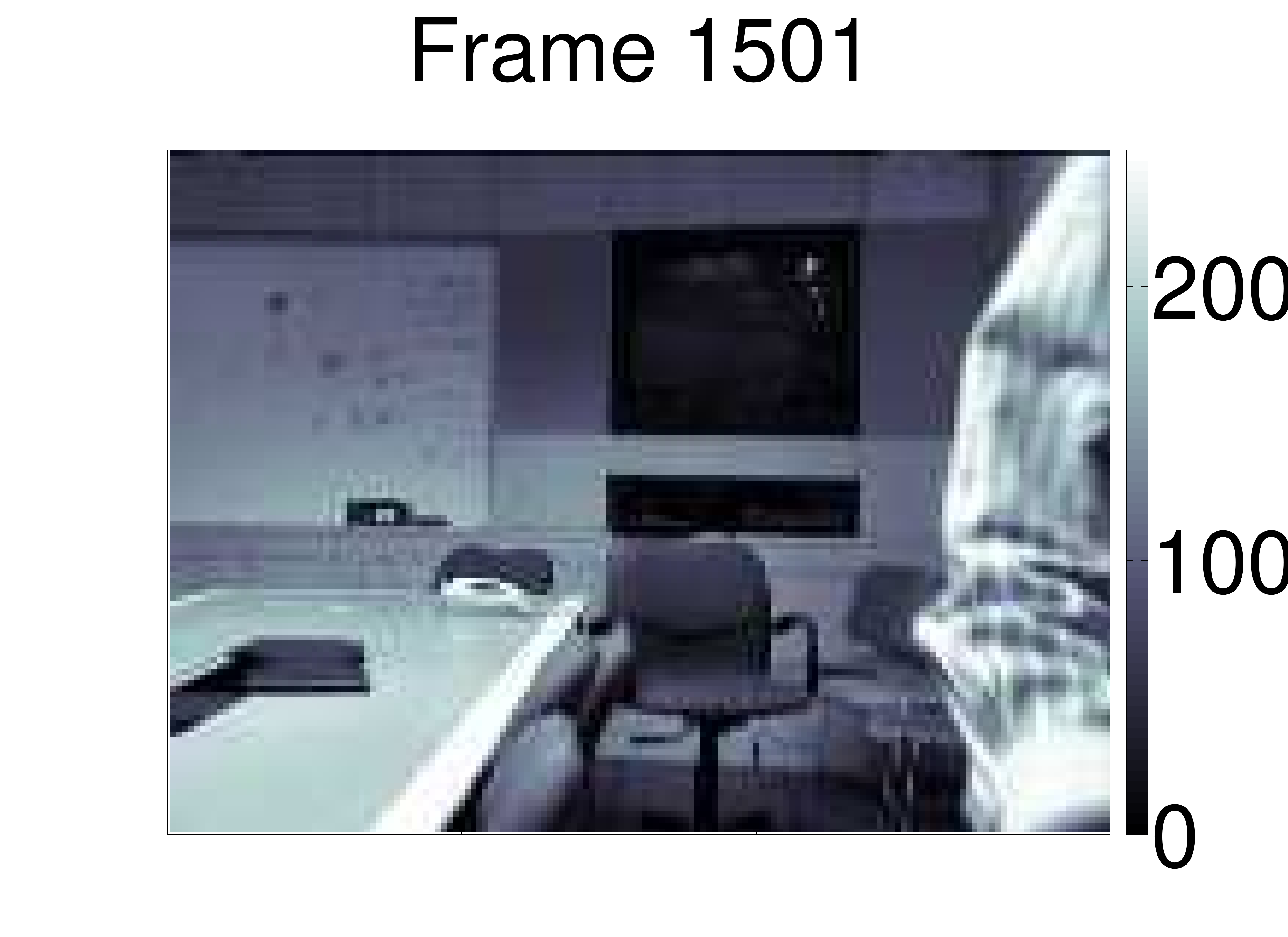} & 
\includegraphics[width=.25\columnwidth]{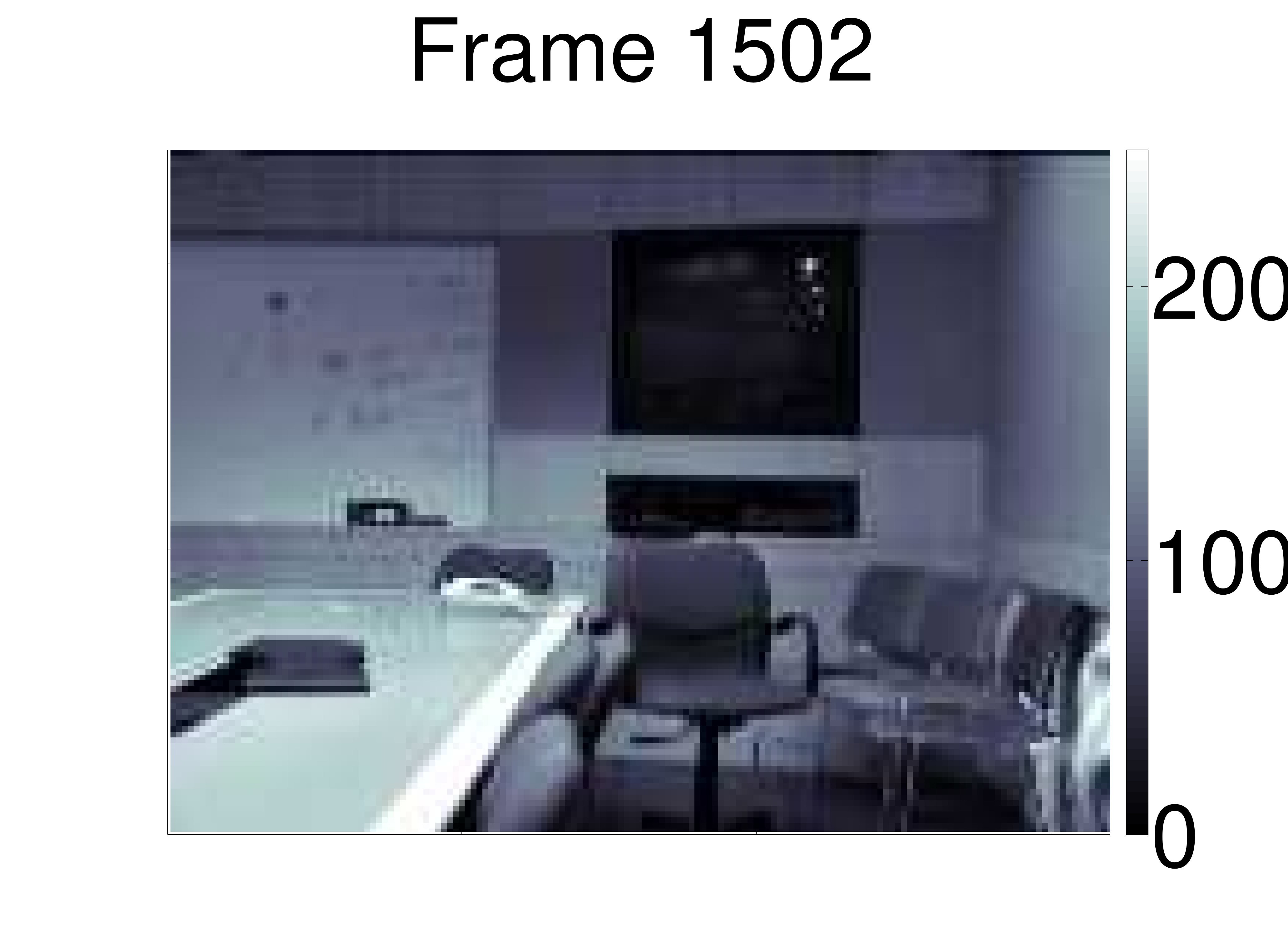} & 
\includegraphics[width=.25\columnwidth]{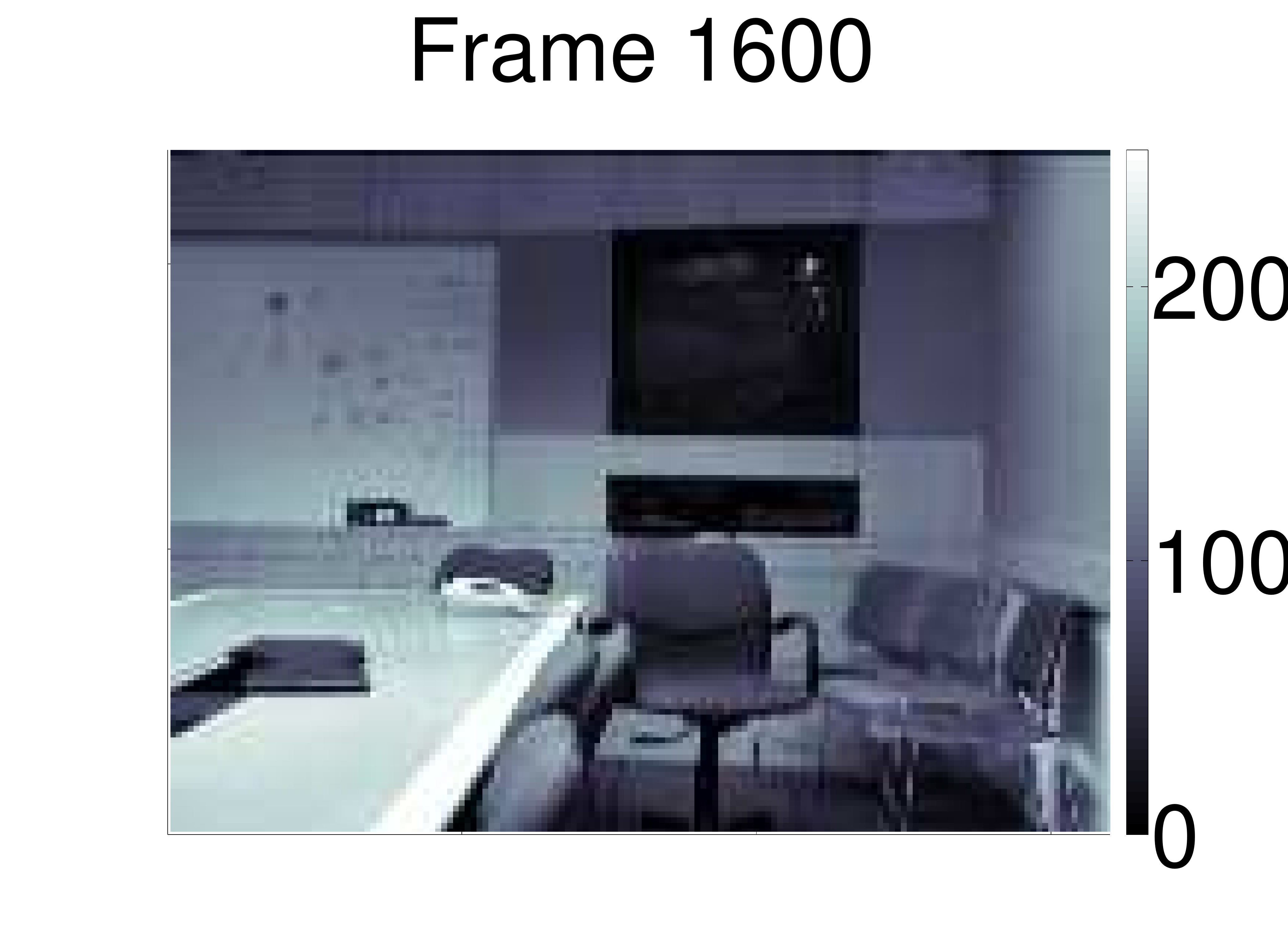} & 
\includegraphics[width=.25\columnwidth]{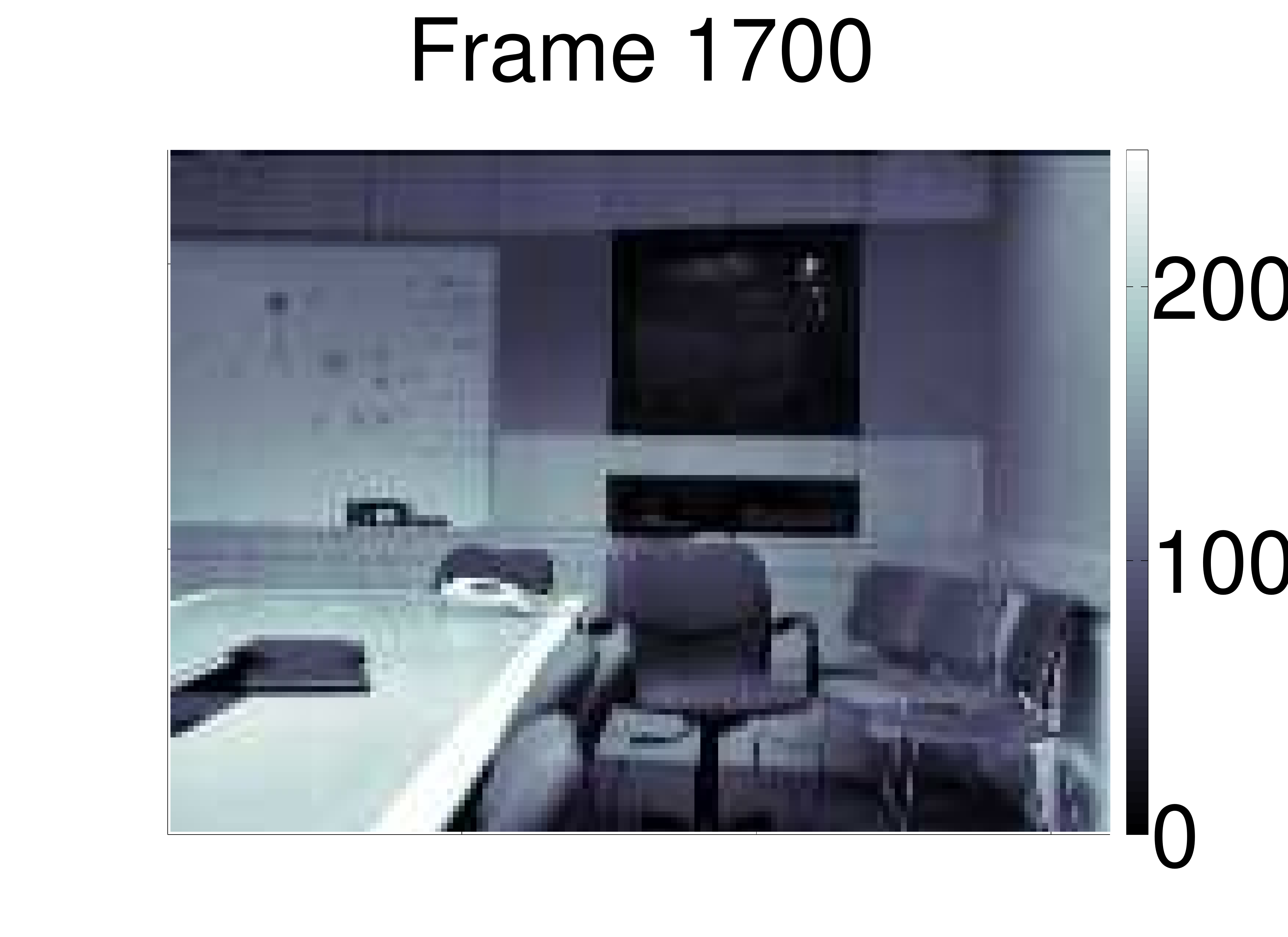} \\
\end{tabular}
\caption{Examples of the original frames of the moved object data set. Frames from 1 to 637, from 892 to 1389, from 1503 to 1744 (end) contain only background (true data). The Person (outlier) is visible in the scene from frame 638 to 891 and from frame 1390 to 1502. We refer to frames 0 to 892 in the following as the reduced moved object data set
}
\label{fig:moorig}
\end{figure}
\clearpage

%%%%%%%%%%%%%%%%%%%%%%%%%%%%%%%%%%%%%%%%%%%%%%%%%%%%
\begin{figure}
\centering
\begin{tabular}{cc}
\includegraphics[width=.5\columnwidth]{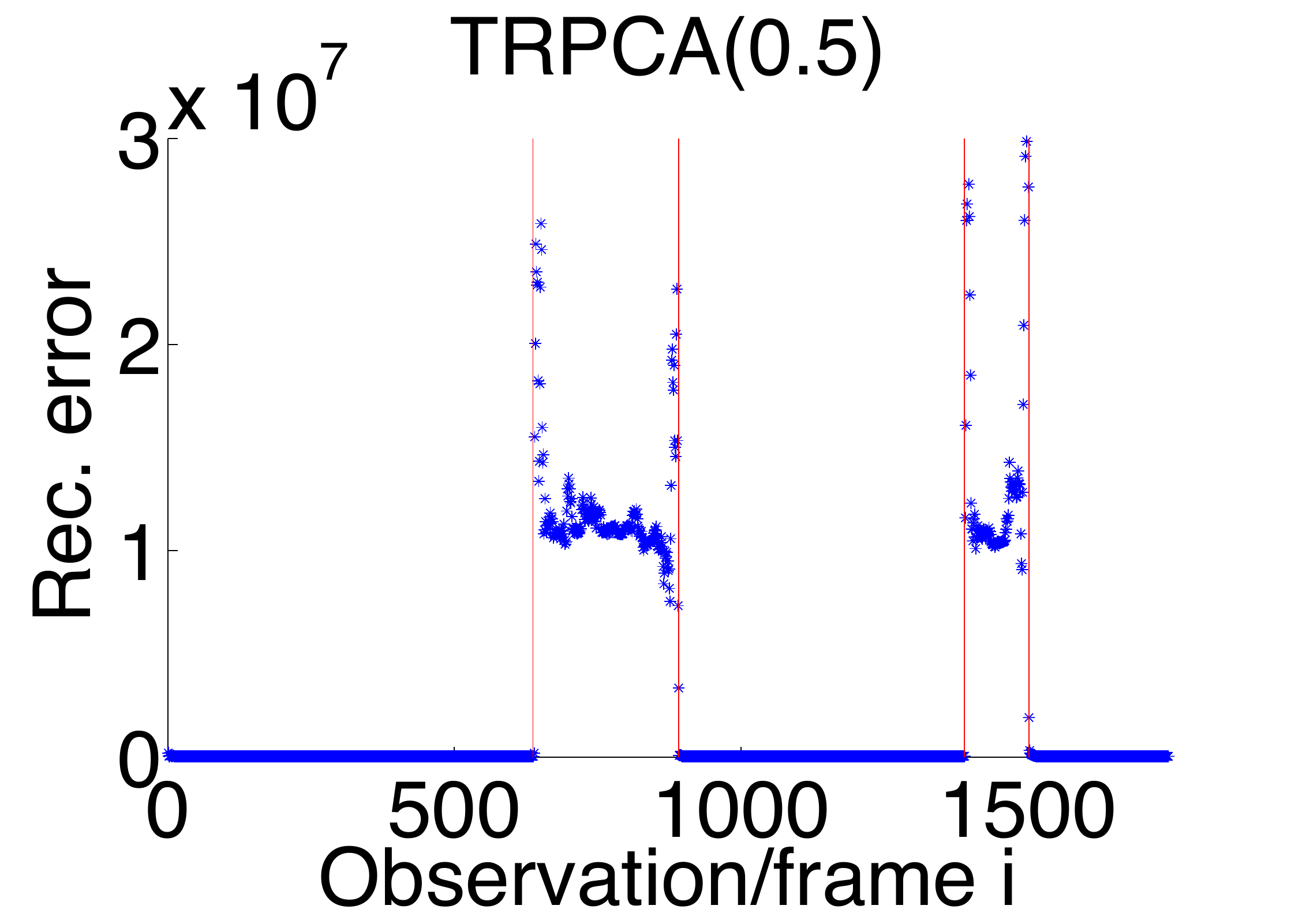} & 
\includegraphics[width=.5\columnwidth]{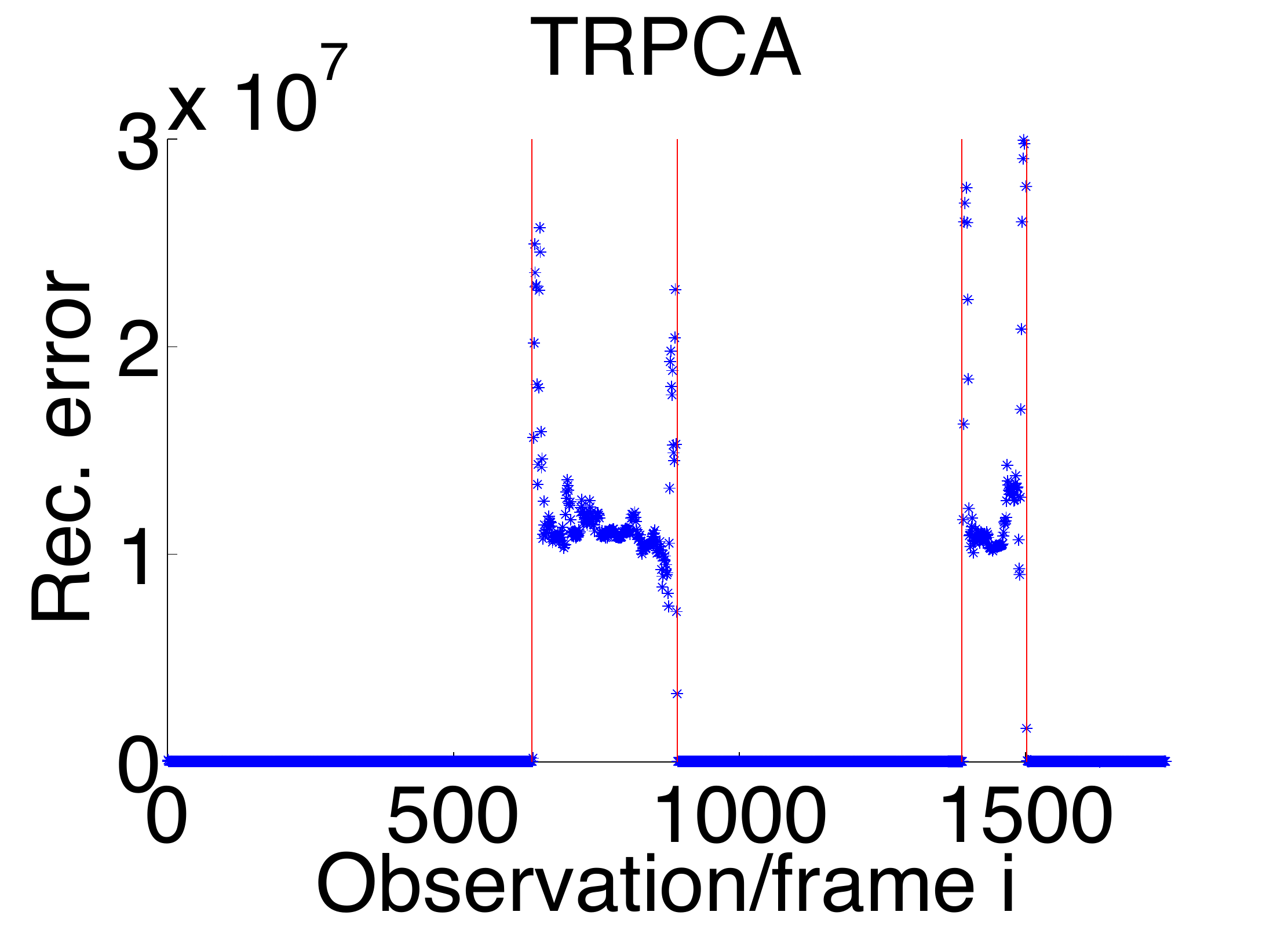} 
\end{tabular}
\caption{The reconstruction error of TPRCA/TRPCA(0.5), by analogy with Fig. \ref{fig:imagesresWS}, for the \textbf{full} moved object data set. The red vertical lines correspond to frames where the person enters/leaves the scene. We do not perform this experiment on the full datset for all other methods given their high runtimes (see Table \ref{tab:runtime}) and instead proceed
with the reduced dataset (see figures below).\newline 
Please note also that there is a small change in the background between frames from 1 to 637 (B1) and frames from 892 to 1389 (B2). Thus the robust PCA components will capture this difference.
This is not a problem for outlier detection (as we can see from the reconstruction errors of our method above) as this change is still small compared to the variation when the person enters the scene but it disturbs the foreground/background detection of all methods. An offline method could detect the scenes with small reconstruction error and do the background/foreground decomposition for each segment separately. The other option
would be to use an online estimation procedure of robust components and center. We do not pursue these directions in this paper as the main purpose of these experiments is an illustration of the differences of the various robust PCA methods in the literature
}
\label{fig:morecerrs}
\end{figure}
\clearpage

%%%%%%%%%%%%%%%%%%%%%%%%%%%%%%%%%%%%%%%%%%%%%%%%%%%%
\begin{figure}
\centering
\begin{tabular}{cccc}
\includegraphics[width=.25\columnwidth]{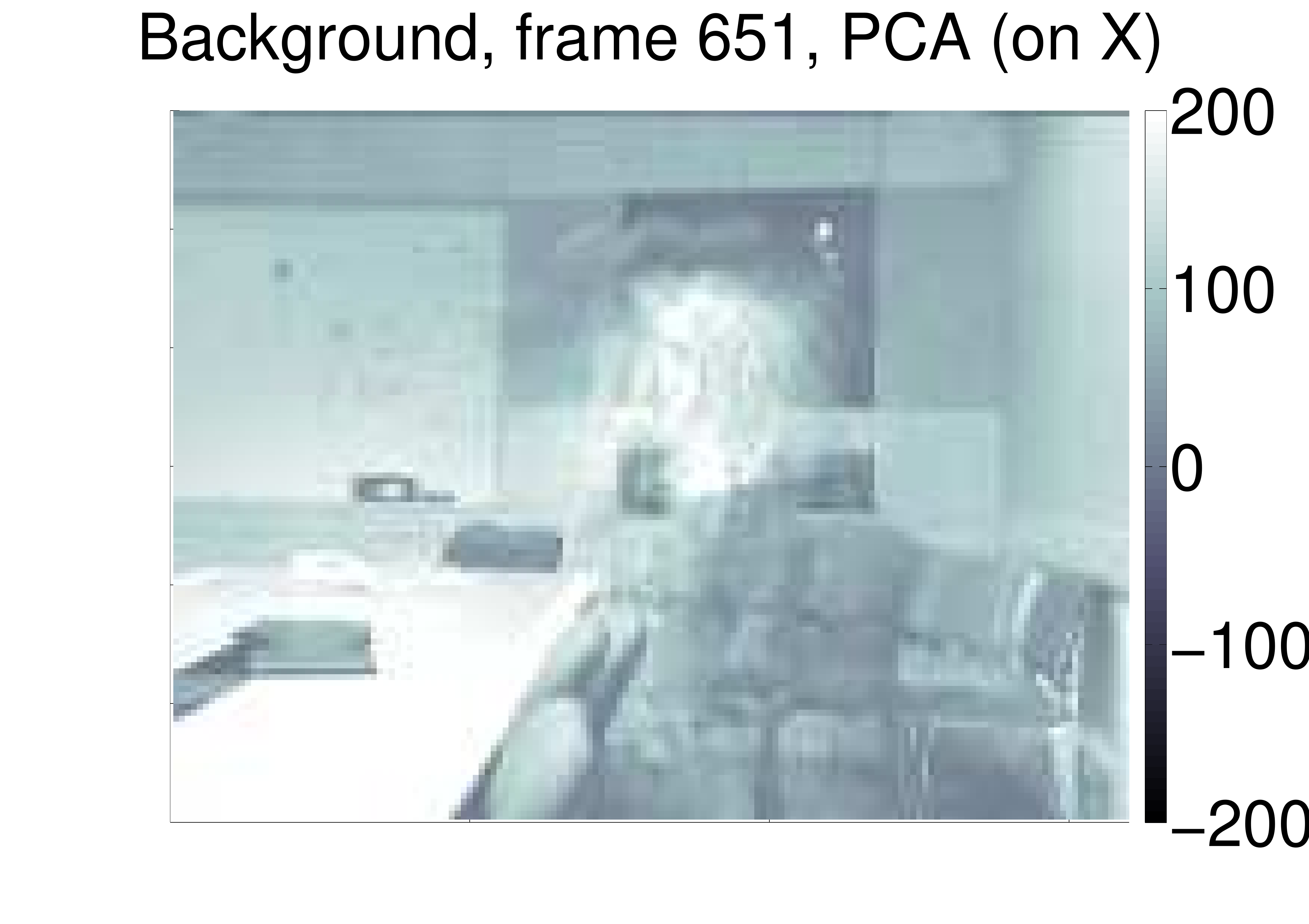} &
\includegraphics[width=.25\columnwidth]{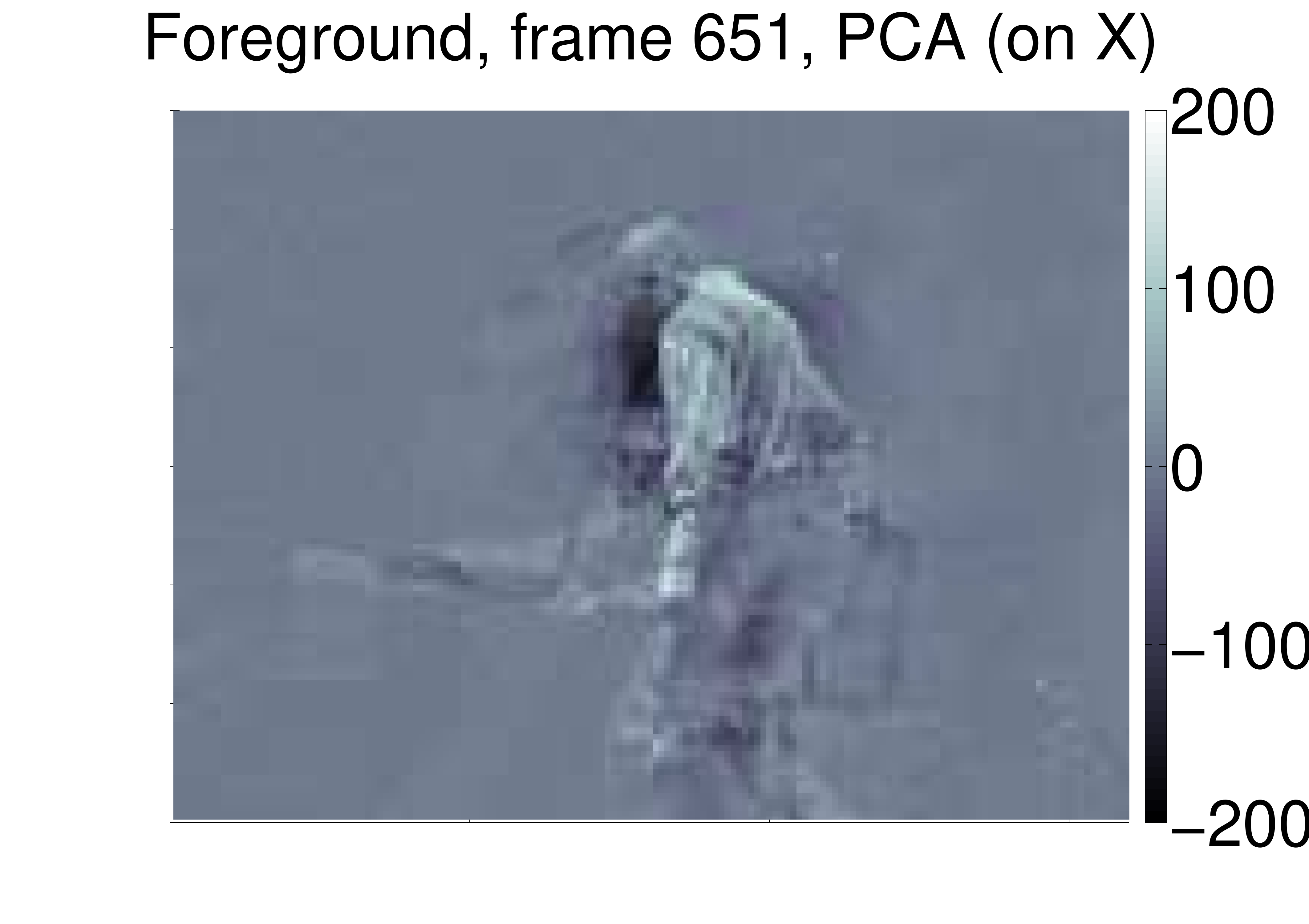} &
\includegraphics[width=.25\columnwidth]{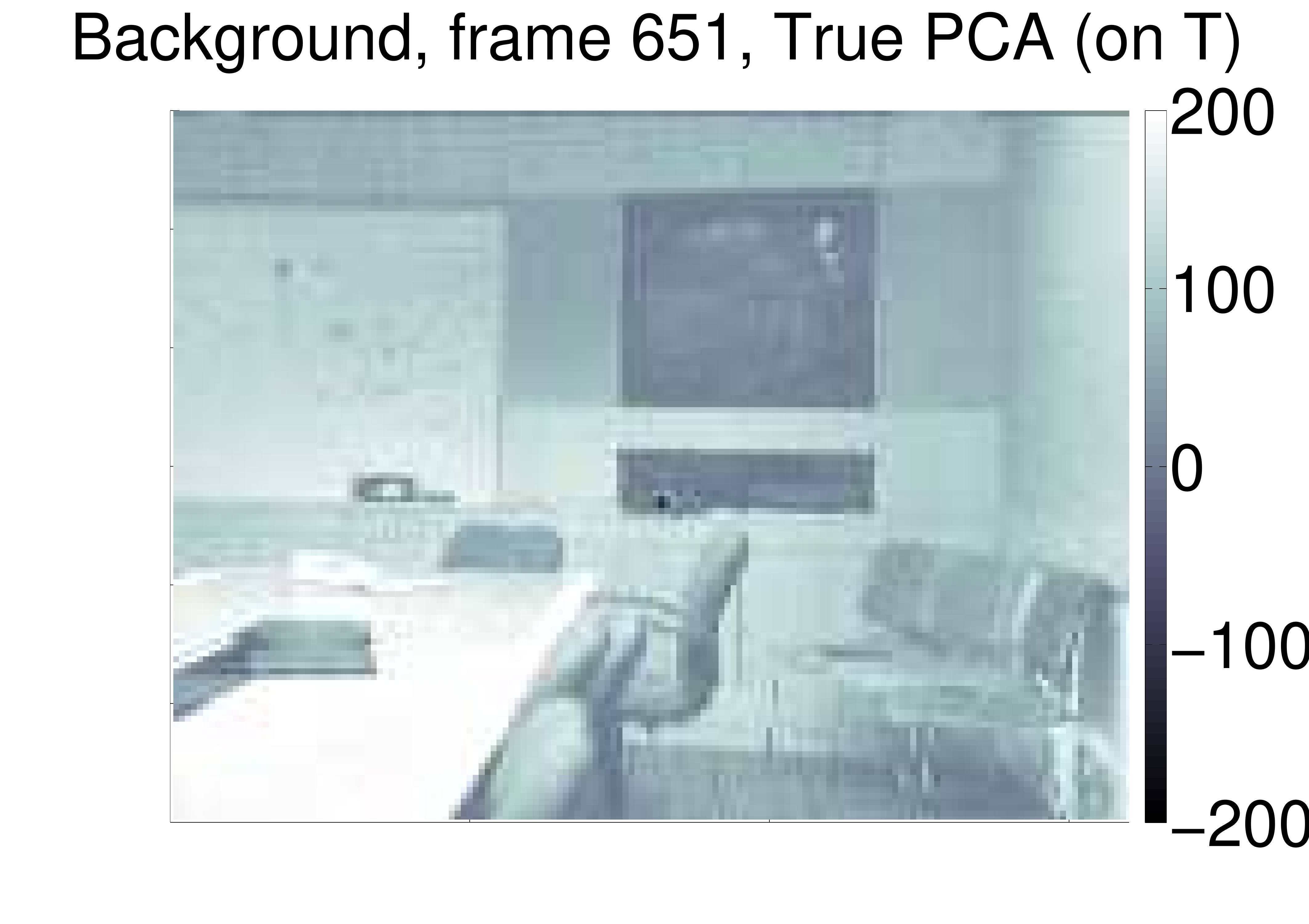} &
\includegraphics[width=.25\columnwidth]{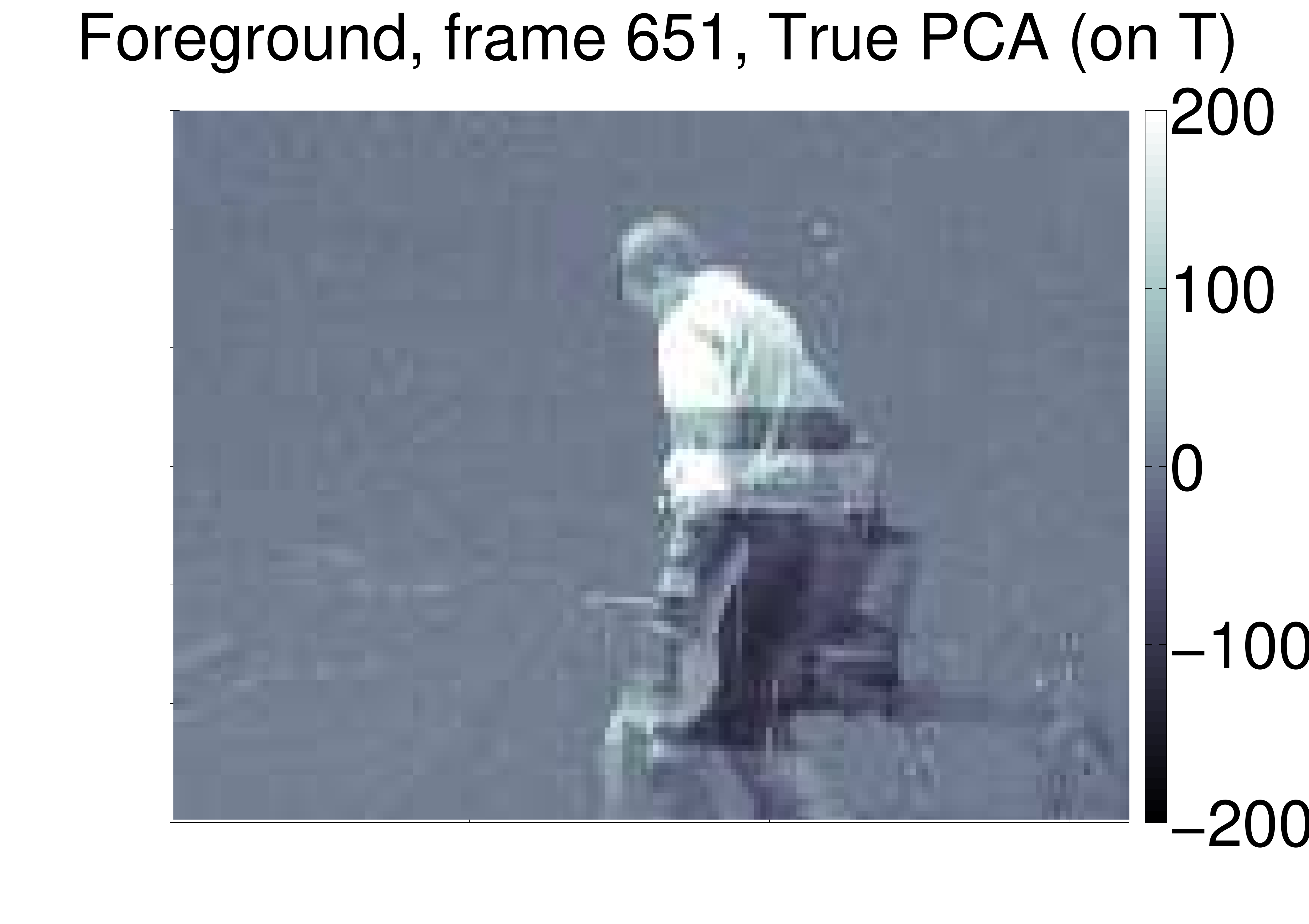} \\ \hline
\includegraphics[width=.25\columnwidth]{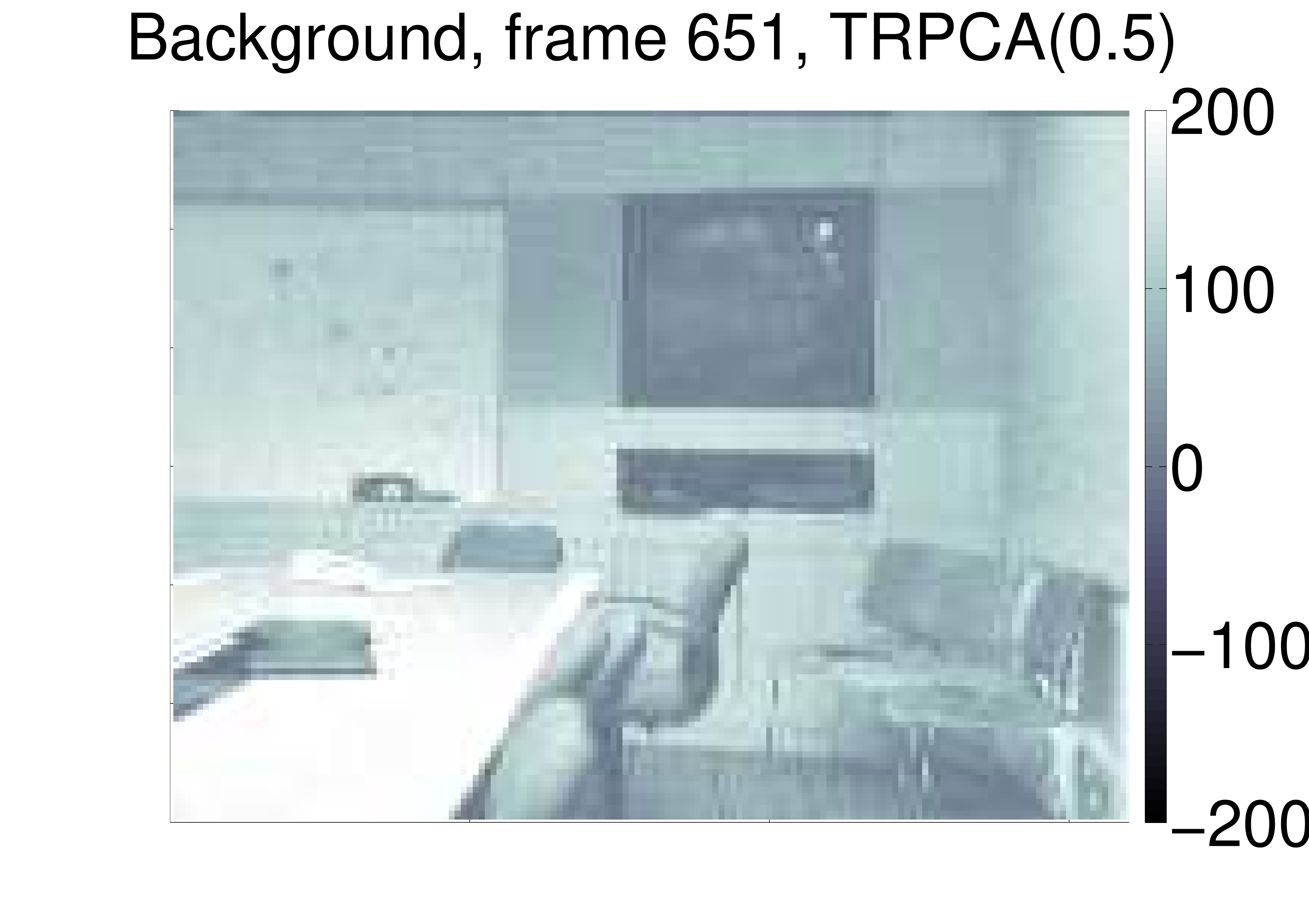} &
\includegraphics[width=.25\columnwidth]{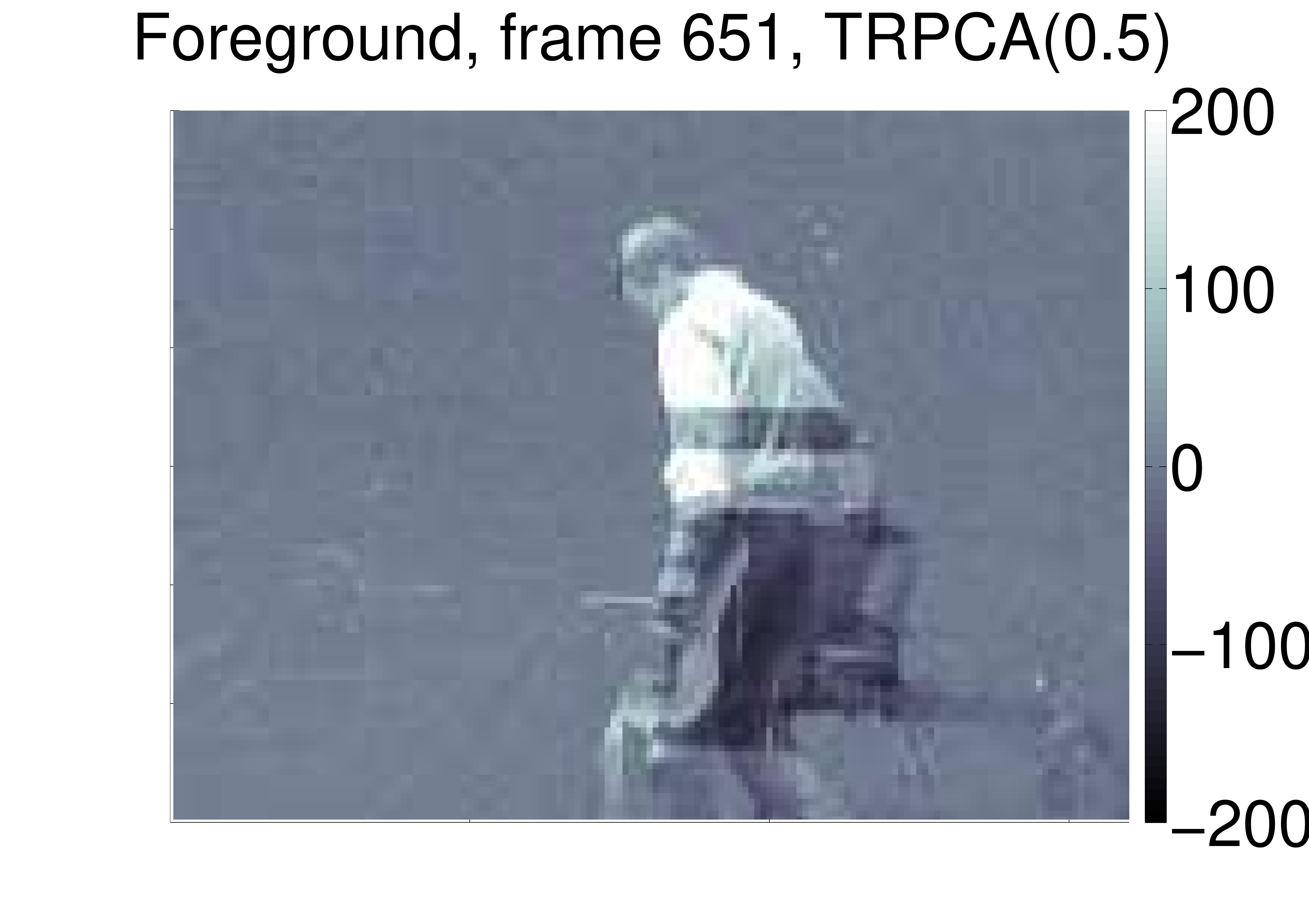} &
\includegraphics[width=.25\columnwidth]{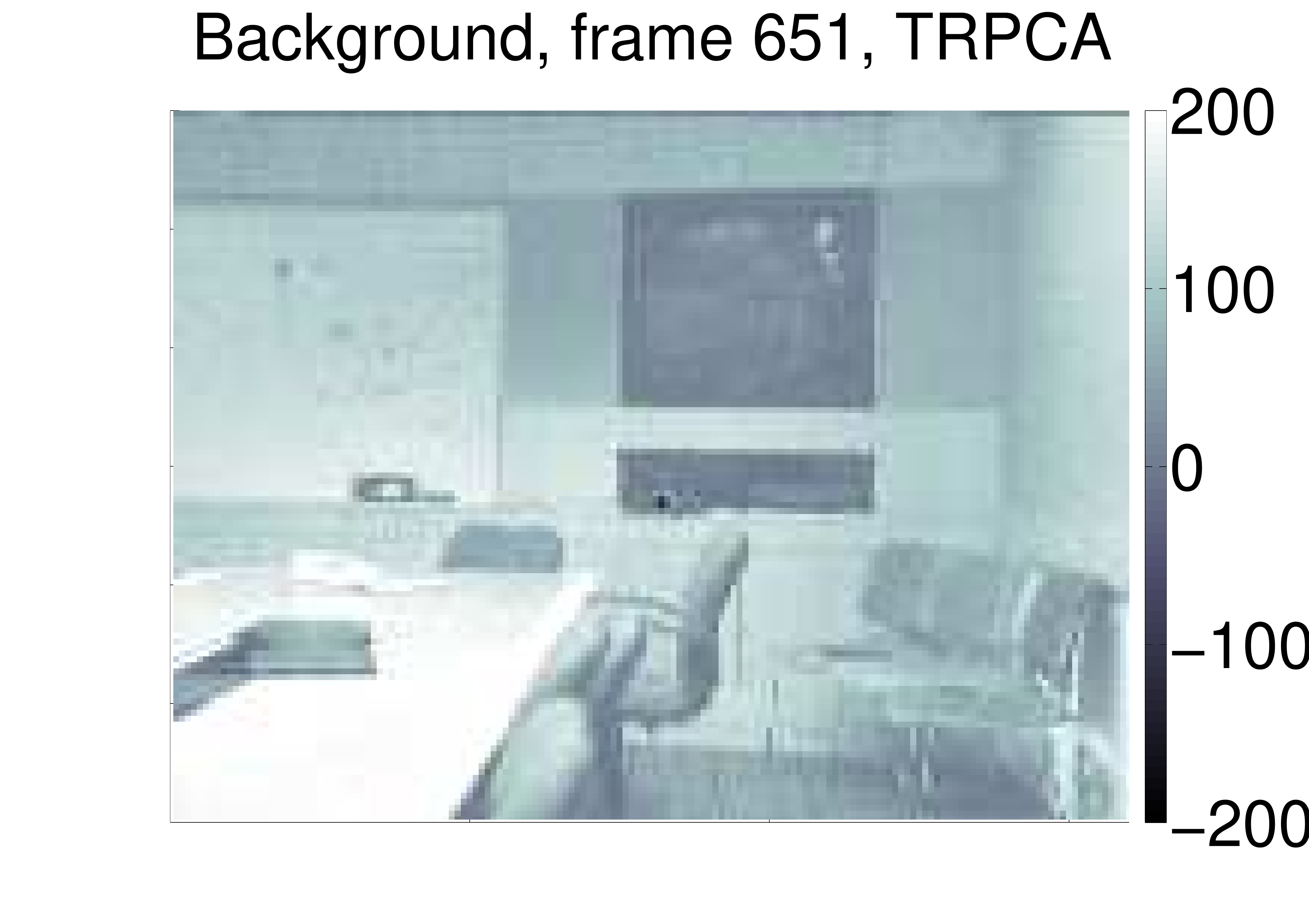} &
\includegraphics[width=.25\columnwidth]{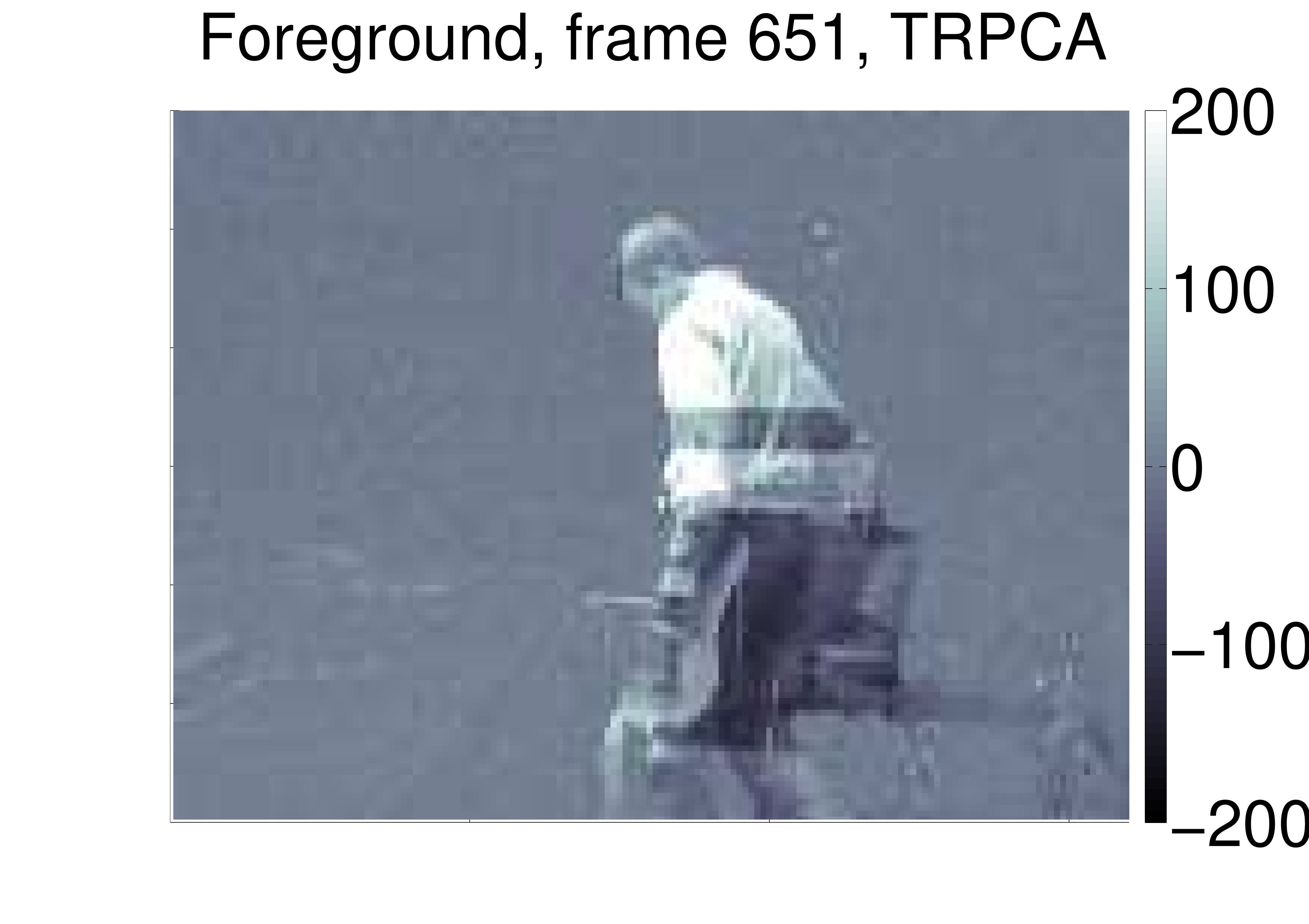} \\
\includegraphics[width=.25\columnwidth]{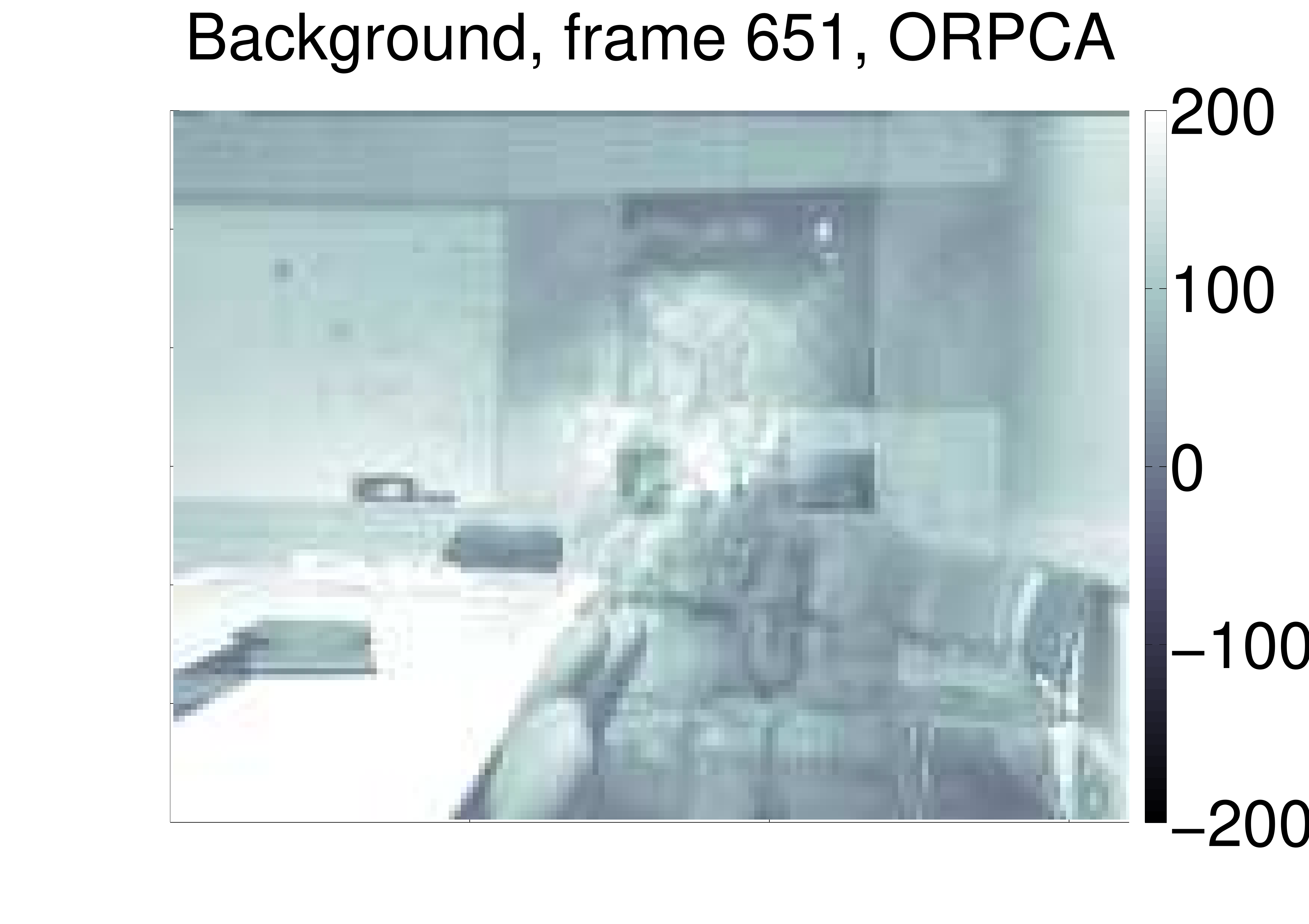} &
\includegraphics[width=.25\columnwidth]{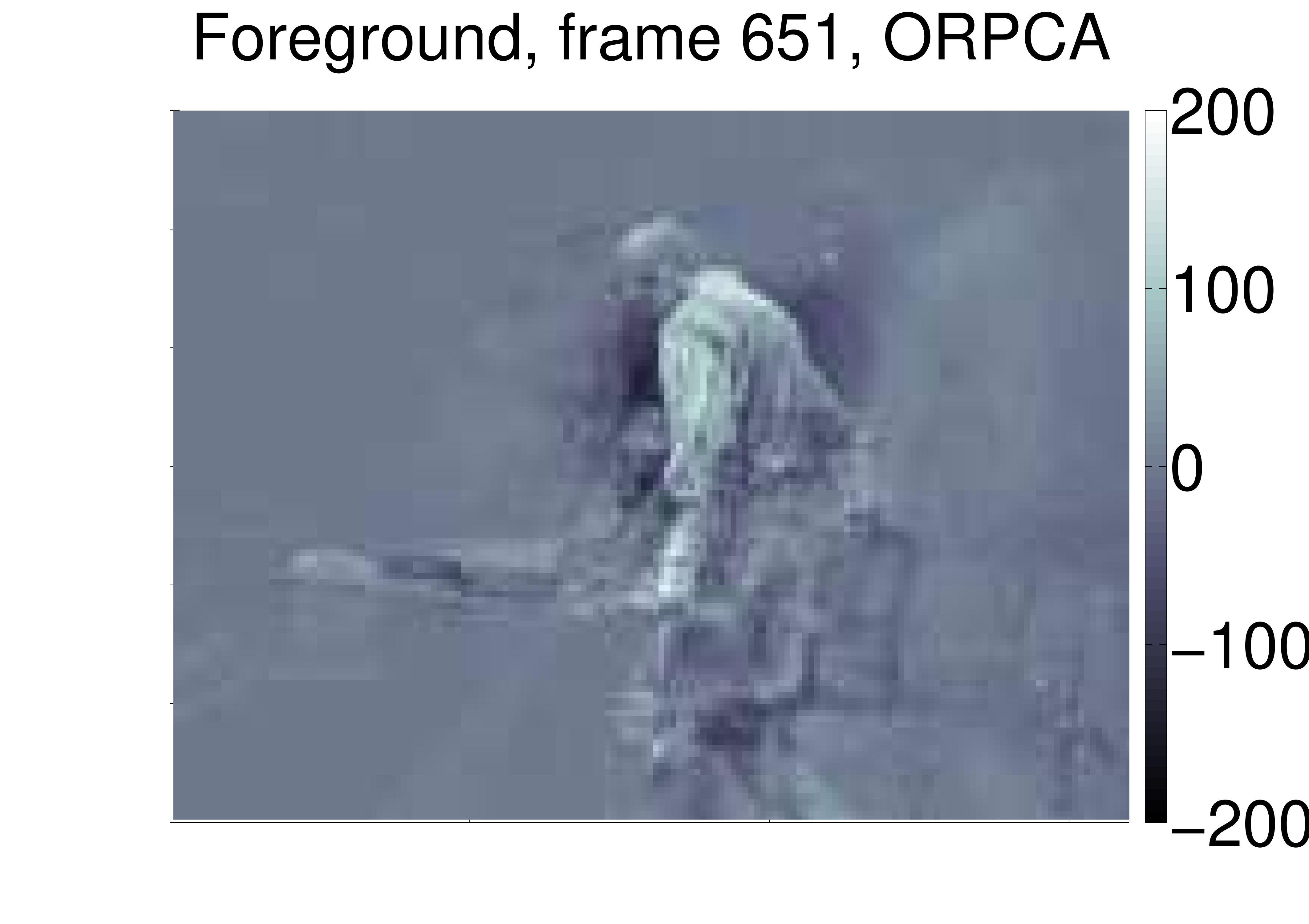} &
\includegraphics[width=.25\columnwidth]{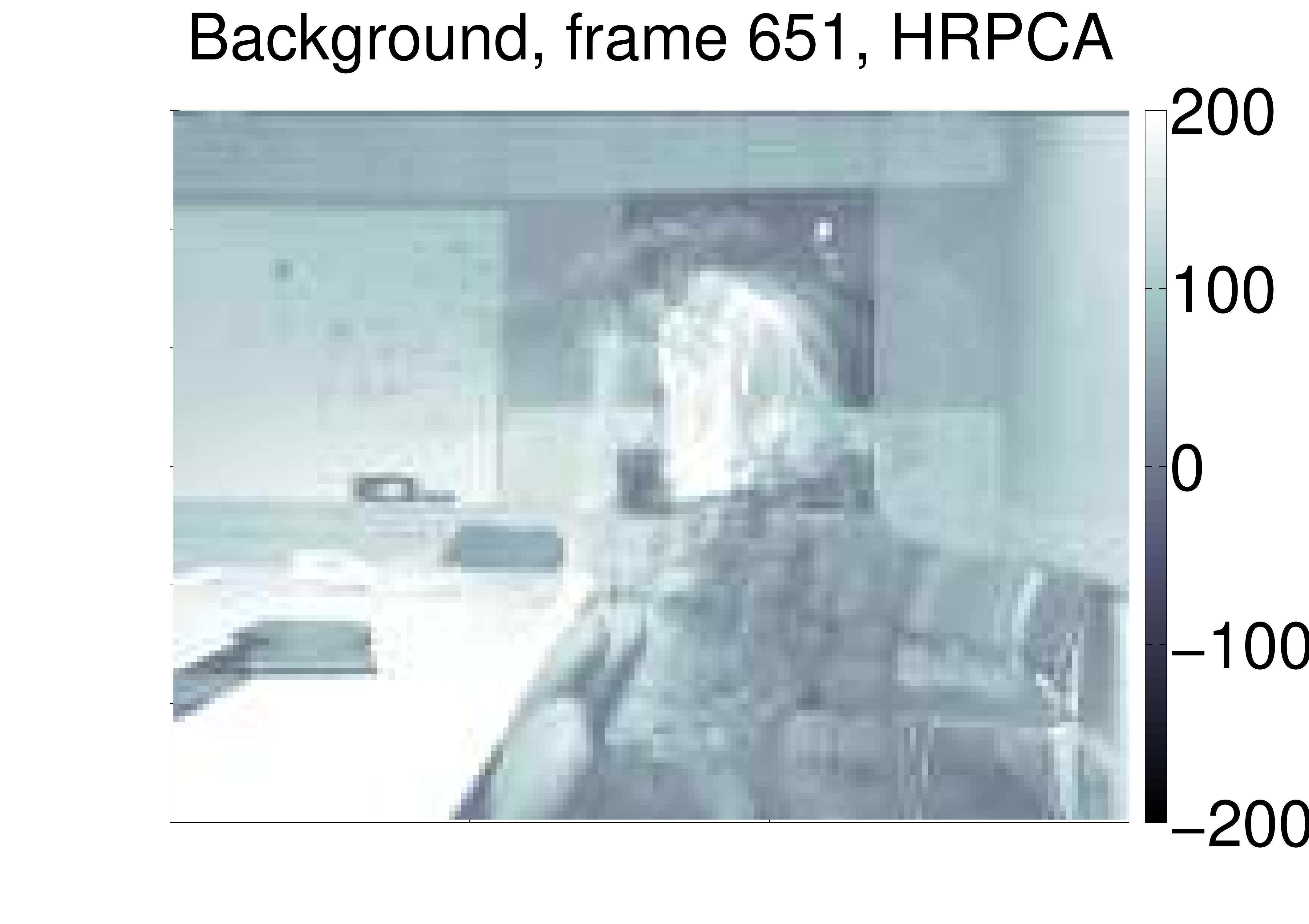} &
\includegraphics[width=.25\columnwidth]{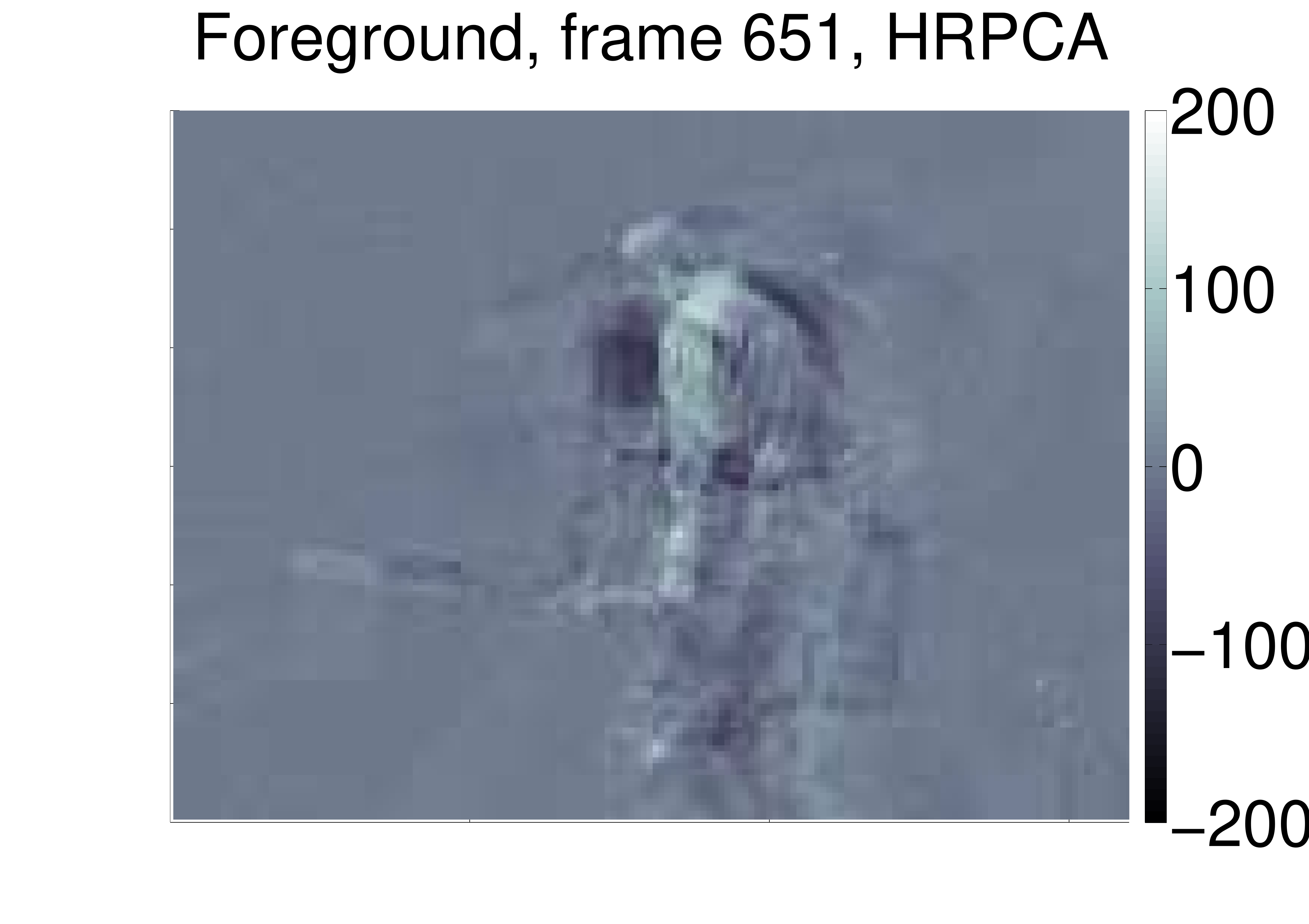} \\
\includegraphics[width=.25\columnwidth]{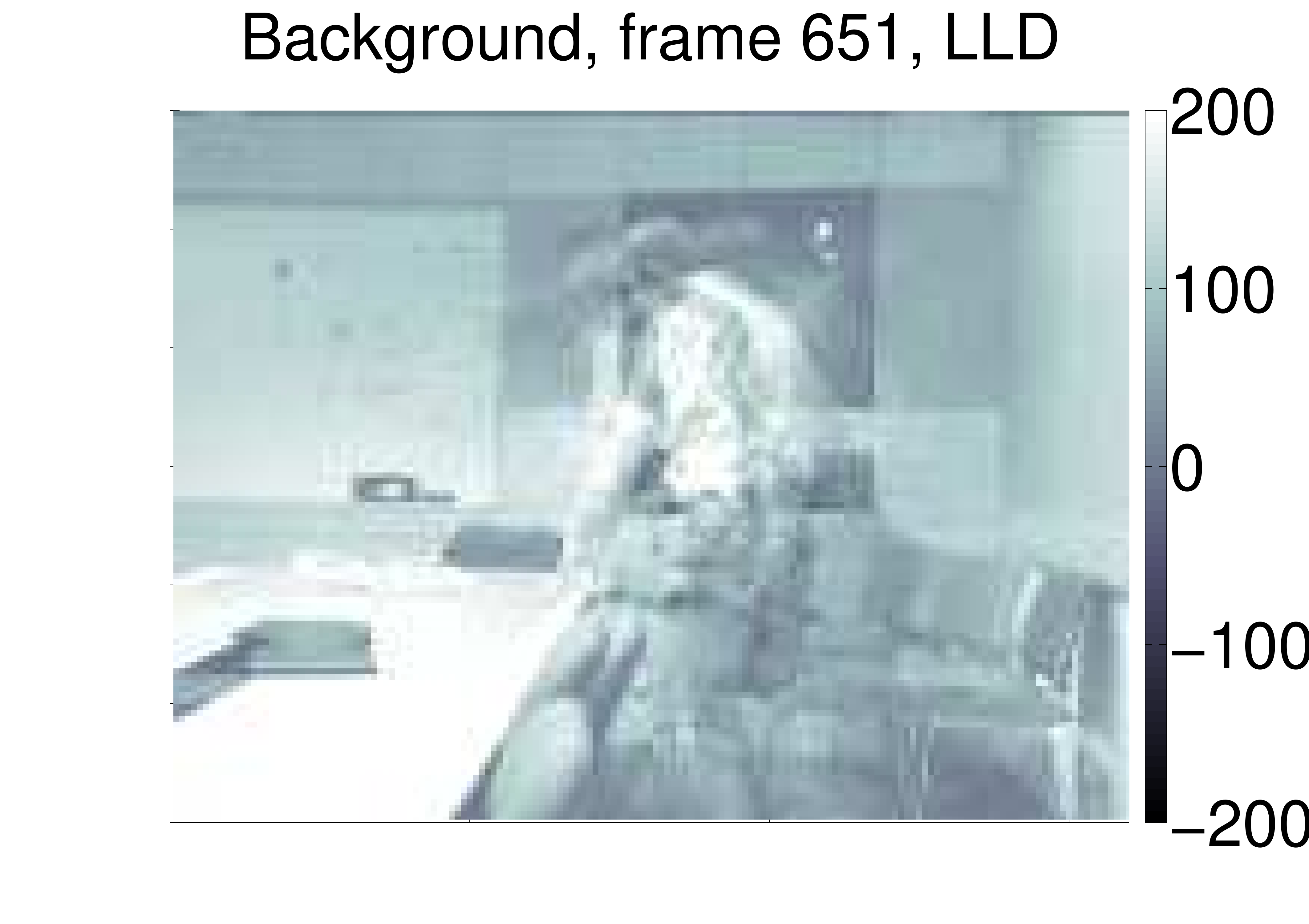} &
\includegraphics[width=.25\columnwidth]{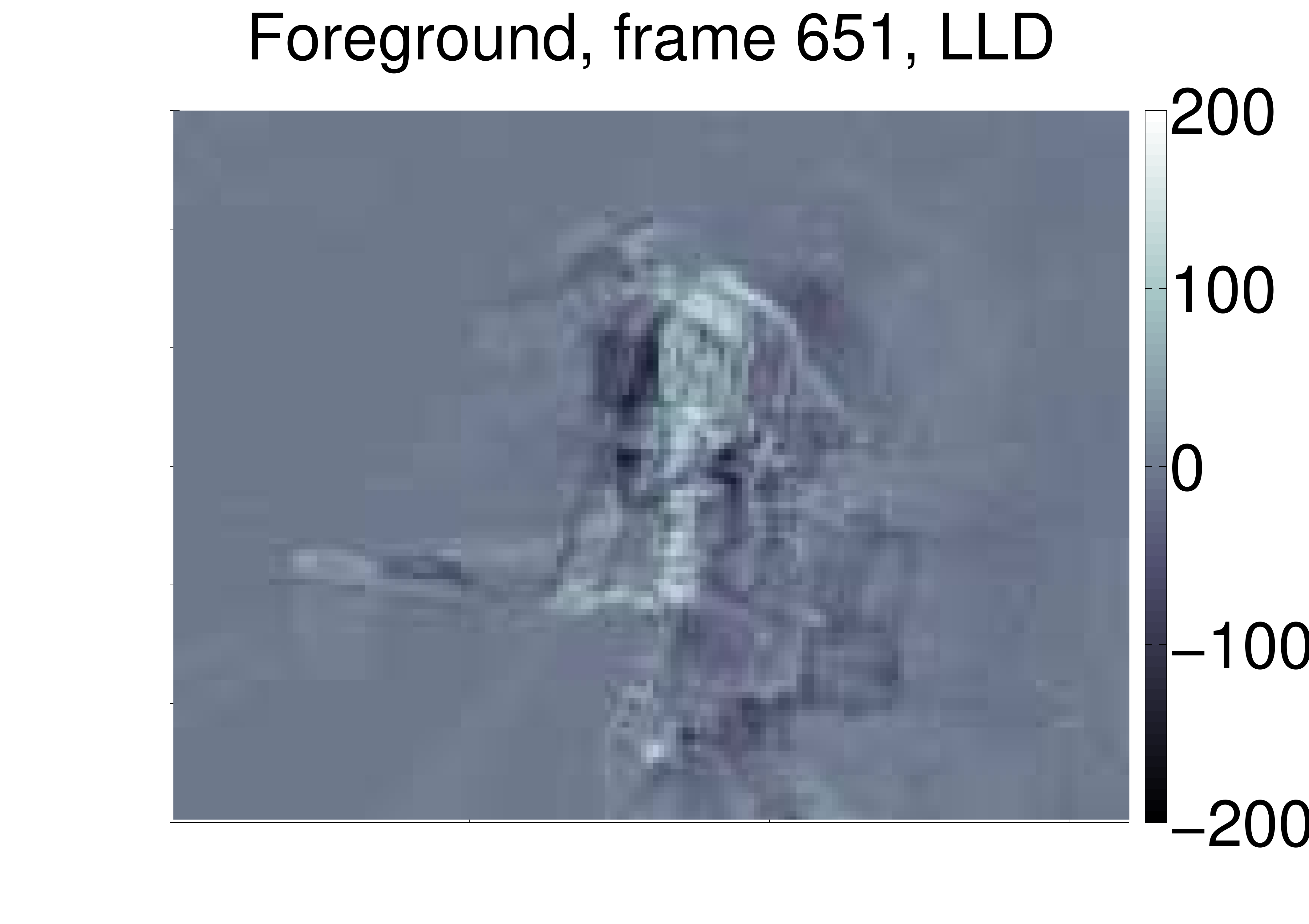} &
\includegraphics[width=.25\columnwidth]{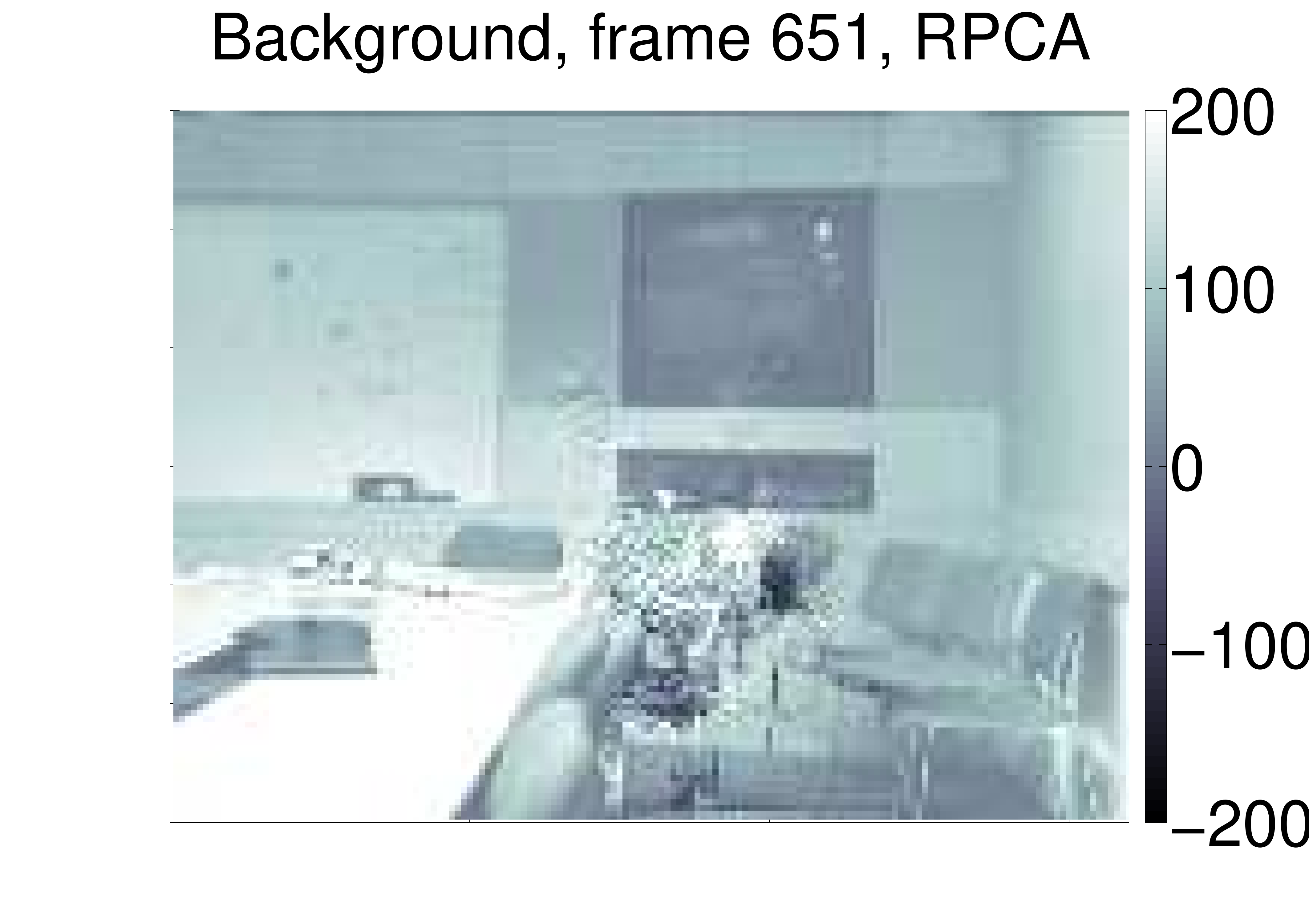} &
\includegraphics[width=.25\columnwidth]{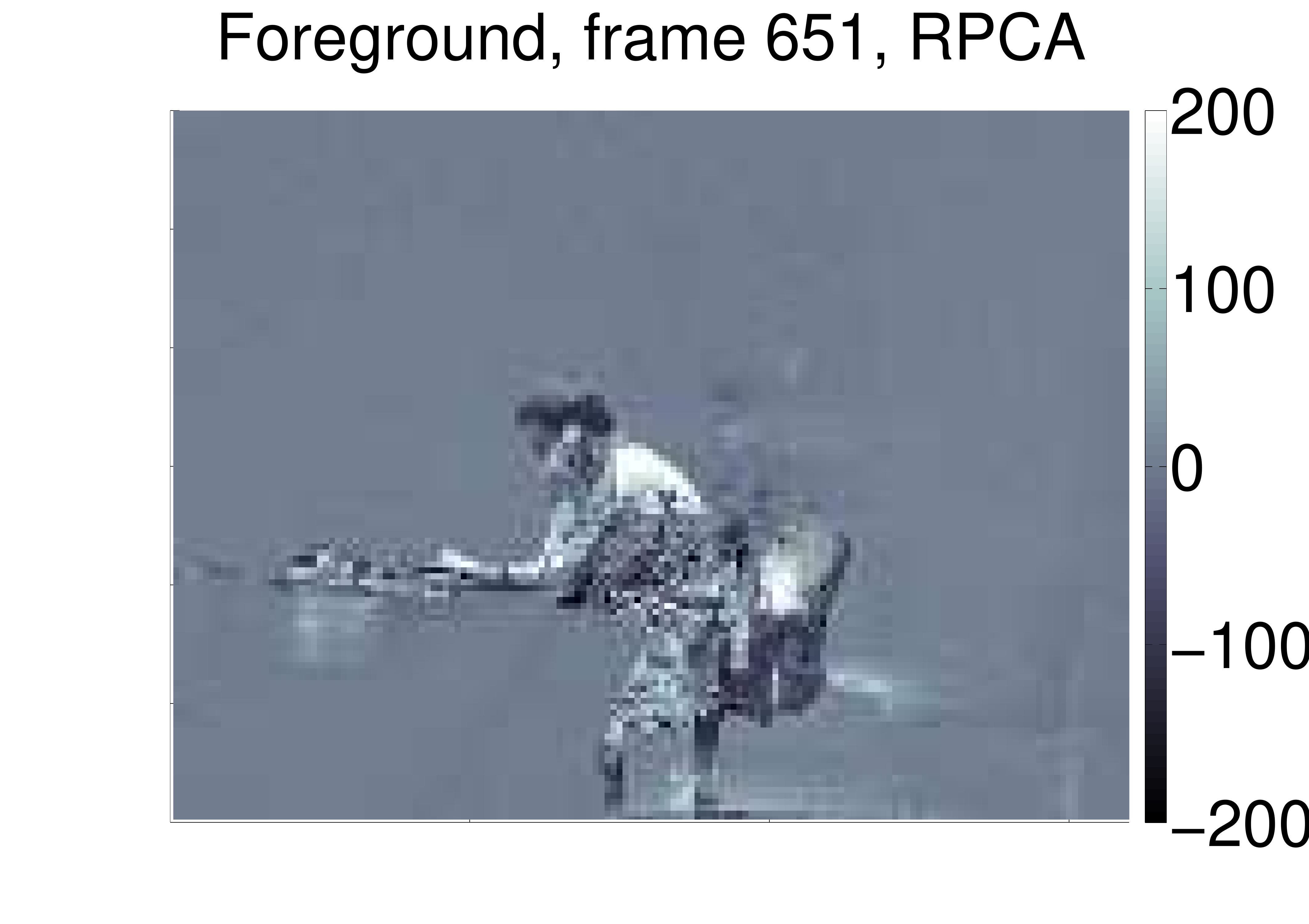}
\end{tabular}
\caption{Extracted background and foreground of frame 651 of the reduced moved object data set. The number of components is $k=10$ (scaled, compare to unscaled version in Fig.~\ref{fig:bfrmo2})
}
\label{fig:bfrmo1}
\end{figure}
\clearpage

%%%%%%%%%%%%%%%%%%%%%%%%%%%%%%%%%%%%%%%%%%%%%%%%%%%%
\begin{figure}
\begin{tabular}{cccc}
\includegraphics[width=.25\columnwidth]{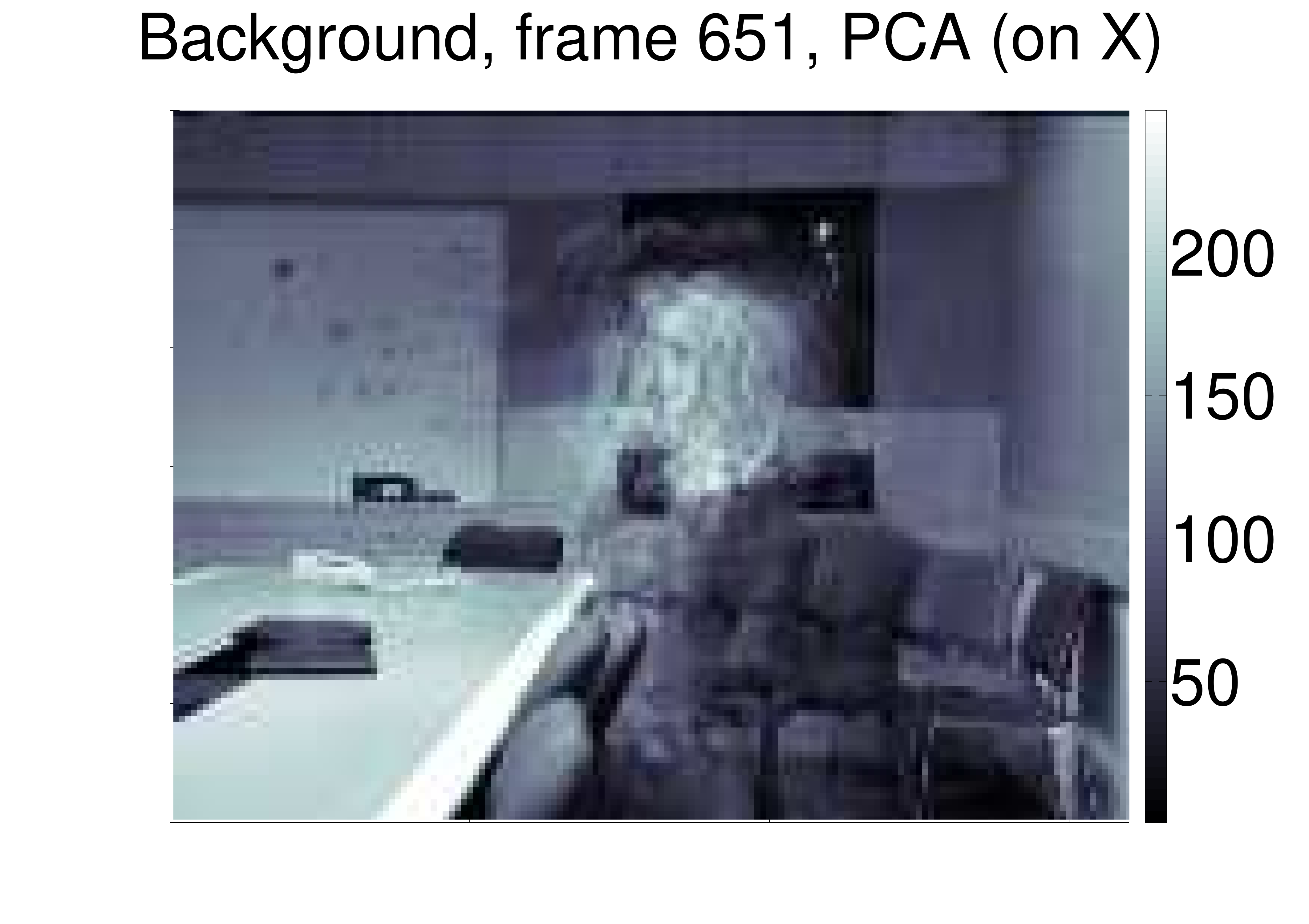} &
\includegraphics[width=.25\columnwidth]{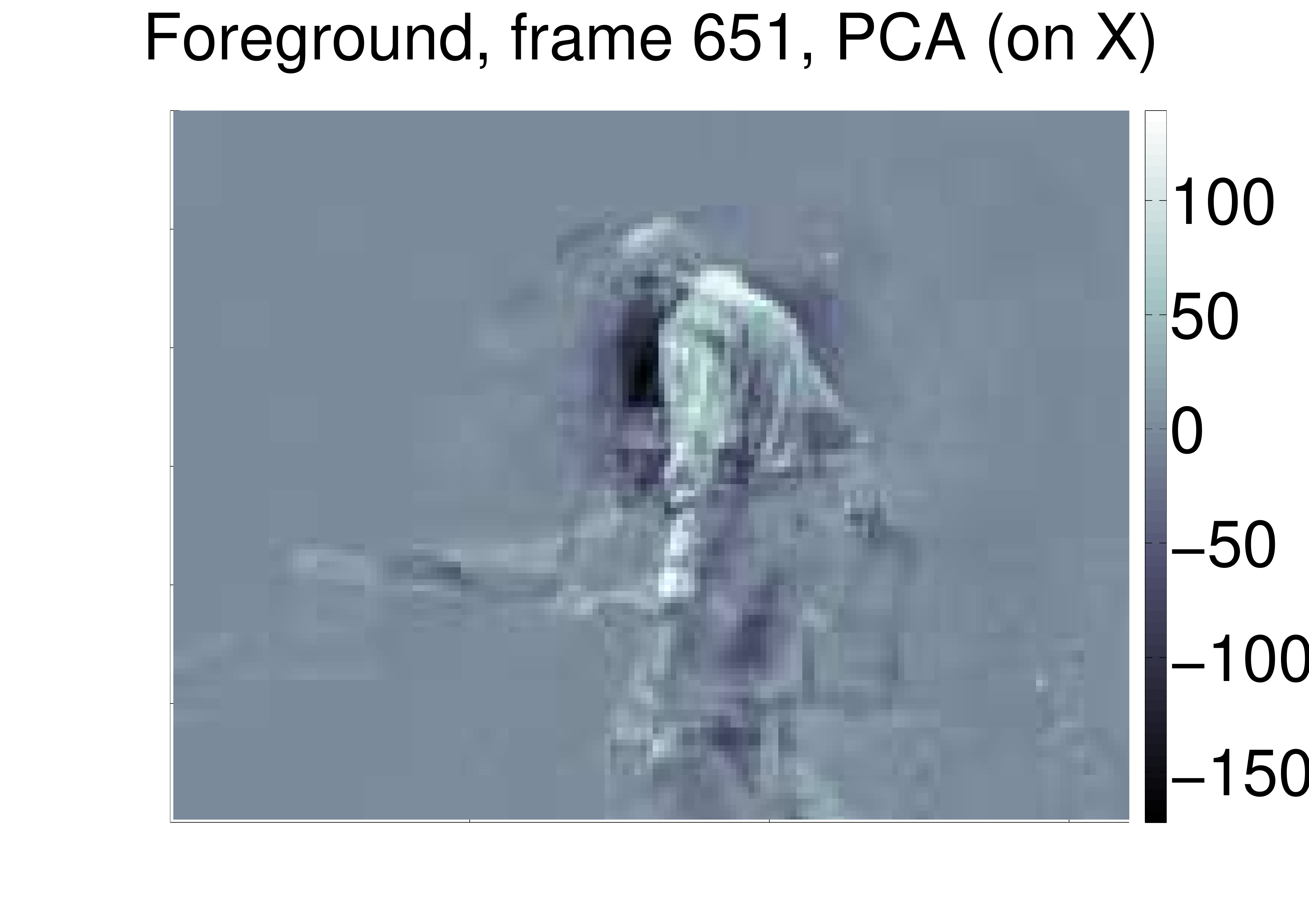} &
\includegraphics[width=.25\columnwidth]{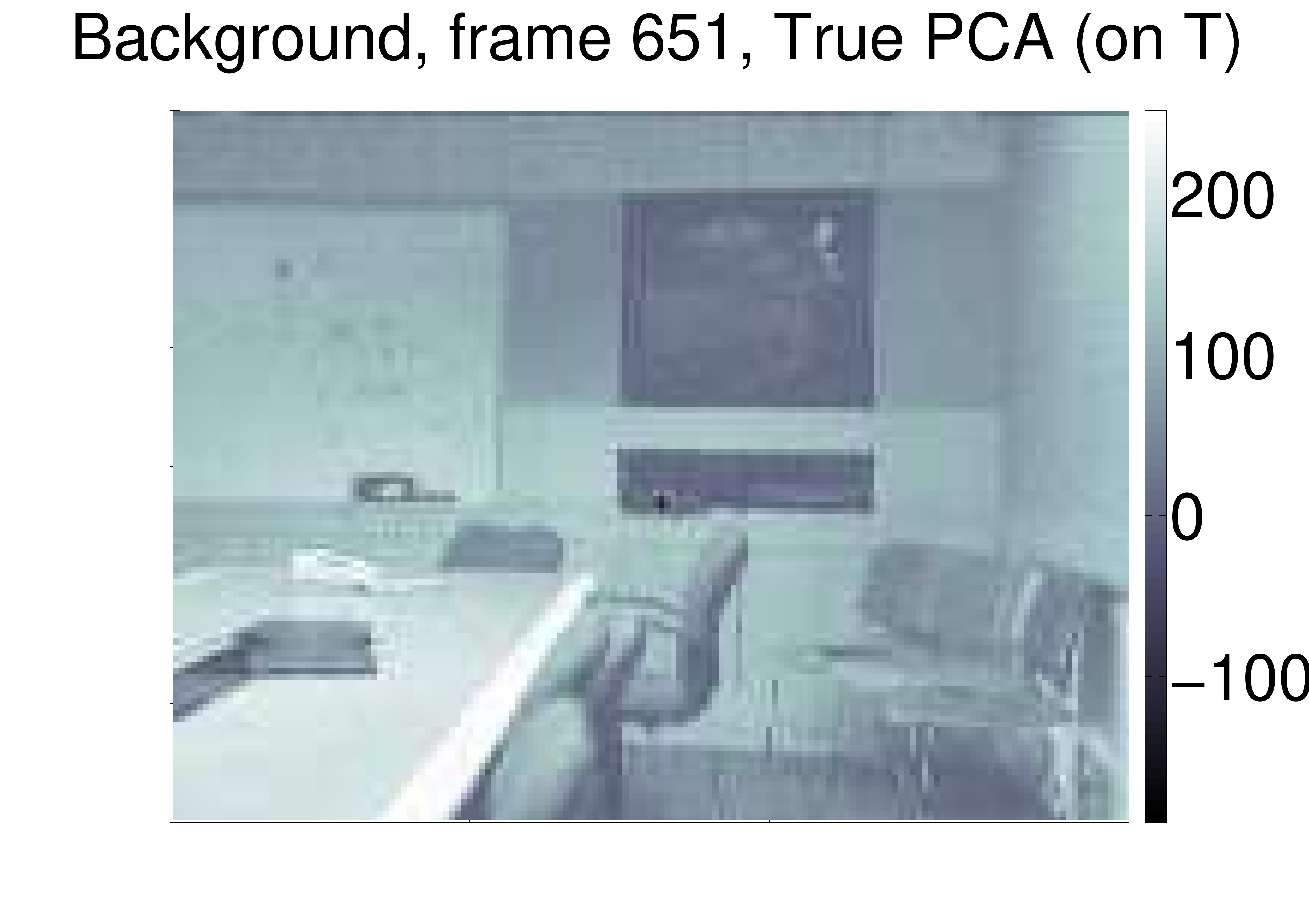} &
\includegraphics[width=.25\columnwidth]{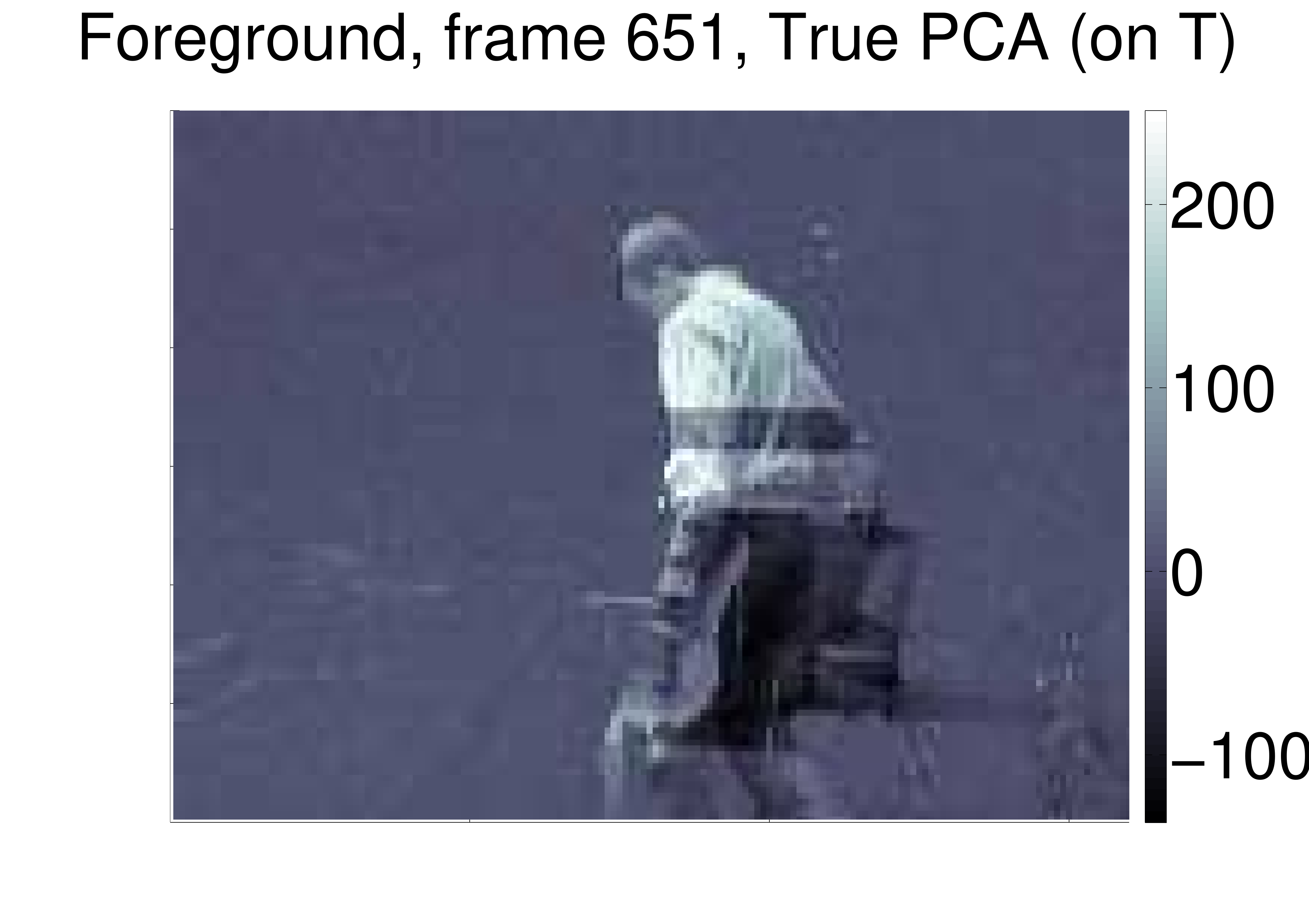} \\ \hline
\includegraphics[width=.25\columnwidth]{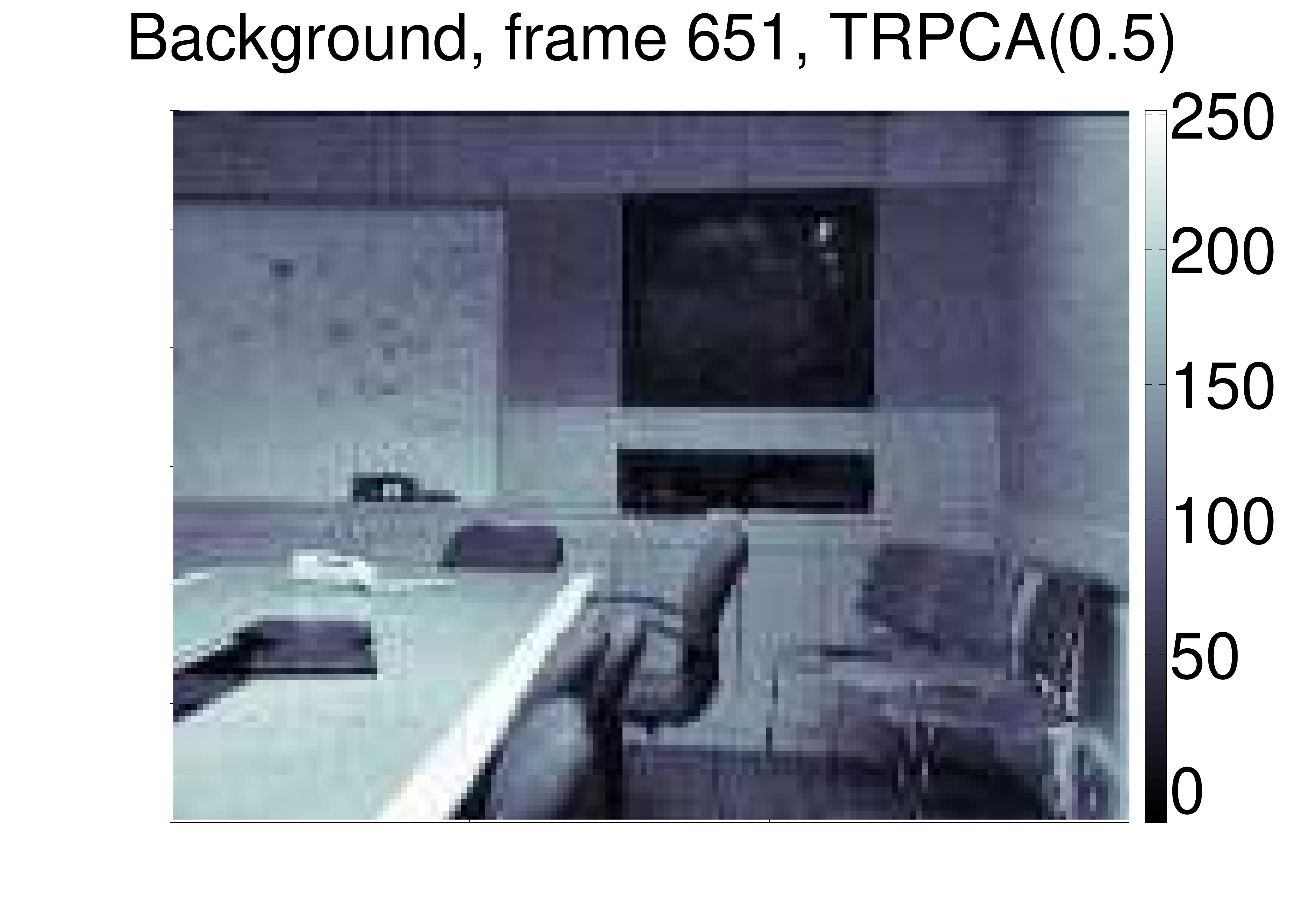} &
\includegraphics[width=.25\columnwidth]{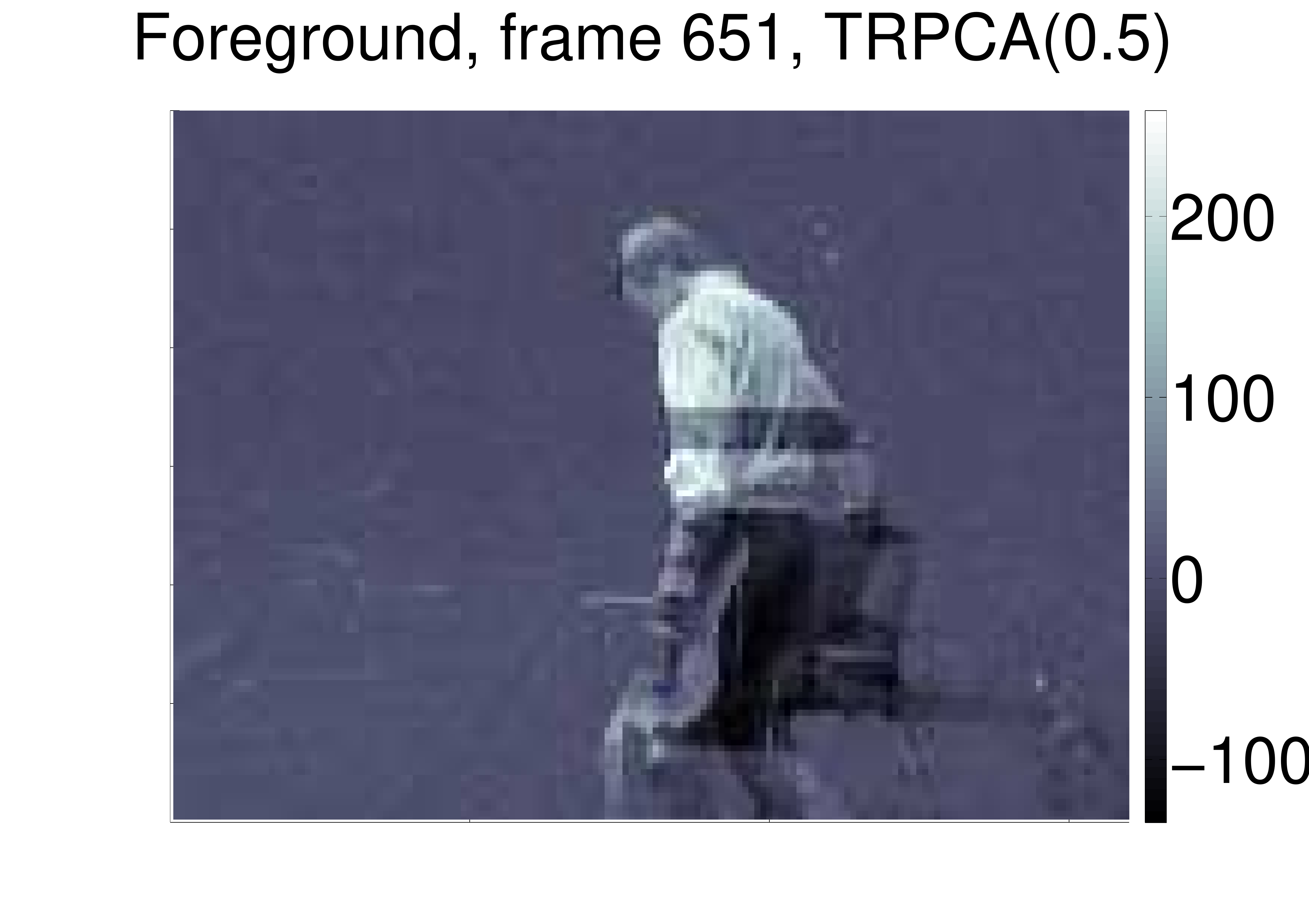} &
\includegraphics[width=.25\columnwidth]{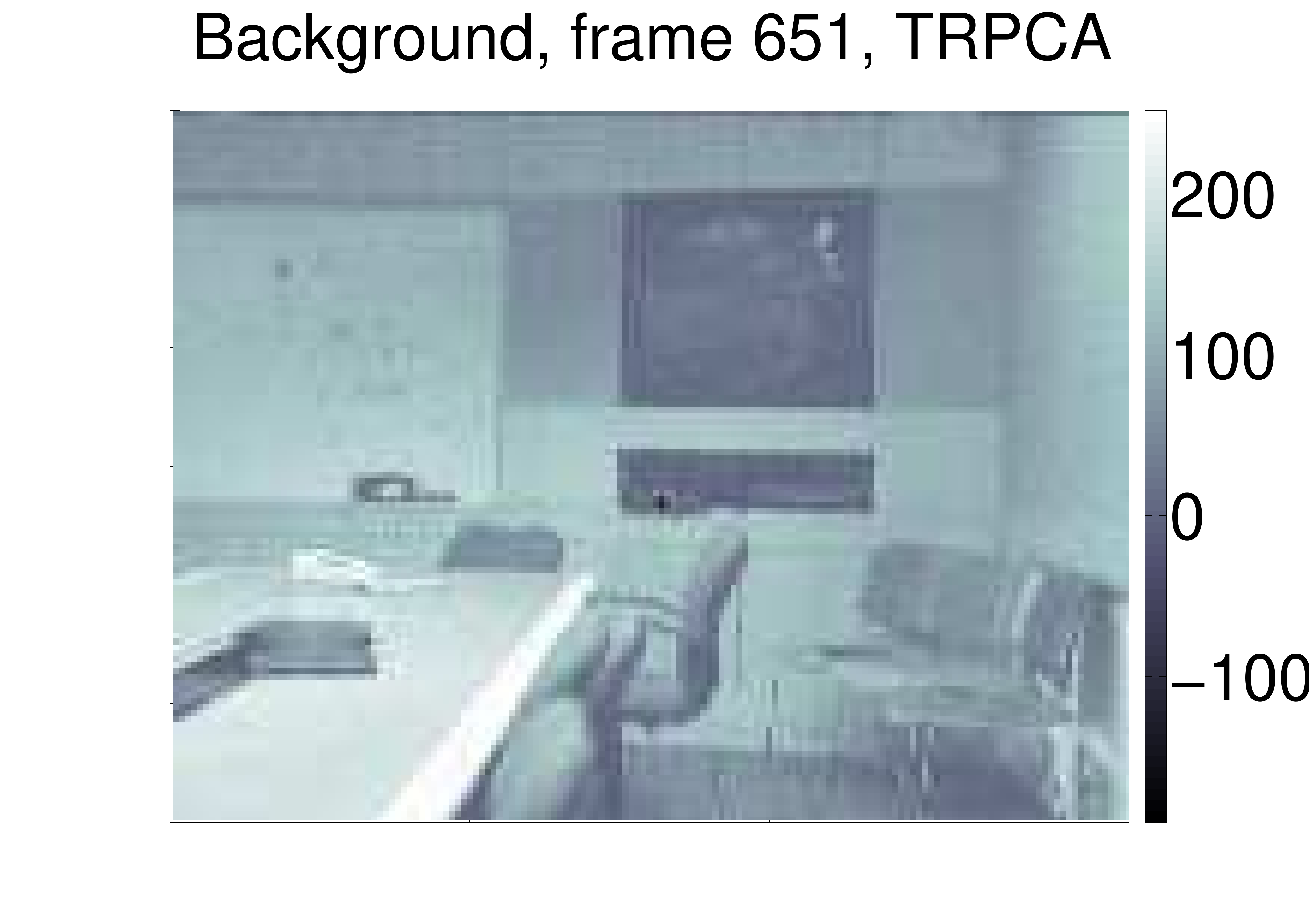} &
\includegraphics[width=.25\columnwidth]{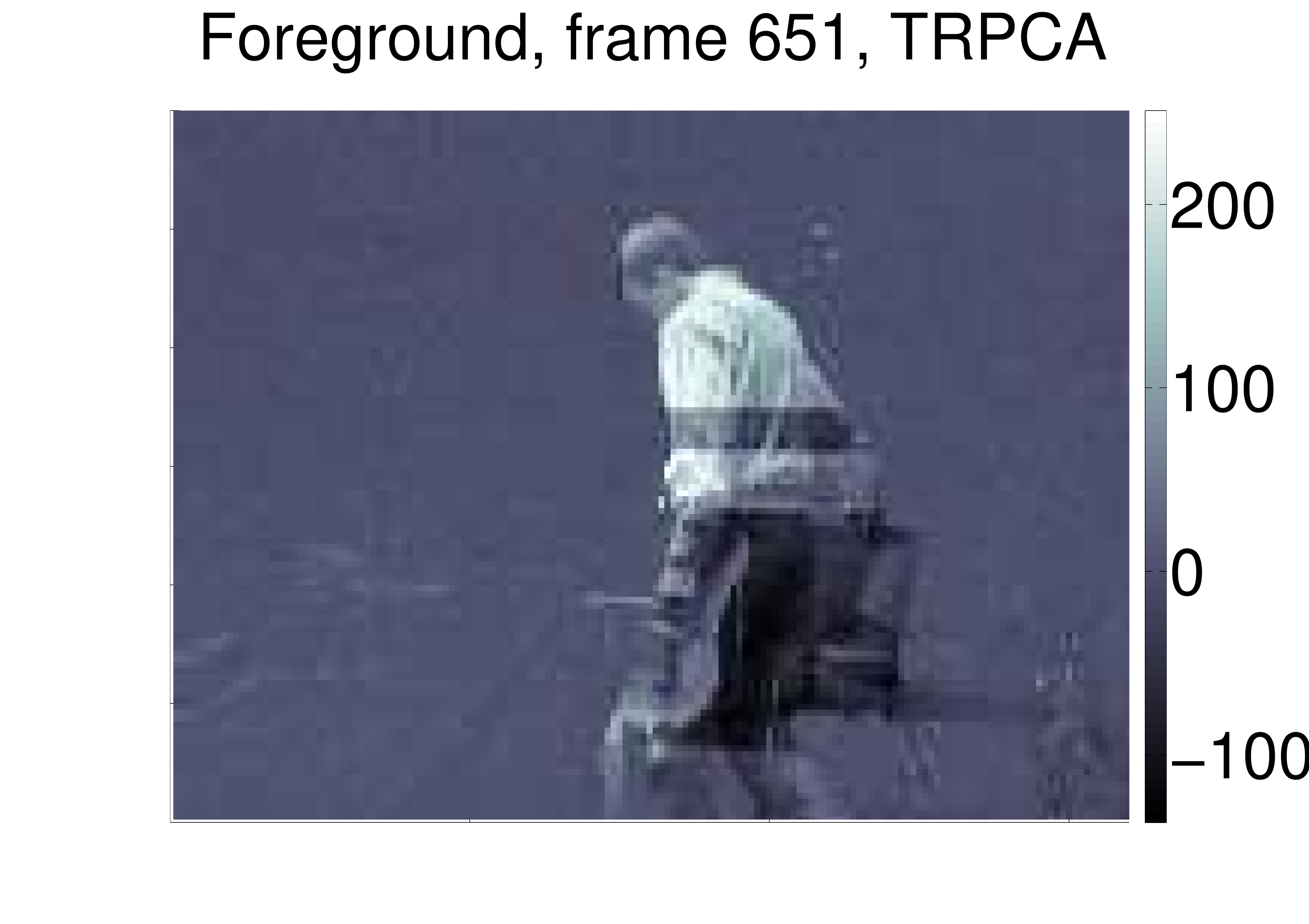} \\
\includegraphics[width=.25\columnwidth]{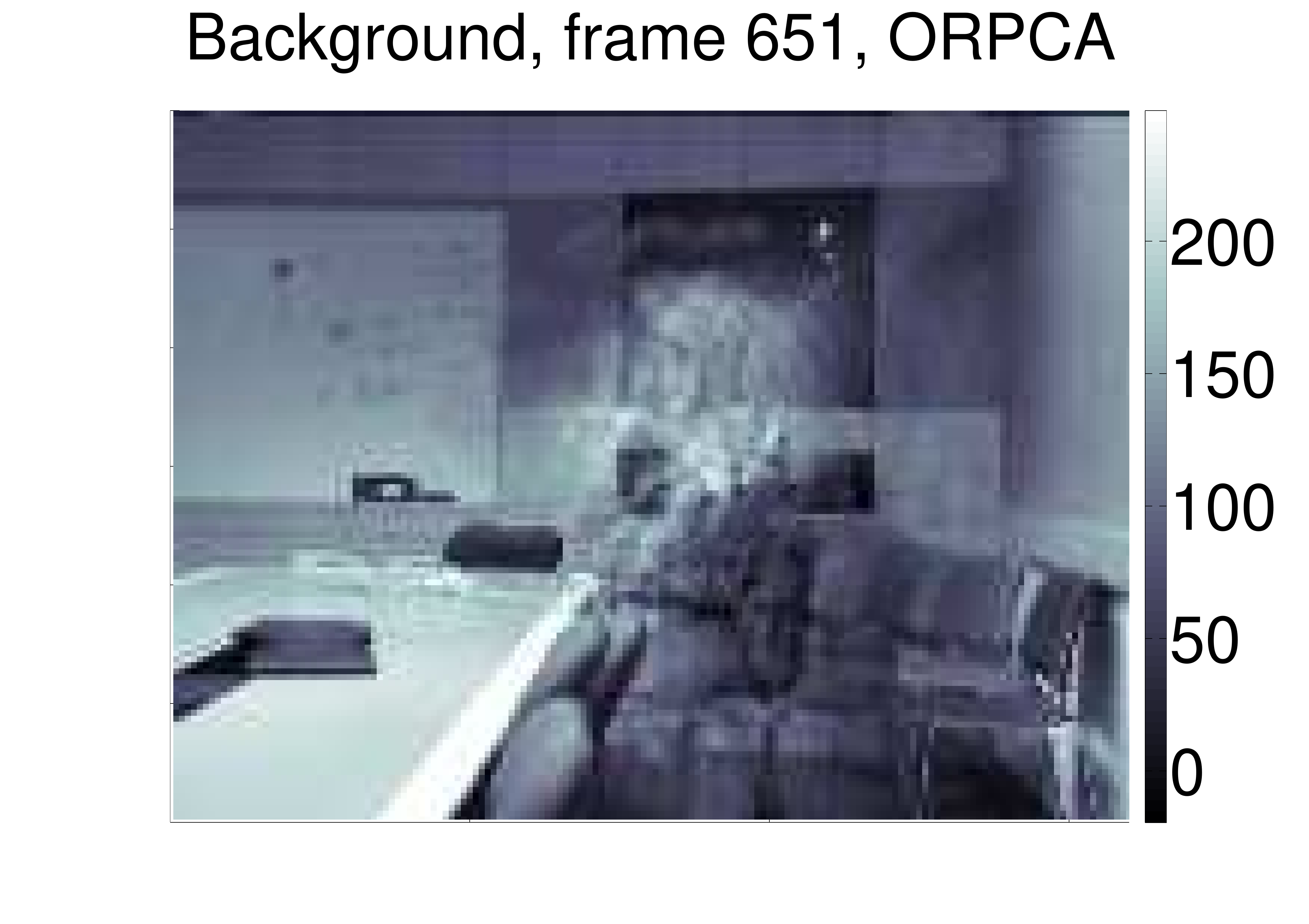} &
\includegraphics[width=.25\columnwidth]{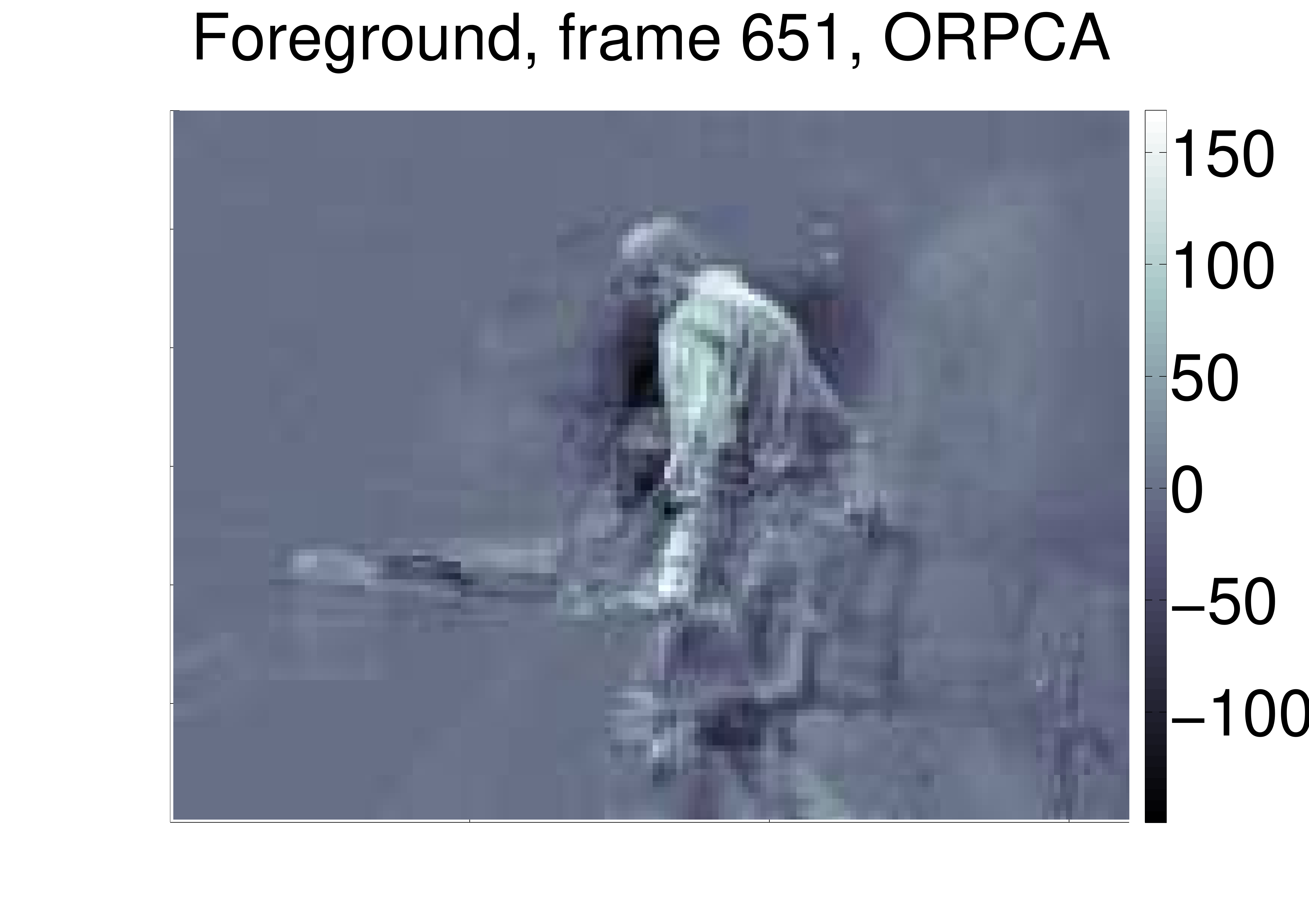} &
\includegraphics[width=.25\columnwidth]{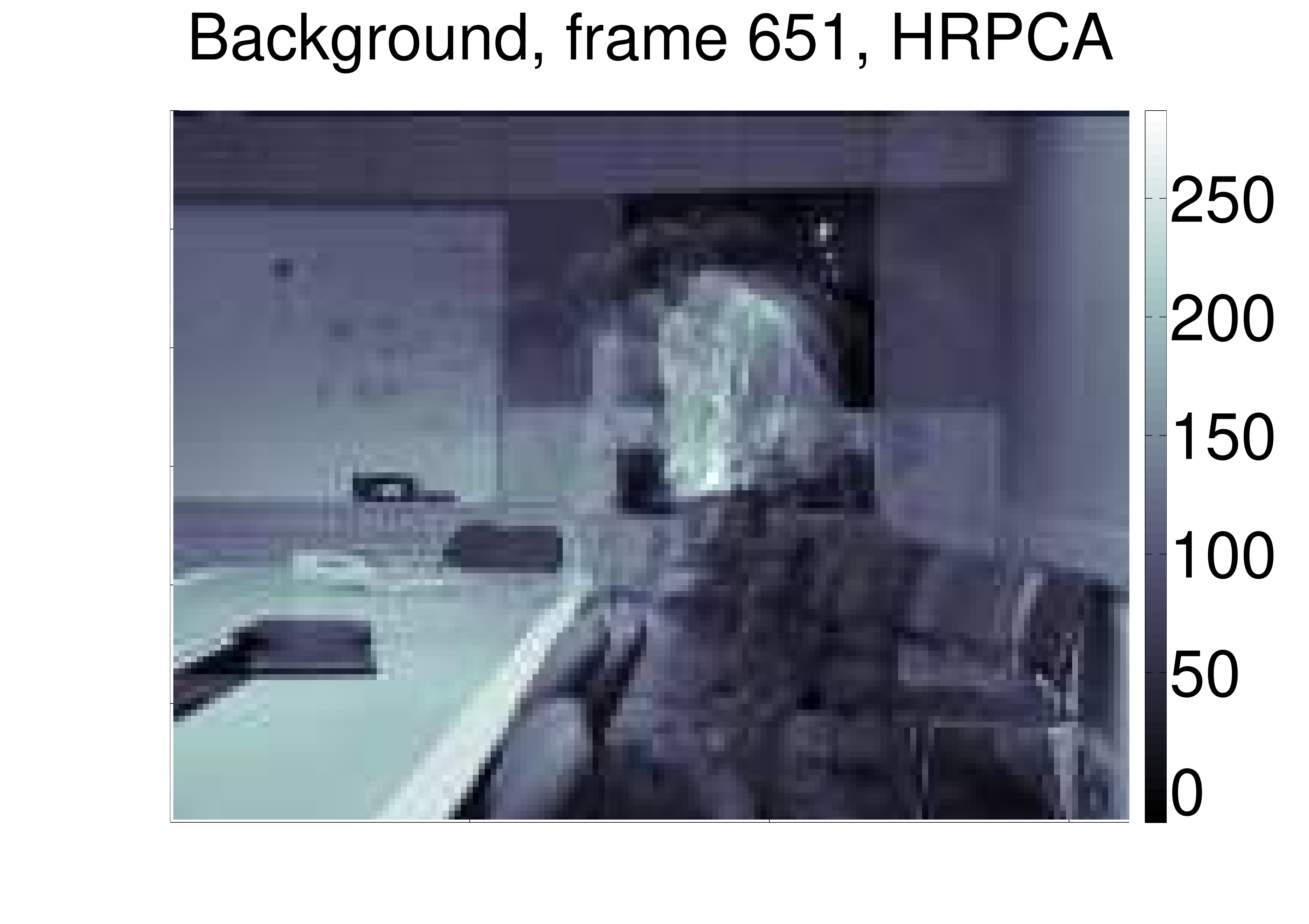} &
\includegraphics[width=.25\columnwidth]{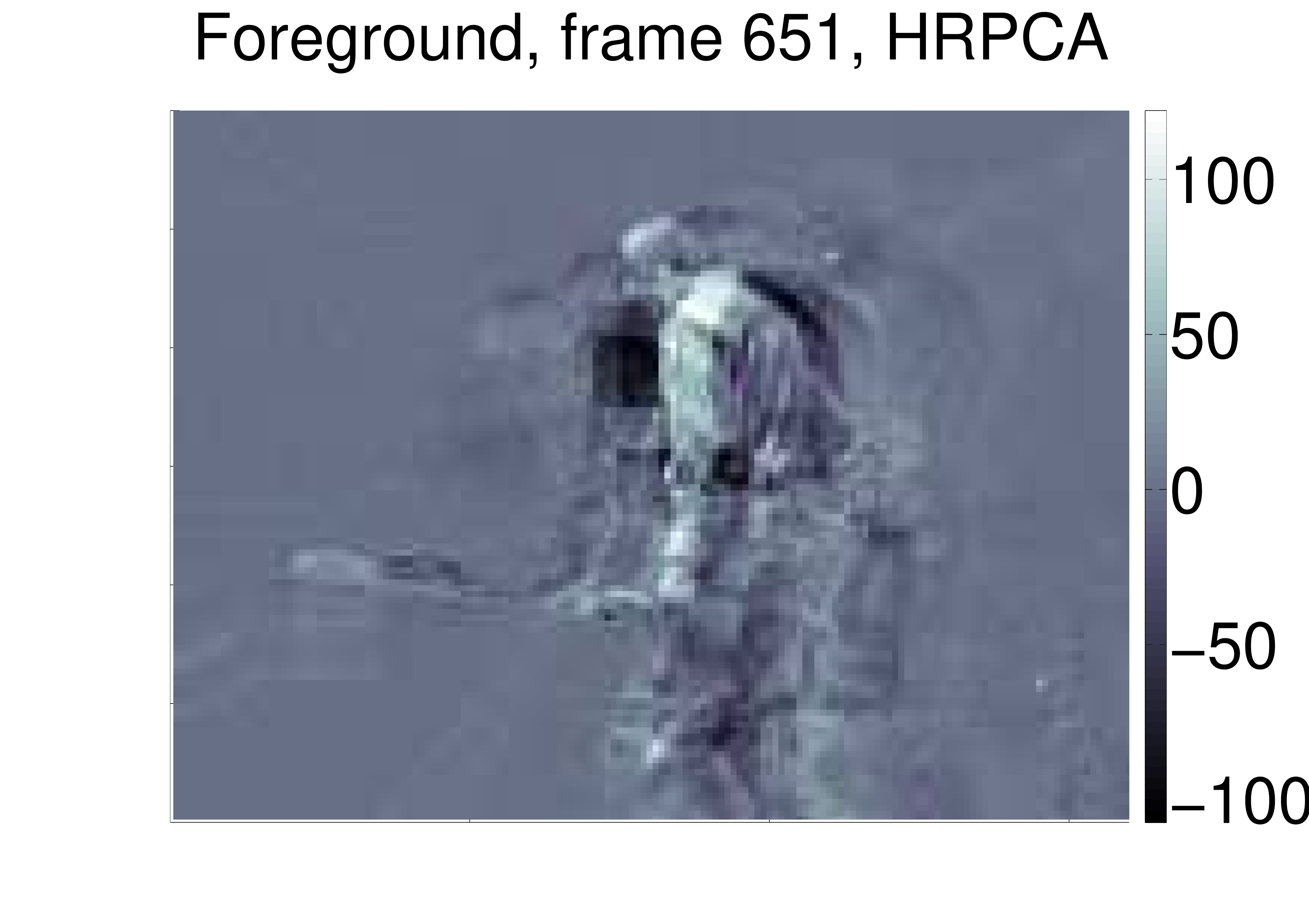} \\
\includegraphics[width=.25\columnwidth]{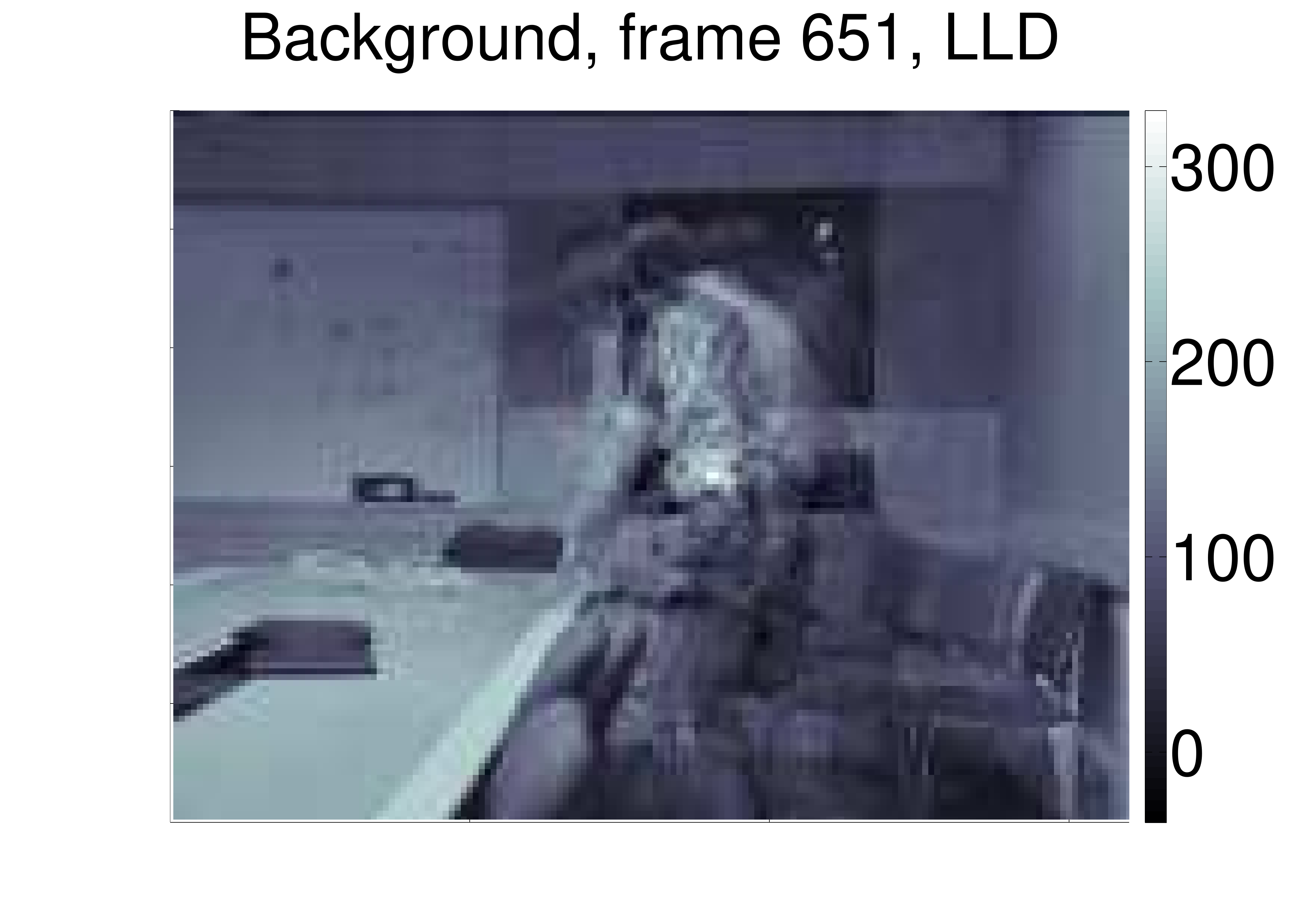} &
\includegraphics[width=.25\columnwidth]{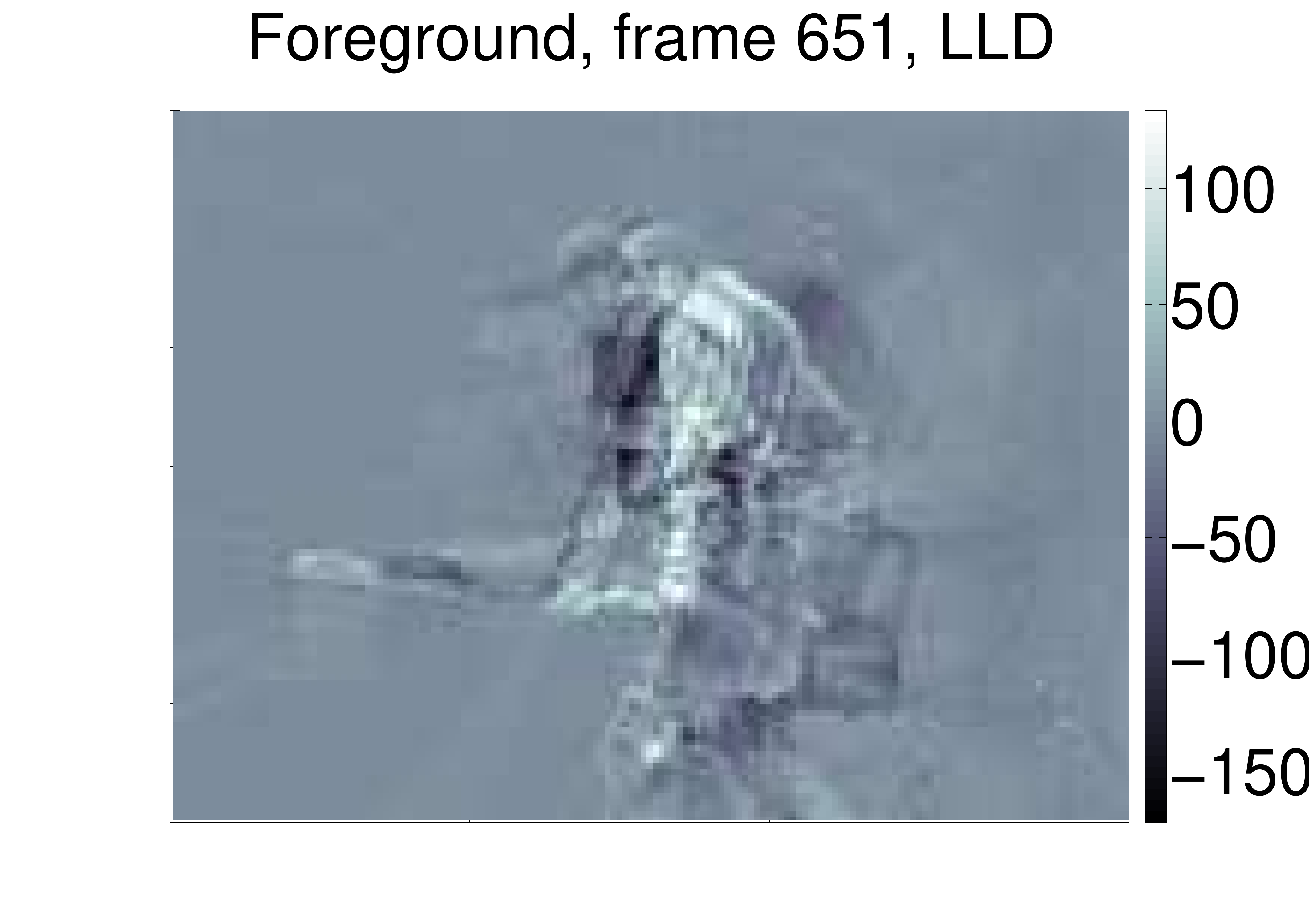} &
\includegraphics[width=.25\columnwidth]{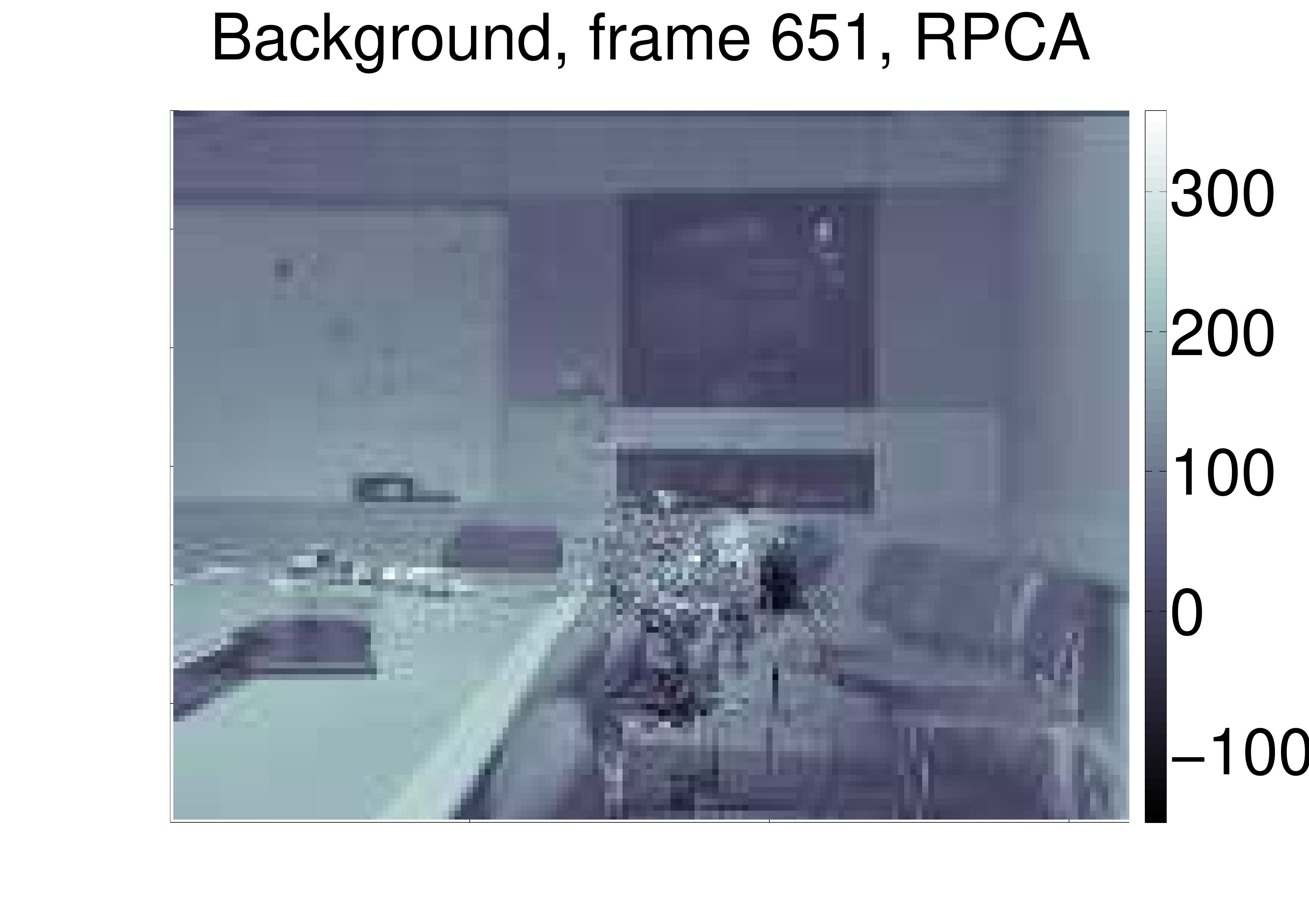} &
\includegraphics[width=.25\columnwidth]{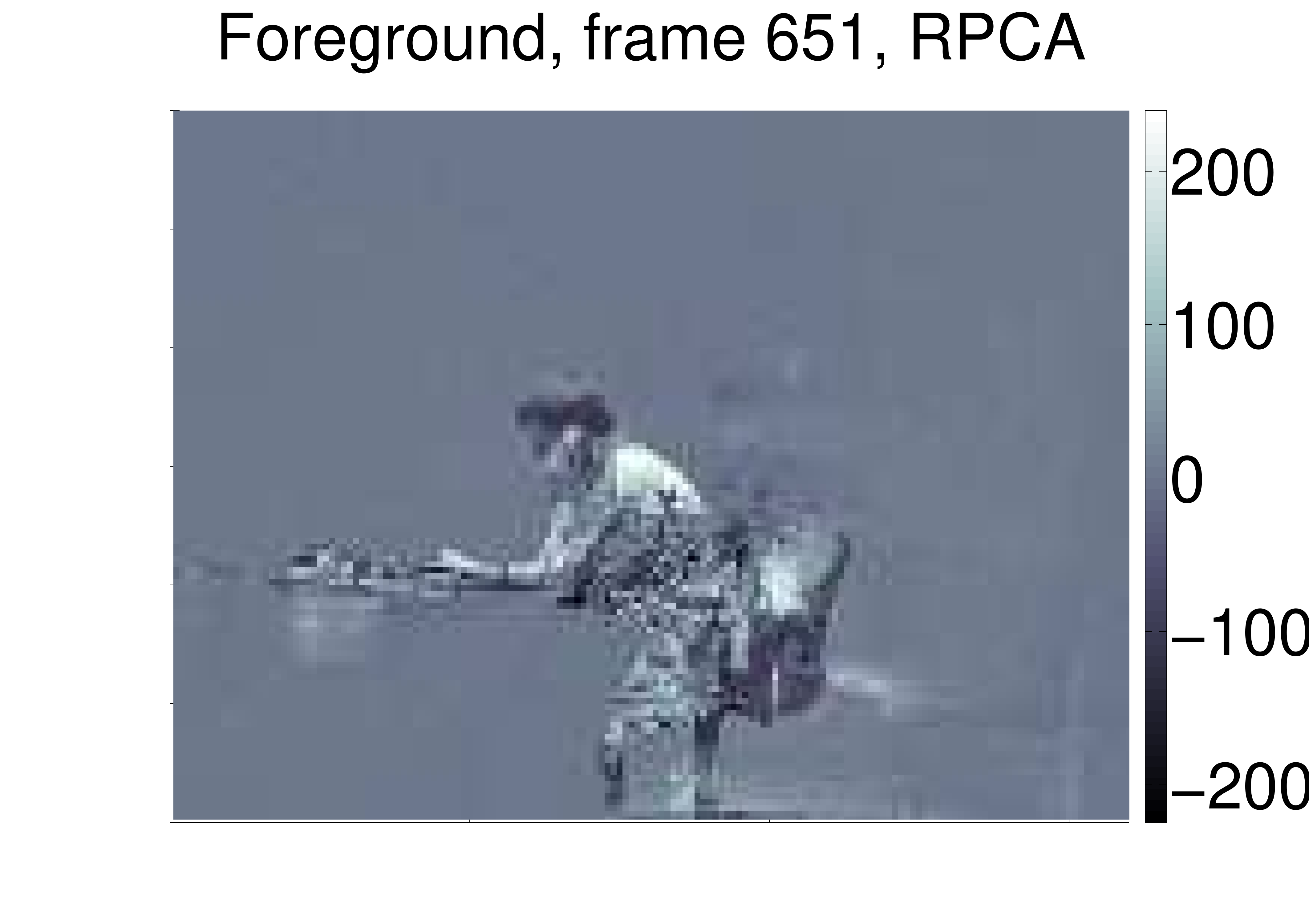}
\end{tabular}
\caption{Extracted background and foreground of frame 651 of the reduced moved object data set. The number of components is $k=10$ (unscaled, compare to scaled version in Fig.~\ref{fig:bfrmo1})
}
\label{fig:bfrmo2}
\end{figure}
\clearpage

%%%%%%%%%%%%%%%%%%%%%%%%%%%%%%%%%%%%%%%%%%%%%%%%%%%%
\begin{figure}
\begin{tabular}{cccc}
\includegraphics[width=.25\columnwidth]{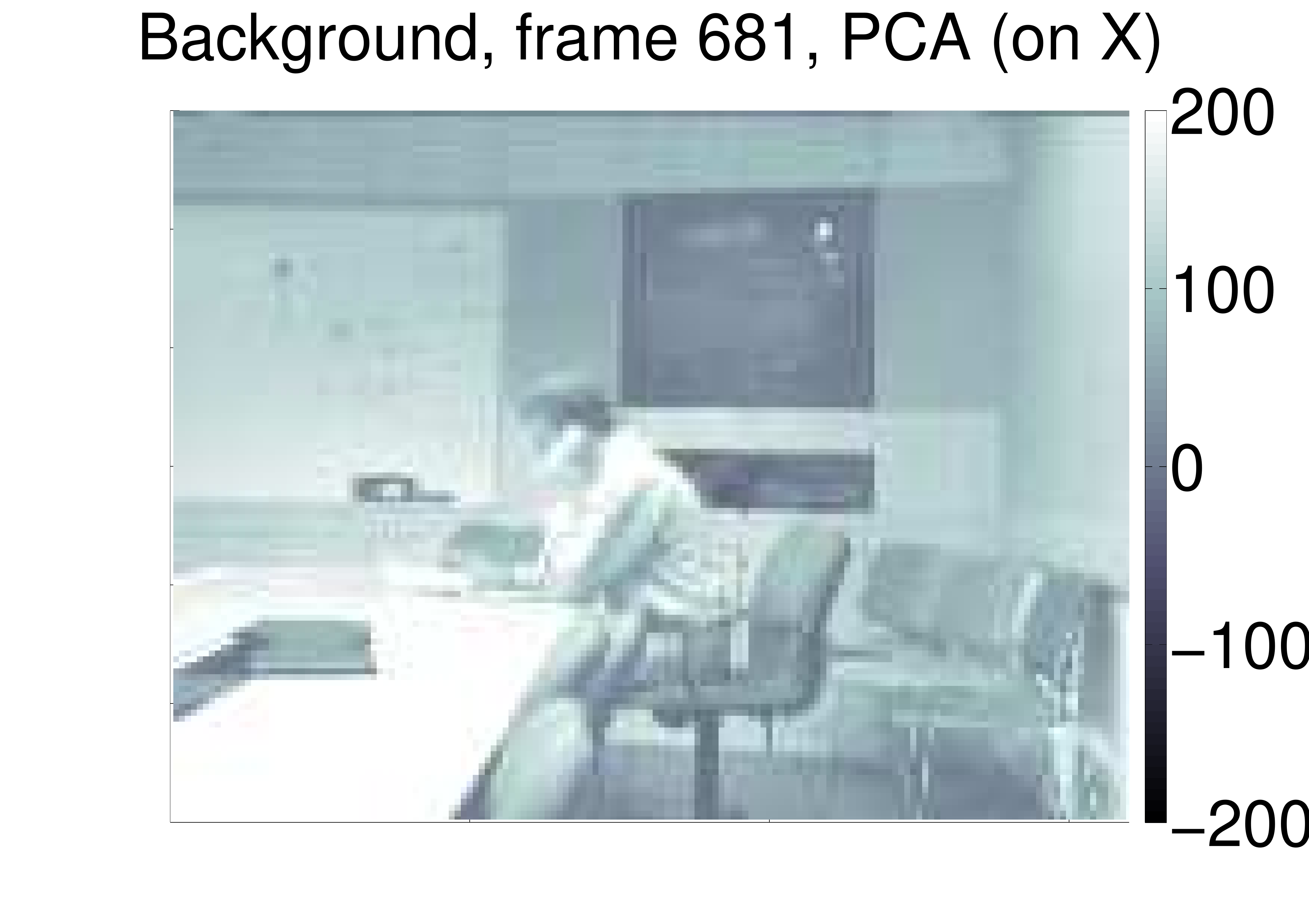} &
\includegraphics[width=.25\columnwidth]{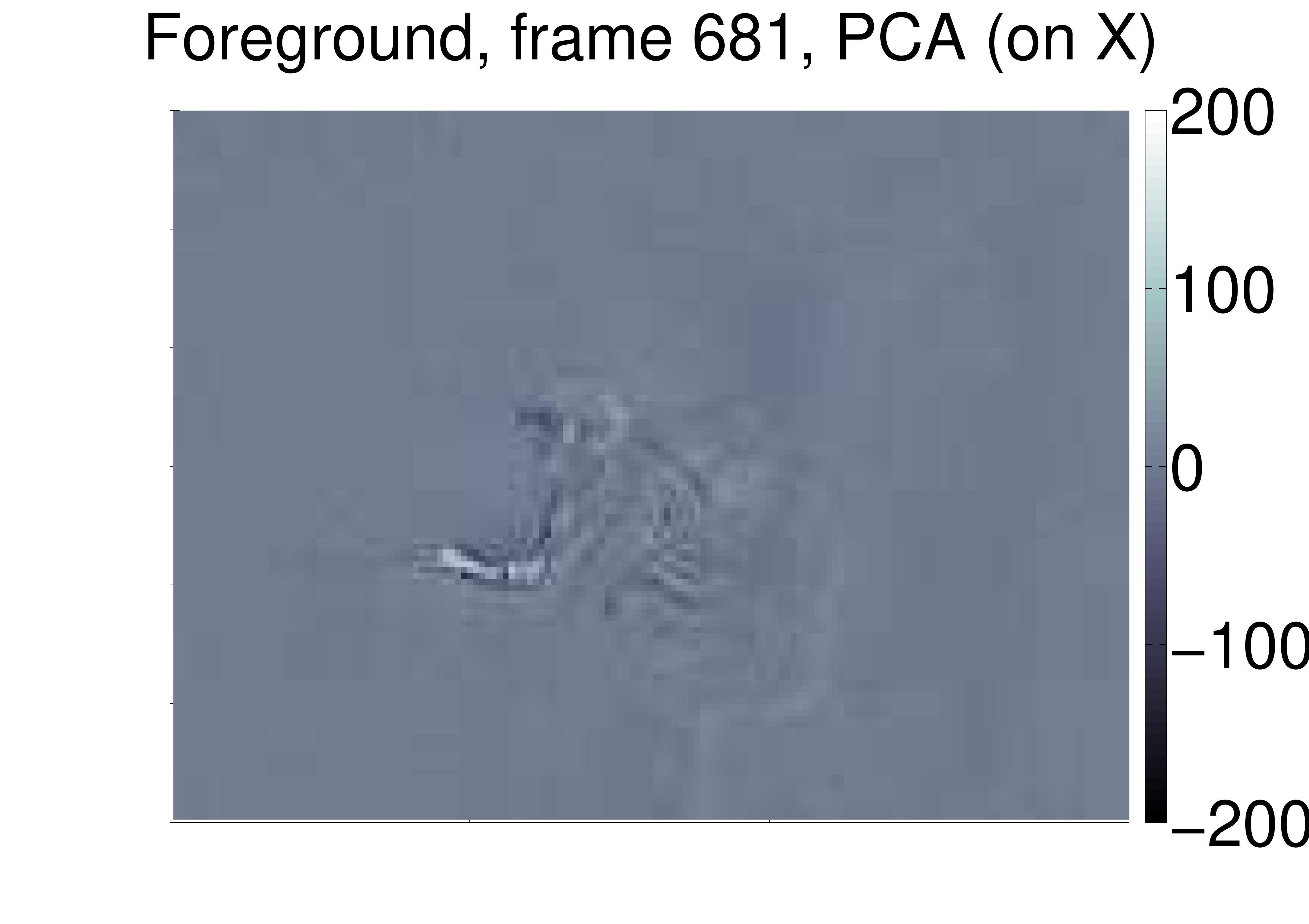} &
\includegraphics[width=.25\columnwidth]{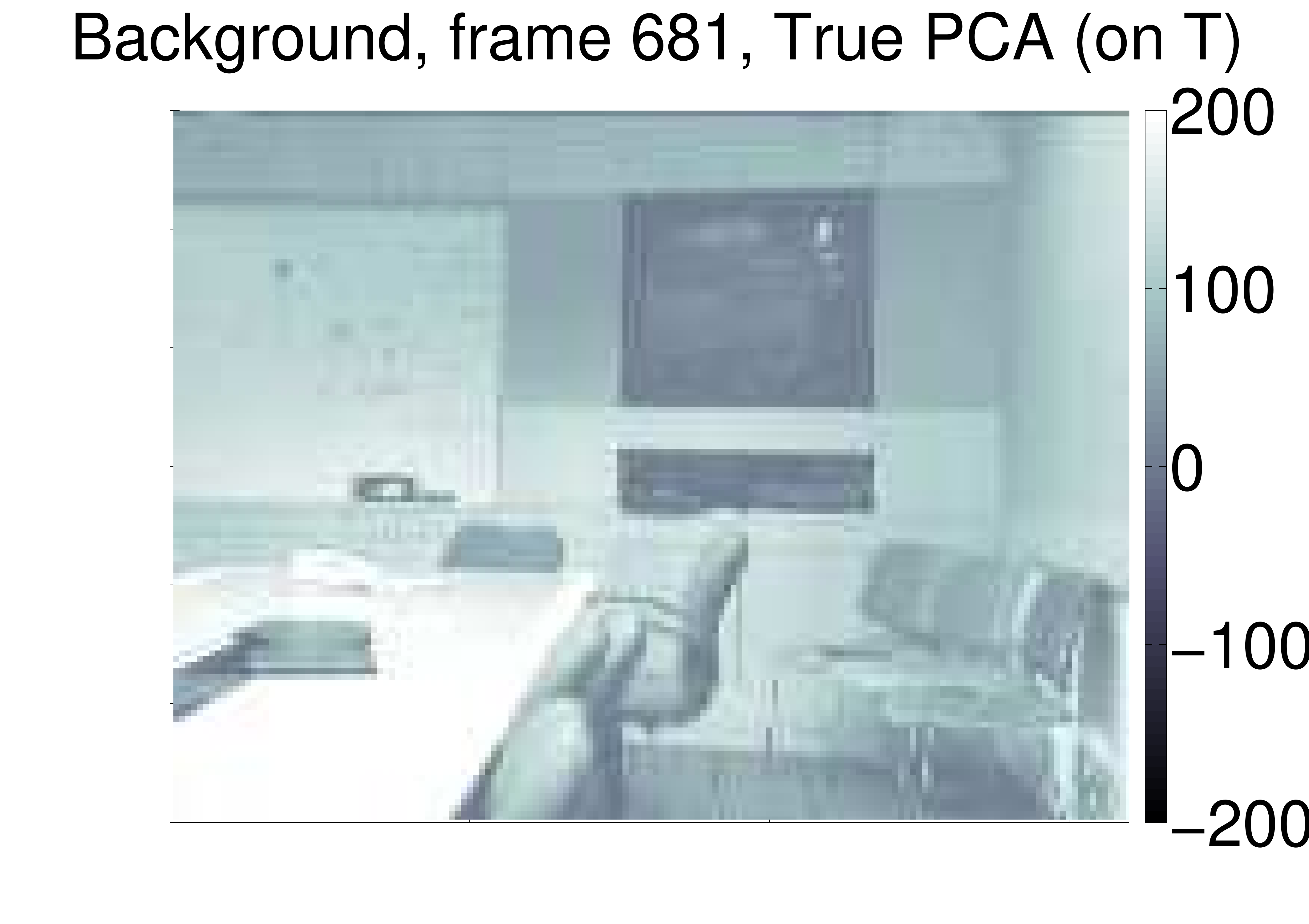} &
\includegraphics[width=.25\columnwidth]{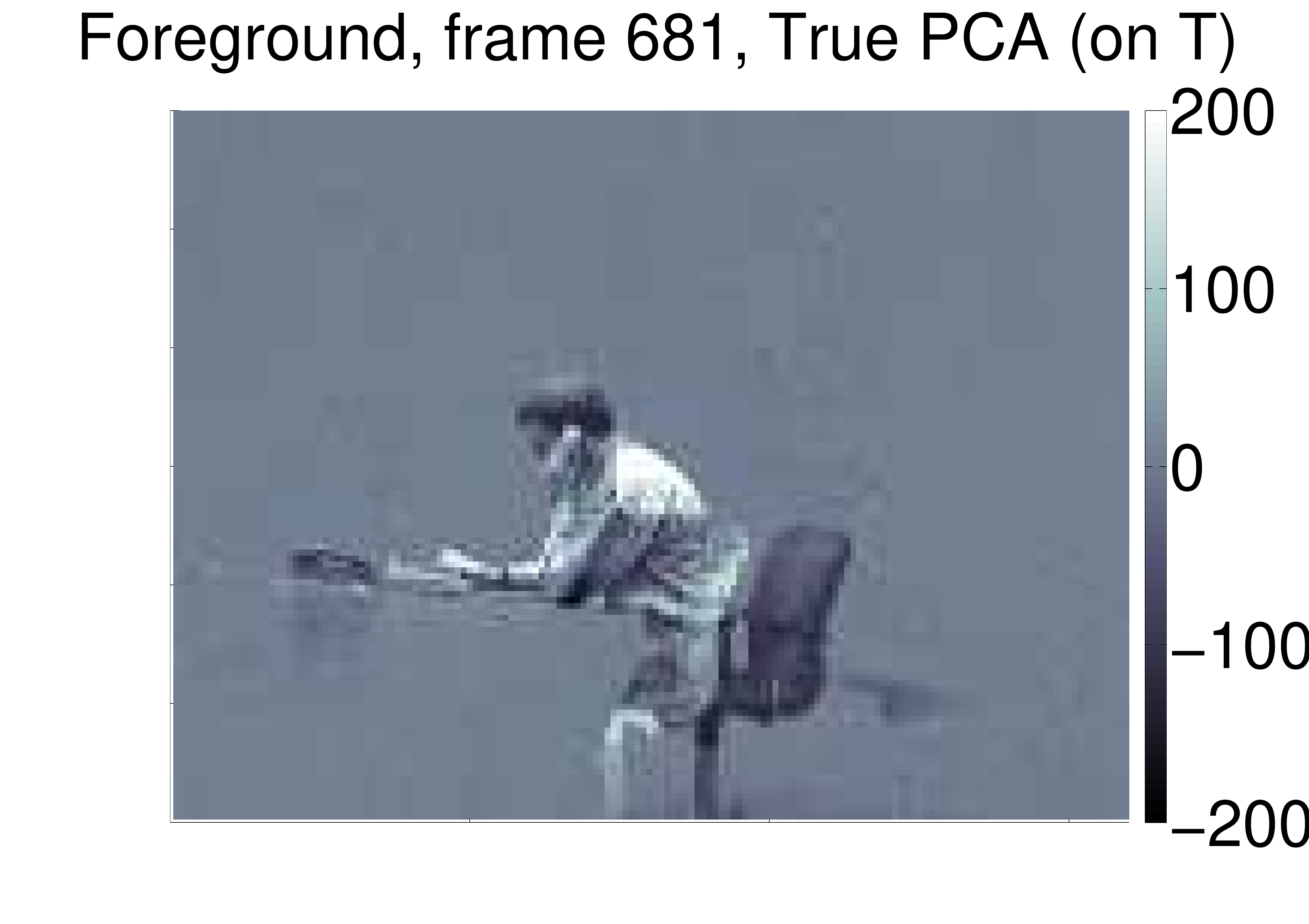} \\ \hline
\includegraphics[width=.25\columnwidth]{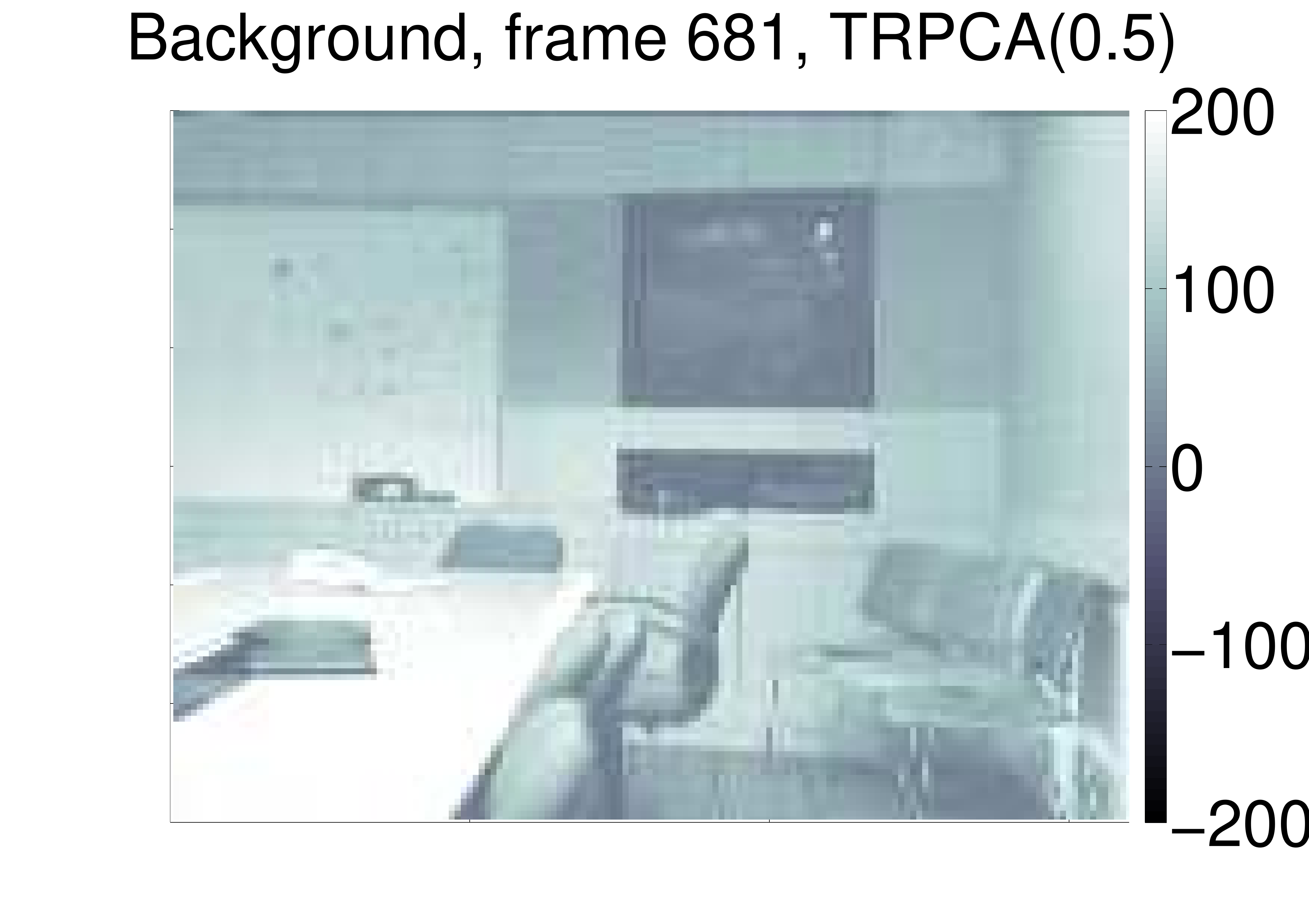} &
\includegraphics[width=.25\columnwidth]{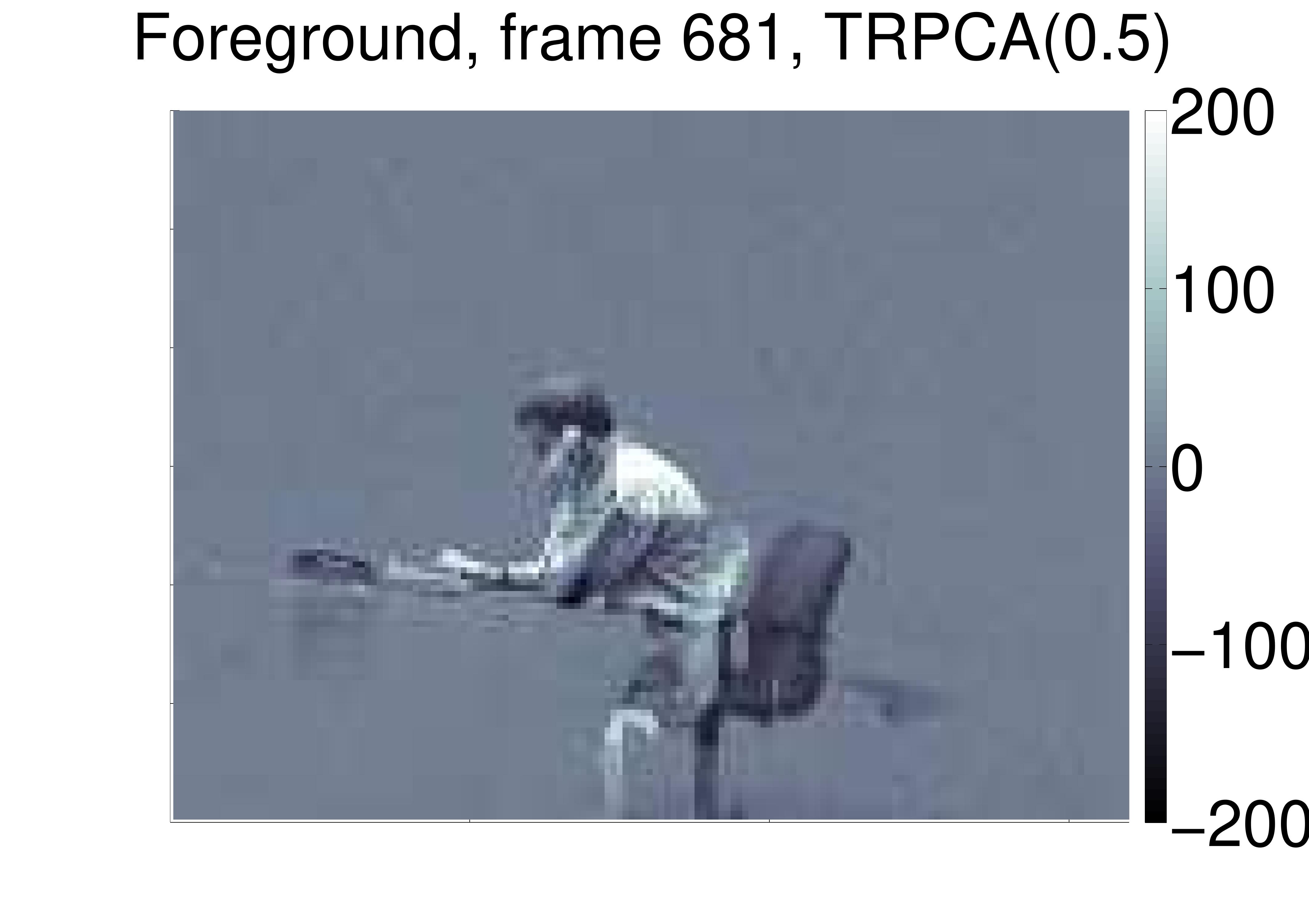} &
\includegraphics[width=.25\columnwidth]{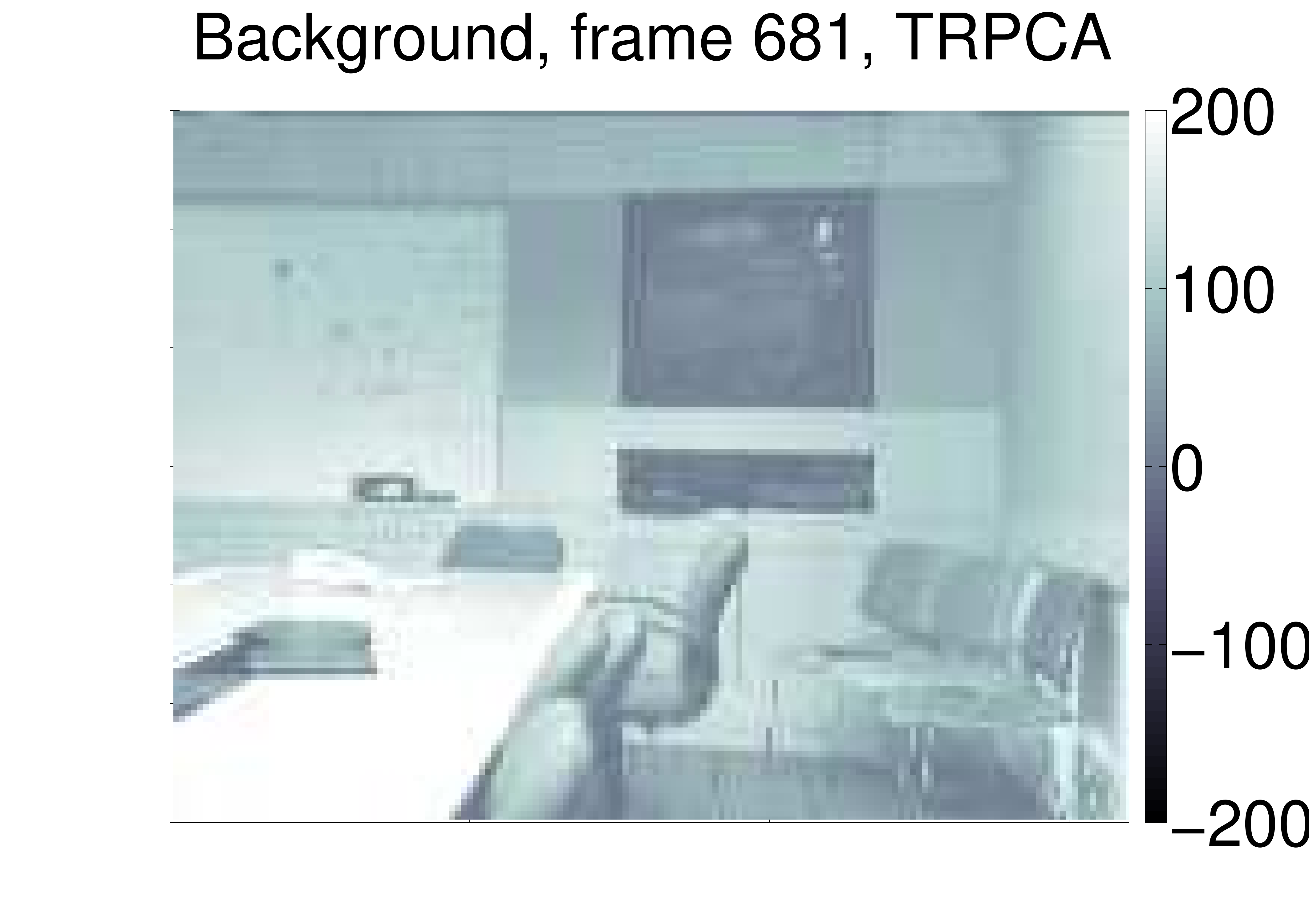} &
\includegraphics[width=.25\columnwidth]{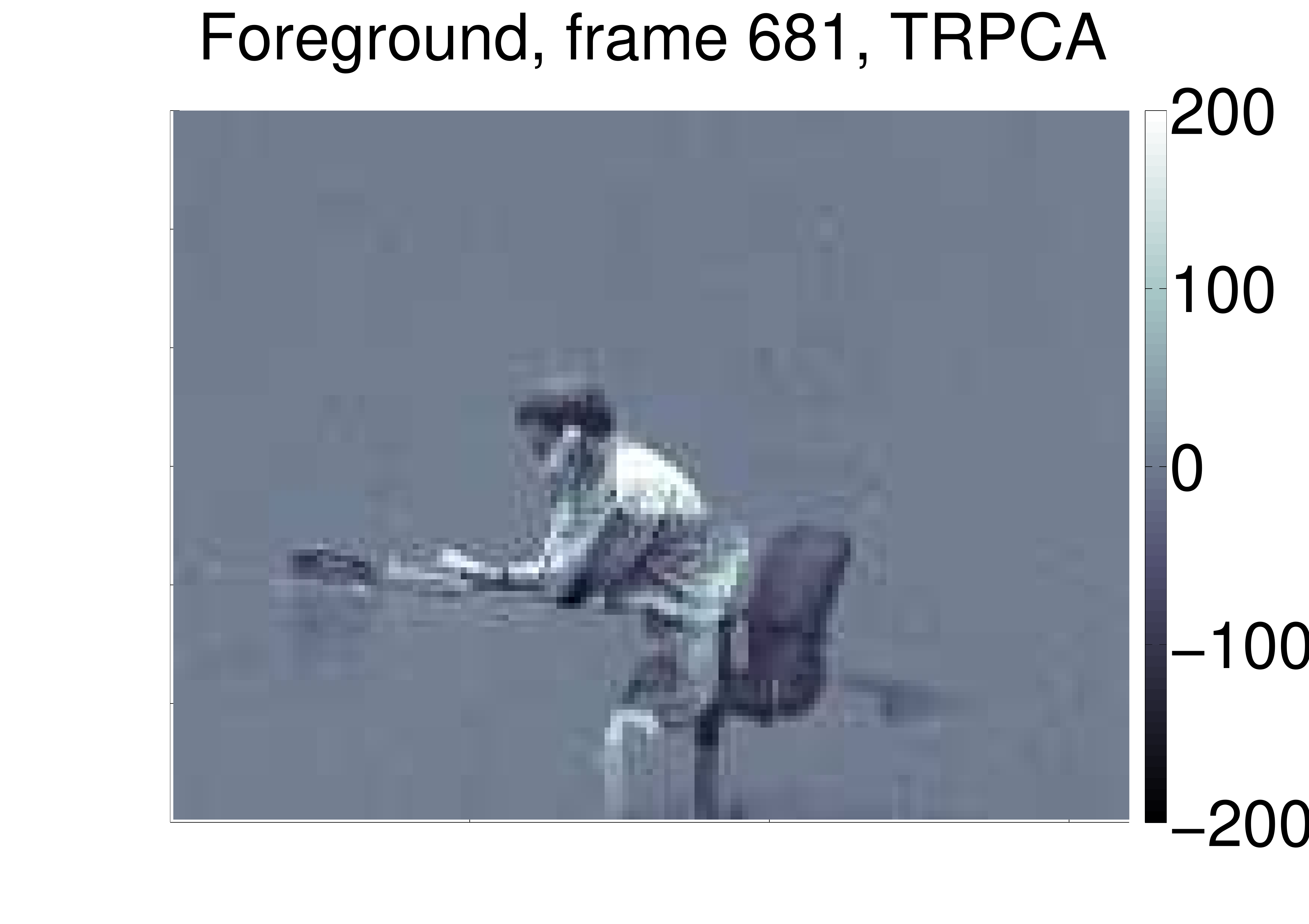} \\
\includegraphics[width=.25\columnwidth]{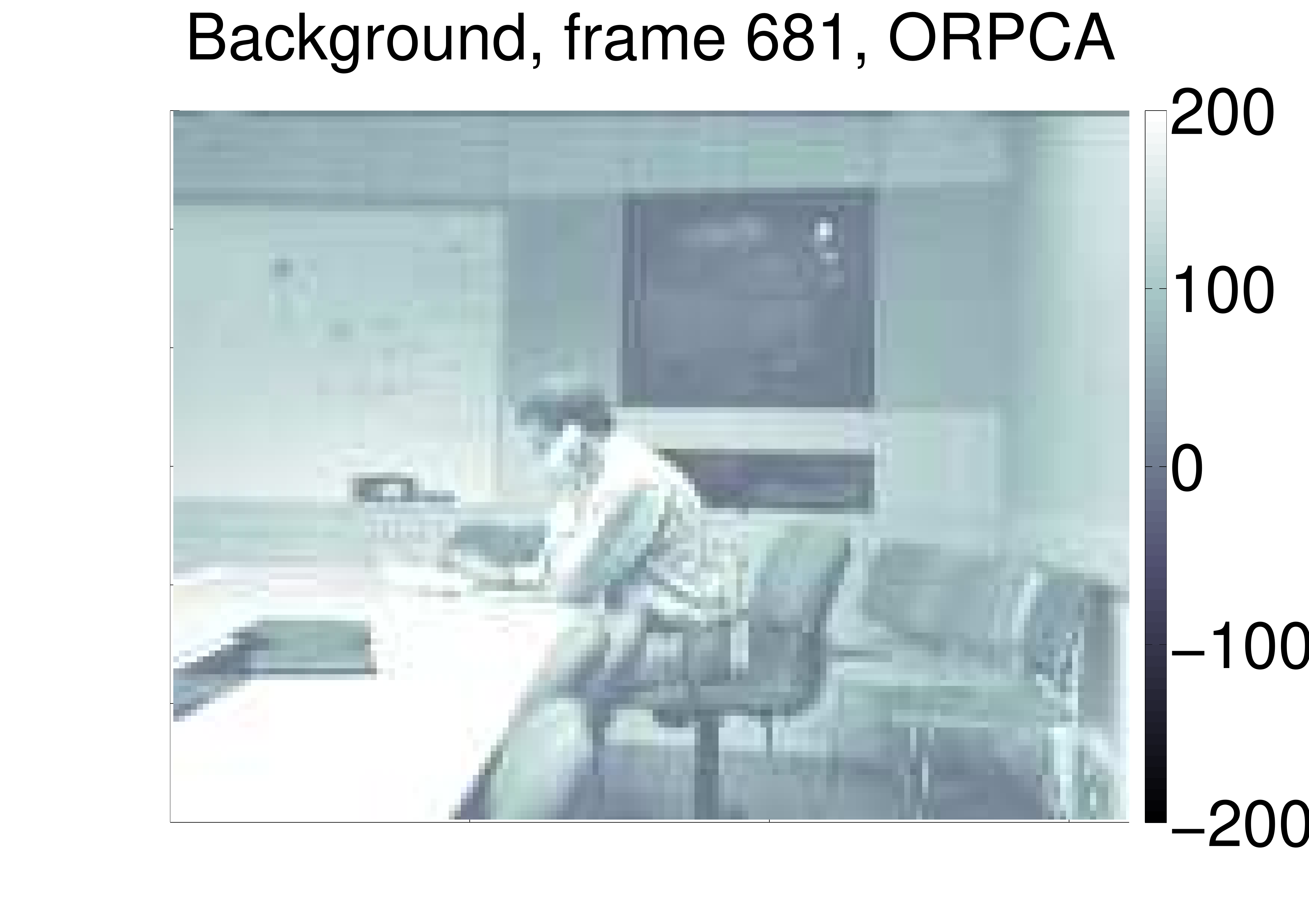} &
\includegraphics[width=.25\columnwidth]{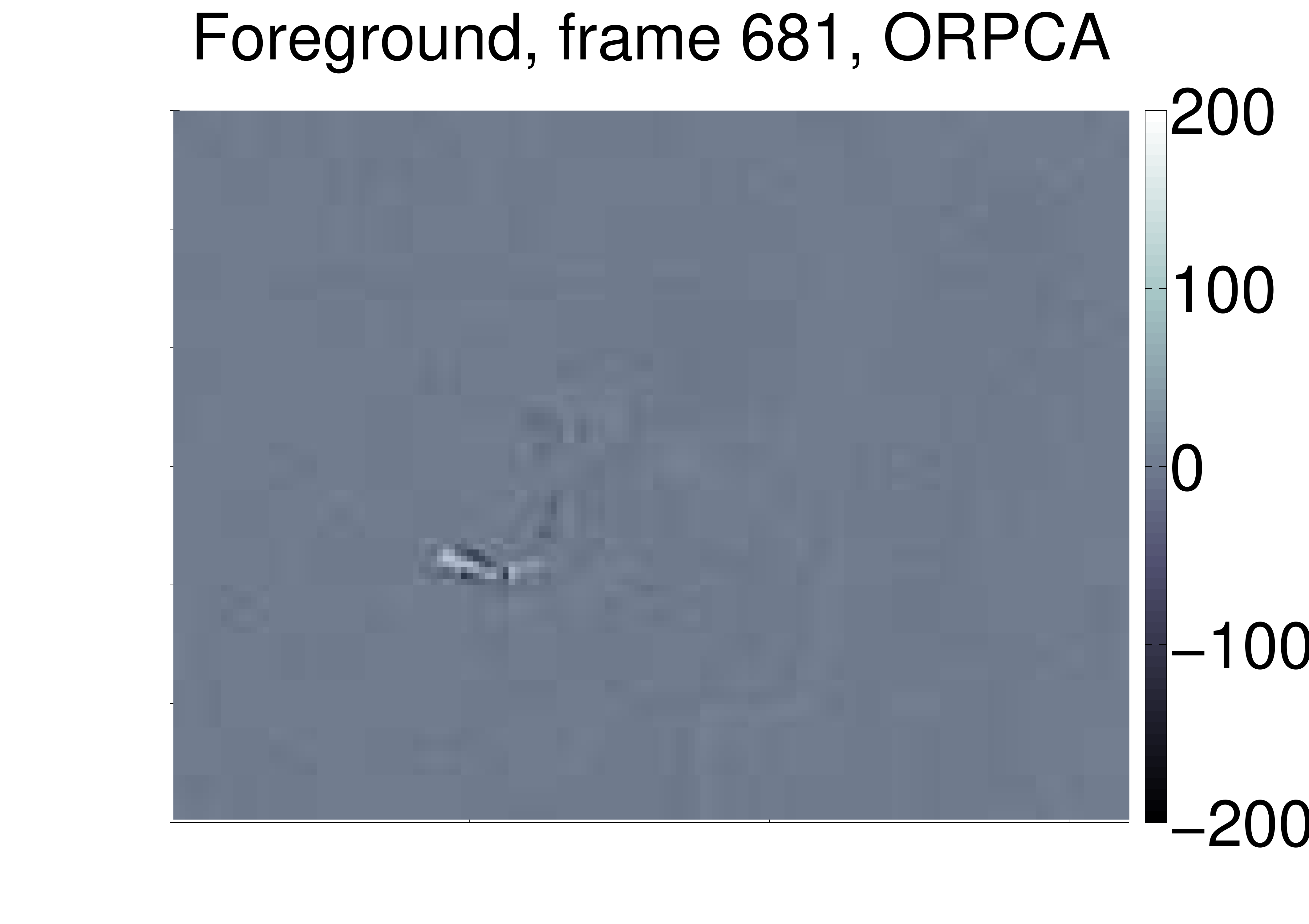} &
\includegraphics[width=.25\columnwidth]{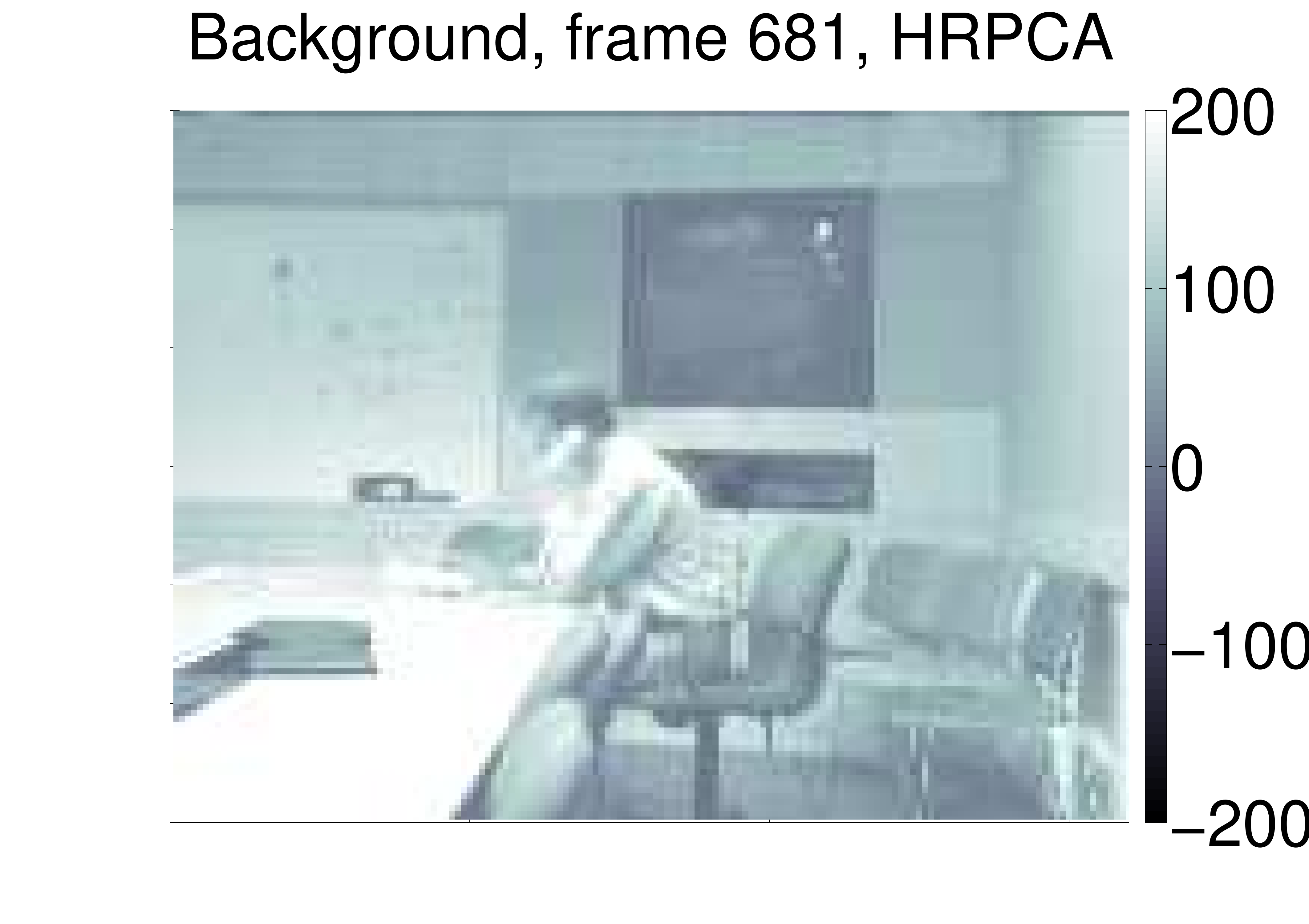} &
\includegraphics[width=.25\columnwidth]{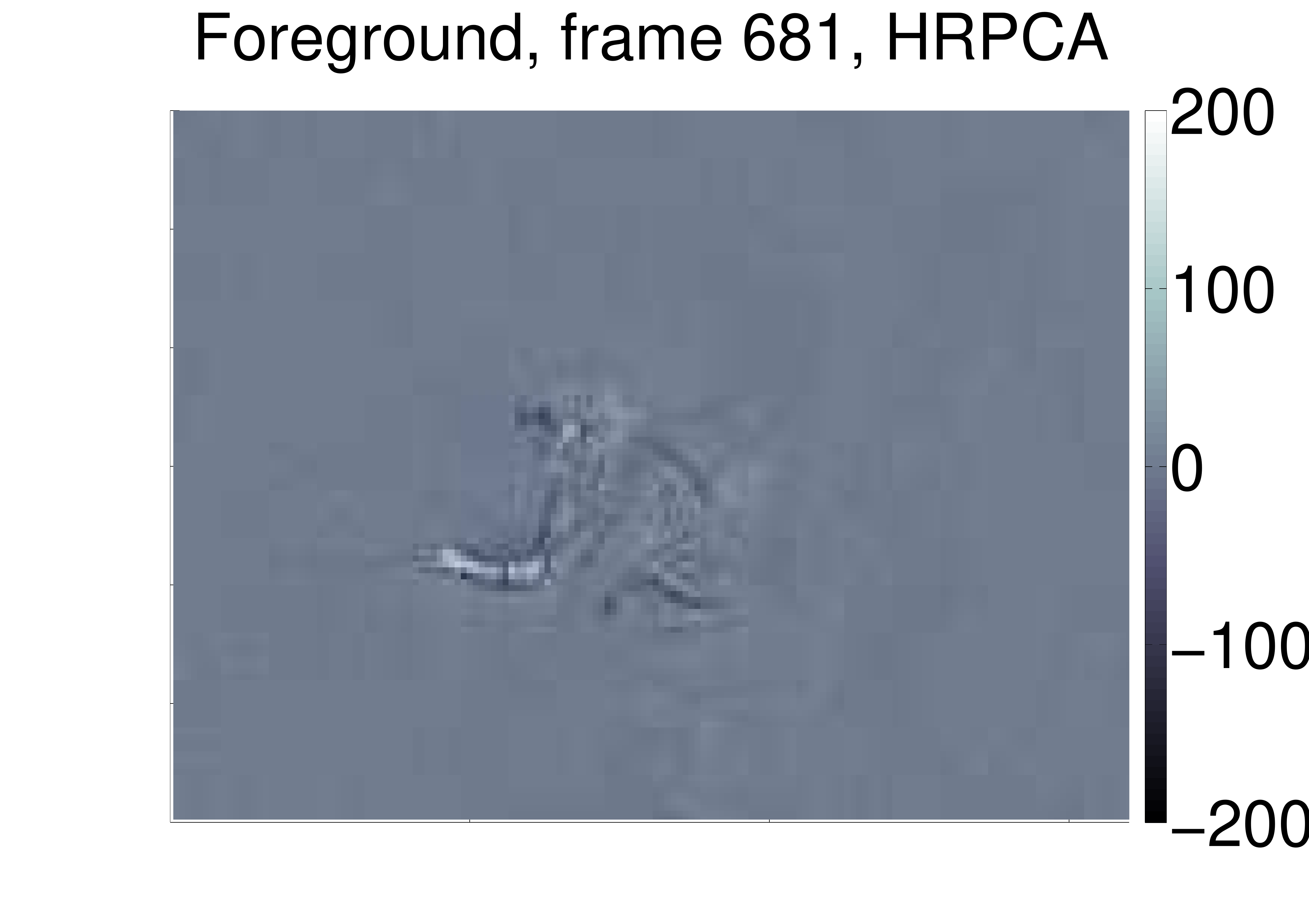} \\
\includegraphics[width=.25\columnwidth]{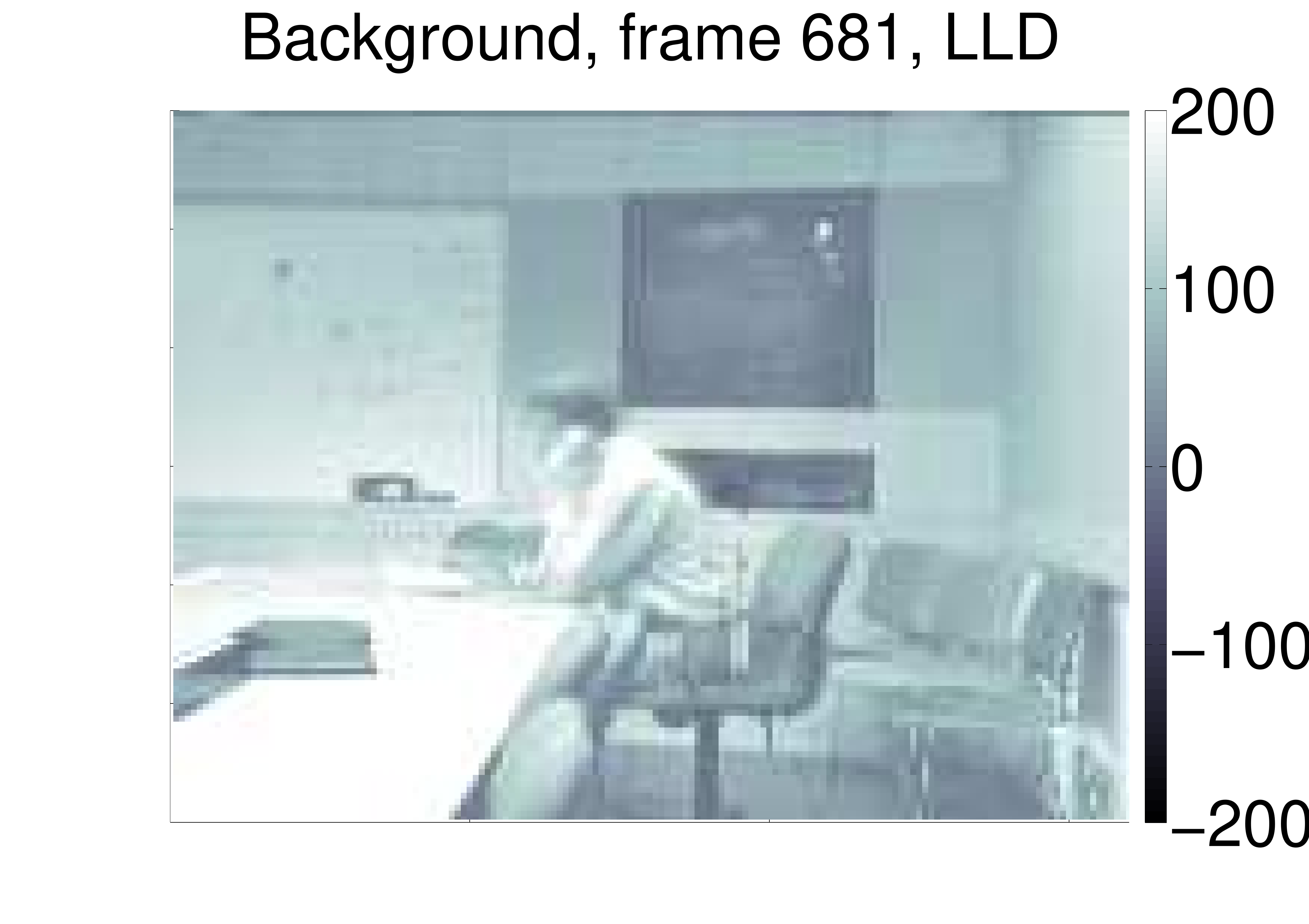} &
\includegraphics[width=.25\columnwidth]{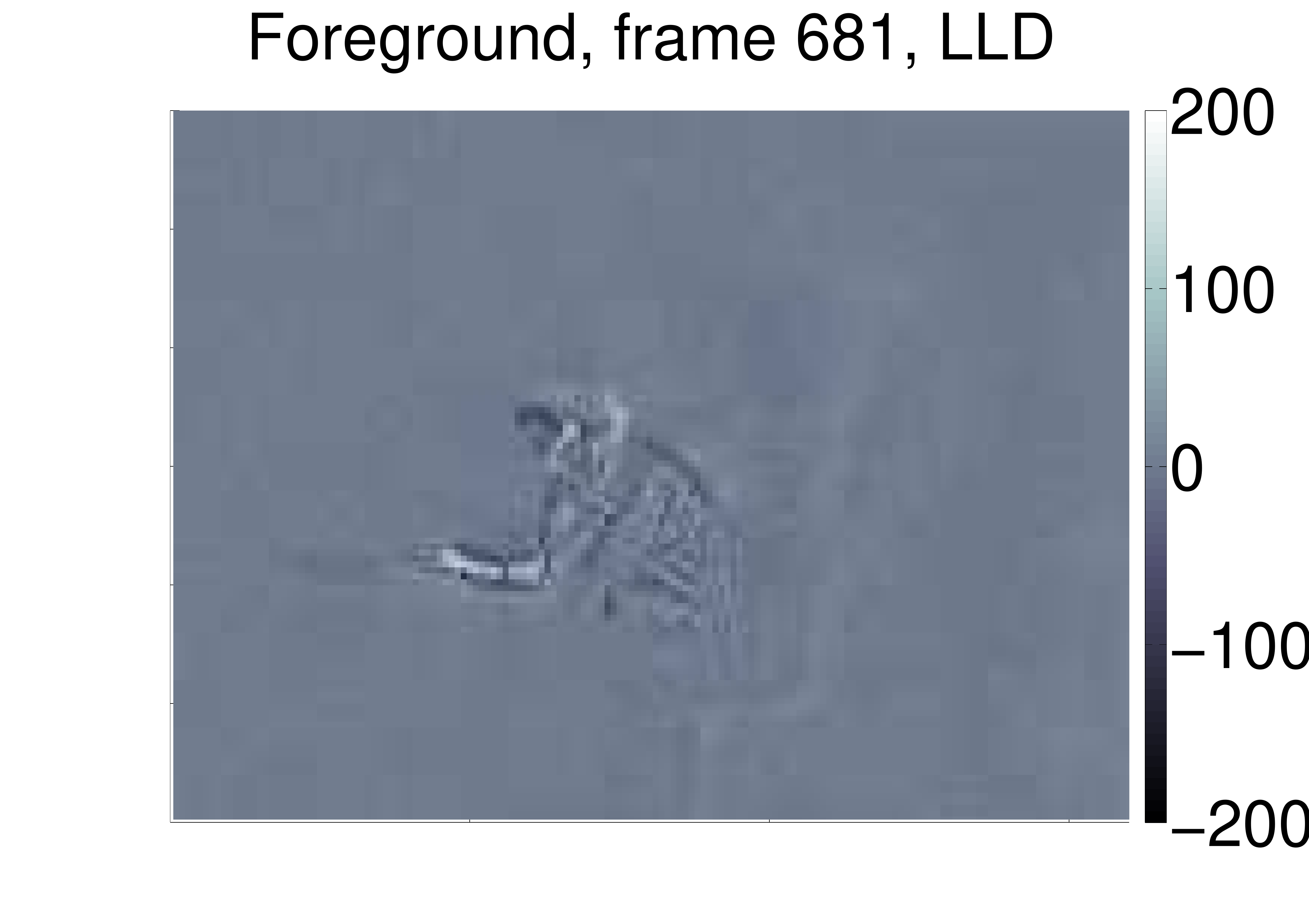} &
\includegraphics[width=.25\columnwidth]{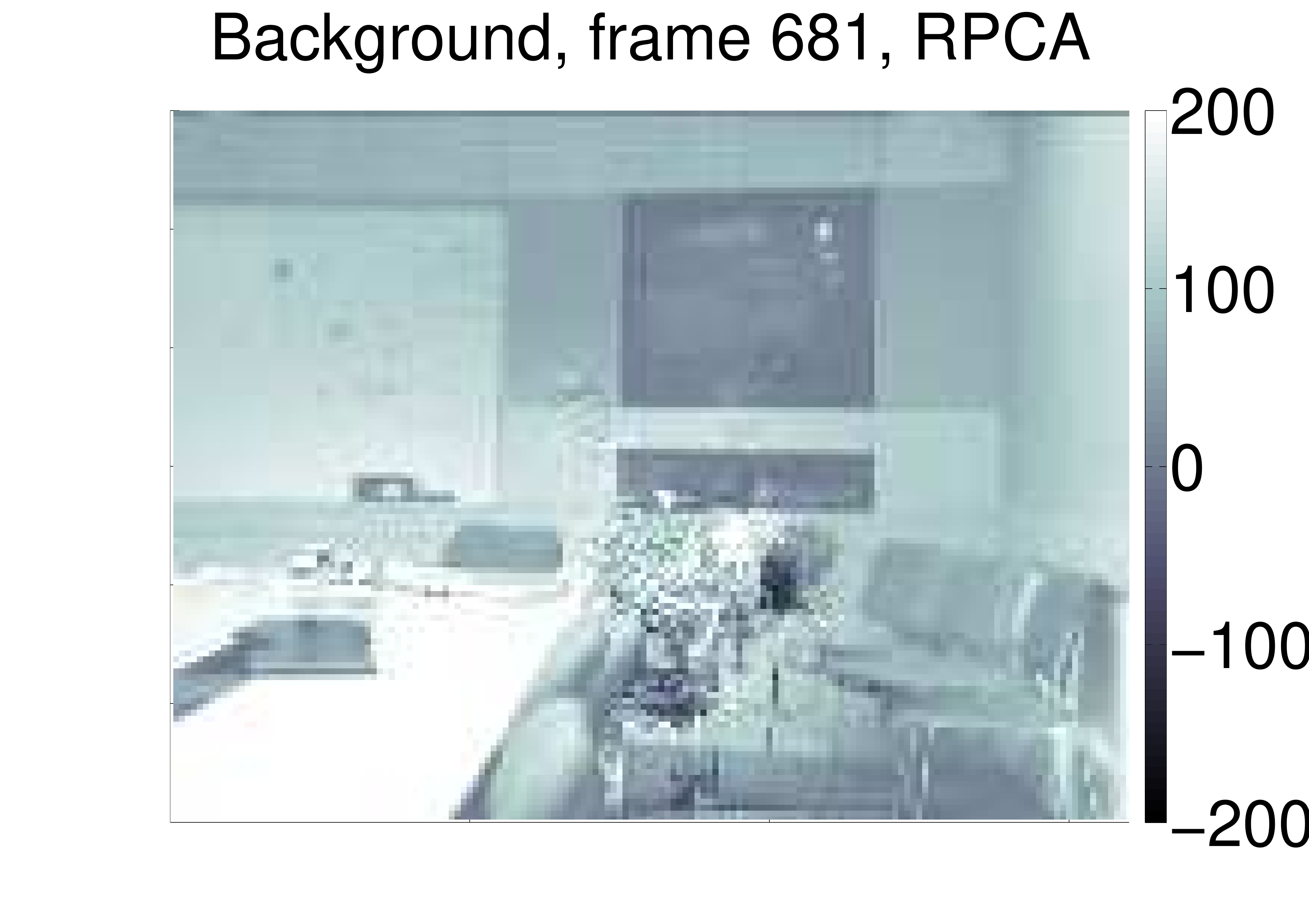} &
\includegraphics[width=.25\columnwidth]{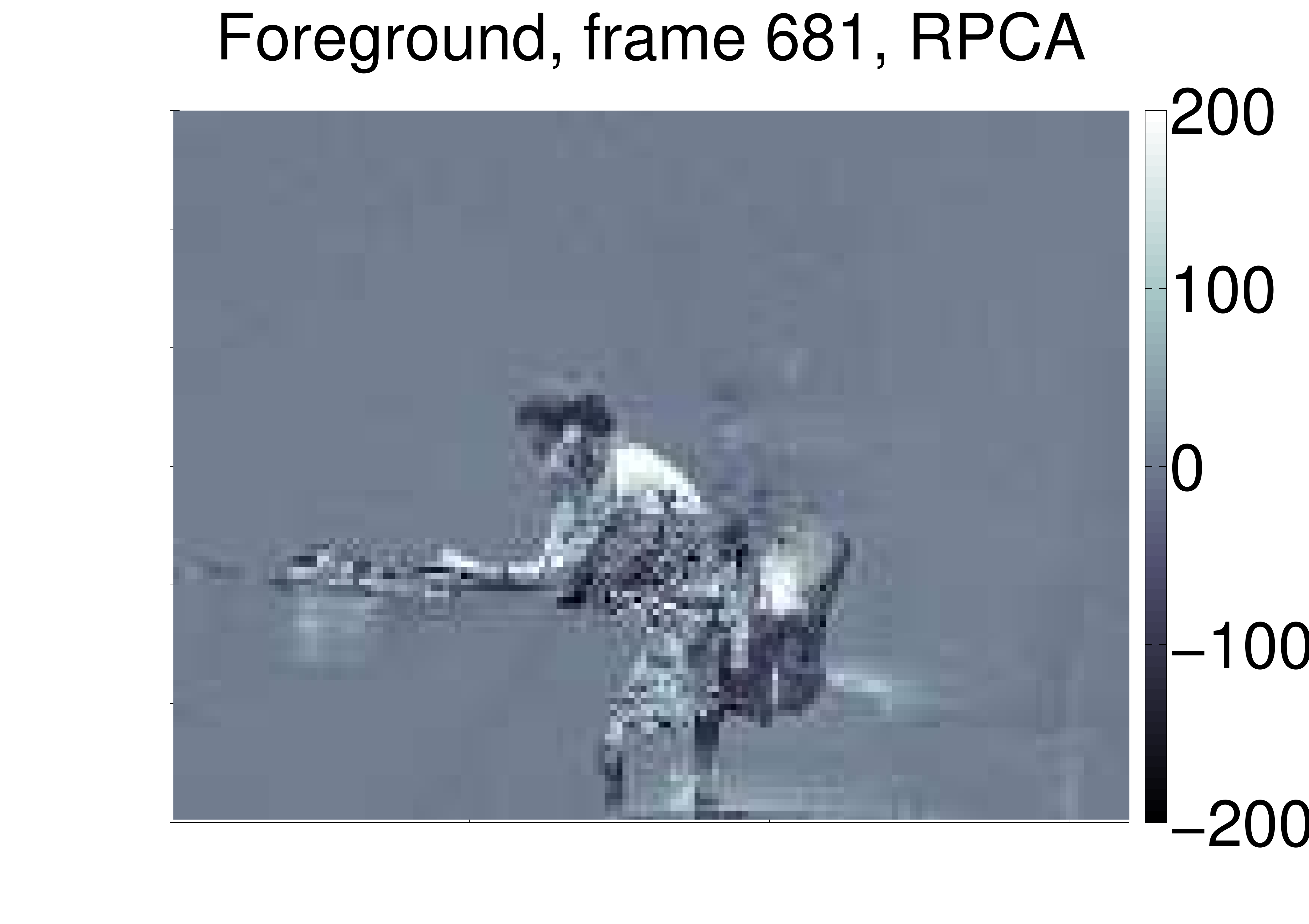}
\end{tabular}
\caption{Extracted background and foreground of frame 681 of the reduced moved object data set. The number of components is $k=10$ (scaled, compare to unscaled version in Fig.~\ref{fig:bfrmo4})
}
\label{fig:bfrmo3}
\end{figure}
\clearpage

%%%%%%%%%%%%%%%%%%%%%%%%%%%%%%%%%%%%%%%%%%%%%%%%%%%%
\begin{figure}
\begin{tabular}{cccc}
\includegraphics[width=.25\columnwidth]{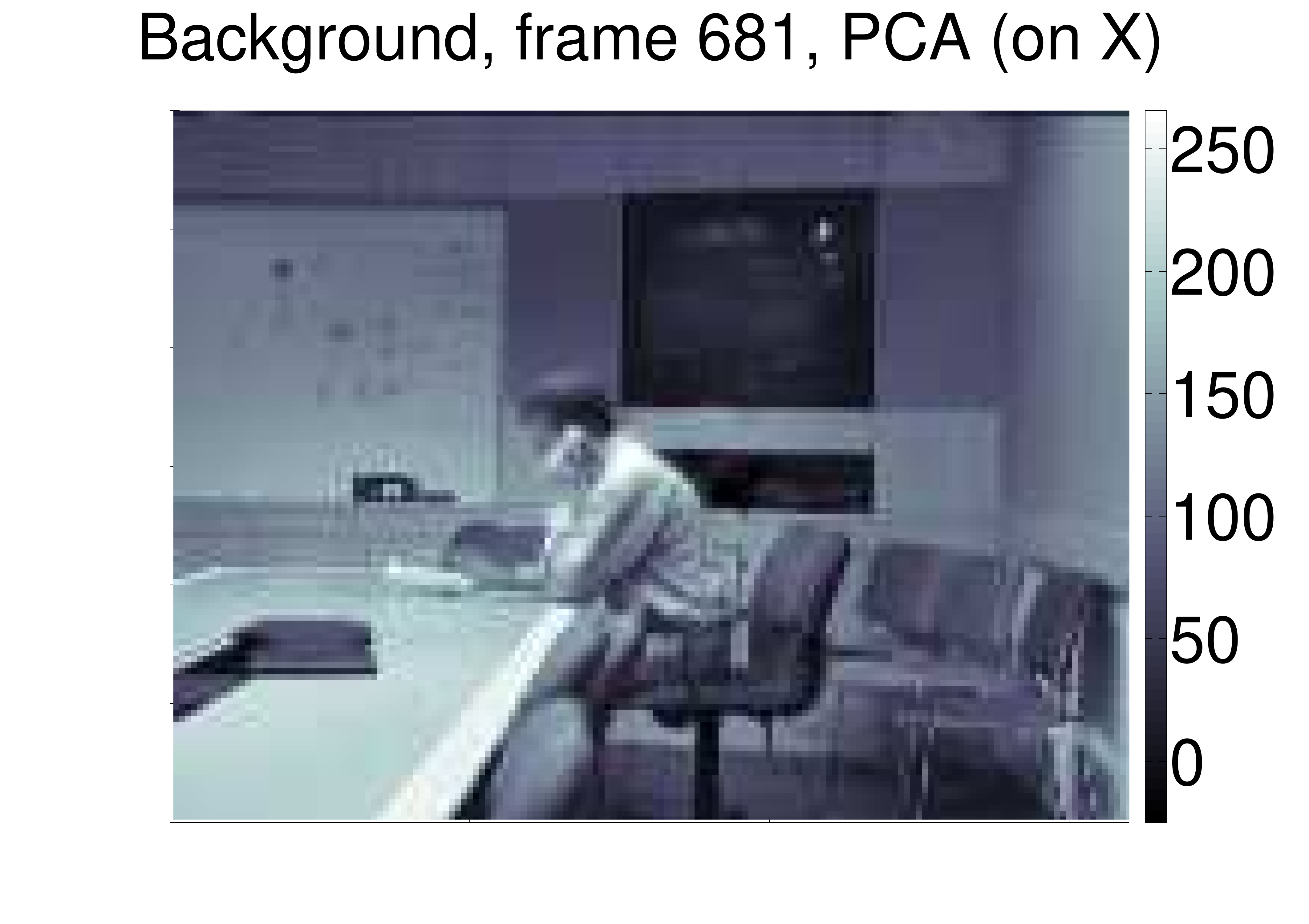} &
\includegraphics[width=.25\columnwidth]{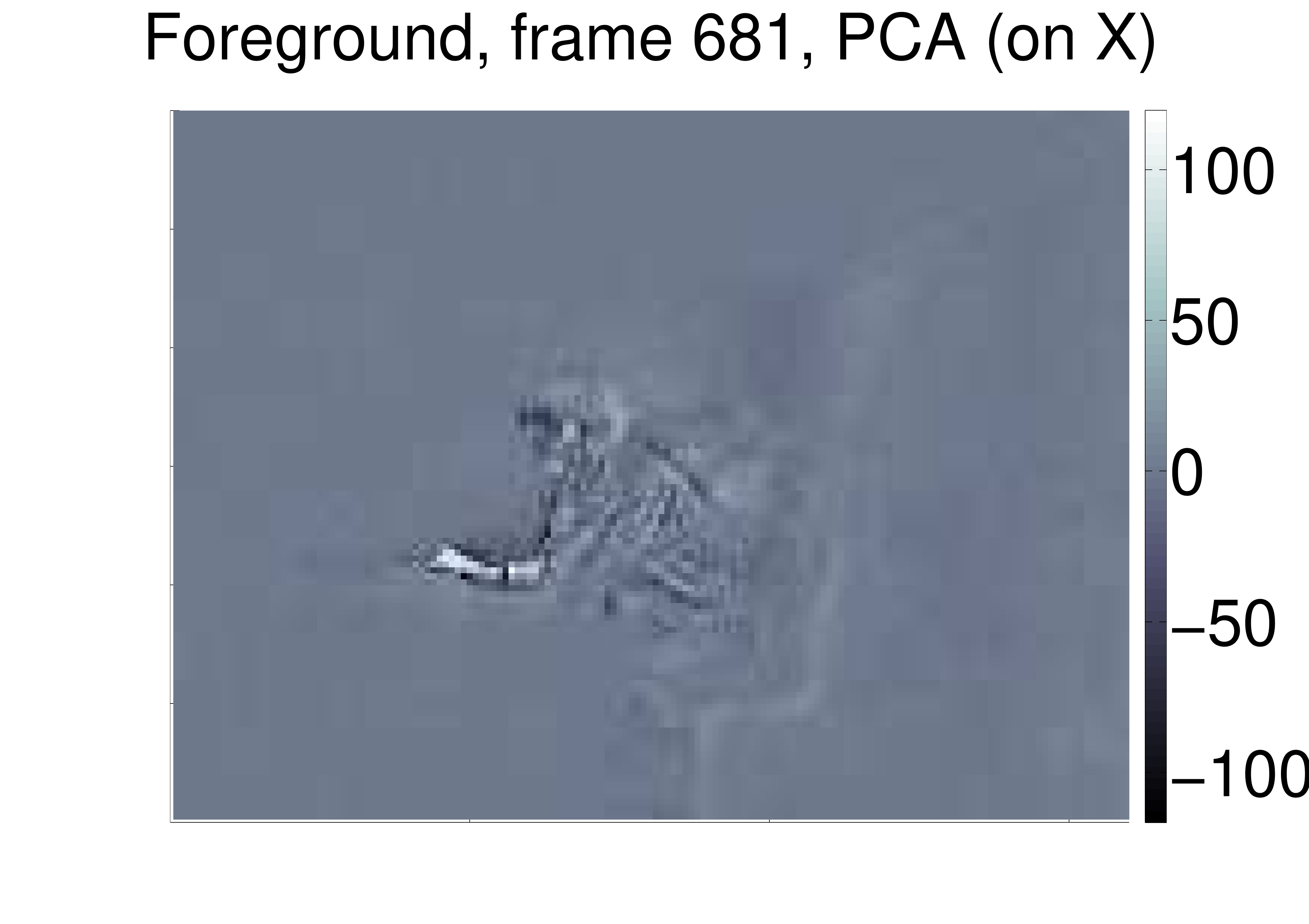} &
\includegraphics[width=.25\columnwidth]{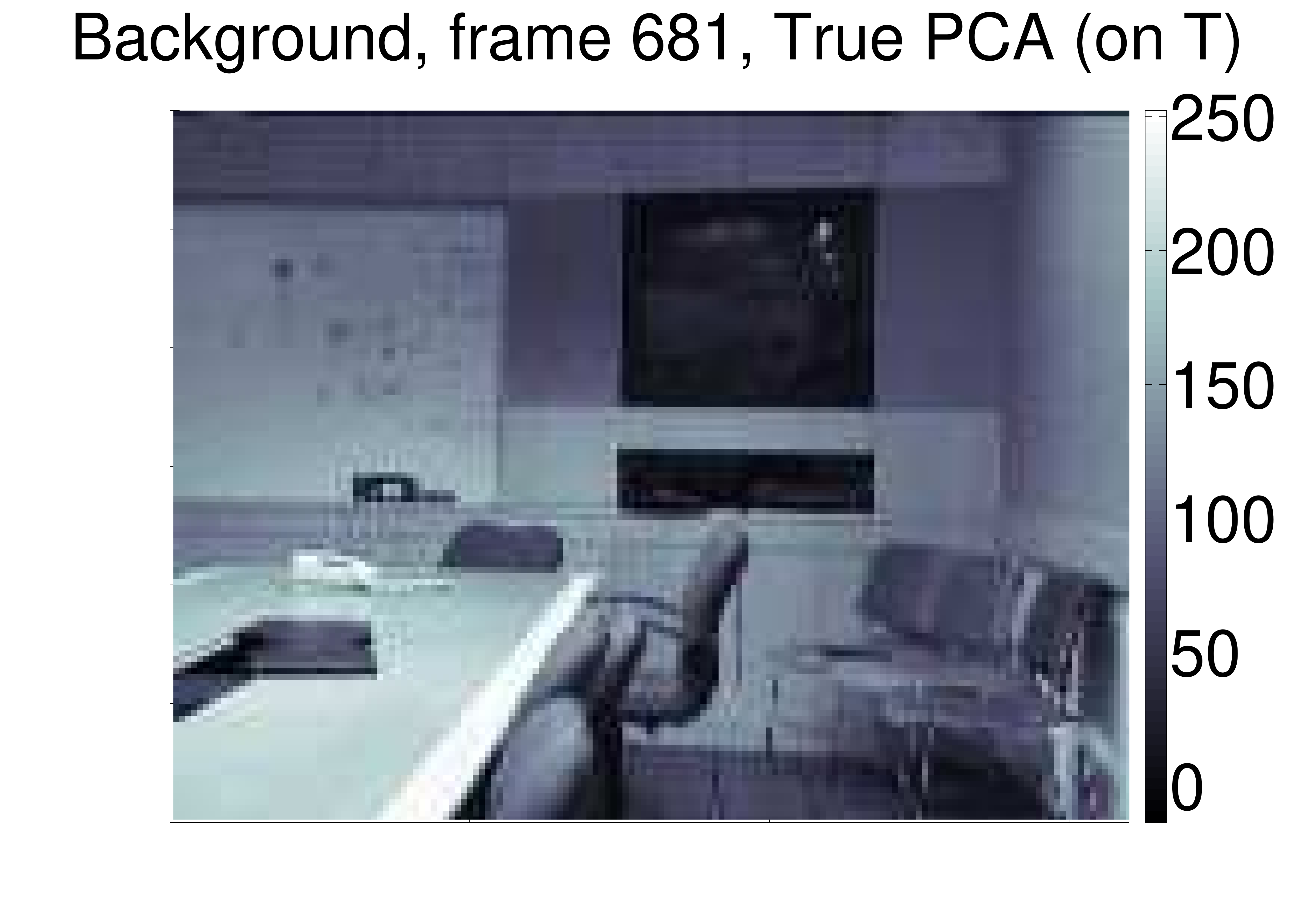} &
\includegraphics[width=.25\columnwidth]{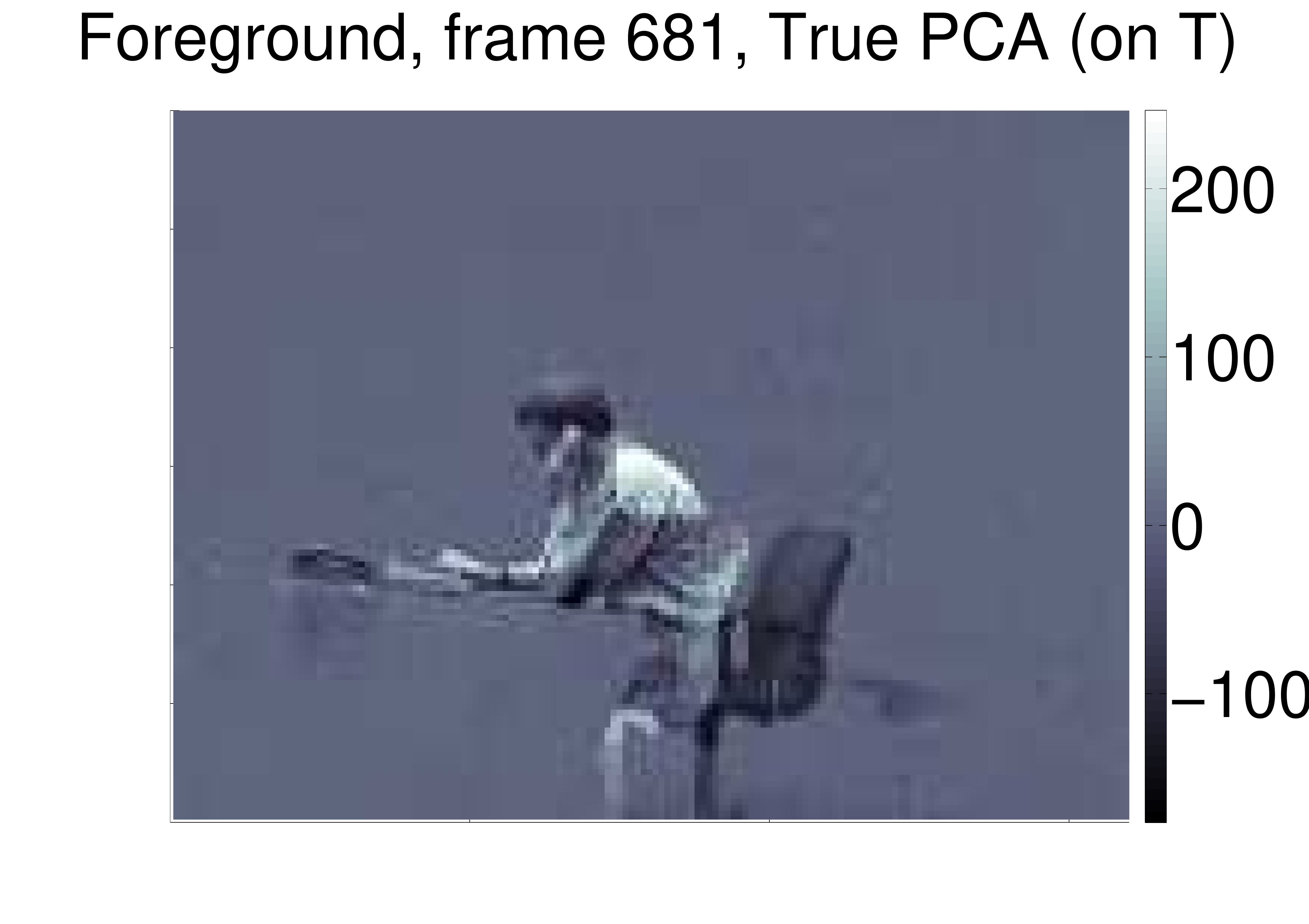} \\ \hline
\includegraphics[width=.25\columnwidth]{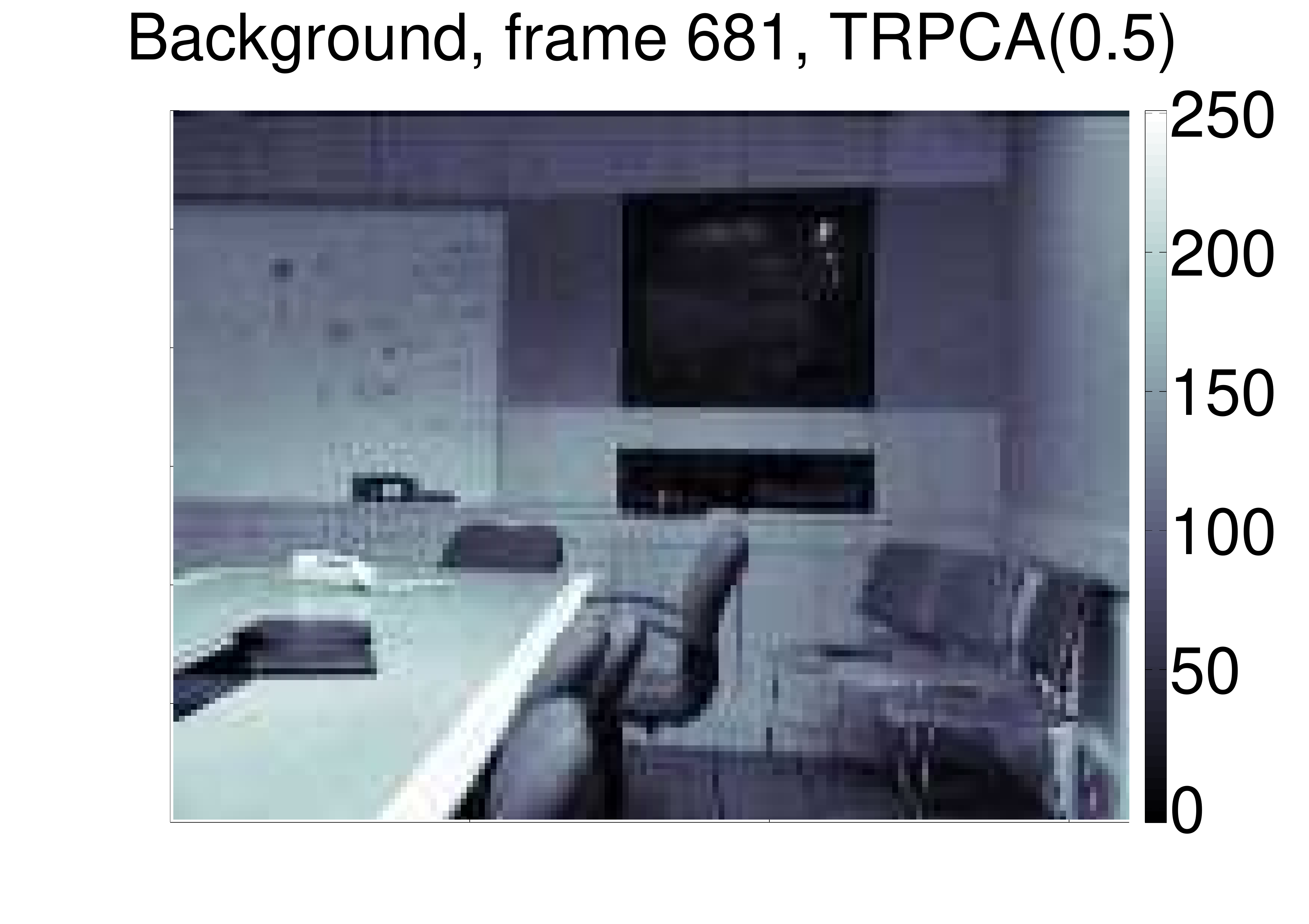} &
\includegraphics[width=.25\columnwidth]{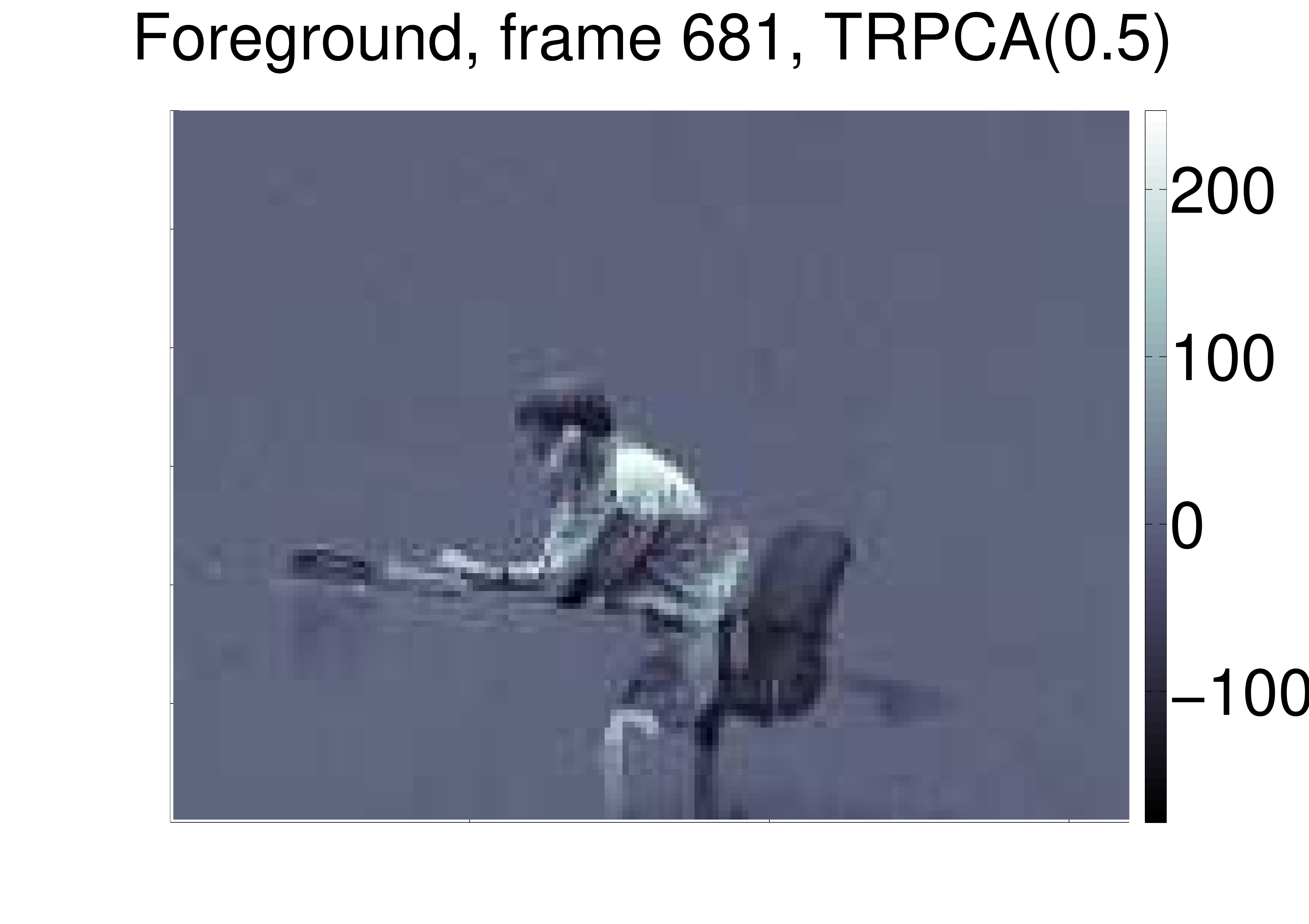} &
\includegraphics[width=.25\columnwidth]{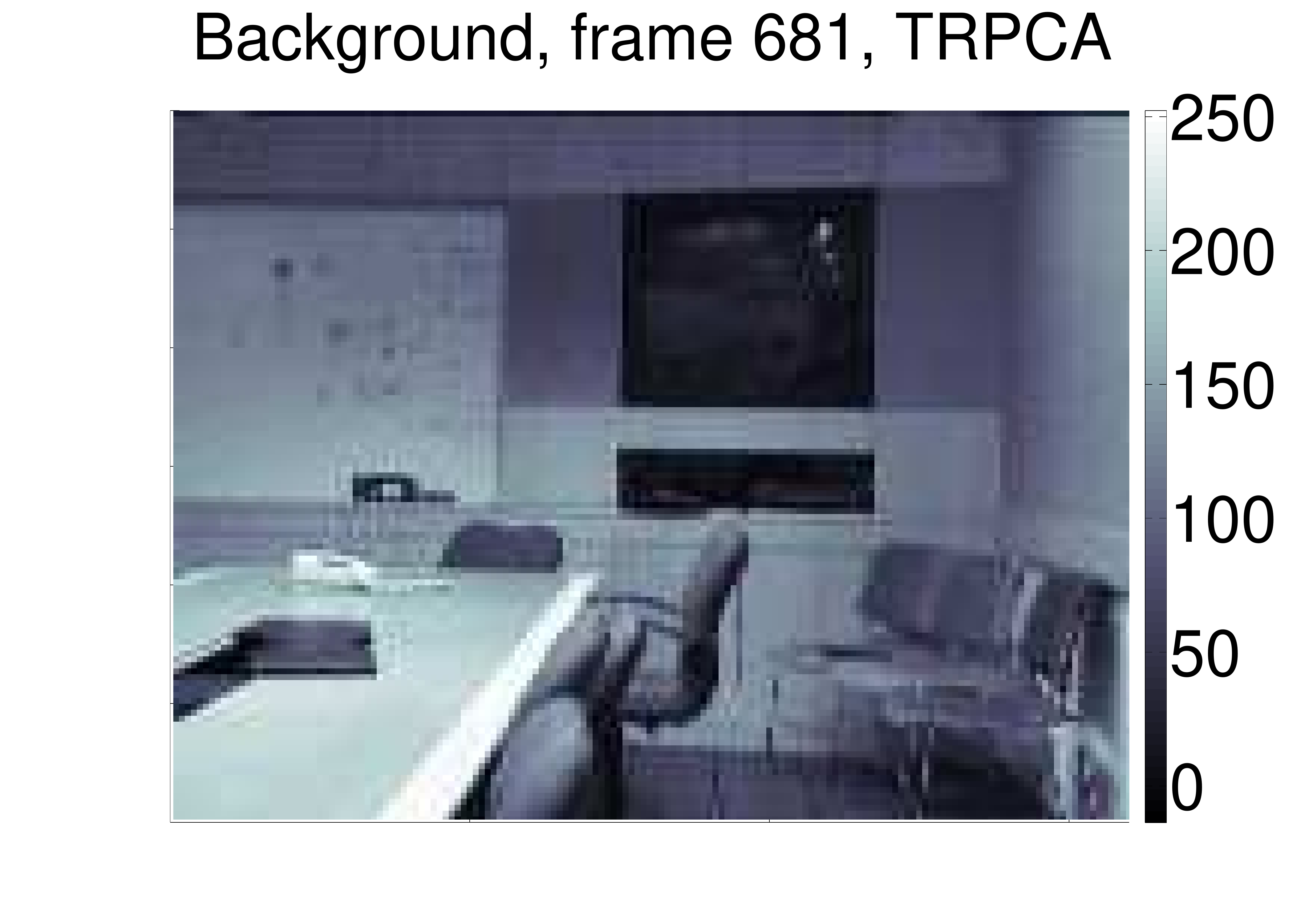} &
\includegraphics[width=.25\columnwidth]{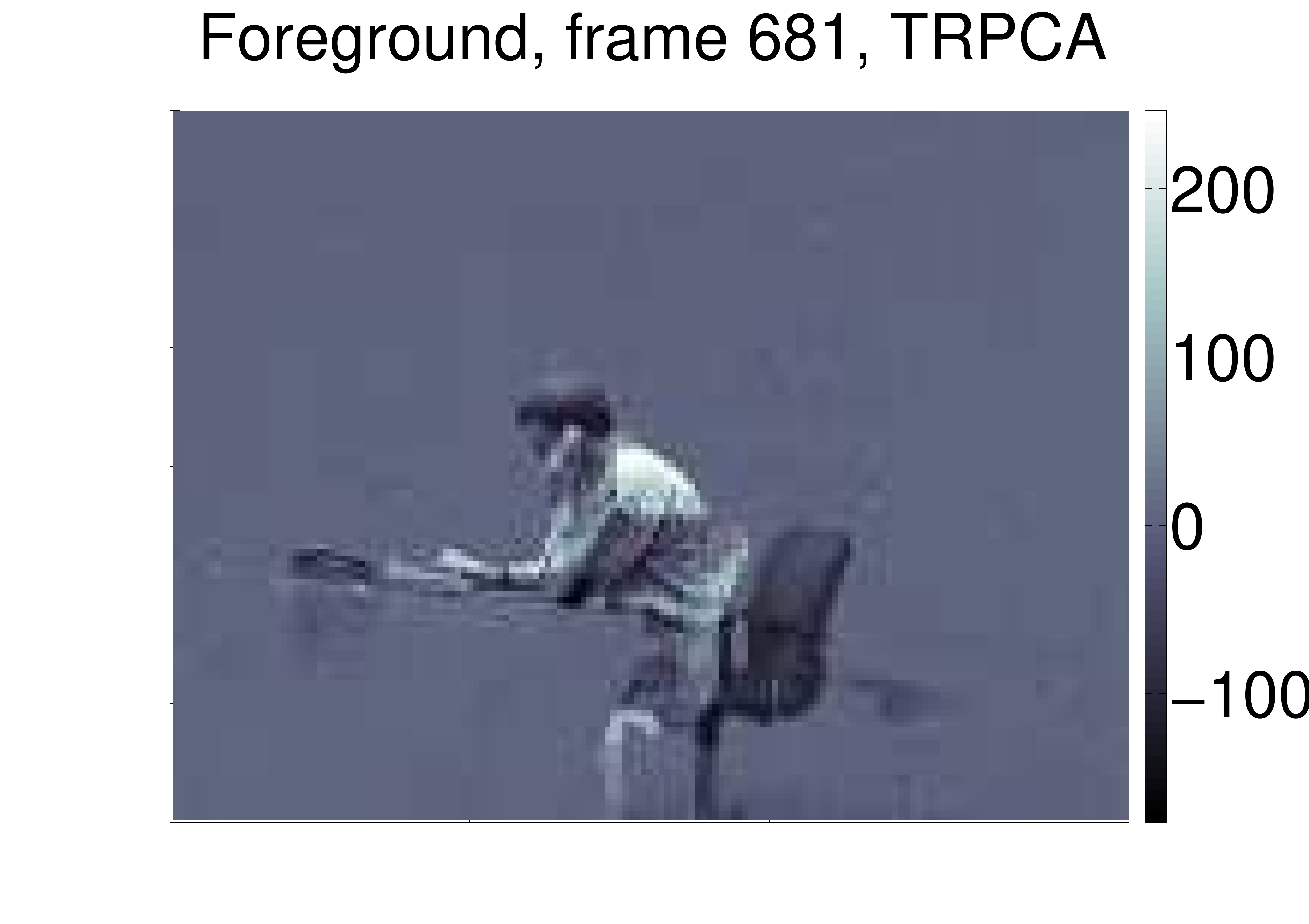} \\
\includegraphics[width=.25\columnwidth]{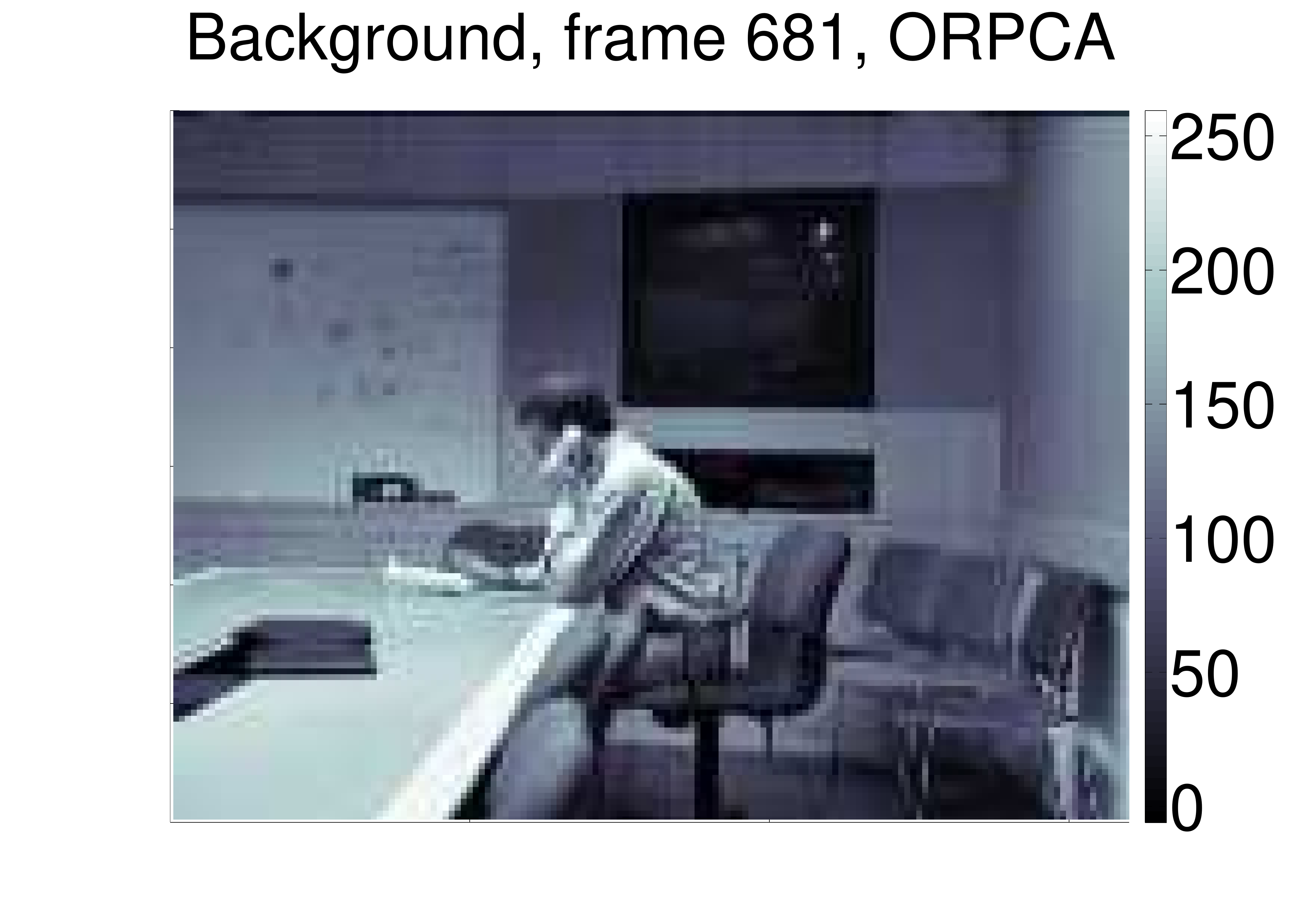} &
\includegraphics[width=.25\columnwidth]{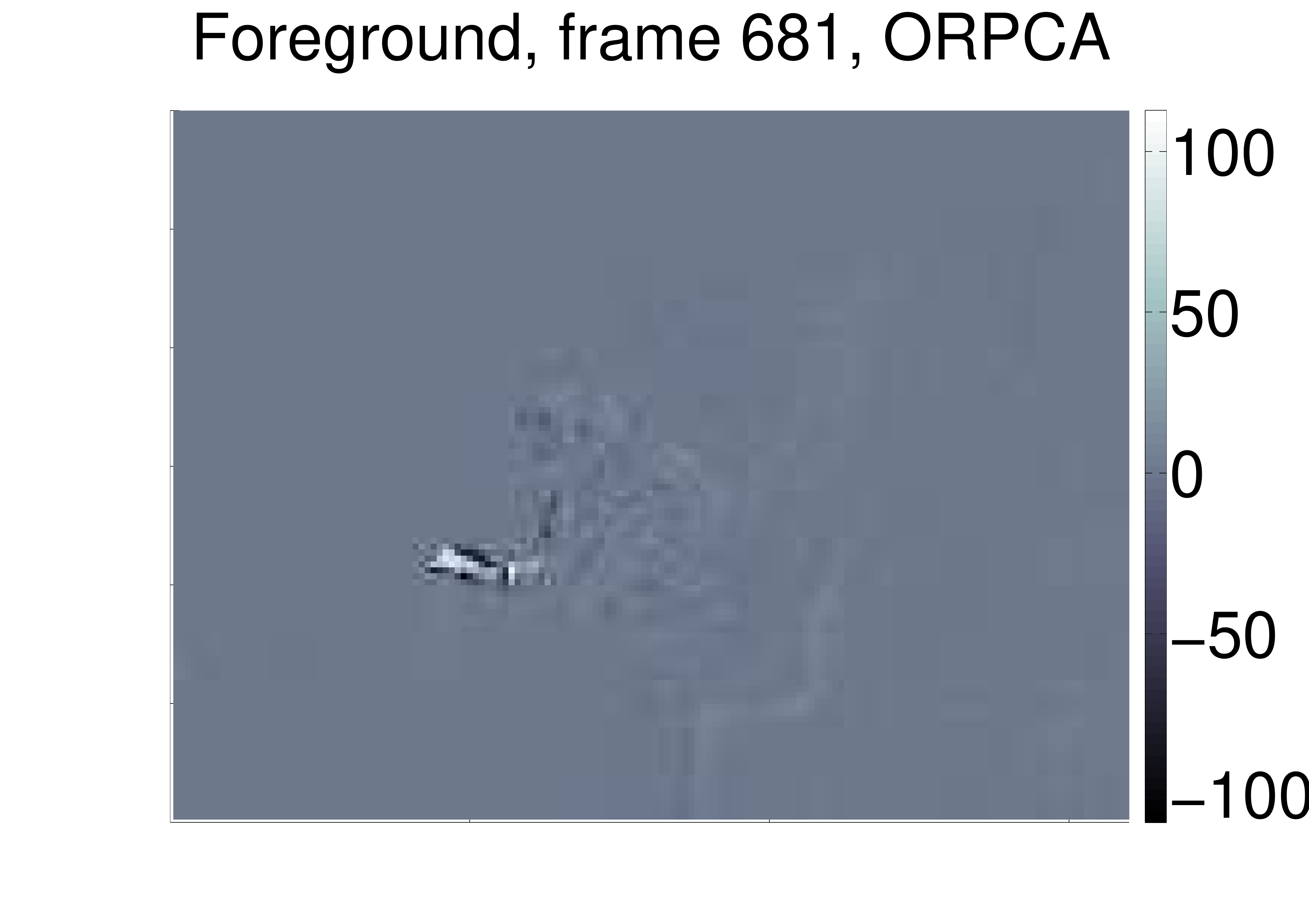} &
\includegraphics[width=.25\columnwidth]{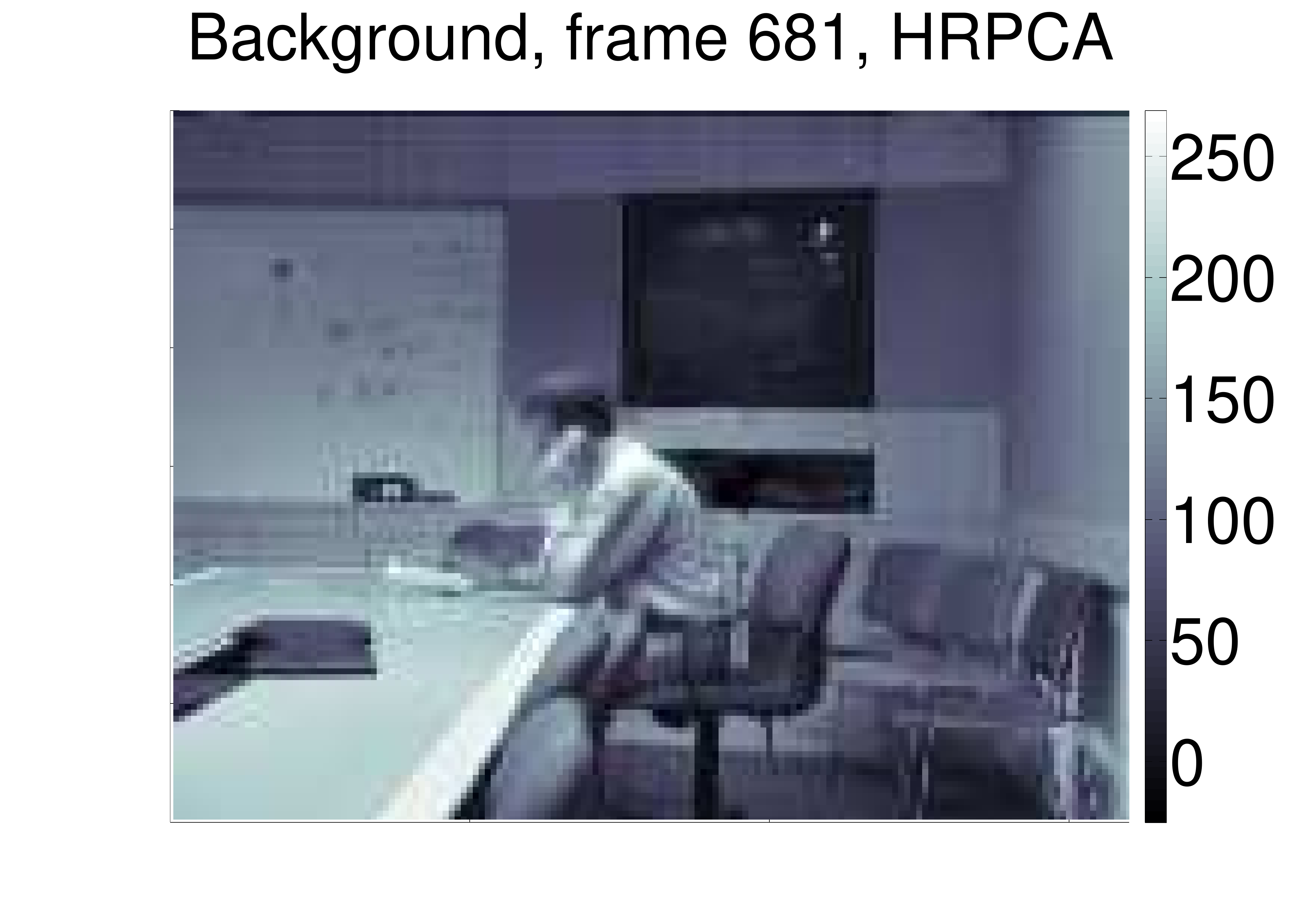} &
\includegraphics[width=.25\columnwidth]{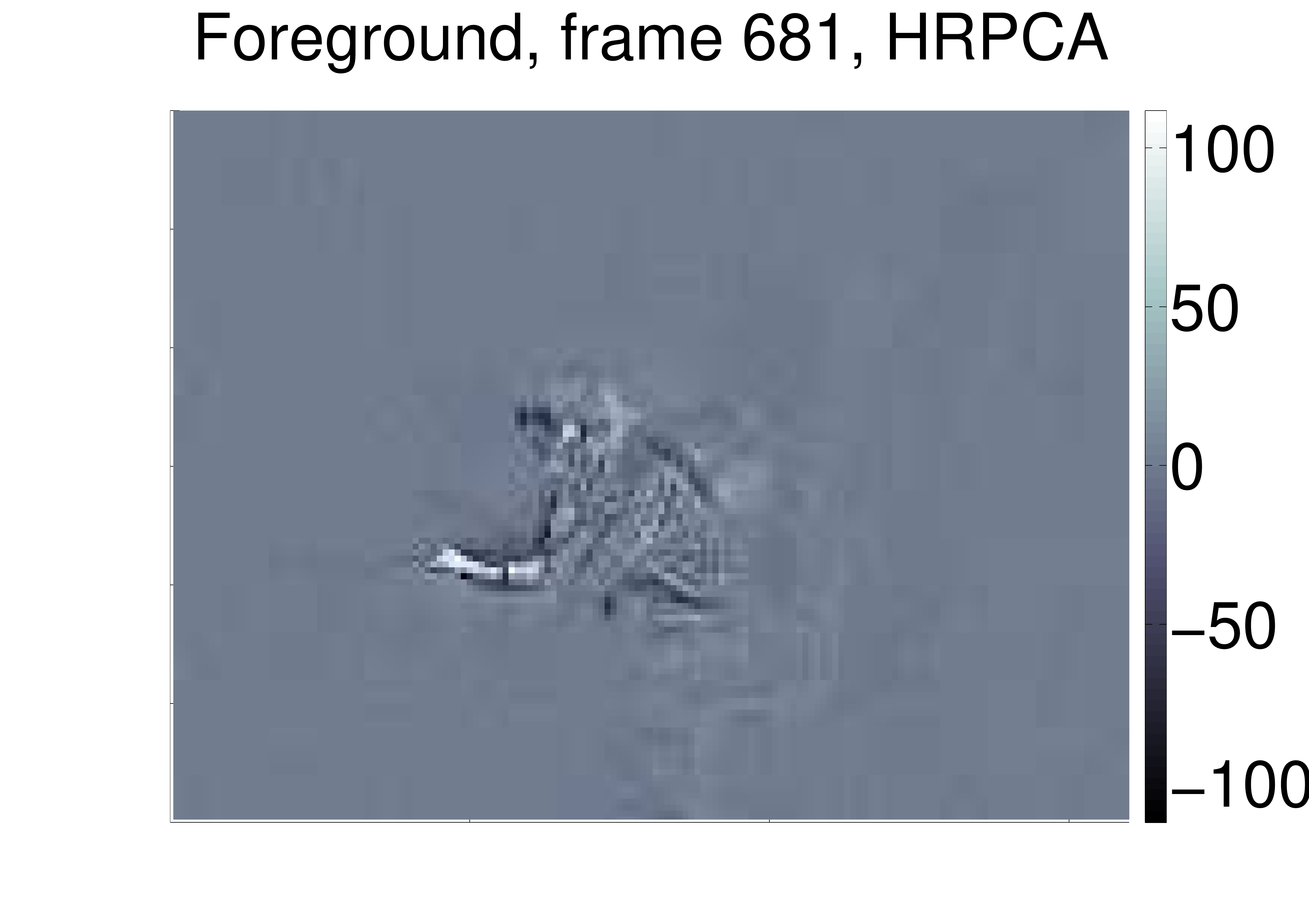} \\
\includegraphics[width=.25\columnwidth]{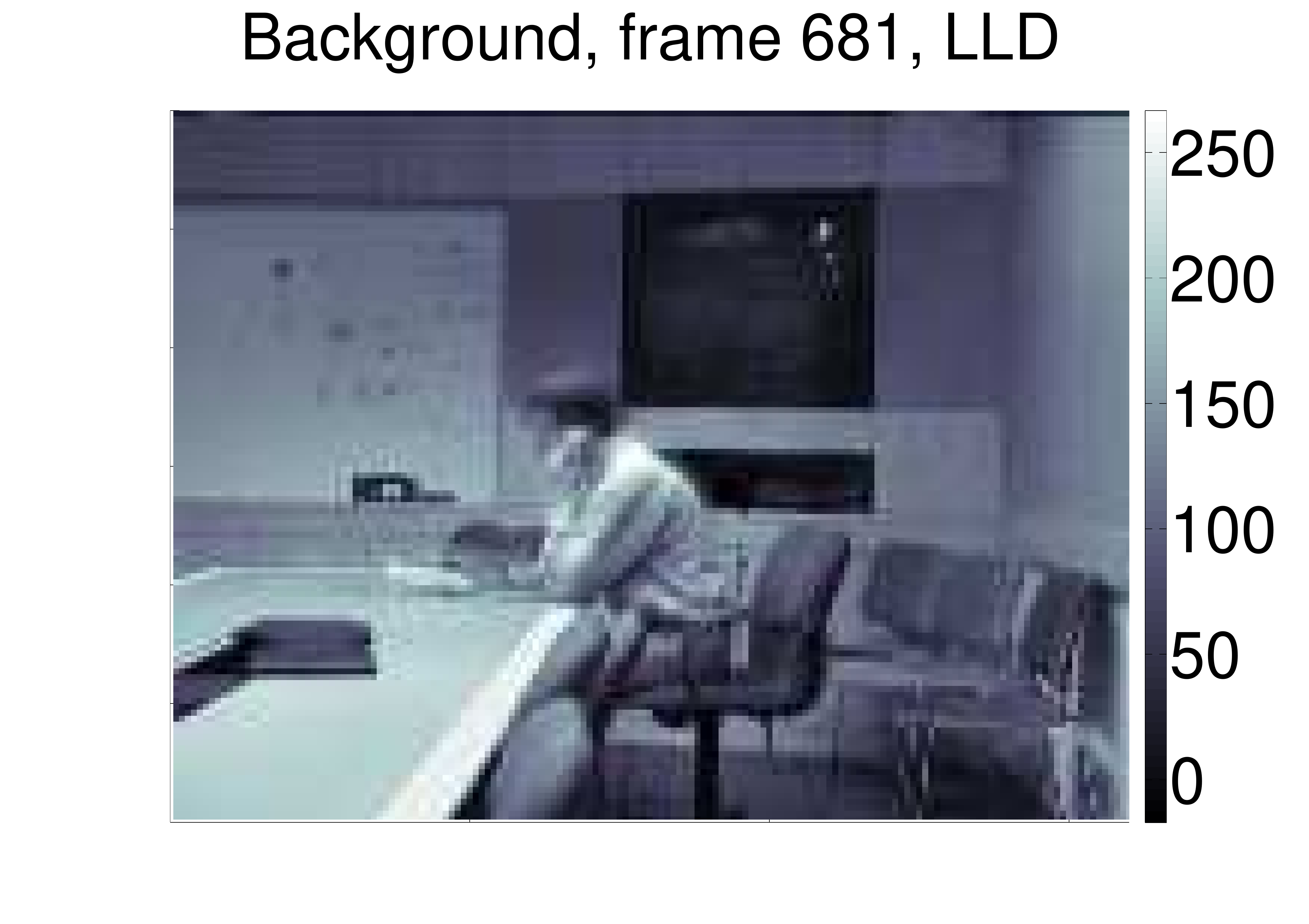} &
\includegraphics[width=.25\columnwidth]{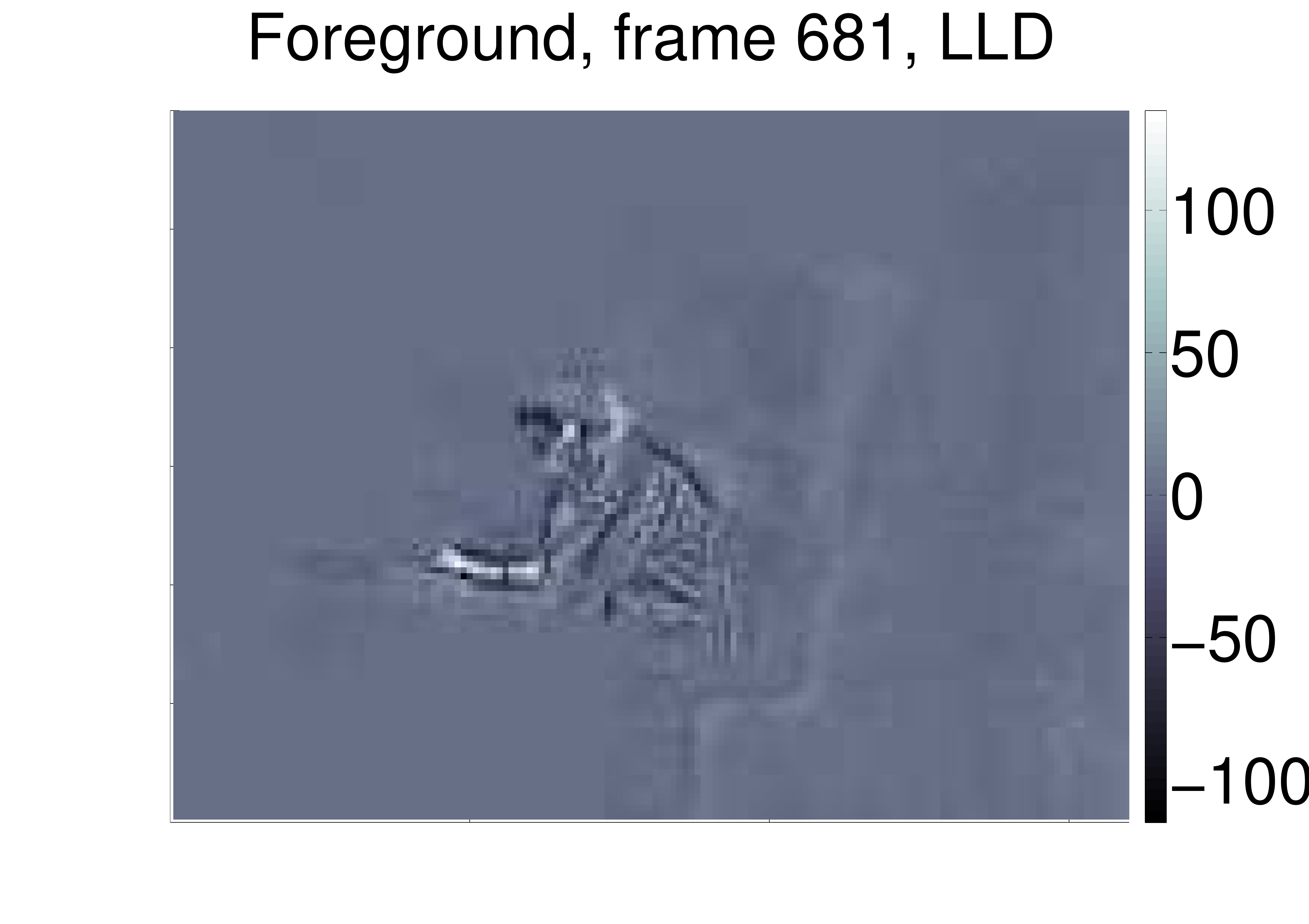} &
\includegraphics[width=.25\columnwidth]{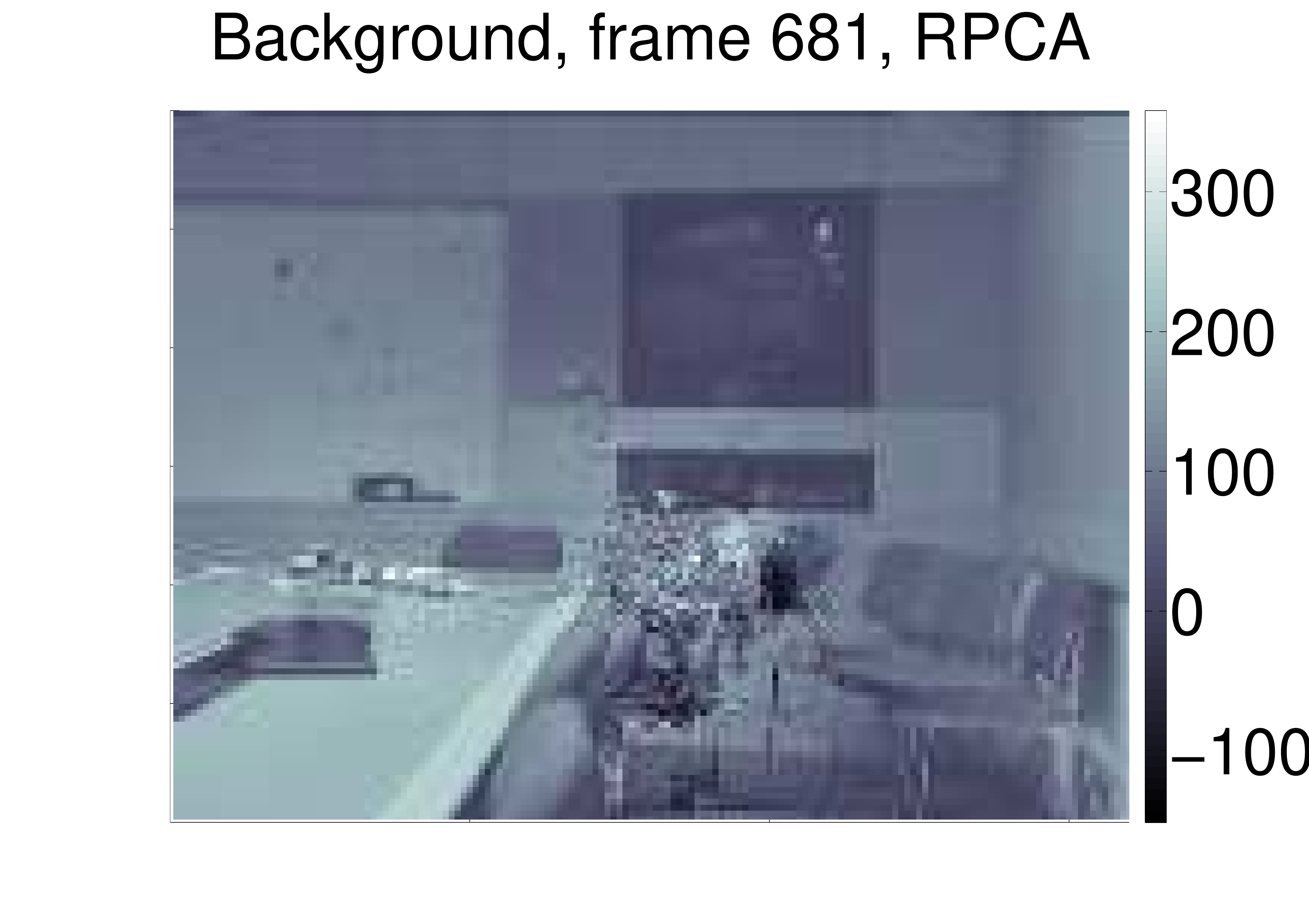} &
\includegraphics[width=.25\columnwidth]{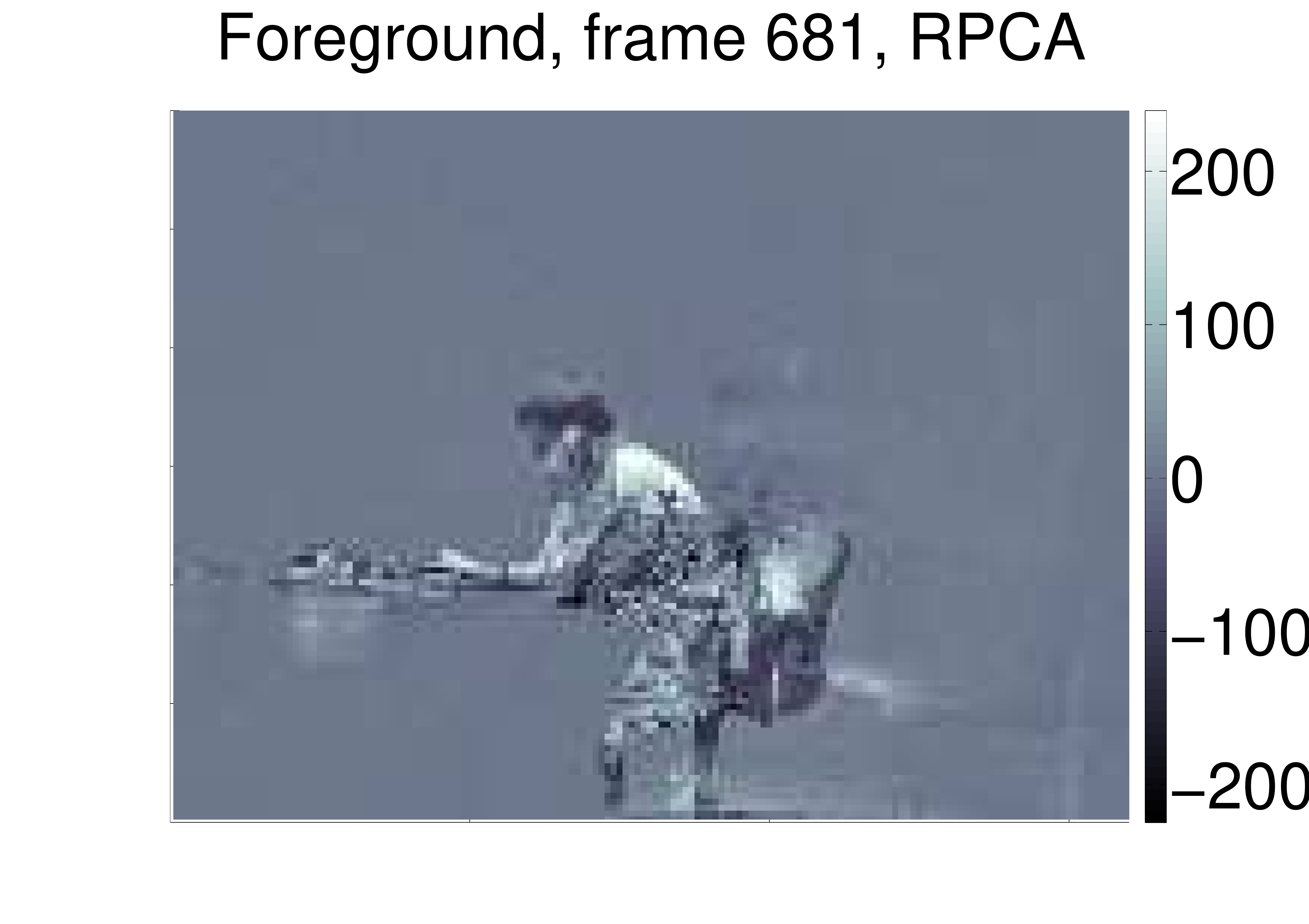} 
\end{tabular}
\caption{Extracted background and foreground of frame 681 of the reduced moved object data set. The number of components is $k=10$ (unscaled, compare to scaled version in Fig.~\ref{fig:bfrmo3})
}
\label{fig:bfrmo4}
\end{figure}
\clearpage

%%%%%%%%%%%%%%%%%%%%%%%%%%%%%%%%%%%%%%%%%%%%%%%%%%%%
\begin{figure}
\begin{tabular}{cccc}
\includegraphics[width=.25\columnwidth]{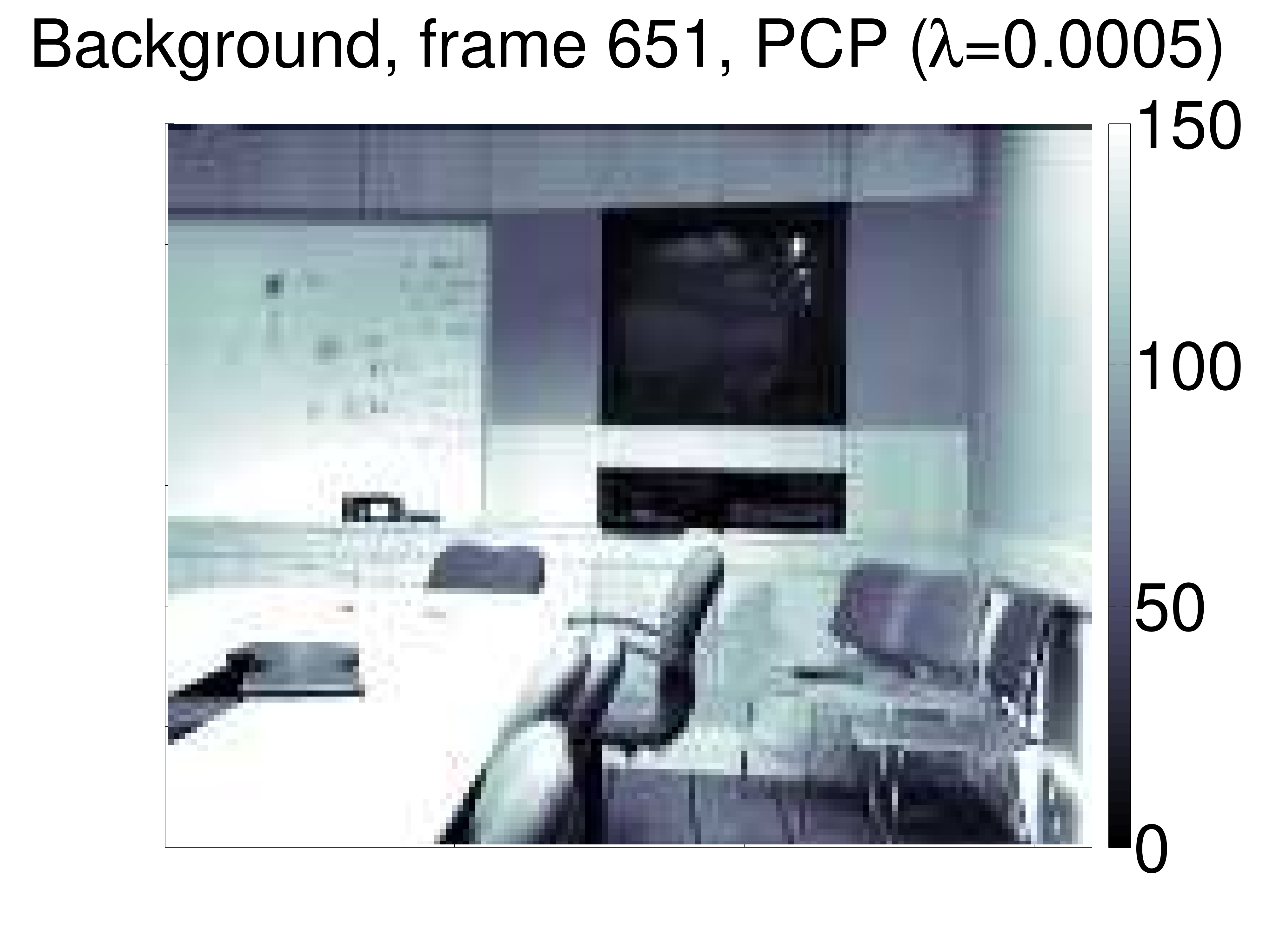} &
\includegraphics[width=.25\columnwidth]{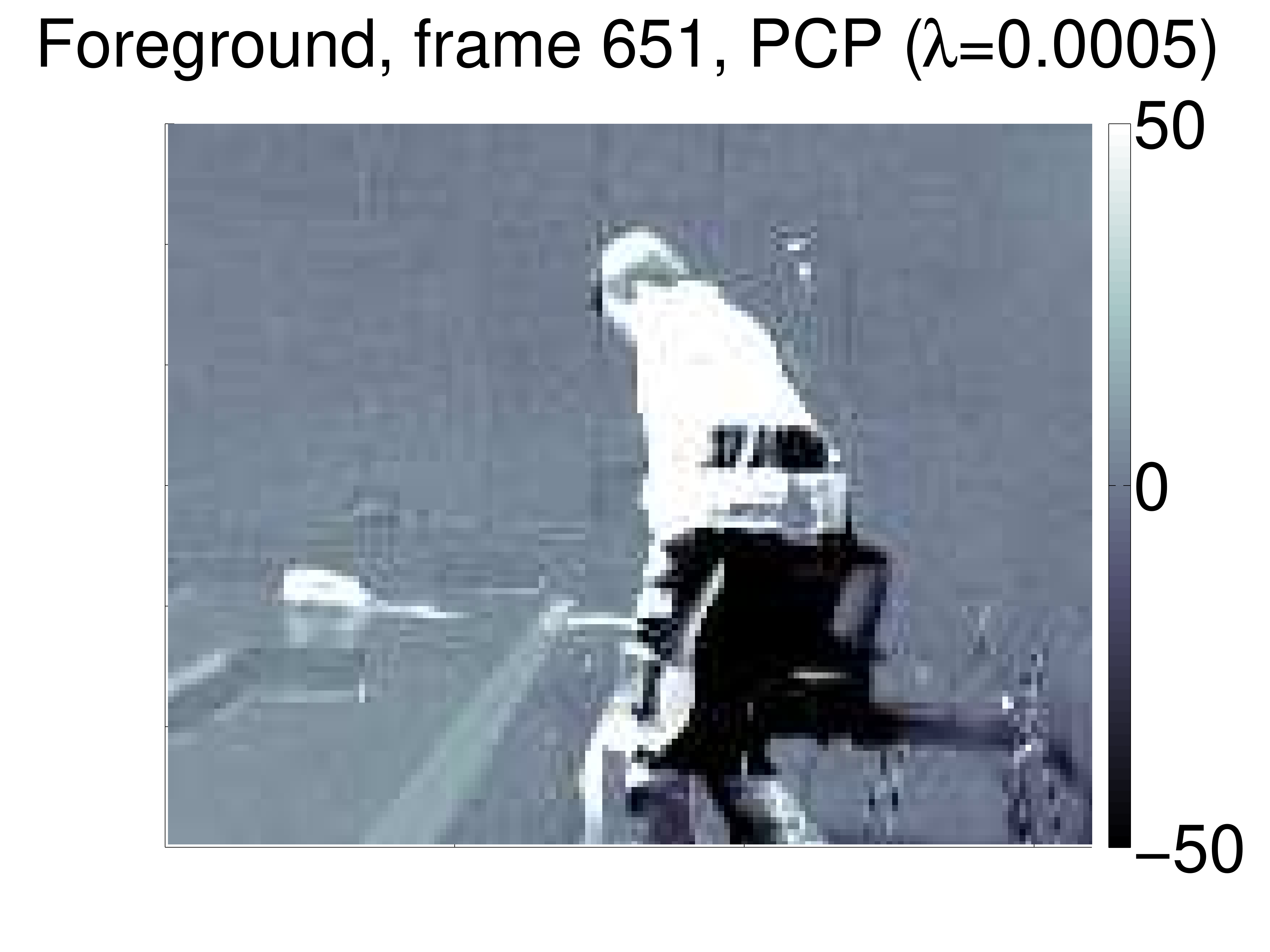} &
\includegraphics[width=.25\columnwidth]{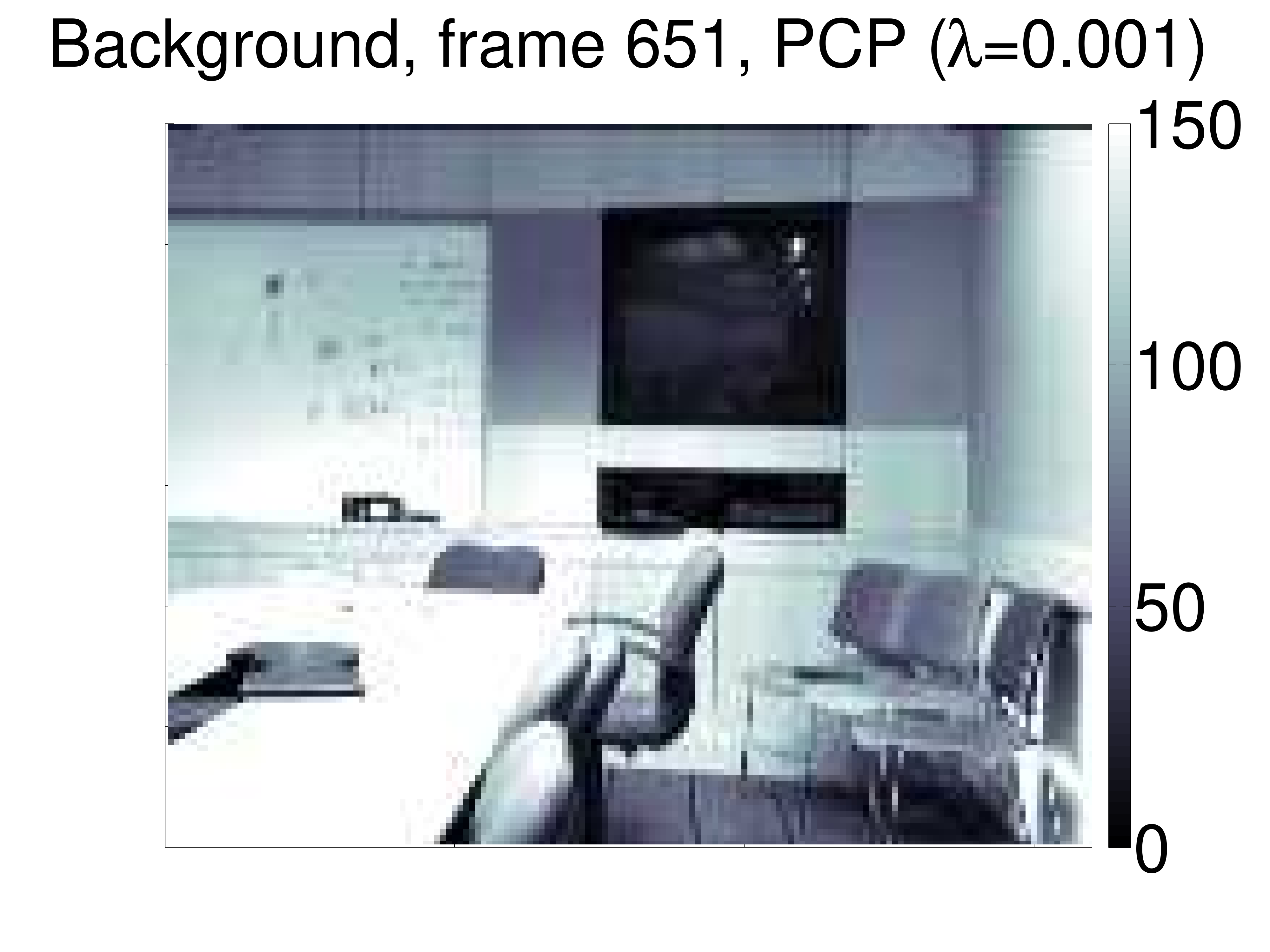} &
\includegraphics[width=.25\columnwidth]{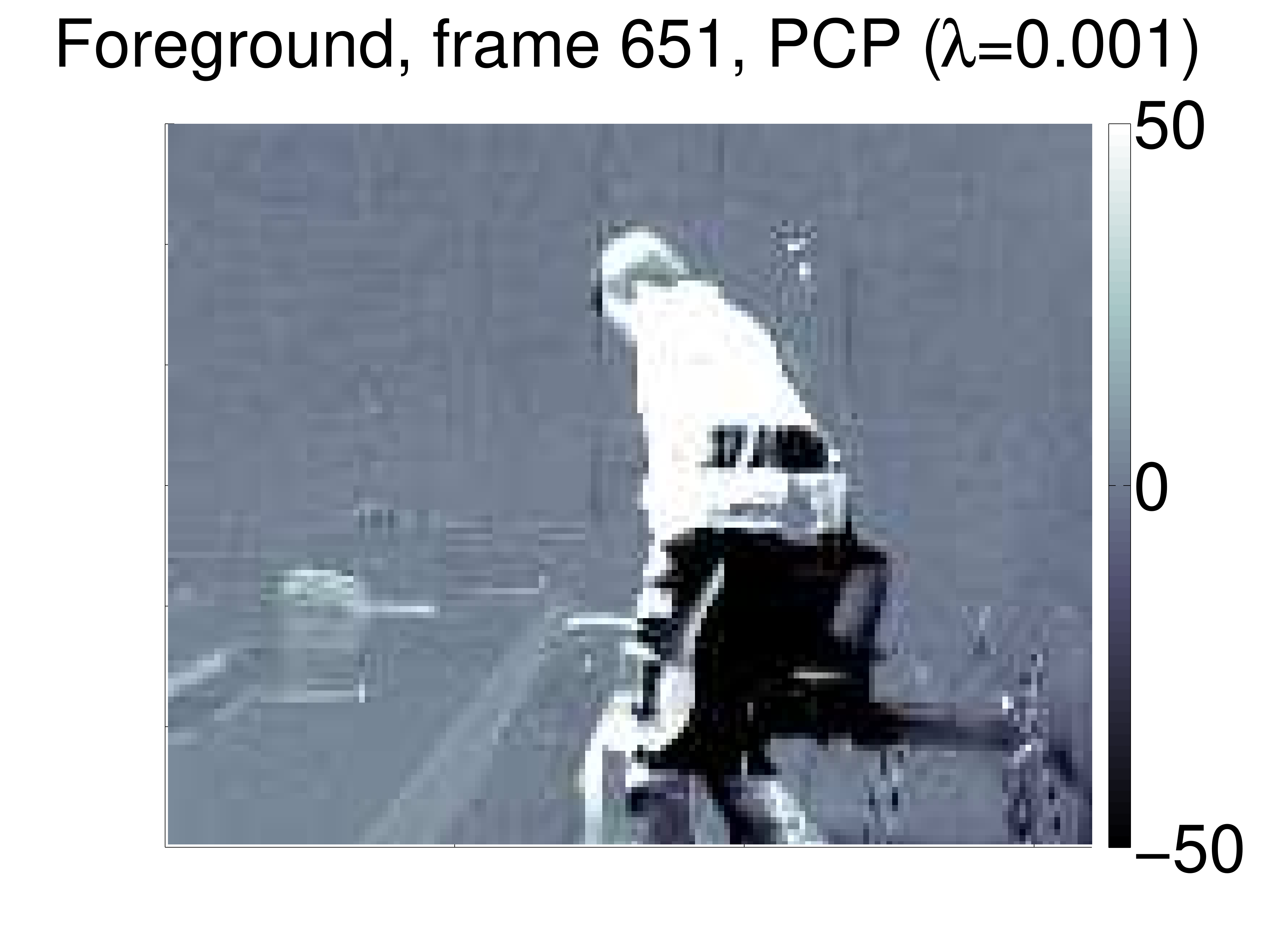} \\
\includegraphics[width=.25\columnwidth]{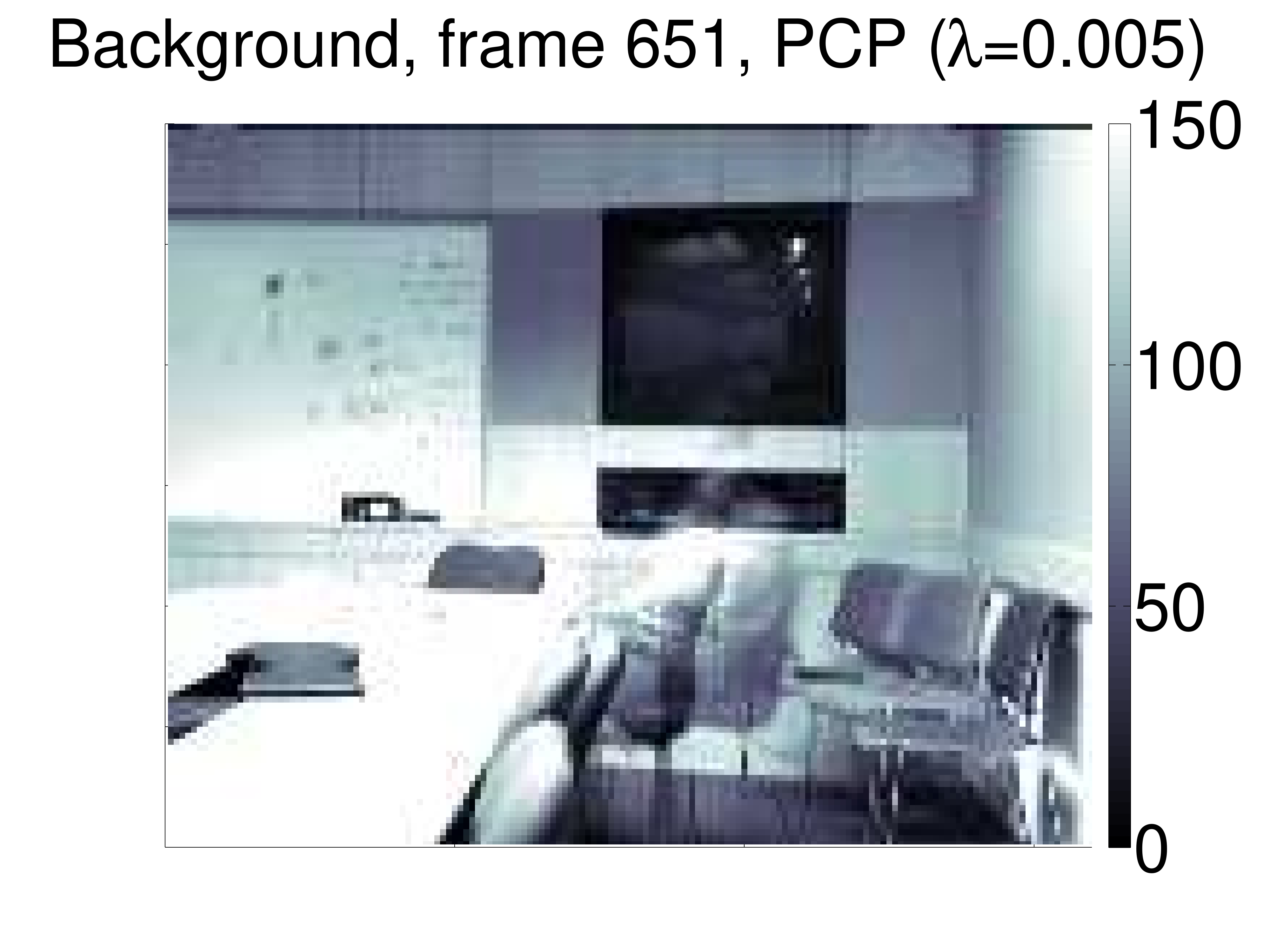} &
\includegraphics[width=.25\columnwidth]{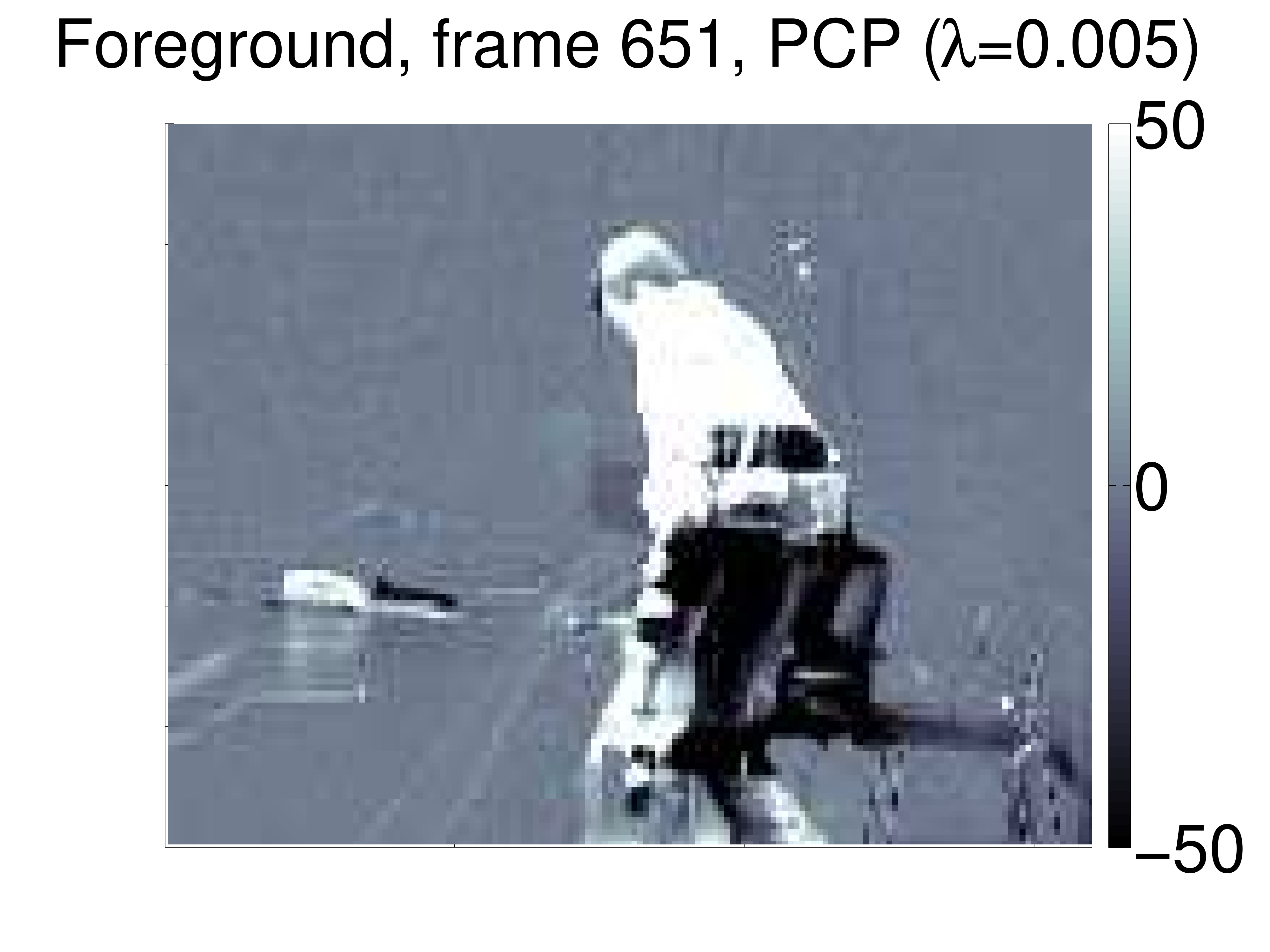} &
\includegraphics[width=.25\columnwidth]{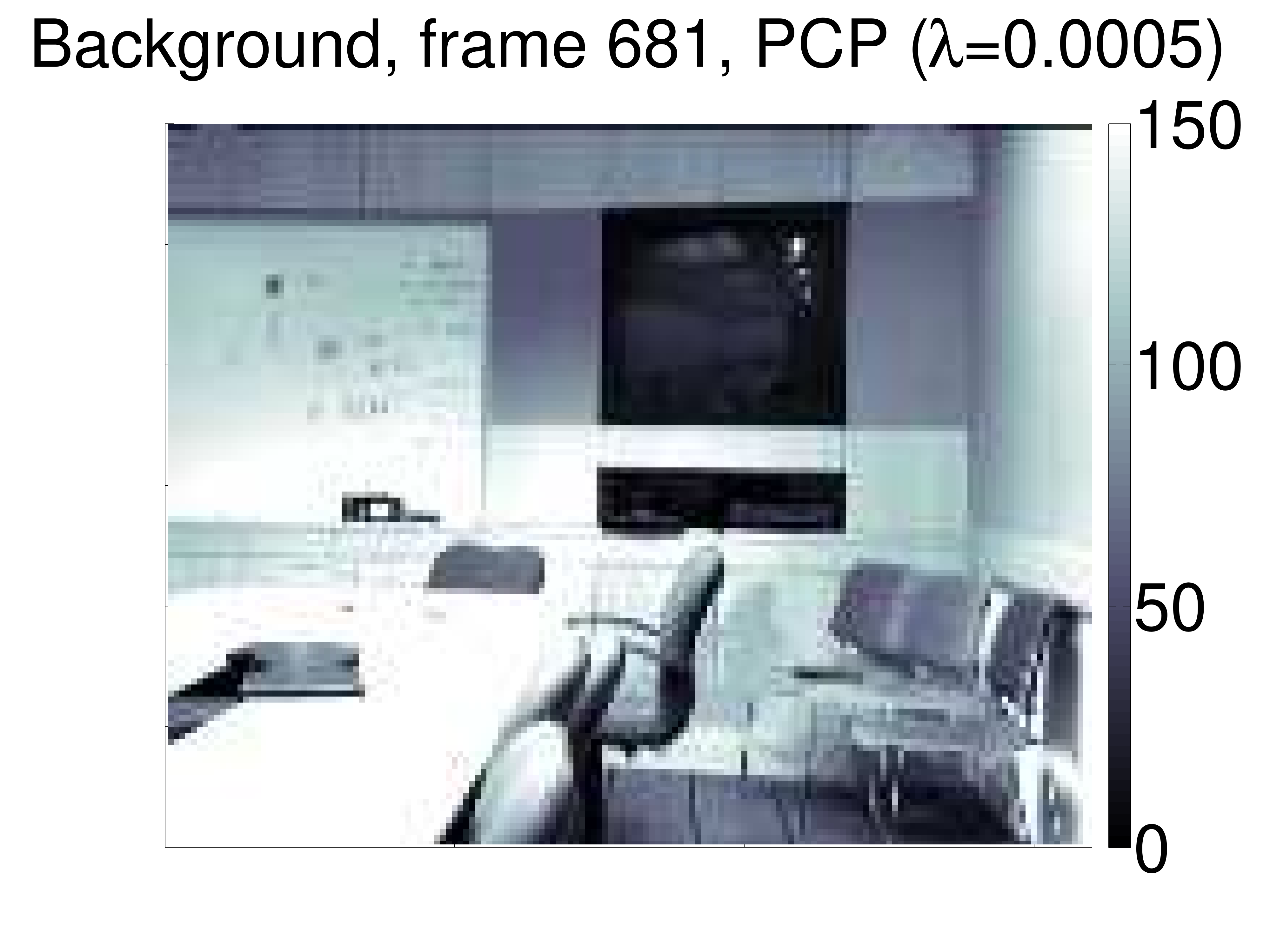}&
\includegraphics[width=.25\columnwidth]{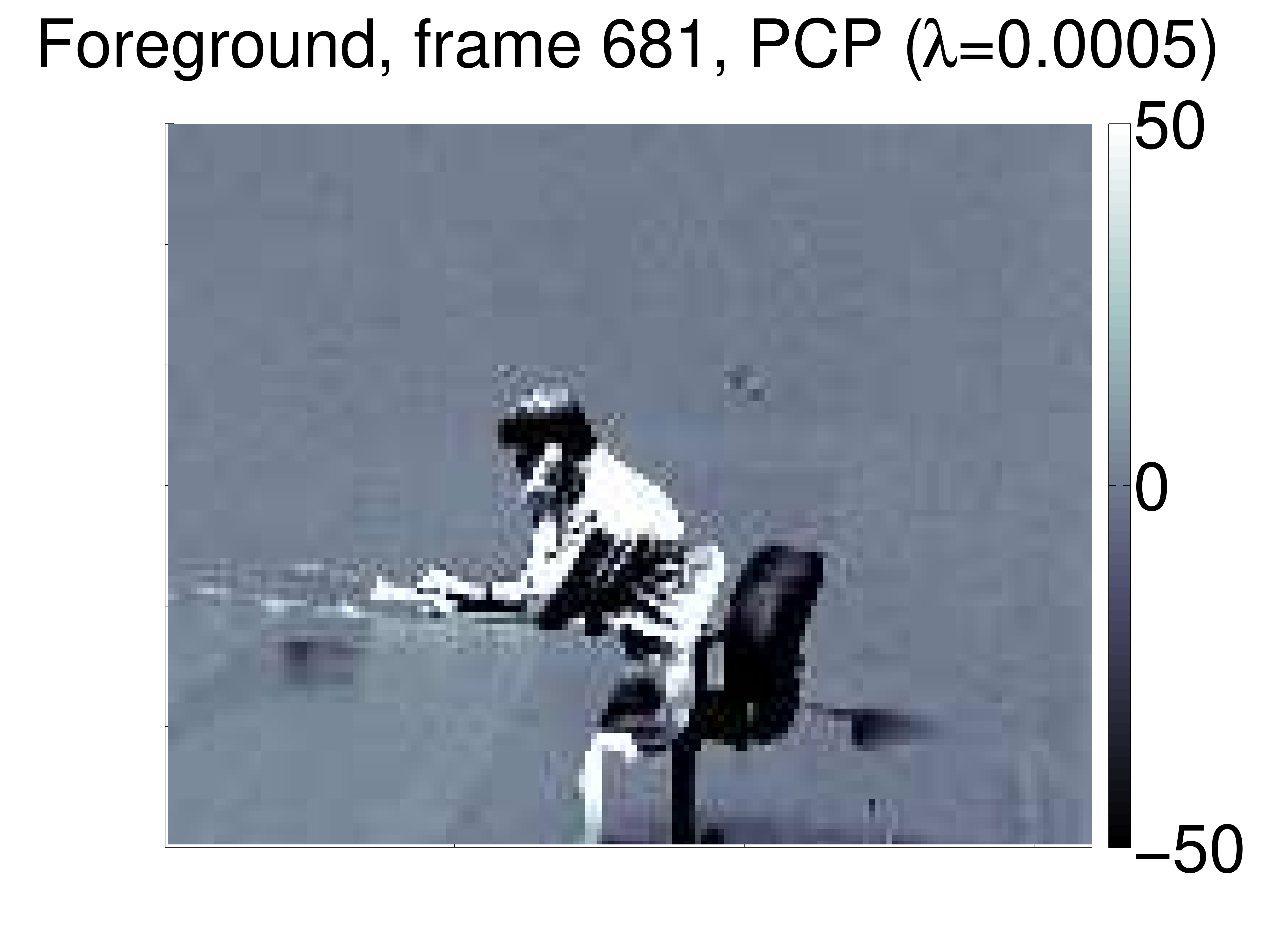}\\
\includegraphics[width=.25\columnwidth]{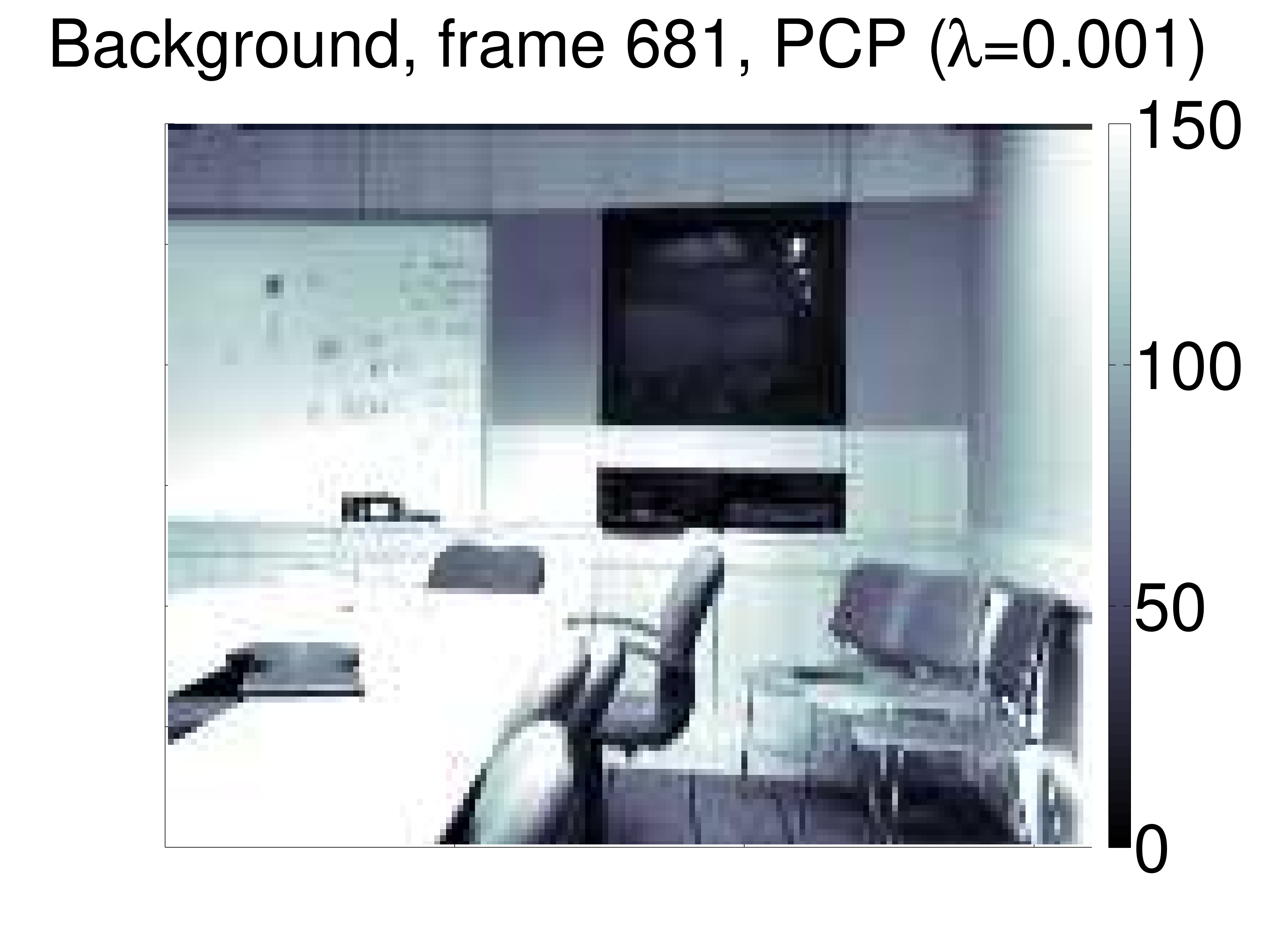}&
\includegraphics[width=.25\columnwidth]{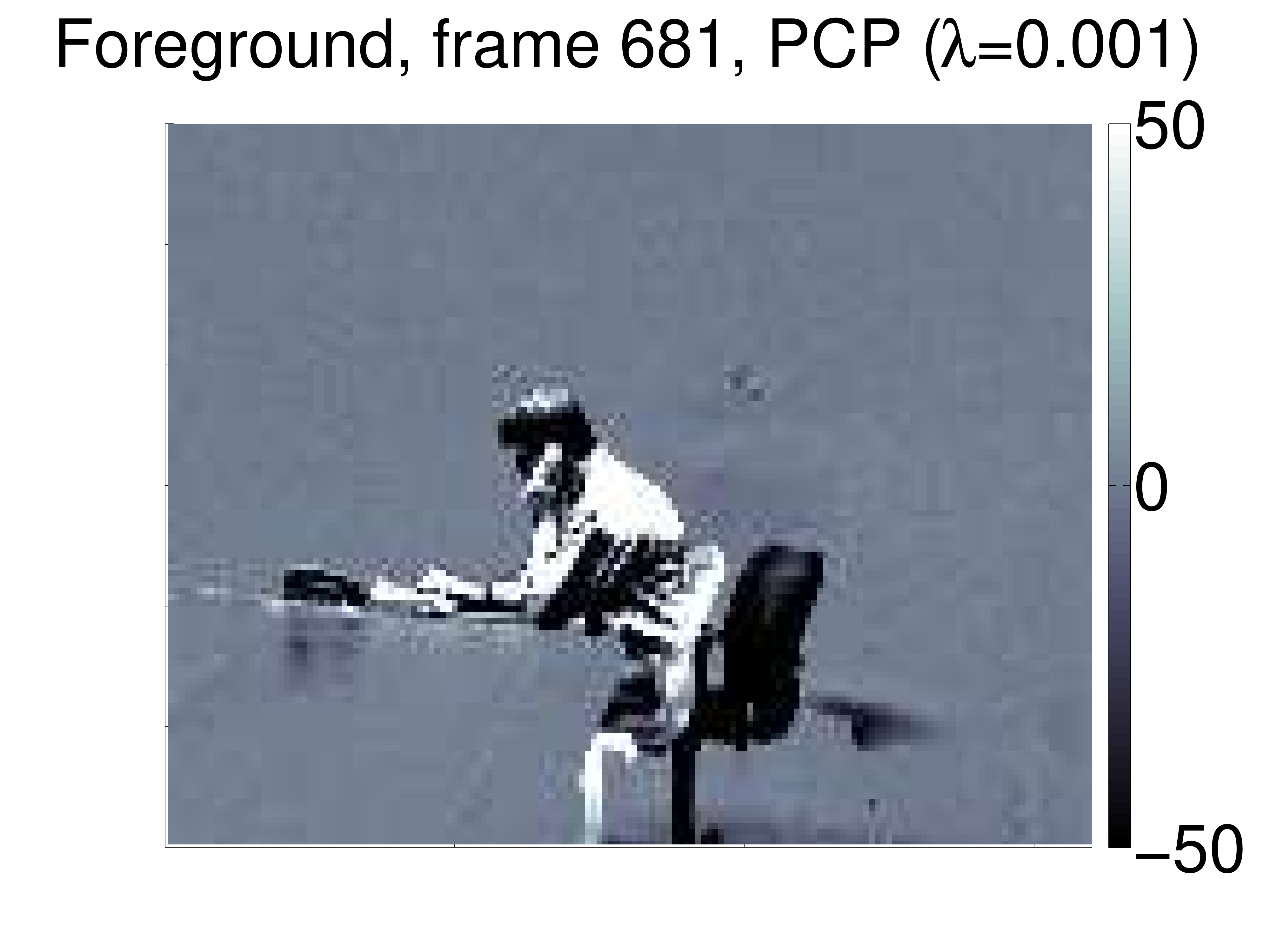}&
\includegraphics[width=.25\columnwidth]{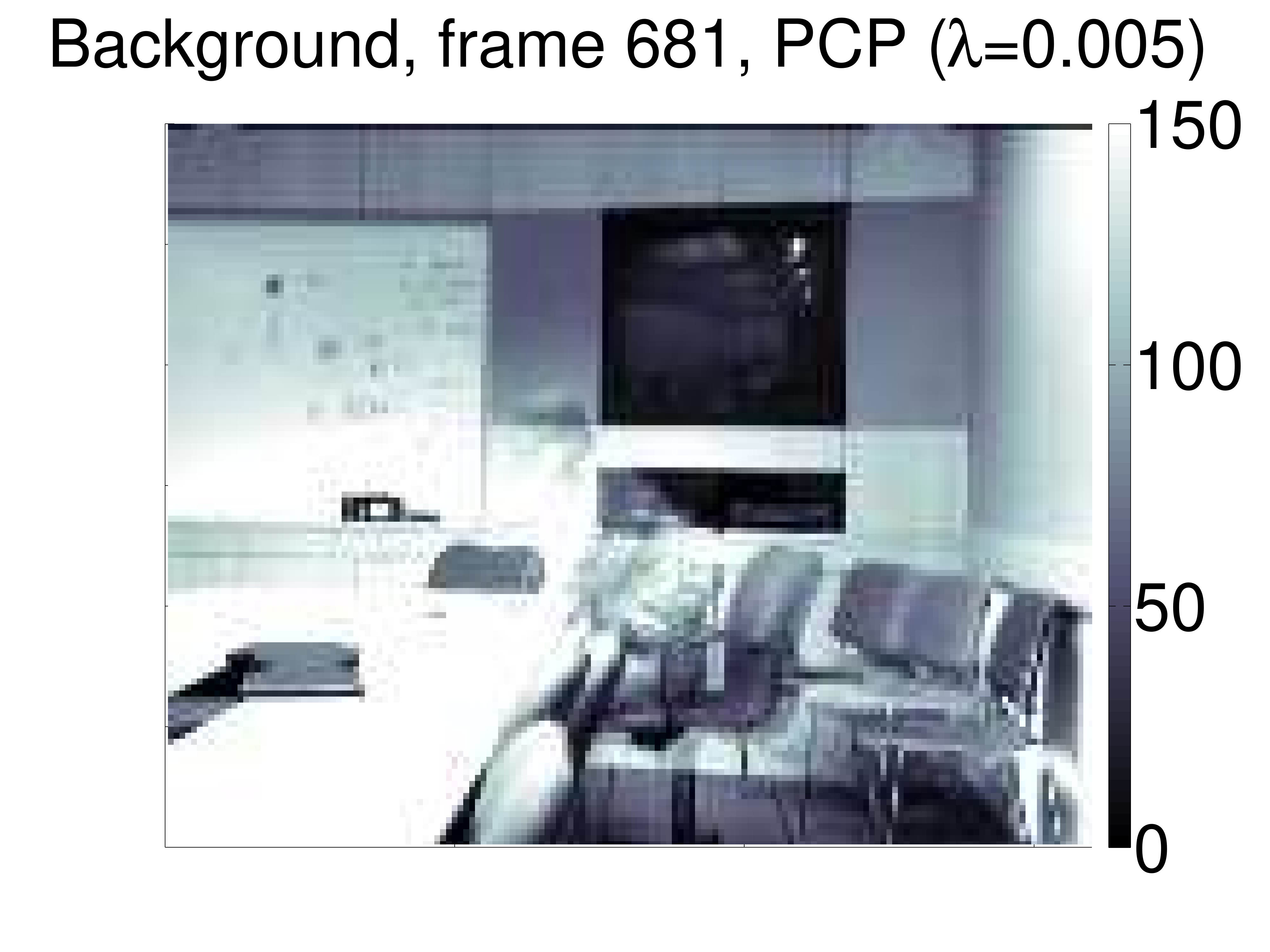}&
\includegraphics[width=.25\columnwidth]{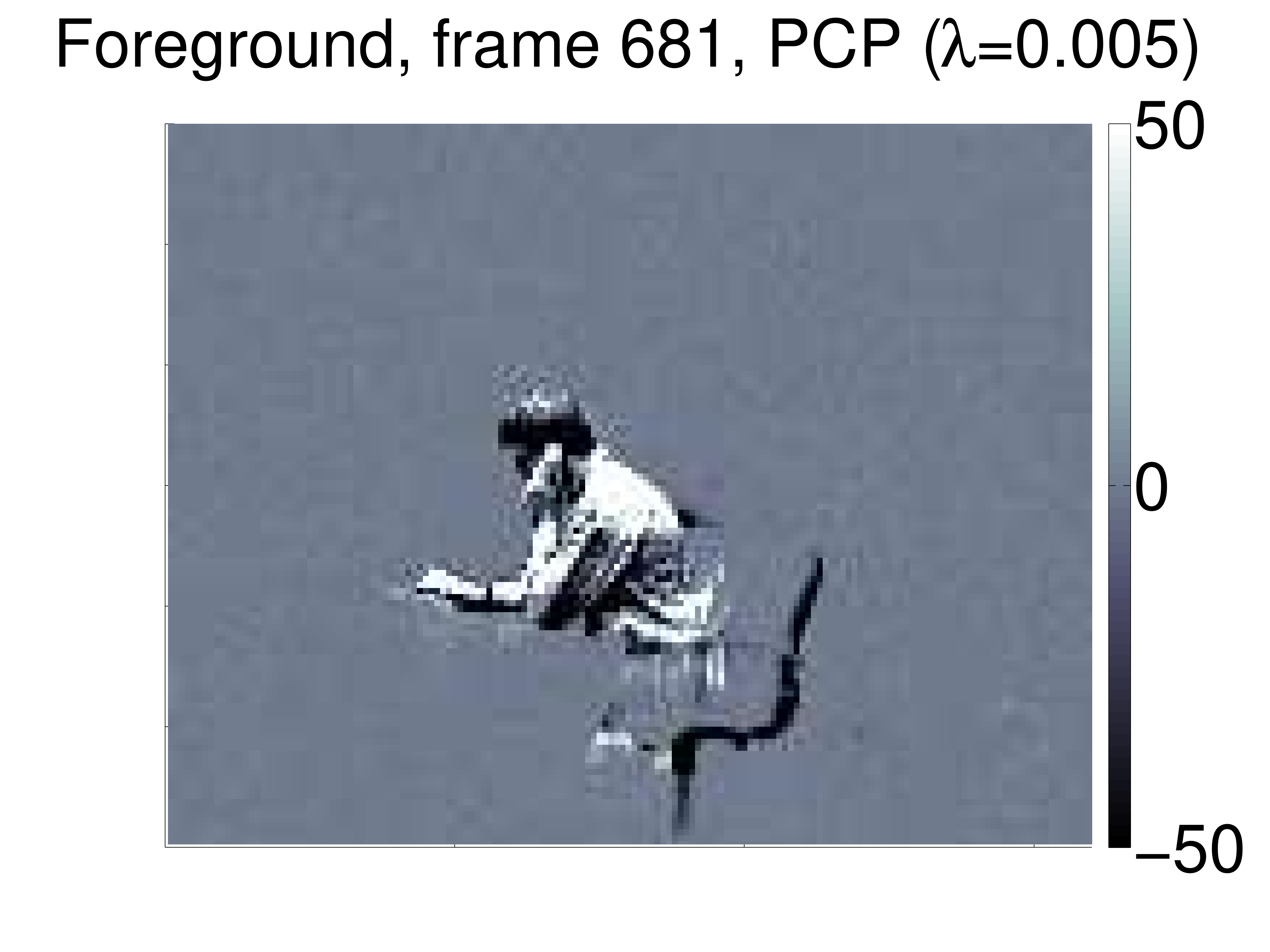}
\end{tabular}
\caption{Extracted background and foreground of frames 651 and 681 of the reduced moved object data set obtained with PCP (scaled, compare to unscaled version in Fig.~\ref{fig:bfrmopcp2})
}
\label{fig:bfrmopcp1}
\end{figure}
\clearpage

%%%%%%%%%%%%%%%%%%%%%%%%%%%%%%%%%%%%%%%%%%%%%%%%%%%%
\begin{figure}
\begin{tabular}{cccc}
\includegraphics[width=.25\columnwidth]{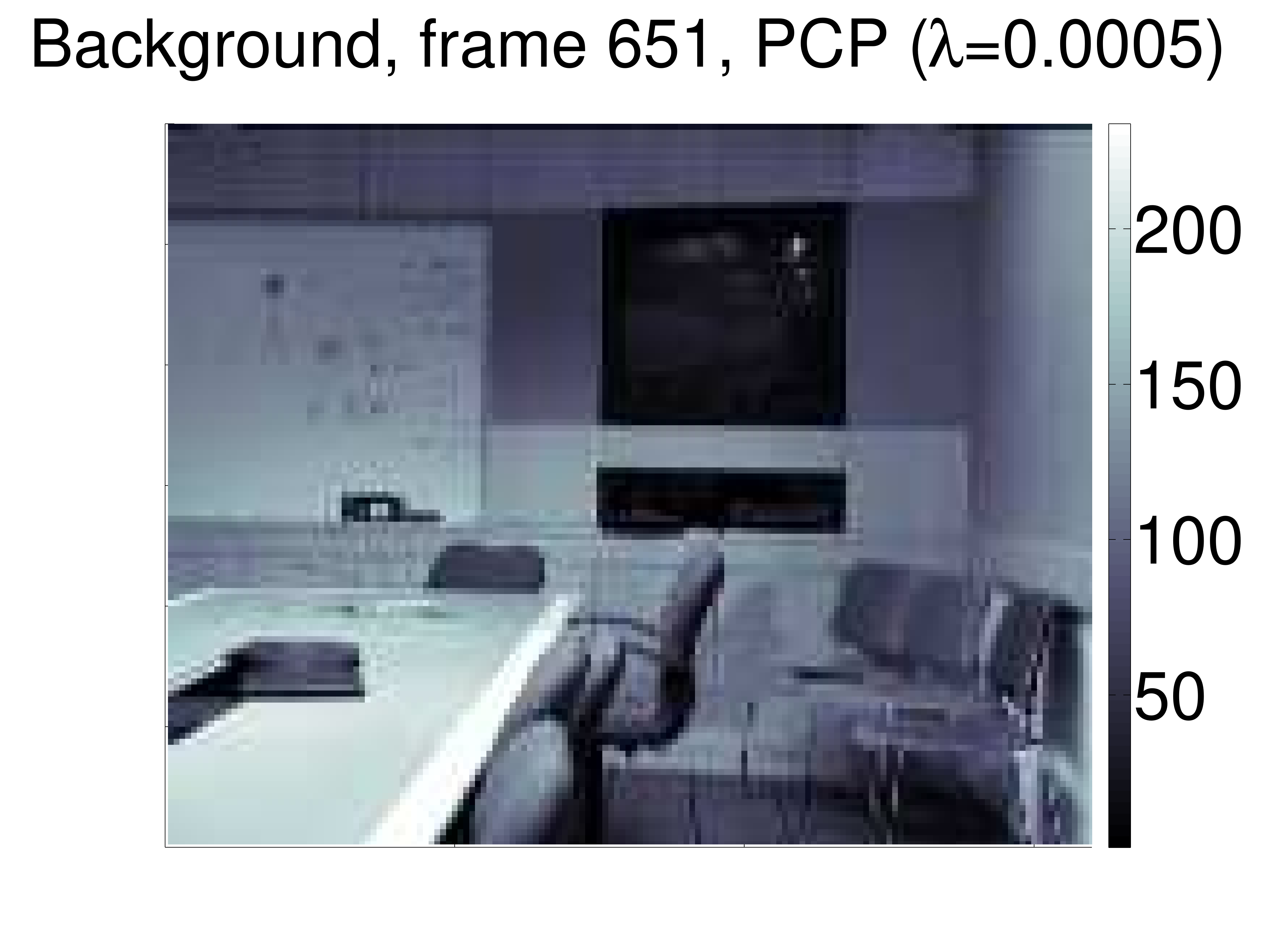} &
\includegraphics[width=.25\columnwidth]{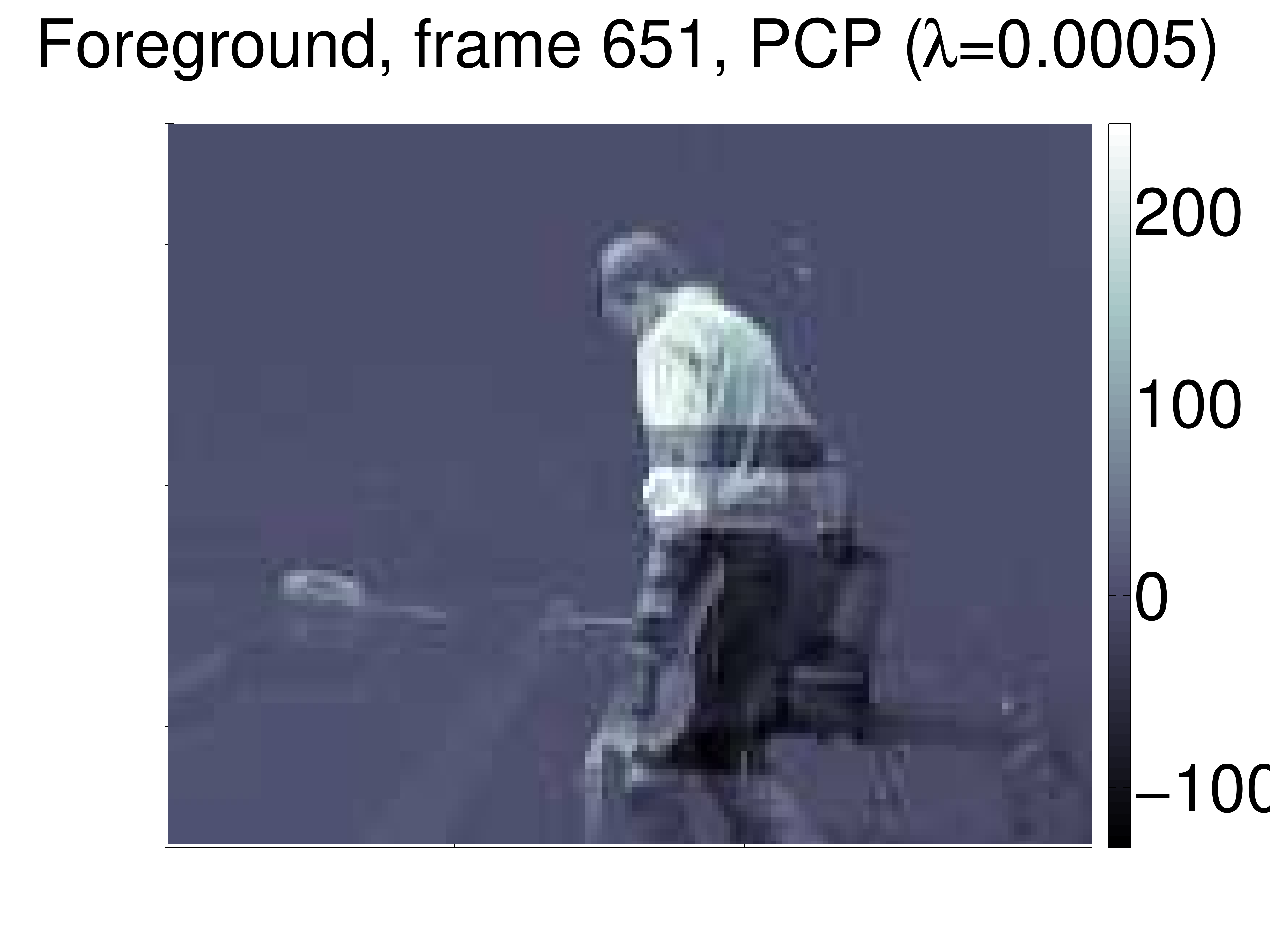} &
\includegraphics[width=.25\columnwidth]{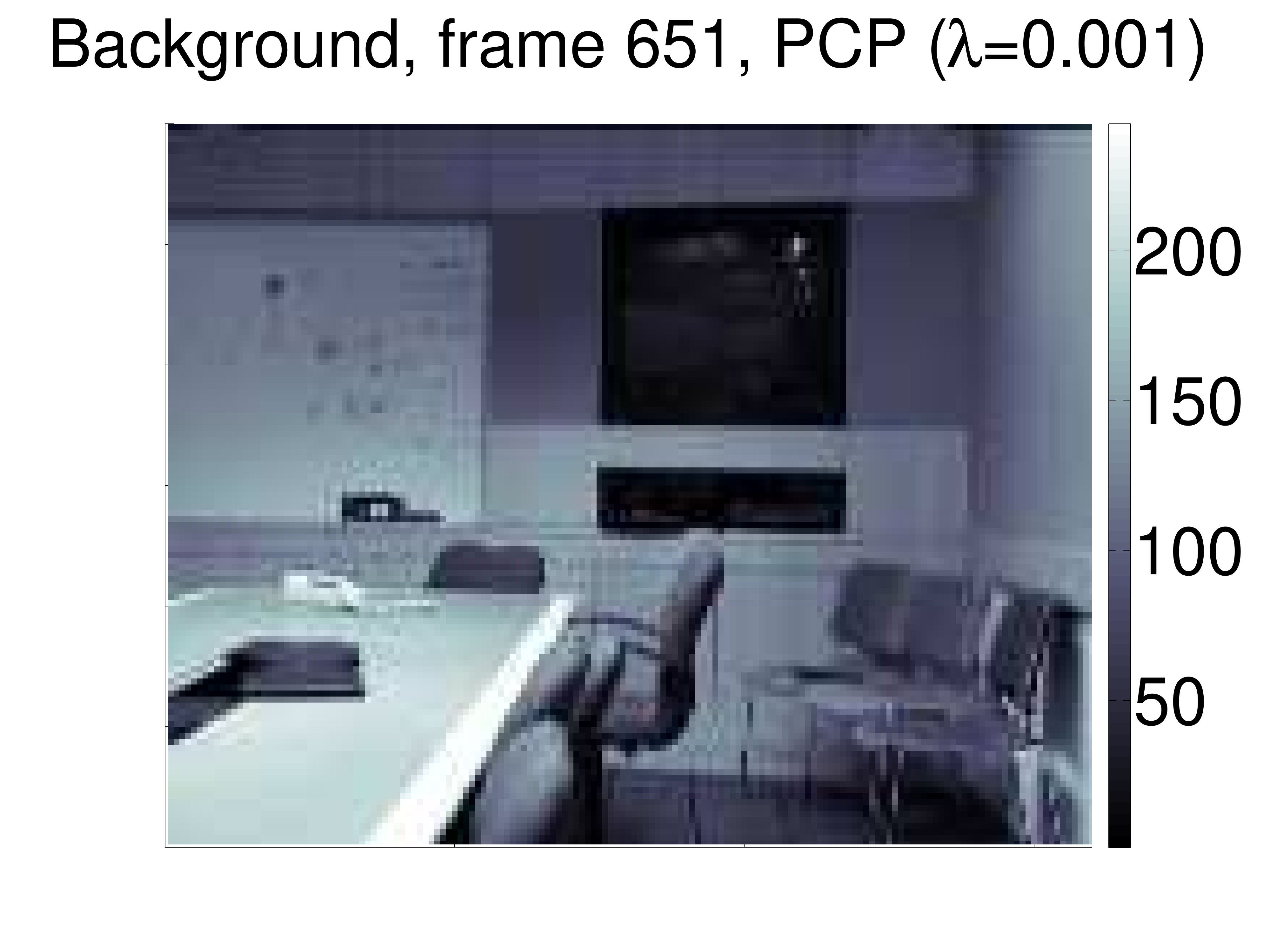} &
\includegraphics[width=.25\columnwidth]{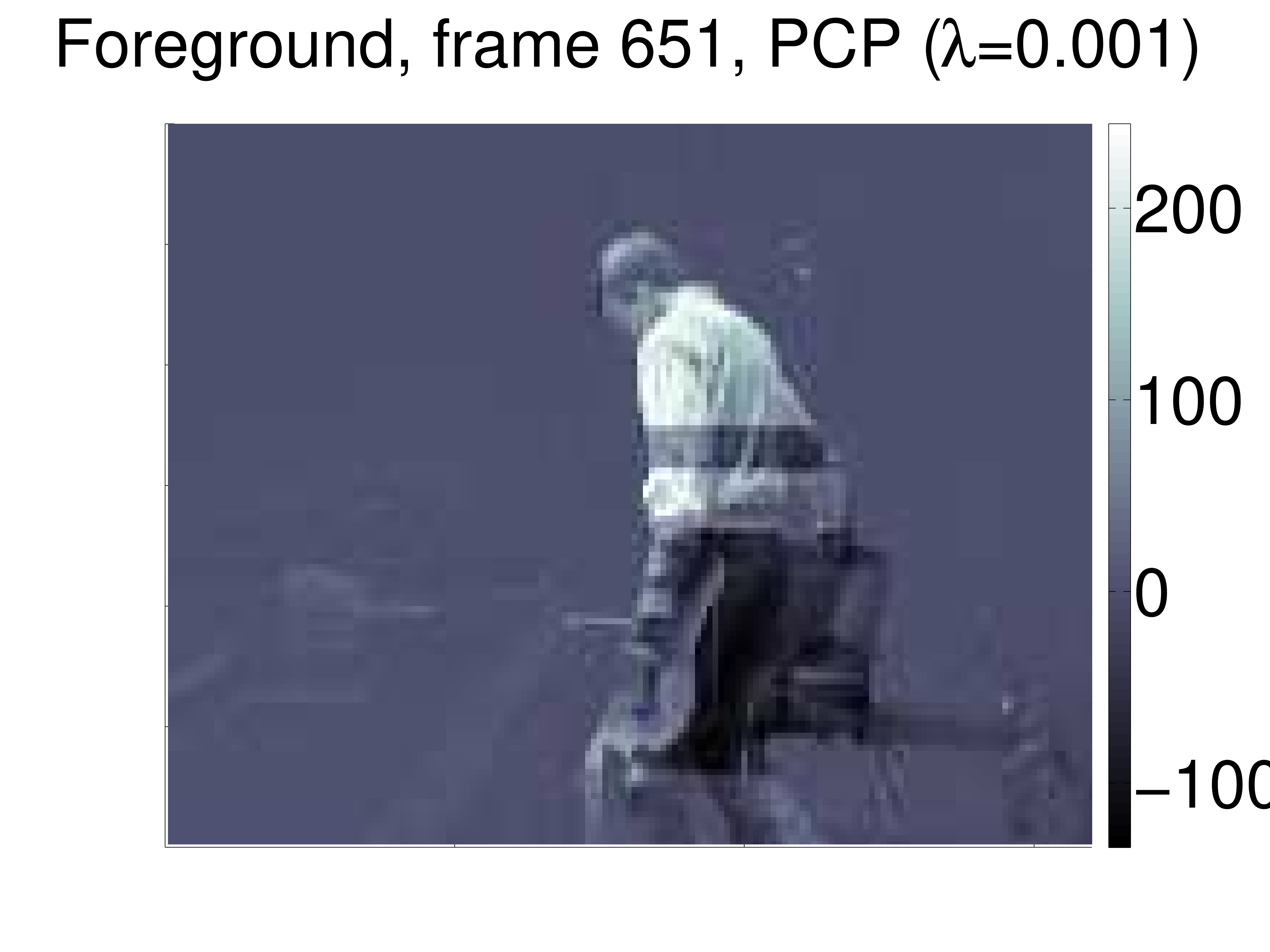} \\
\includegraphics[width=.25\columnwidth]{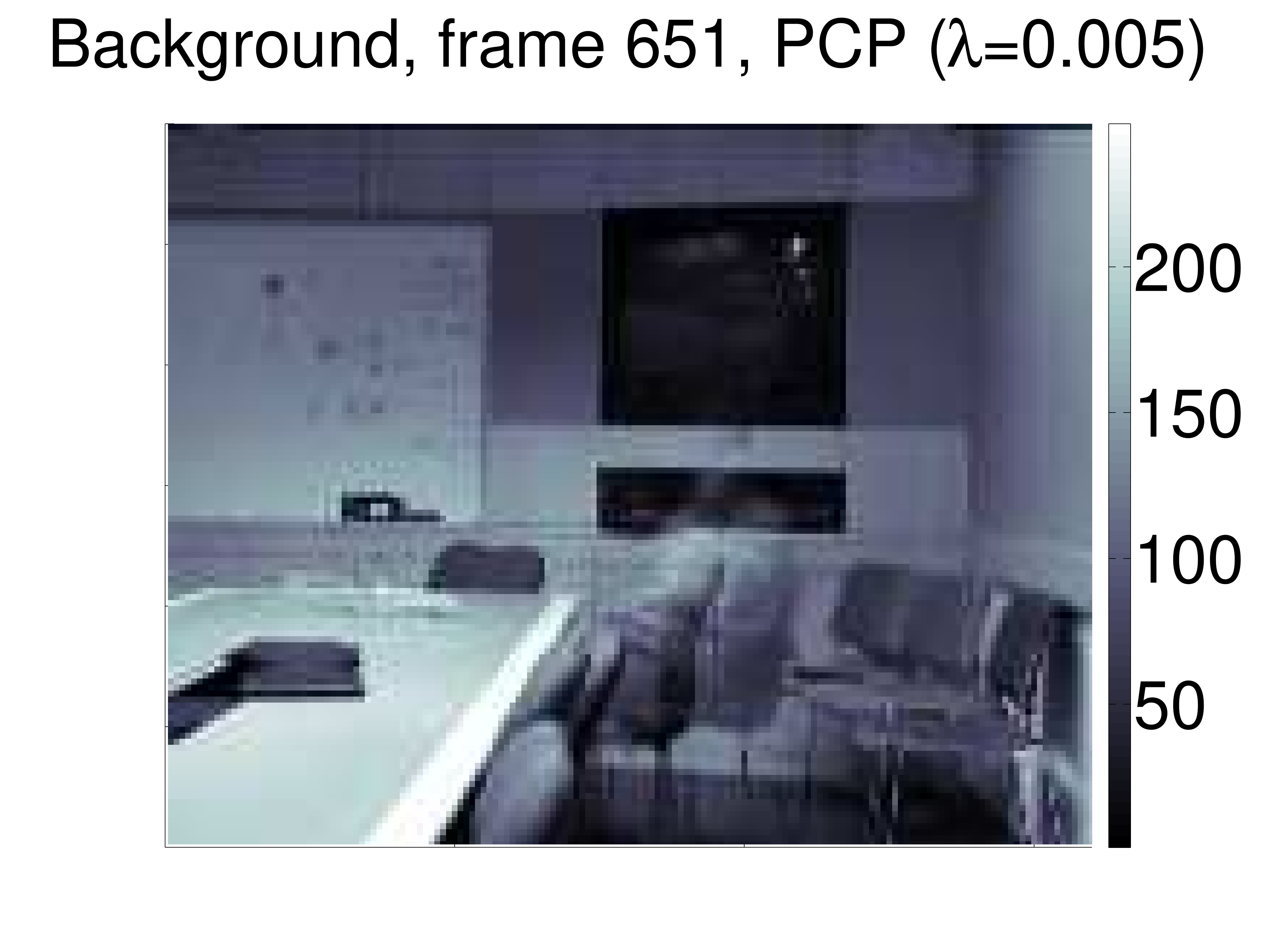} &
\includegraphics[width=.25\columnwidth]{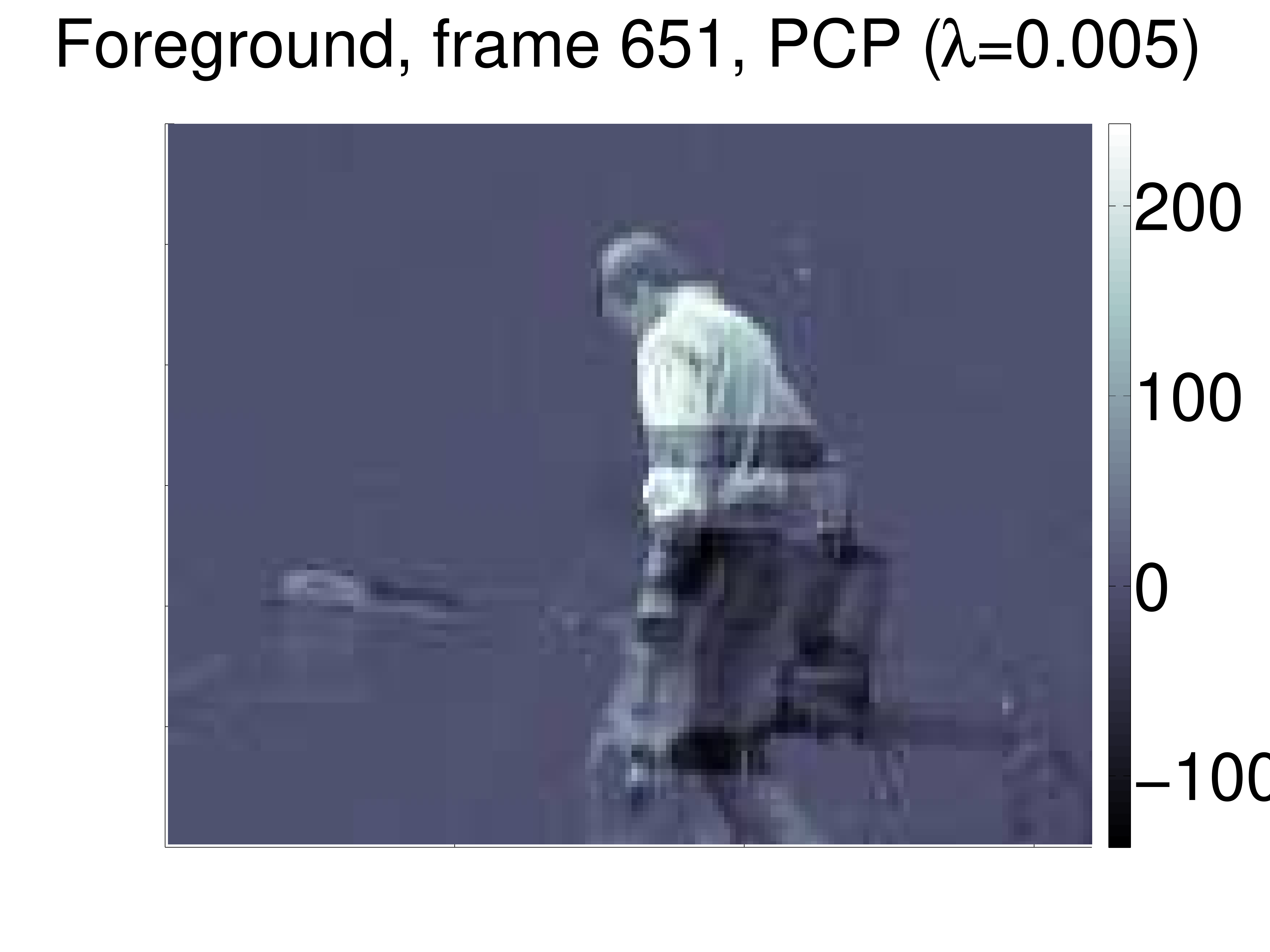} &
\includegraphics[width=.25\columnwidth]{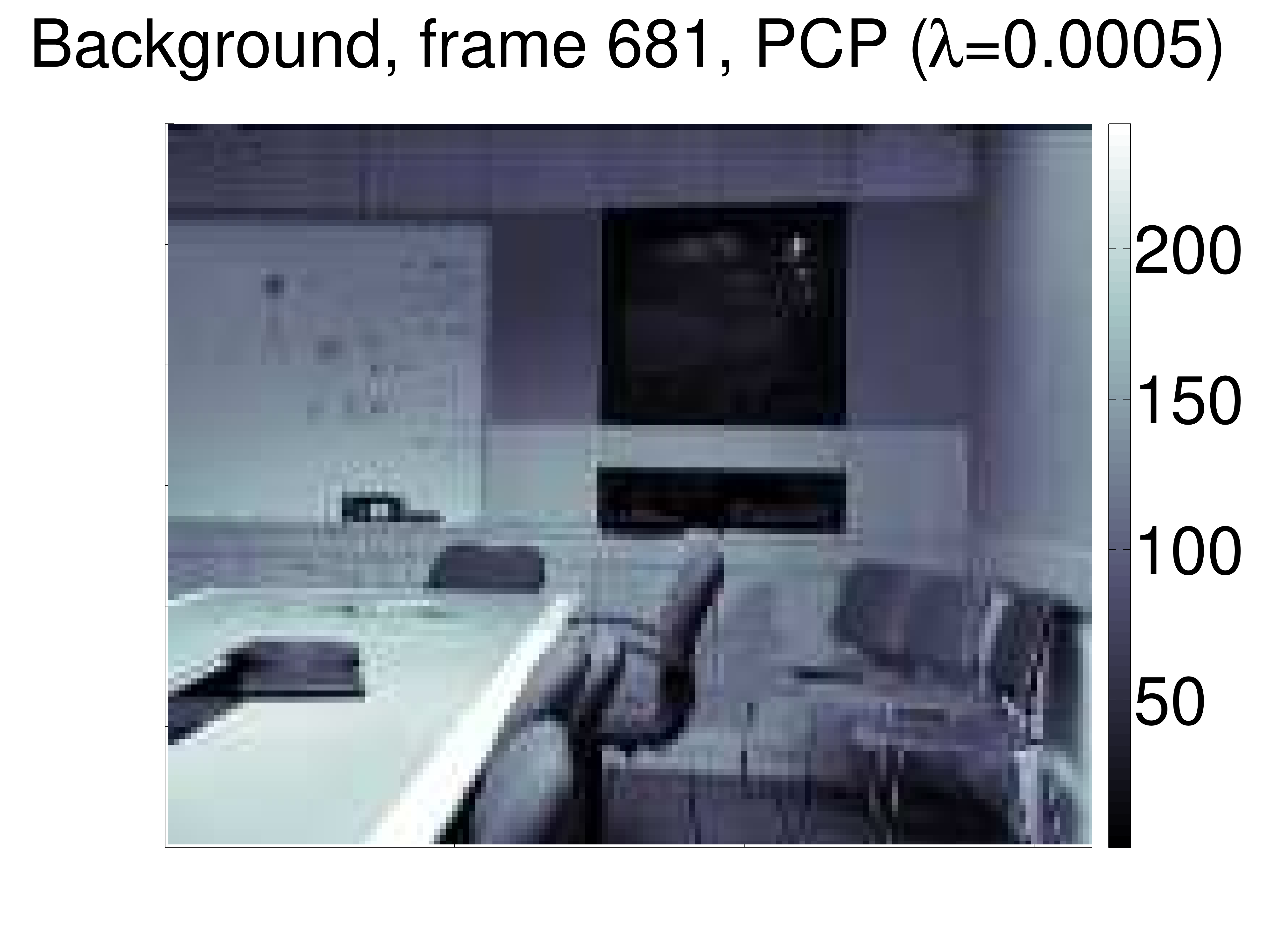}&
\includegraphics[width=.25\columnwidth]{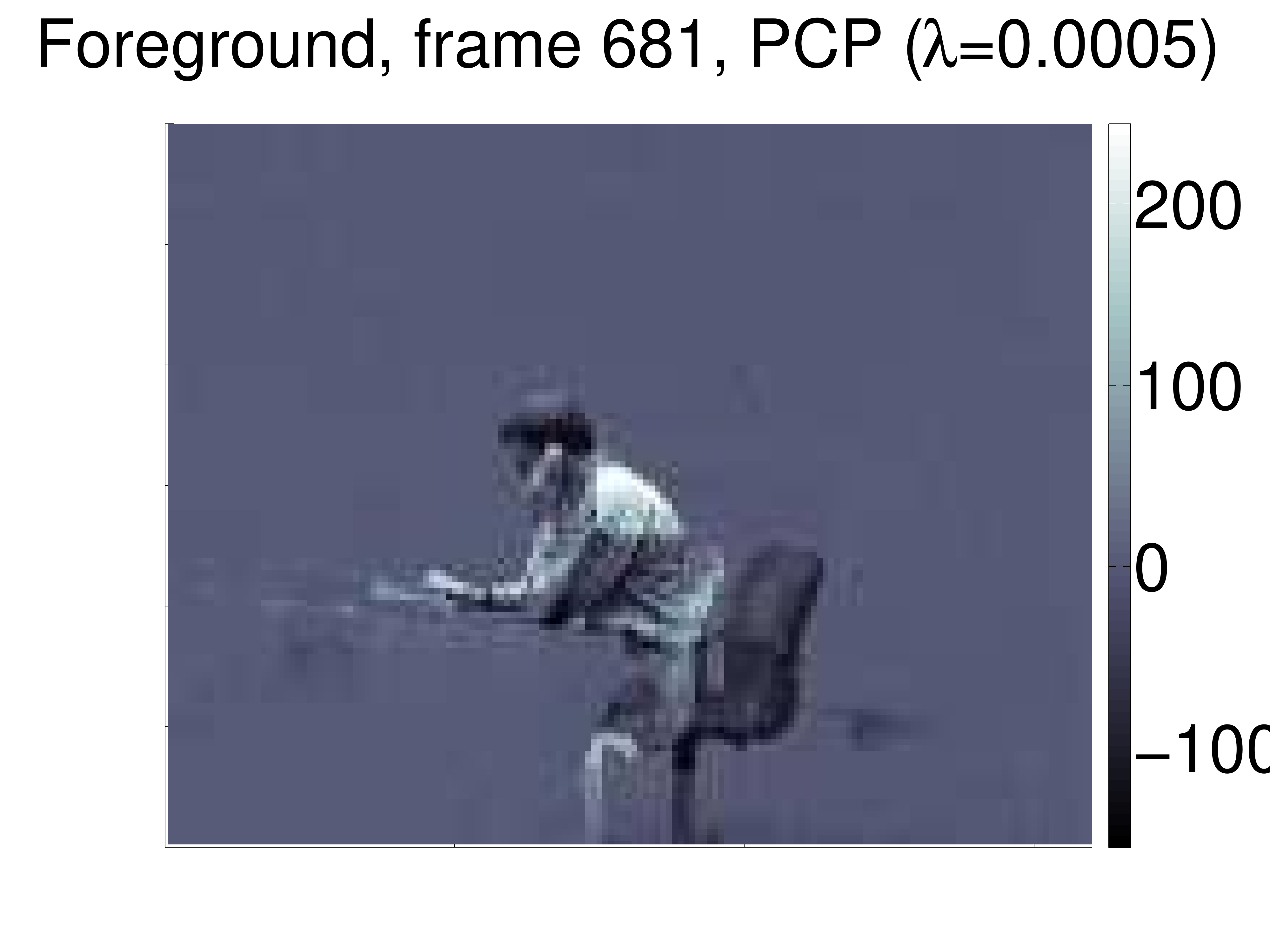}\\
\includegraphics[width=.25\columnwidth]{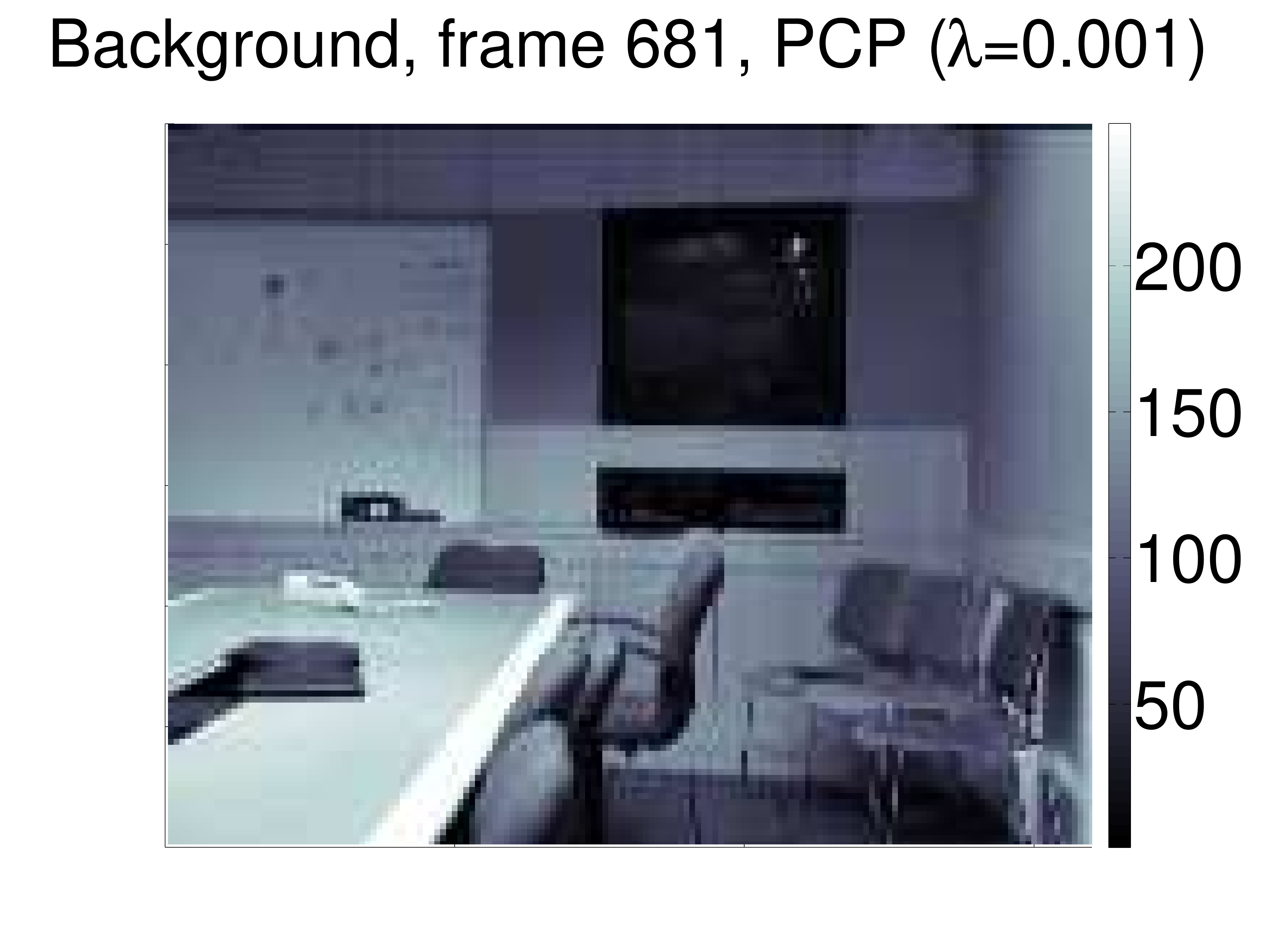}&
\includegraphics[width=.25\columnwidth]{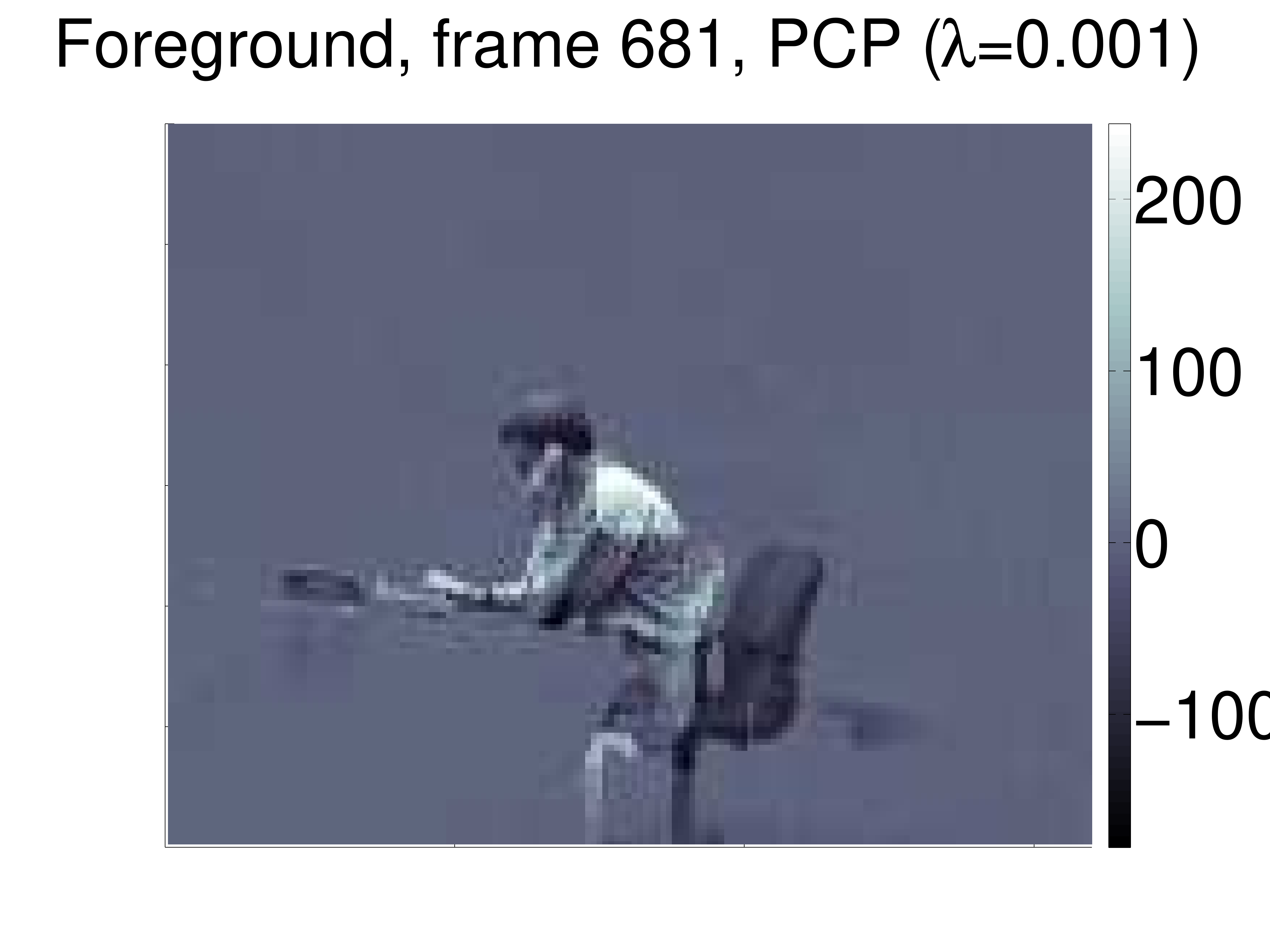}&
\includegraphics[width=.25\columnwidth]{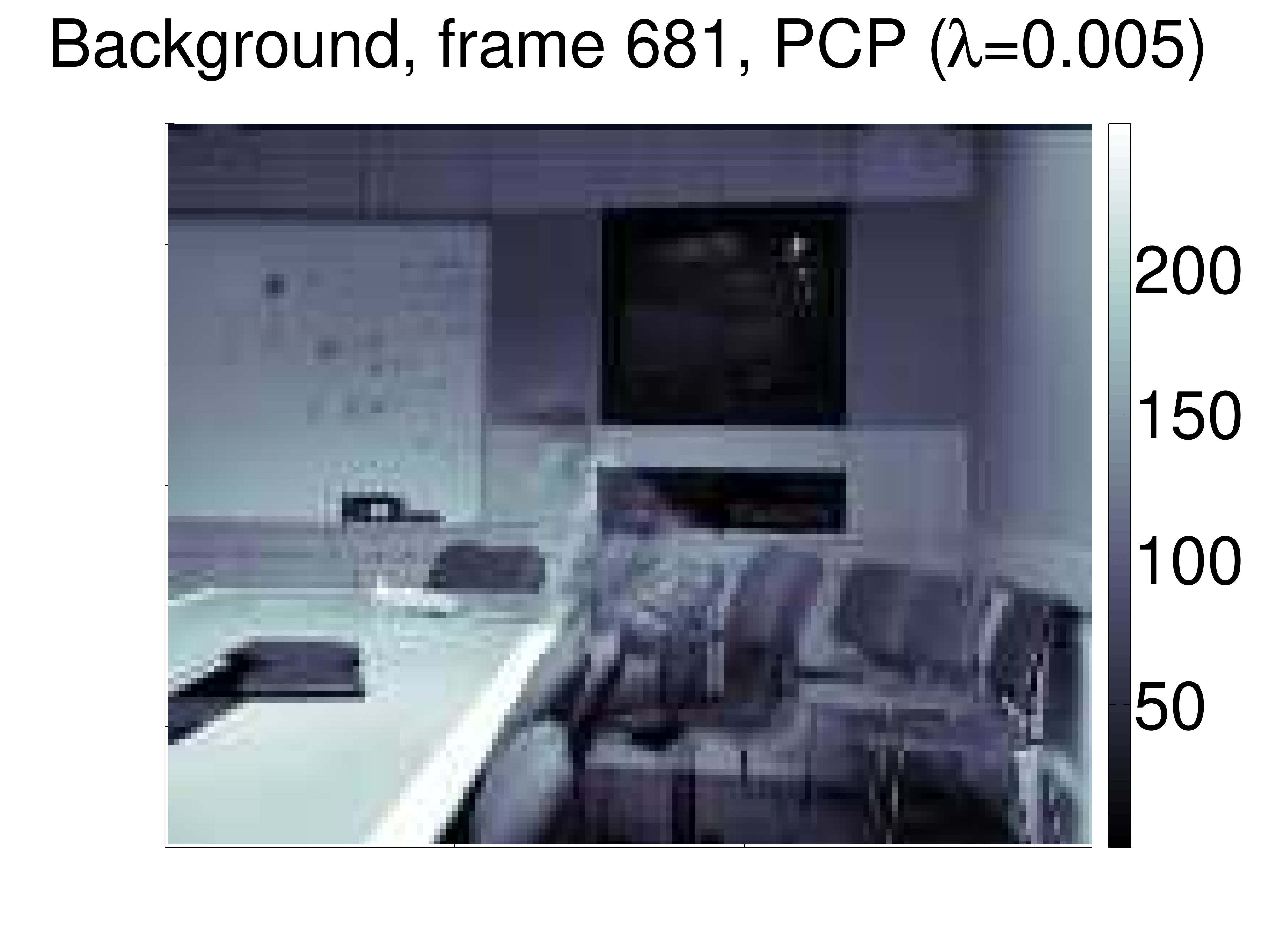}&
\includegraphics[width=.25\columnwidth]{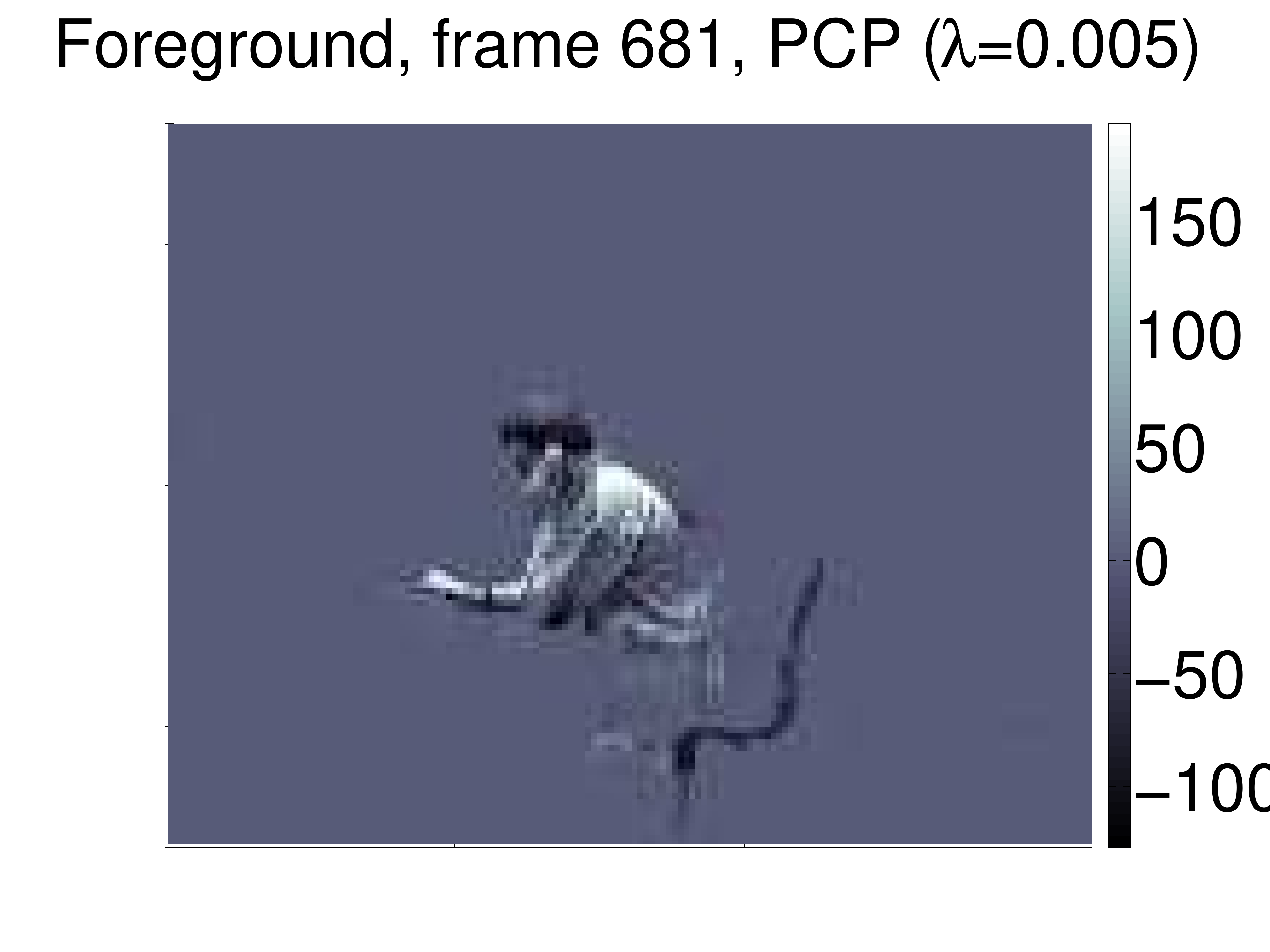}
\end{tabular}
\caption{Extracted background and foreground of frames 651 and 681 of the reduced moved object data set obtained with PCP (unscaled, compare to scaled version in Fig.~\ref{fig:bfrmopcp1})
}
\label{fig:bfrmopcp2}
\end{figure}
\clearpage

%%%%%%%%%%%%%%%%%%%%%%%%%%%%%%%%%%%%%%%%%%%%%%%%%%%%
\begin{figure}
\centering
\begin{tabular}{cc}
\includegraphics[width=.5\columnwidth]{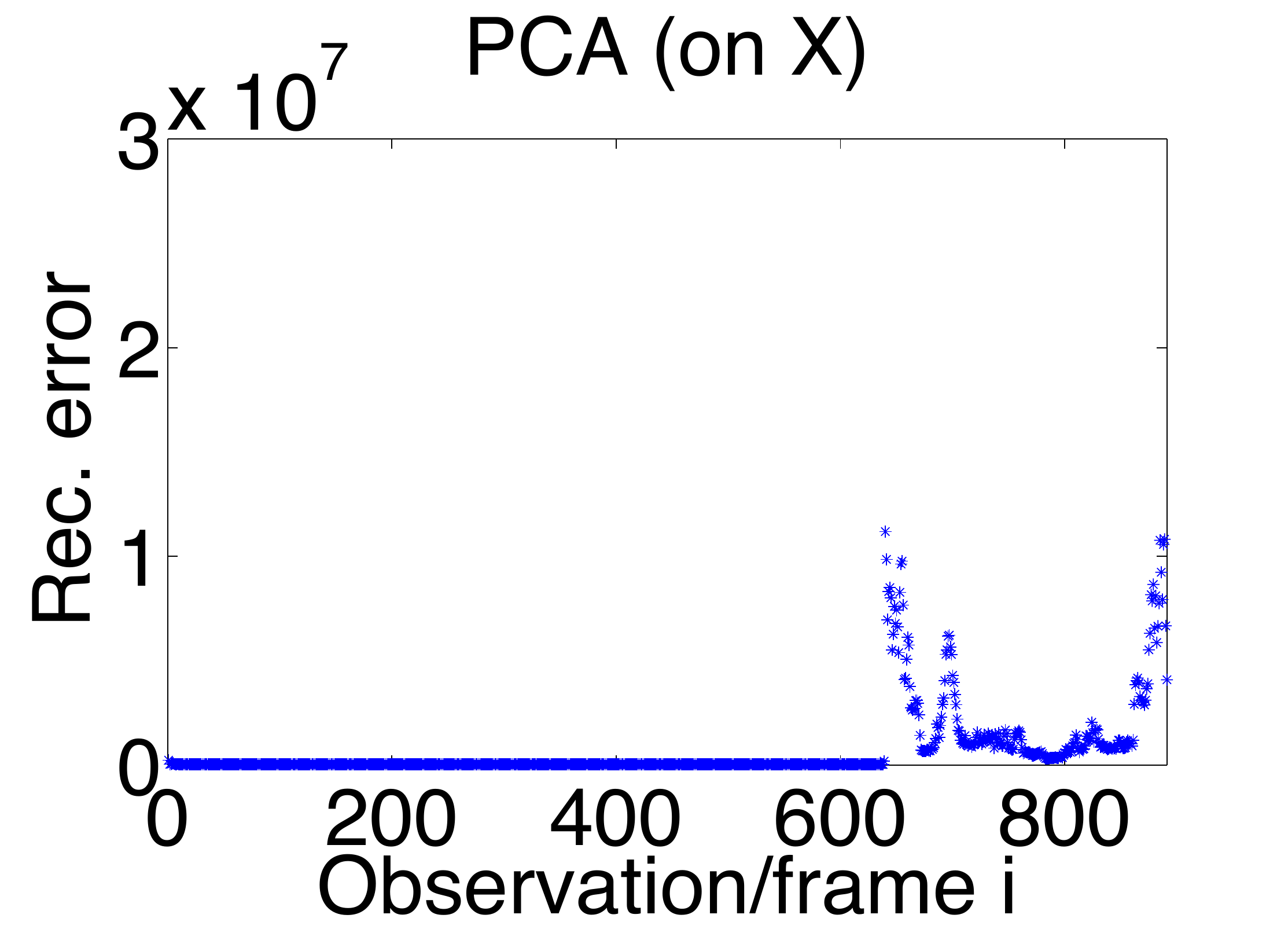} &
\includegraphics[width=.5\columnwidth]{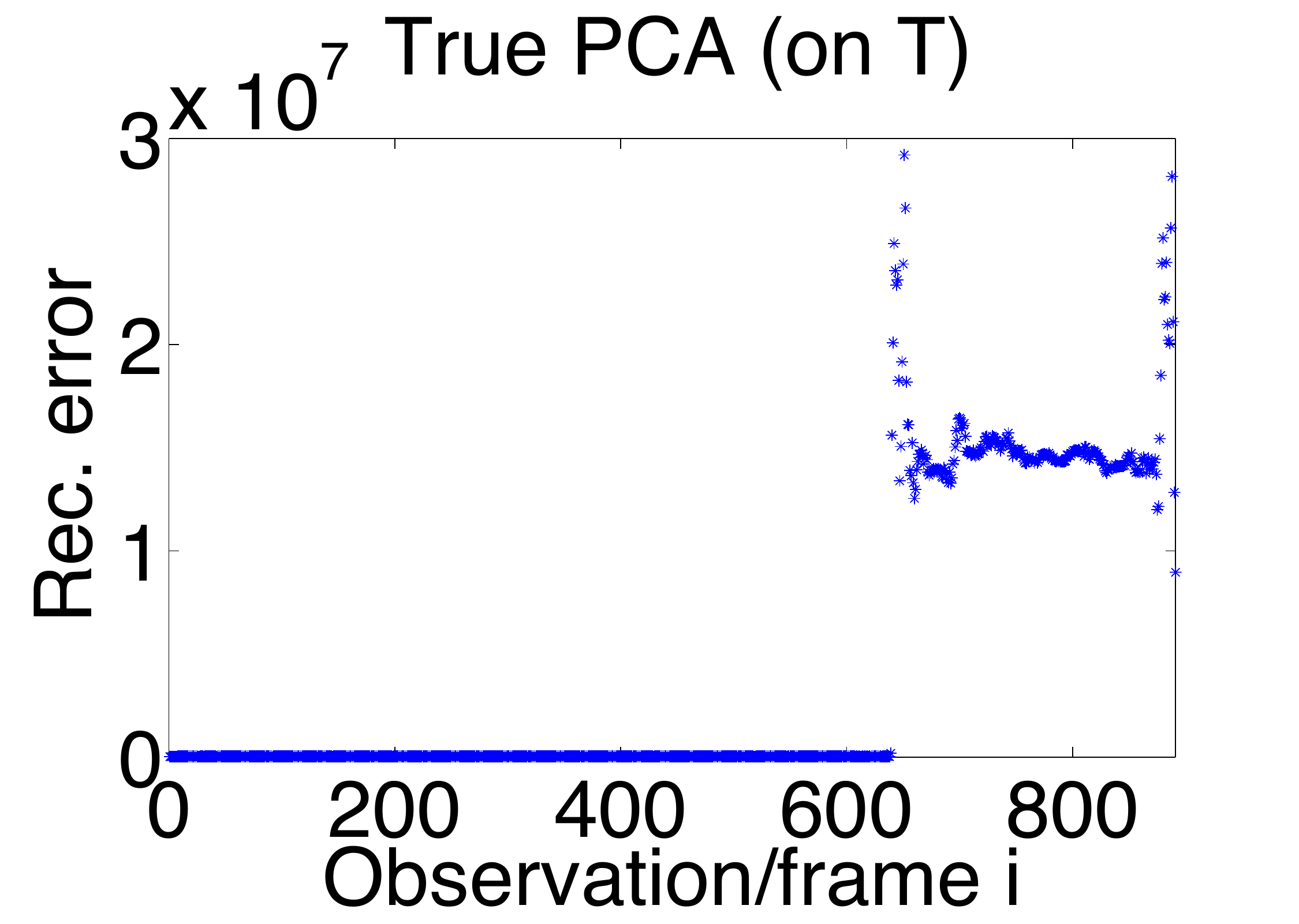} \\ \hline
\includegraphics[width=.5\columnwidth]{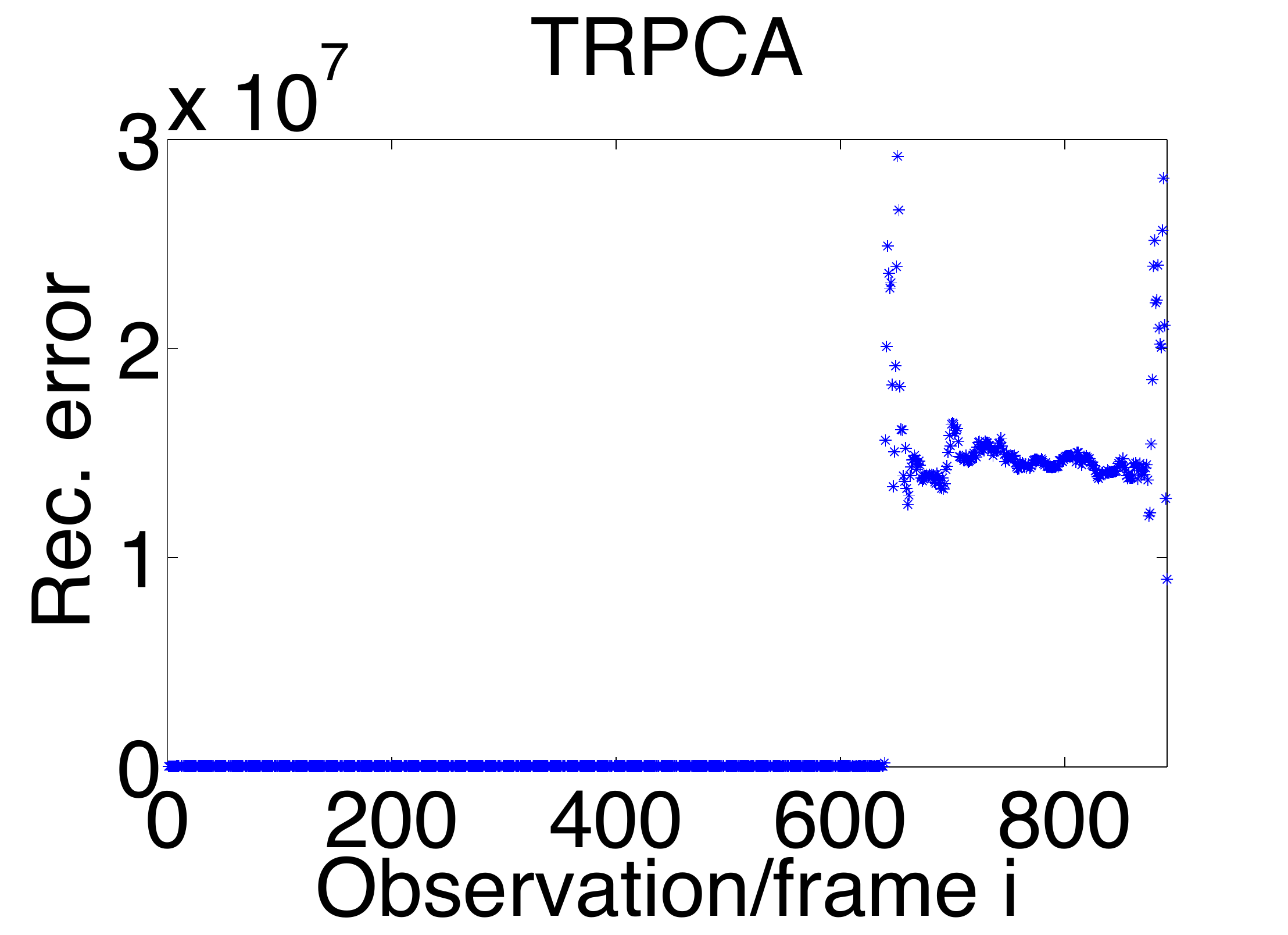} &
\includegraphics[width=.5\columnwidth]{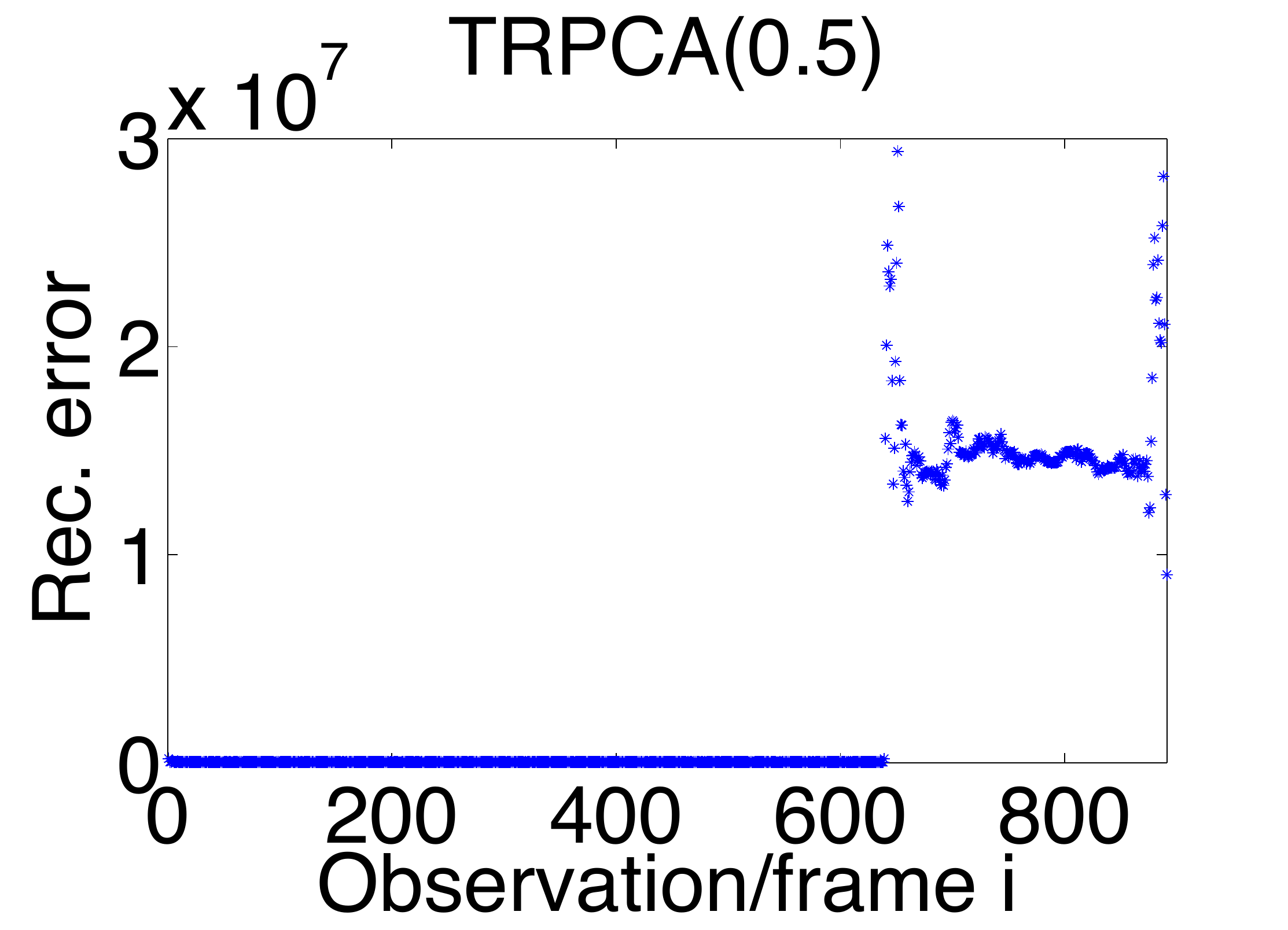} \\
\includegraphics[width=.5\columnwidth]{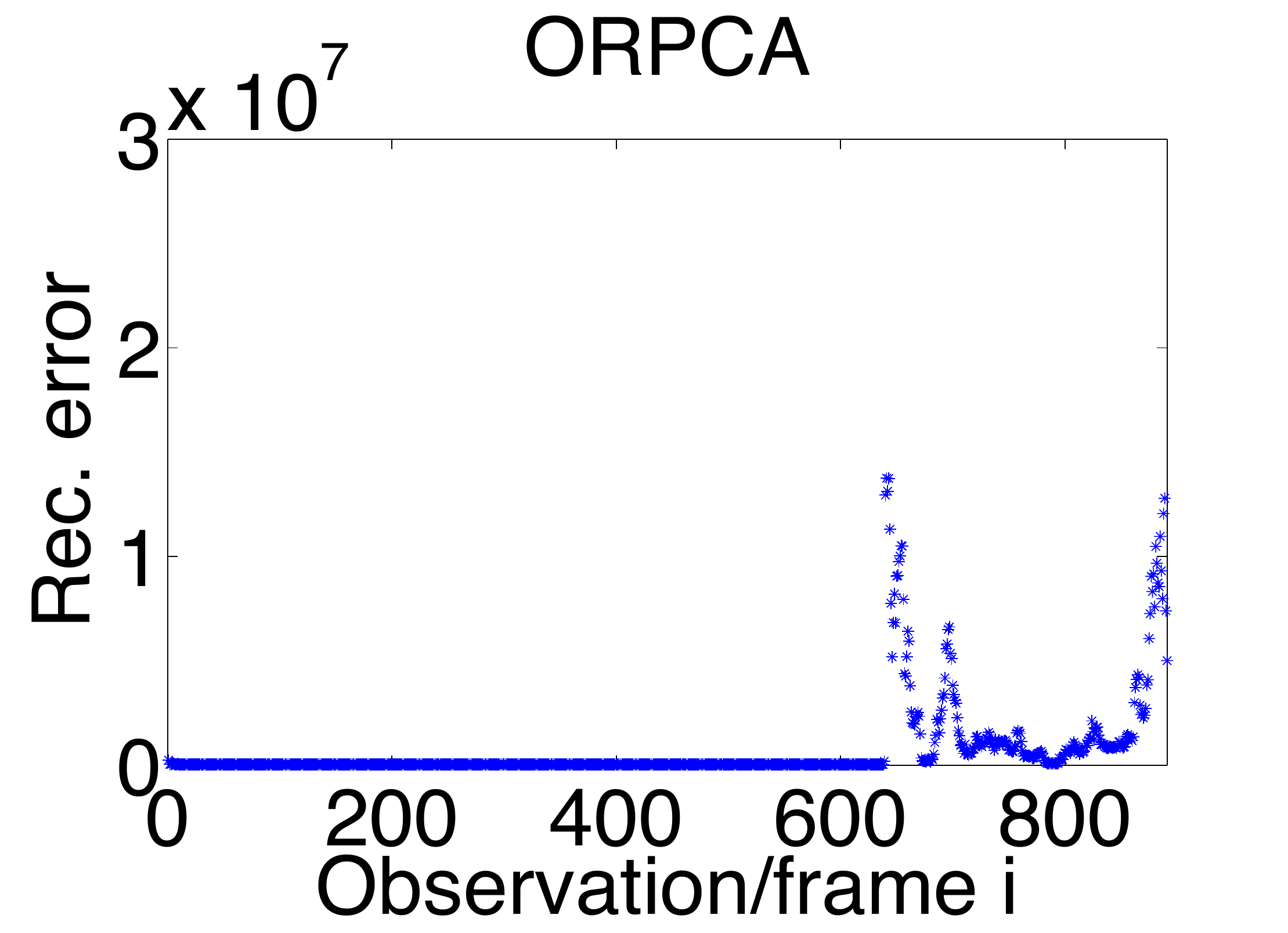} &
\includegraphics[width=.5\columnwidth]{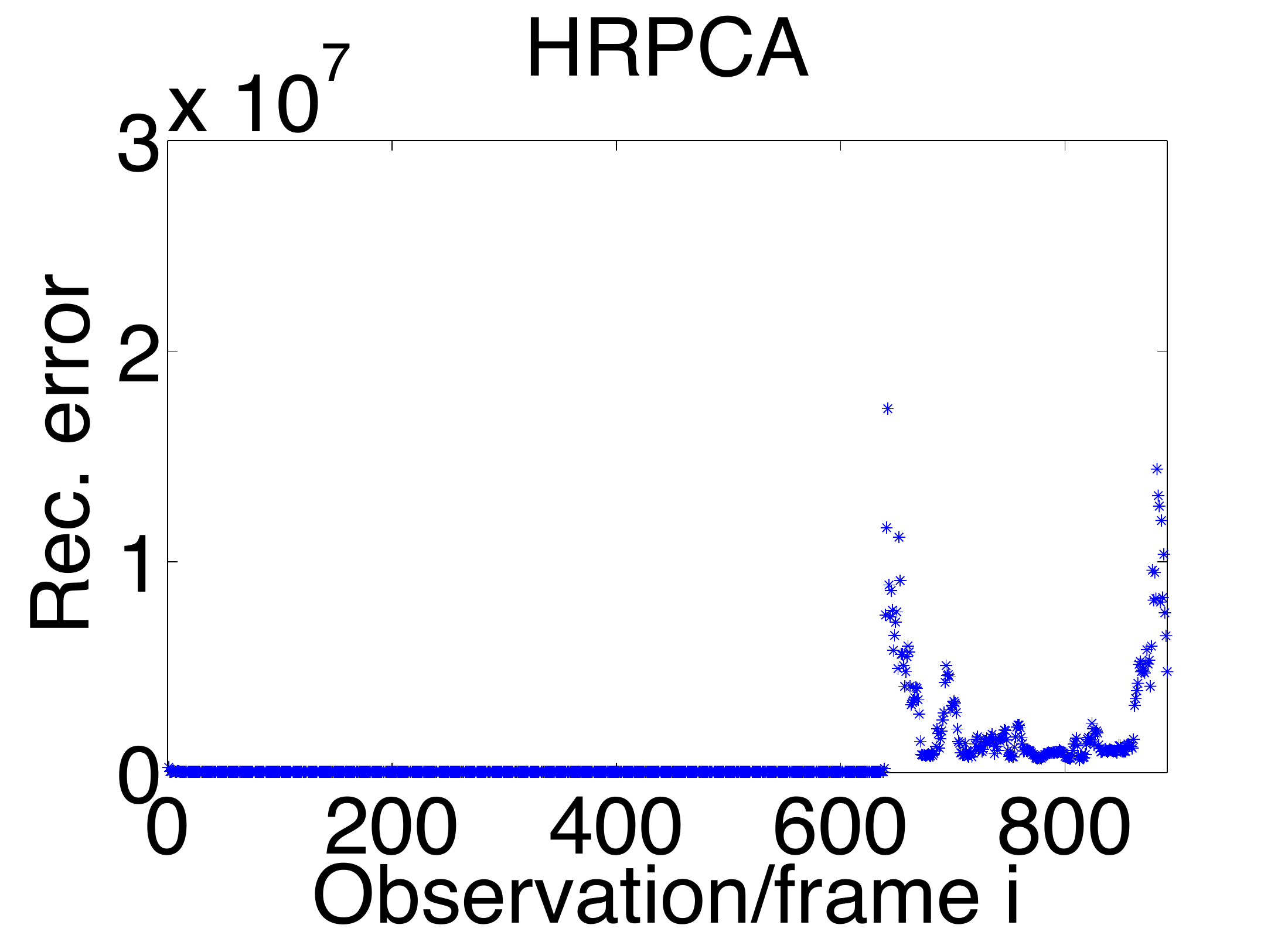} \\
\includegraphics[width=.5\columnwidth]{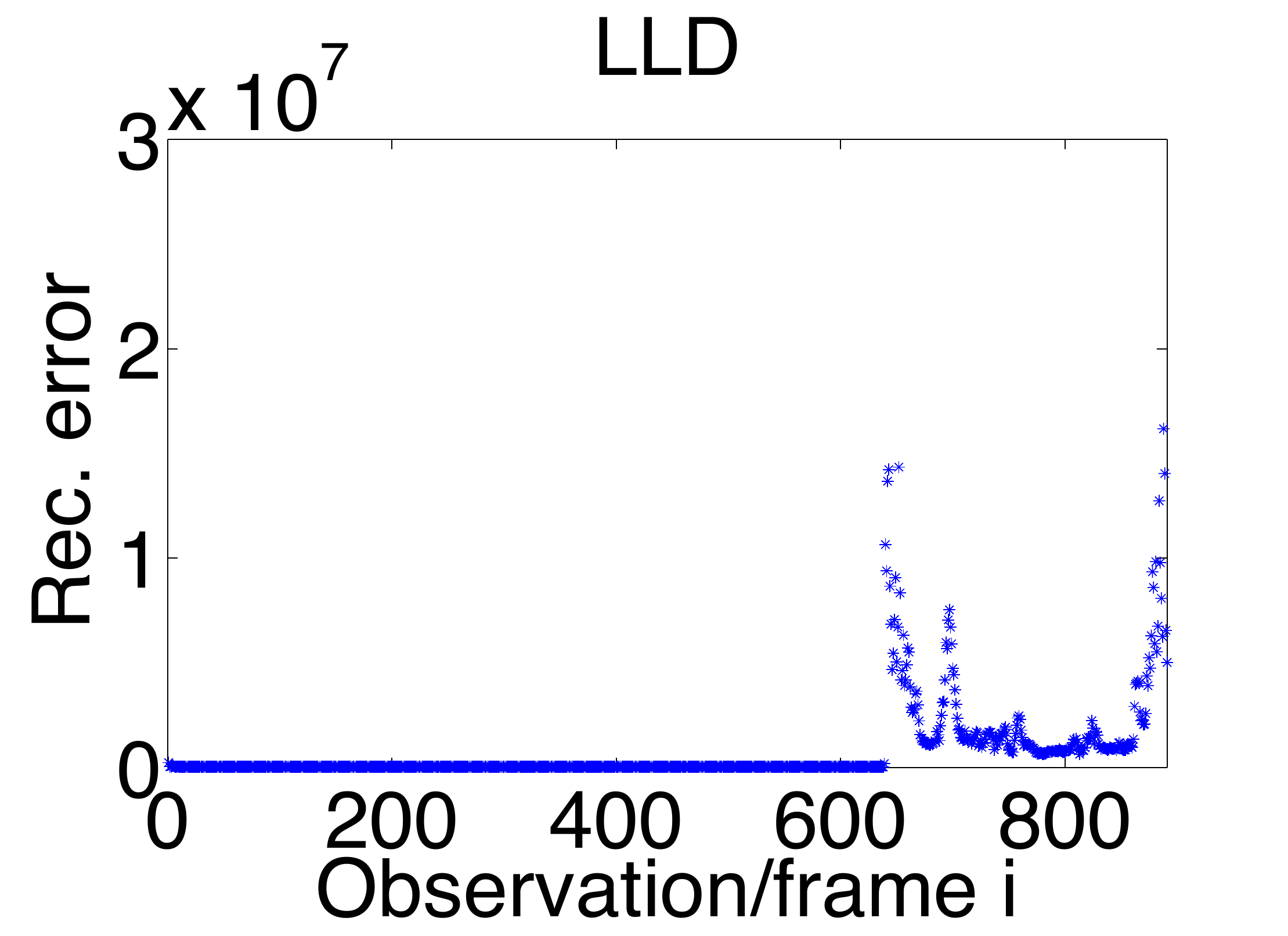} &
\includegraphics[width=.5\columnwidth]{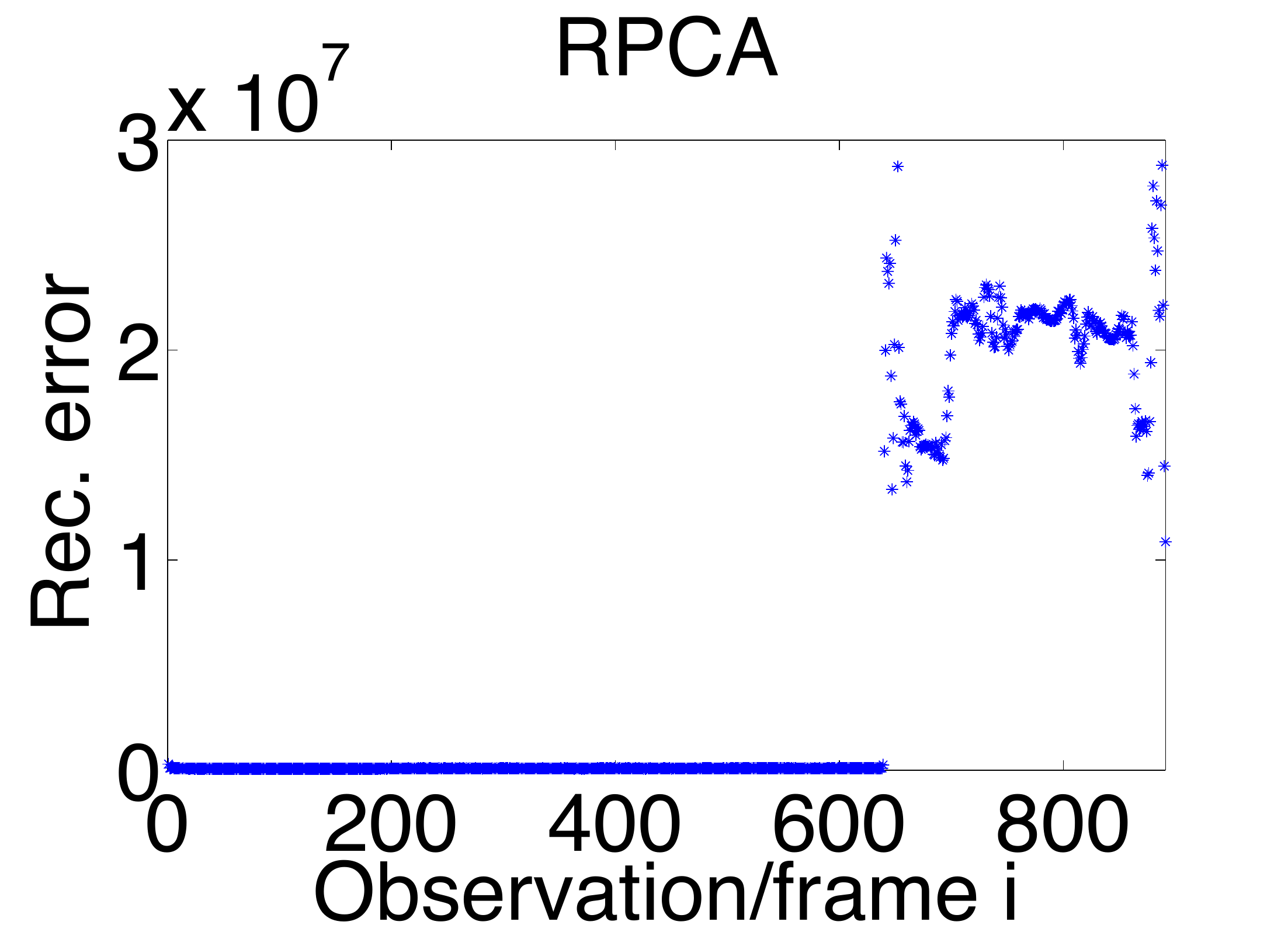} 
\end{tabular}
\caption{Reconstruction errors of different methods on the reduced moved object data set (analogous to Fig.~\ref{fig:imagesresWS}). One can see that TRPCA/TRPCA(0.5) again recovers the reconstruction errors of the true data almost perfectly as opposed to all other methods. However, note that RPCA does also well in having large reconstruction error for all frames containing the person
}
\label{fig:rmores}
\end{figure}
\clearpage

%\end{comment}

\end{document}